\documentclass{article} 

\usepackage{amsmath,amsfonts,bm}

\def\eqref#1{equation~\ref{#1}}

\def\ceil#1{\lceil #1 \rceil}

\def\1{\bm{1}}

\DeclareMathAlphabet{\mathsfit}{\encodingdefault}{\sfdefault}{m}{sl}
\SetMathAlphabet{\mathsfit}{bold}{\encodingdefault}{\sfdefault}{bx}{n}

\DeclareMathOperator*{\argmax}{arg\,max}
\DeclareMathOperator*{\argmin}{arg\,min}

\usepackage[utf8]{inputenc} 
\usepackage[T1]{fontenc}    
\usepackage[colorlinks=true, linkcolor=blue, citecolor=blue]{hyperref}
\usepackage{url}            
\usepackage{booktabs}       
\usepackage{amsfonts}       
\usepackage{nicefrac}       
\usepackage{microtype}      
\usepackage{xcolor}         
\usepackage{natbib}
\usepackage{fullpage}

\usepackage{algorithm}
\usepackage{algorithmic}
\usepackage{amsmath}
\usepackage{amssymb}
\usepackage{amsthm}
\usepackage{bm}
\allowdisplaybreaks[4]
\usepackage{graphicx}
\usepackage{mathtools}
\usepackage{enumitem}
\usepackage{xspace}
\usepackage{authblk}

\newtheorem{lemma}{Lemma}[section]
\newtheorem{theorem}{Theorem}[section]

\newtheorem{definition}{Definition}[section]
\newtheorem{remark}{Remark}[section]

\newcommand{\mypar}[1]{\noindent\textbf{#1}}
\newcommand{\ouralg}{\textsc{GeRL\_MG2}\xspace}

\newcommand*{\affaddr}[1]{#1}
\newcommand*{\affmark}[1][*]{\textsuperscript{#1}}
\newcommand*{\email}[1]{\texttt{#1}}

\title{Representation Learning for General-sum\\ Low-rank  Markov Games}

\author{
Chengzhuo Ni\affmark[1], Yuda Song\affmark[2], Xuezhou Zhang\affmark[1], Chi Jin\affmark[1], and Mengdi Wang\affmark[1]\\
\affaddr{\affmark[1]Department of Electrical and Computer Engineering, Princeton University}\\
\email{\{cn10,xz7392,chij,mengdiw\}@princeton.edu}\\
\affaddr{\affmark[2]Carnegie Mellon University}\\
\email{yudas@andrew.cmu.edu}
}
\date{}

\begin{document}

\maketitle

\begin{abstract}
    We study multi-agent general-sum Markov games with nonlinear function approximation. We focus on low-rank Markov games whose transition matrix admits a hidden low-rank structure on top of an unknown non-linear representation. The goal is to design an algorithm that (1) finds an $\varepsilon$-equilibrium policy sample efficiently without prior knowledge of the environment or the representation, and (2) permits a deep-learning friendly implementation. We leverage representation learning and present a model-based and a model-free approach to construct an effective representation from the collected data. 
    For both approaches, the algorithm achieves a sample complexity of poly$(H,d,A,1/\varepsilon)$, where
    $H$ is the game horizon, $d$ is the dimension of the feature vector, $A$ is the size of the joint action space and $\varepsilon$ is the optimality gap. 
    When the number of players is large, the above sample complexity can scale exponentially with the number of players in the worst case. To address this challenge, we consider Markov games with a factorized transition structure and present an algorithm that escapes such exponential scaling. To our best knowledge, this is the first sample-efficient algorithm for multi-agent general-sum Markov games that incorporates (non-linear) function approximation. We accompany our theoretical result with a neural network-based implementation of our algorithm and evaluate it against the widely used deep RL baseline, DQN with fictitious play. \looseness=-1
\end{abstract}

\section{Introduction}
Multi-agent reinforcement learning (MARL) studies the problem where multiple agents learn to make sequential decisions in an unknown environment to maximize their (own) cumulative rewards. Recently, MARL has achieved remarkable empirical success, such as in traditional games like GO \citep{silver2016mastering,silver2017mastering} and Poker \citep{moravvcik2017deepstack}, real-time video games such as Starcraft and Dota 2\citep{vinyals2019grandmaster, berner2019dota}, decentralized controls or multi-agent robotics systems \citep{brambilla2013swarm} and autonomous driving \citep{shalev2016safe}. 

On the theoretical front, however, provably sample-efficient algorithms for Markov games have been largely restricted to either two-player zero-sum games \citep{bai2020near,xie2020learning,chen2021almost,jin2021power} or general-sum games with small and finite state and action spaces \citep{bai2020provable, liu2021sharp, jin2021v}. These algorithms typically do not permit a scalable implementation applicable to real-world games, due to either (1) they only work for tabular or linear Markov games which are too restrictive to model real-world games, or (2) the ones that do handle rich non-linear function approximation \citep{jin2021power} are not computationally efficient. This motivates us to ask the following question:

\begin{center}


\emph{Can we design an efficient algorithm that
(1) provably learns multi-player general-sum Markov games with rich nonlinear function approximation and (2) permits scalable implementations?}
\end{center}

This paper presents the first positive answer to the above question. In particular, we make the following contributions:
\begin{enumerate}[leftmargin=*,itemsep=0pt]
    \item We design a new centralized self-play meta algorithm for multi-agent low-rank Markov games: \textbf{Ge}neral \textbf{R}epresentation \textbf{L}earning for \textbf{M}ulti-player \textbf{G}eneral-sum \textbf{M}arkov \textbf{G}ame (\textbf{\ouralg}). We present a model-based and a model-free instantiation of \ouralg which differ by the way function approximation is used, and a clean analysis for both approaches.
    \item We show that the model-based variant requires access to an MLE oracle and a NE/CE/CCE oracle for matrix games, and enjoys a $\tilde{O}\left(H^6d^4A^2\log(|\Phi||\Psi|)/\varepsilon^2\right)$ sample complexity to learn an $\varepsilon$-NE/CE/CCE equilibrium policy, where $d$ is the dimension of the feature vector, $A$ is the size of the joint action space, $H$ is the game horizon, $\Phi$ and $\Psi$ are the function classes for the representation and emission process. The model-free variant replaces model-learning with solving a minimax optimization problem, and enjoys a sample complexity of $\tilde{O}\left(H^6d^4A^3M\log(|\Phi|)/\varepsilon^2\right)$ for a slightly restricted class of Markov game with latent block structure.
    \item Both of the above algorithms have sample complexities scaling with the joint action space size, which is exponential in the number of players. This unfavorable scaling is referred to as the \textit{curse of multi-agent}. We consider a spatial factorization structure where the transition of each player's local state is directly affected only by at most $L=O(1)$ players in its adjacency. Given this additional structure, we provide an algorithm that achieves $\tilde{O}(M^4H^6d^{2(L+1)^2}\tilde{A}^{2(L+1)}/\varepsilon^2)$ sample complexity, where $\tilde{A}$ is the size of a single player's action space, thus escaping the exponential scaling to the number of agents. \looseness=-1
    \item Finally, we provide an efficient implementation of our model-free algorithm, and show that it achieves superior performance against traditional deep RL baselines without principled representation learning. \looseness=-1
\end{enumerate}

\subsection{Related Works}
\paragraph{Markov games} Markov games \citep{littman1994Markov,shapley1953stochastic} is an extensively used framework introduced for game playing with sequential decision making. Previous works \citep{littman1994Markov,hu2003nash,hansen2013strategy} studied how to find the Nash equilibrium of a Markov game when the transition matrix and reward function are known. When the dynamic of the Markov game is unknown, recent works provide a line of finite-sample guarantees for learning Nash equilibrium in two-player zero-sum Markov games \citep{bai2020provable,xie2020learning,bai2020near,zhang2020model,liu2021sharp,jin2021power,huang2021towards} and learning various equilibriums (including NE,CE,CCE, which are standard solution notions in games \citep{roughgarden2010algorithmic}) in general-sum Markov games \citep{liu2021sharp,bai2021sample,jin2021v}. Some of the analysis in these works are based on the techniques for learning single-agent Markov Decision Processes (MDPs) \citep{azar2017minimax,jin2018q,jin2020provably}. 

\paragraph{RL with Function Approximation} Function approximation in reinforcement learning has been extensively studied in recent years. For the single-agent Markov decision process, function approximation is adopted to achieve a better sample complexity that depends on the complexity of function approximators rather than the size of the state-action space. For example, \citep{yang2019sample,jin2020provably,zanette2020frequentist} considered the linear MDP model, where the transition probability function and reward function are linear in some feature mapping over state-action pairs. Another line of works
\citep[see, e.g.,][]{jiang2017contextual,jin2021bellman,du2021bilinear,foster2021statistical}
studied the MDPs with general nonlinear function approximations. 

When it comes to Markov game, \citep{chen2021almost,xie2020learning,jia2019feature} studied the Markov games with linear function approximations. Recently, \citep{huang2021towards} and \citep{jin2021power} proposed the first algorithms for two-player zero-sum Markov games with general function approximation, and provided a sample complexity governed by the minimax Eluder dimension. However, technical difficulties prevent extending these results to multi-player general-sum Markov games with nonlinear function approximation. The results for linear function approximation assume a known state-action feature, and are unable to solve the Markov games with a more general non-linear approximation where both the feature and function parameters are unknown. For the general function class works, their approaches rely heavily on the two-player nature, and it's not clear how to apply their methods to the general multi-player setting. 

\paragraph{Representation Learning in RL} Our work is closely related to representation learning in single-agent RL, where the study mainly focuses on the low-rank MDPs. A low-rank MDP is strictly more general than a linear MDP which assumes the representation is known a priori. Several related works studied low-rank MDPs with provable sample complexities. \citep{agarwal2020flambe,ren2021free} and \citep{uehara2021representation} consider the model-based setting, where the algorithm learns the representation with the model class of the transition probability given. \citep{modi2021model} provided a representation learning algorithm under the model-free setting and proved its sample efficiency when the MDP satisfies the minimal reachability assumption. \citep{zhang2022efficient} proposed a model-free method for the more restricted MDP class called Block MDP, but does not rely on the reachability assumption, which is also studied in papers including \citep{du2019provably3} and \citep{misra2020kinematic}. A concurrent work \citep{qiu2022contrastive} studies representation learning in RL with contrastive learning and extends their algorithm to the Markov game setting. However, their method requires strong data assumption and does not provide any practical implementation in the Markov game setting.

\section{Problem Settings}
A general-sum Markov game with $M$ players is defined by a tuple $(\mathcal{S},\{\mathcal{A}_i\}_{i=1}^M,P^\star,\{r_i\}_{i=1}^M,H,d_1)$. Here $\mathcal{S}$ is the state space, $\mathcal{A}_i$ is the action space for player $i$, $H$ is the time horizon of each episode and $d_1$ is the initial state distribution. We let $\mathcal{A}=\mathcal{A}_1\times\ldots\times\mathcal{A}_M$ and use $\bm{a}=(a_1,a_2,\ldots,a_M)$ to denote the joint actions by all $M$ players. Denote $\tilde A = \max_i |\mathcal{A}_i|$ and $A=|\mathcal{A}|$. $P^\star=\{P^\star_h\}_{h=1}^H$ is a collection of transition probabilities, so that $P^\star_h(\cdot\vert s,\bm{a})$ gives the distribution of the next state if actions $\bm{a}$ are taken at state $s$ and step $h$. And $r_i=\{r_{h,i}\}_{h=1}^H$ is a collection of reward functions, so that $r_{h,i}(s,\bm{a})$ gives the reward received by player $i$ when actions $\bm{a}$ are taken at state $s$ and step $h$. 
\subsection{Solution Concepts}
The policy of player $i$ is denoted as $\pi_i:=\{\pi_{h,i}:\mathcal{S}\rightarrow\Delta_{\mathcal{A}_i}\}_{h\in[H]}$. We denote the product policy of all the players as $\pi:=\pi_1\times\ldots\times\pi_M$, here ``product'' means that conditioned on the same state, the action of each player is sampled independently according to their own policy. We denote the policy of all the players except player $i$ as $\pi_{-i}$. We define $V^{\pi}_{h,i}(s)$ as the expected cumulative reward that will be received by player $i$ if starting at state $s$ at step $h$ and all players following policy $\pi$. For any strategy $\pi_{-i}$, there exists a best response policy of player $i$, which is a policy $\mu^\dagger(\pi_{-i})$ satisfying $V_{h,i}^{\mu^\dagger(\pi_{-i}),\pi_{-i}}(s)=\max_{\pi_i}V_{h,i}^{\pi_i,\pi_{-i}}(s)$ for any $(s,h)\in\mathcal{S}\times[H]$. We denote $V^{\dagger,\pi_{-i}}_{h,i}:=V^{\mu^\dagger(\pi_{-i}),\pi_{-i}}_{h,i}$. Let $v^{\dagger,\pi_{-i}}_i:=\mathbb{E}_{s\sim d_1}\left[V_{1,i}^{\dagger,\pi_{-i}}(s)\right],v^\pi_i:=\mathbb{E}_{s\sim d_1}\left[V_{1,i}^\pi(s)\right]$. 
\begin{definition}[NE]
A product policy $\pi$ is a Nash equilibrium (NE) if $v_i^\pi=v_i^{\dagger,\pi_{-i}},\forall i\in[M]$. And we call $\pi$ an $\varepsilon$-approximate NE if $\max_{i\in[M]}\{v_i^{\dagger,\pi_{-i}}-v_i^\pi\}<\varepsilon$. 
\end{definition}
The coarse correlated equilibrium (CCE) is a relaxed version of Nash equilibrium in which we consider general correlated policies instead of product policies. 
\begin{definition}[CCE]
A correlated policy $\pi$ is a CCE if $V_{h,i}^{\dagger,\pi_{-i}}(s)\leq V_{h,i}^\pi(s)$ for all $s\in\mathcal{S},h\in[H],i\in[M]$. And we call $\pi$ an $\varepsilon$-approximate CCE if $\max_{i\in[M]}\{v_i^{\dagger,\pi_{-i}}-v_i^\pi\}<\varepsilon$. 
\end{definition}
The correlated equilibrium (CE) is another relaxation of the Nash equilibrium. To define CE, we first introduce the concept of strategy modification: A strategy modification $\omega_i:=\{\omega_{h,i}\}_{h\in[H]}$ for player $i$ is a set of $H$ functions from $\mathcal{S}\times\mathcal{A}_i$ to $\mathcal{A}_i$. Let $\Omega_i:=\{\Omega_{h,i}\}_{h\in[H]}$ denote the set of all possible strategy modifications for player $i$. One can compose a strategy modification $\omega_i$ with any Markov policy $\pi$ and obtain a new policy $\omega_i\circ\pi$ such that when policy $\pi$ chooses to play $\bm{a}:=(a_1,\ldots,a_M)$ at state $s$ and step $h$, policy $\omega_i\circ\pi$ will play $(a_1,\ldots,a_{i-1},\omega_{h,i}(s,a_i), a_{i+1},\ldots,a_M)$ instead.
\begin{definition}[CE]
A correlated policy $\pi$ is a CE if $\max_{\omega_i\in\Omega_i}V_{h,i}^{\omega_i\circ\pi}(s)\leq V_{h,i}^\pi(s)$ for all $(s,h)\in\mathcal{S}\times[H],i\in[M]$. And we call $\pi$ an $\varepsilon$-approximate CE if $\max_{i\in[M]}\{\max_{\omega_i\in\Omega_i}v_i^{\omega_i\circ\pi}-v_i^\pi\}<\varepsilon$. 
\end{definition}

\begin{remark}
For general-sum Markov Games, we have $\{\mathrm{NE}\}\subseteq \{\textrm{CE}\}\subseteq \{\textrm{CCE}\}$, so that they form a nested set of notions of equilibria \citep{roughgarden2010algorithmic}. While there exist algorithms to approximately compute the Nash equilibrium \citep{berg2017exclusion}, the computation of NE for general-sum games in the worst case is still PPAD-hard \citep{daskalakis2013complexity}. On the other hand, CCE and CE can be solved in polynomial time using linear programming (Examples include \cite{papadimitriou2008computing,blum2008regret}). Therefore, in this paper we study both NE and these weaker equilibrium concepts that permit more computationally efficient solutions.
\end{remark}

\subsection{Low-Rank Markov Games}
In this paper, we consider the class of low-rank Markov games. A Markov game is called a low-rank Markov game if the transition probability at any time step $h$ has a latent low-rank structure. 
\begin{definition}[Low-Rank Markov Game] We call a Markov game a low-rank Markov game if for any $s,s^\prime\in\mathcal{S},\bm{a}\in\mathcal{A},h\in[H],i\in[M]$, we have $P_h^\star(s^\prime\vert s,\bm{a})=\phi_h^\star(s,\bm{a})^\top w_h^\star(s^\prime)$, where $\Vert\phi_h^\star(s,\mathbf{a})\Vert_2\leq 1$ and $\Vert w_h^\star(s^\prime)\Vert_2\leq\sqrt{d}$ for all $(s,\bm{a},s^\prime)$.
\end{definition}
A special case of low-rank Markov game is the Block Markov game:
\begin{definition}[Block Markov Game] \label{def:Markov_game}
Consider any $h\in[H]$. A Block Markov game has an emission distribution $o_h(\cdot\vert z)\in\Delta_\mathcal{S}$ and a latent state space transition $T_h(z^\prime\vert z,\bm{a})$, such that for any $s\in\mathcal{S},o_h(s\vert z)>0$ for a unique latent state $z\in\mathcal{Z}$, denoted as $\psi_h^\star(s)$. Denote $Z=\vert\mathcal{Z}\vert$. Together with the ground truth decoder $\psi_h^\star$, it defines the transitions $P_h^\star(s^\prime\vert s,a)=\sum_{z^\prime\in\mathcal{Z}}o_h(s^\prime\vert z^\prime)T_h(z^\prime\vert\psi_h^\star(s),a)$.
\end{definition}
With the definition of the Block Markov game, one can naturally derive a feature vector that in addition takes the one-hot form: we just need to let the ground truth $\phi^\star_h(s,\bm{a})$ at step $h$ be a $Z\cdot A$-dimensional vector $e_{(\psi_h^\star(s),\bm{a})}$ where $e_i$ is the $i$-th basis vector. Correspondingly, for any $s\in\mathcal{S},w^\star_h(s)$ is a $Z\cdot A$ dimensional vector such that the $(z,\bm{a})$-th entry is $\sum_{z^\prime\in\mathcal{Z}}o_h(s\vert z^\prime)T_h(z^\prime\vert z,\bm{a})$. Then $P_h^\star(s^\prime\vert s,\bm{a})=\phi^\star_h(s,\bm{a})^\top w^\star_h(s^\prime)$, so that the Block Markov game is a low-rank Markov game with rank $d=Z\cdot A$. \looseness=-1

\paragraph{Learning Objective}
The goal of multi-agent reinforcement learning is to design algorithms for Markov games that find an $\varepsilon$-approximate equilibrium (NE, CCE, CE) from a small number of interactions with the environment. We focus on the low-rank Markov games whose feature vector $\phi^\star$ and transition probability $P^\star$ are both \textit{unknown}, and the goal is to identify a $\varepsilon$-approximate equilibrium policy with a number of interactions scaling polynomially with $d,A,H,\frac{1}{\varepsilon}$ and the log-cardinality of the function class, without depending on the number of raw states which could be infinite. \looseness=-1

\section{Algorithm Description}

\begin{algorithm}[t]
    \caption{Model-based Representation Learning for Multi-player General-sum Low-Rank Markov Game with UCB-driven Exploration (\textsc{MBRL\_MG2})}
    \label{alg:mb}
    \begin{algorithmic}[1]
        \STATE\textbf{Input:} Regularizer $\lambda$, iteration $N$, parameter $\{\alpha^{(n)}\}_{n=1}^N,\{\zeta^{(n)}\}_{n=1}^N$.
        \STATE Initialize $\pi^{(0)}$ to be uniform; set $\mathcal{D}^{(0)}_h=\emptyset$, $\tilde{\mathcal{D}}^{(0)}_h=\emptyset,\ \forall h\in[H]$.
        \FOR{episode $n=1,2,\cdots,N$} 
            \FOR{step $h=H,H-1\ldots,1$}
                \STATE Collect two triples $(s,\bm{a},s^\prime),(\tilde{s}^\prime,\tilde{\bm{a}}^\prime,\tilde{s}^{\prime\prime})$ with \\
                $
                    s\sim d_{P^\star,h}^{\pi^{(n-1)}},\ \bm{a}\sim U(\mathcal{A}),\  s^\prime\sim P_h^\star(s,\bm{a}),$\\ 
                    $\tilde{s}\sim d_{P^\star,h-1}^{\pi^{(n-1)}},\ \tilde{\bm{a}}\sim U(\mathcal{A}),\ \tilde{s}^\prime\sim P_{h-1}^\star(\tilde{s},\tilde{\bm{a}}),\ \tilde{\bm{a}}^\prime\sim U(\mathcal{A}),\ \tilde{s}^{\prime\prime}\sim P^\star_h(\tilde{s}^\prime,\tilde{\bm{a}}^\prime).
                $
                \STATE Update datasets: $\mathcal{D}^{(n)}_h=\mathcal{D}^{(n-1)}_h\cup\{(s,\bm{a},s^\prime)\},\ \tilde{\mathcal{D}}^{(n)}_h=\tilde{\mathcal{D}}^{(n-1)}_h\cup\{(\tilde{s}^\prime,\tilde{\bm{a}}^\prime,\tilde{s}^{\prime\prime})\}$.
                \STATE Model-based representation learning via MLE:  \\
                \centering
                $(\hat{w}_h^{(n)},\hat{\phi}_h^{(n)}):=\argmax_{(w,\phi)\in\mathcal{M}_h}\mathbb{E}_{\mathcal{D}_h^{(n)}\cup\tilde{\mathcal{D}}_h^{(n)}}\left[\log w(s^\prime)^\top\phi(s,\bm{a})\right],\hat P_h^{(n)}(s^\prime\vert s,\bm{a})=\hat{\phi}^{(n)}_h(s,\bm{a})^\top\hat{w}_h^{(n)}(s^\prime)$.
                \ENDFOR
            \STATE Compute $\hat{\beta}_h^{(n)}$ from \eqref{eq:bonus}. 
            \STATE Compute the equilibrium policy through the planning oracle:
            \begin{align*}
                \pi^{(n)}= \textsc{NE/CE/CCE}\left(\left\{\hat{P}^{(n)}_h\right\}_{h\in[H]} ,\left\{r_{h,i}+\beta_h^{(n)}\right\}_{h\in[H],i\in[M]}\right).
            \end{align*}
            \STATE Compute the gap:
            $\Delta^{(n)}=\max_{i\in[M]}\left\{v^{\pi^{(n)}}_{\hat{P}^{(n)},r_i+\beta^{(n)}}-v^{\pi^{(n)}}_{\hat{P}^{(n)},r_i-\beta^{(n)}}\right\}+2H\sqrt{A\zeta^{(n)}}$.
        \ENDFOR 
        \STATE\textbf{Return } $\hat{\pi}=\pi^{(n^\star)}$ where $n^\star=\argmin_{n\in[N]}\Delta^{(n)}$.
    \end{algorithmic}
\end{algorithm}

\begin{algorithm}[th!]
    \caption{Model-free Representation Learning for Multi-player General-sum Low-Rank Markov Game with UCB-driven Exploration (\textsc{MFRL\_MG2})}
    \label{alg:mf}
    \begin{algorithmic}[1]
        \STATE\textbf{Input:} Regularizer $\lambda$, iteration $N$, parameter $\{\alpha^{(n)}\}_{n=1}^N,\{\zeta^{(n)}\}_{n=1}^N$.
        \STATE Initialize $\pi^{(0)}$ to be uniform; set $\mathcal{D}^{(0)}_h=\emptyset$, $\tilde{\mathcal{D}}^{(0)}_h=\emptyset,\ \forall h\in[H]$.
        \FOR{episode $n=1,2,\cdots,N$} 
            \STATE Set $\overline{V}^{(n)}_{H+1,i}\leftarrow 0,\ \underline{V}^{(n)}_{H+1,i}\leftarrow 0$, $\forall i\in[m]$.
            \FOR{step $h=H,H-1\ldots,1$}
                \STATE Collect two triples $(s,\bm{a},s^\prime),(\tilde{s}^\prime,\tilde{\bm{a}}^\prime,\tilde{s}^{\prime\prime})$ with \\
                $
                    s\sim d_{P^\star,h}^{\pi^{(n-1)}},\ \bm{a}\sim U(\mathcal{A}),\  s^\prime\sim P_h^\star(s,\bm{a}),$\\ 
                    $\tilde{s}\sim d_{P^\star,h-1}^{\pi^{(n-1)}},\ \tilde{\bm{a}}\sim U(\mathcal{A}),\ \tilde{s}^\prime\sim P_{h-1}^\star(\tilde{s},\tilde{\bm{a}}),\ \tilde{\bm{a}}^\prime\sim U(\mathcal{A}),\ \tilde{s}^{\prime\prime}\sim P^\star_h(\tilde{s}^\prime,\tilde{\bm{a}}^\prime).
                $
                \STATE Update datasets: $\mathcal{D}^{(n)}_h=\mathcal{D}^{(n-1)}_h\cup\{(s,\bm{a},s^\prime)\},\ \tilde{\mathcal{D}}^{(n)}_h=\tilde{\mathcal{D}}^{(n-1)}_h\cup\{(\tilde{s}^\prime,\tilde{\bm{a}}^\prime,\tilde{s}^{\prime\prime})\}$.
                \STATE Model-free representation learning: 
                \begin{align*} \hat{\phi}_h^{(n)}=\argmin_{\phi\in\Phi_h}\max_{f\in\mathcal{F}_h}[\min_\theta\mathcal{L}_{\lambda,\mathcal{D}_h^{(n)}}(\phi,\theta,f)-\min_{\tilde{\phi}\in\Phi_h,\tilde{\theta}}\mathcal{L}_{\lambda,\mathcal{D}_h^{(n)}}(\tilde{\phi},\tilde{\theta},f)]
                \end{align*}
                where $\mathcal{L}_{\lambda,\mathcal{D}}(\phi,\theta,f):=\mathbb{E}_{\mathcal{D}}\left[\left(\phi(s,\bm{a})^\top\theta-f(s^\prime)\right)^2\right]+\lambda\Vert\theta\Vert_2^2$ denotes the ridge regression loss.

            \STATE Compute $\hat{\beta}_h^{(n)}$ from \eqref{eq:bonus}. 
            \STATE 
            Compute data covariance: $\Lambda_h^{(n)}\leftarrow\sum_{(\tilde{s},\tilde{\bm{a}})\in\mathcal{D}_h^{(n)}\cup\tilde{\mathcal{D}}_h^{(n)}}\hat{\phi}_h^{(n)}(\tilde{s},\tilde{\bm{a}})\hat{\phi}_h^{(n)}(\tilde{s},\tilde{\bm{a}})^\top+\lambda I_d$.
            \STATE Estimate optimistic and pessimistic Q-functions:
            \begin{align*} \overline{Q}_{h,i}^{(n)}(\cdot,\cdot)&\leftarrow r_{h,i}(\cdot,\cdot)+\hat{\phi}_h^{(n)}(\cdot,\cdot)^\top\left(\Lambda_h^{(n)}\right)^{-1}\sum_{(\tilde{s},\tilde{\bm{a}},\tilde{s}^\prime)\in\mathcal{D}_h^{(n)}\cup\tilde{\mathcal{D}}_h^{(n)}}\hat{\phi}_h^{(n)}(\tilde{s},\tilde{\bm{a}})\overline{V}_{h+1,i}^{(n)}(\tilde s^\prime)+\hat{\beta}^{(n)}_h(\cdot,\cdot),\forall i\in[m]\\ \underline{Q}_{h,i}^{(n)}(\cdot,\cdot)&\leftarrow r_{h,i}(\cdot,\cdot)+\hat{\phi}_h^{(n)}(\cdot,\cdot)^\top\left(\Lambda_h^{(n)}\right)^{-1}\sum_{(\tilde{s},\tilde{\bm{a}},\tilde{s}^\prime)\in\mathcal{D}_h^{(n)}\cup\tilde{\mathcal{D}}_h^{(n)}}\hat{\phi}_h^{(n)}(\tilde{s},\tilde{\bm{a}})\underline{V}_{h+1,i}^{(n)}(\tilde s^\prime)-\hat{\beta}^{(n)}_h(\cdot,\cdot),\forall i\in[m]
            \end{align*}
            \STATE Find the policy $\pi_h^{(n)}$ by calling NFG equilibrium oracles according to \eqref{eq:nash},\ref{eq:cce},\ref{eq:ce}.
            \STATE Compute the V-functions:
            \begin{align*} \overline{V}_{h,i}^{(n)}(\cdot)\leftarrow\left(\mathbb{D}_{\pi^{(n)}_h}\overline{Q}_{h,i}^{(n)}\right)(\cdot),\quad\underline{V}_{h,i}^{(n)}(\cdot)\leftarrow\left(\mathbb{D}_{\pi^{(n)}_h}\underline{Q}_{h,i}^{(n)}\right)(\cdot),\forall i\in[m]
            \end{align*}
            where $(\mathbb{D}_{\pi}f)(s):=\mathbb{E}_{\bm{a}\sim\pi(s)}\left[f(s,\bm{a})\right],\forall f:\mathcal{S}\times\mathcal{A}\rightarrow\mathbb{R}$.
        \ENDFOR
            \STATE Compute the gap:
            $ \qquad\qquad\Delta^{(n)}=\max_{i\in[M]}\left\{\overline{v}^{(n)}_i-\underline{v}^{(n)}_i\right\}+2H\sqrt{A\zeta^{(n)}}
            $
            \\
            where $\overline{v}^{(n)}_i=\int_{\mathcal{S}}\overline{V}^{(n)}_{1,i}(s)d_1(s)\mathrm{d}s$, and $ \underline{v}^{(n)}_i{=}{}\int_{\mathcal{S}}\underline{V}^{(n)}_{1,i}(s)d_1(s)\mathrm{d}s$.
        \ENDFOR 
        \STATE\textbf{Return } $\hat{\pi}=\pi^{(n^\star)}$ where $n^\star=\argmin_{n\in[N]}\Delta^{(n)}$.
    \end{algorithmic}
\end{algorithm}

In this section, we present our algorithm \ouralg (see Alg. \ref{alg:mb} and \ref{alg:mf}). The algorithm comes in two different versions, depending on whether we learn the representation using the model-based or model-free method. Both versions share the same structure, which mainly consists of two modules: the representation learning module and the planning module. Denote $d_{P,h}^\pi$ as the state distribution under transition probability $P$ and policy $\pi$ at step $h$. 

\subsection{Representation Learning}
In the representation learning module, the main goal is to learn a representation function $\hat{\phi}$ to approximate $\phi^\star$, using the data collected so far. In each episode, the algorithm first collects some new data using the policy derived from the previous episode. Note that in our data collection scheme, for each time step $h$, we maintain two buffers $\mathcal{D}_h^{(n)}$ and $\tilde{\mathcal{D}}_h^{(n)}$ of transition tuples $(s,a,s^\prime)$ (line 6 of Alg. \ref{alg:mb} or line 7 of Alg. \ref{alg:mf}) which draw the state $s$ from slightly different distributions. Based on the data collected in history, the representation learning module estimates the feature $\hat{\phi}^{(n)}$ and transition probability $\hat{P}^{(n)}$. Our algorithm comes in two versions (model-based,  Alg. \ref{alg:mb}; model-free,  Alg. \ref{alg:mf}) based on whether we are given the full model class $\mathcal{M}_h$ of the transition probability, or only the function class of the state-action features $\Phi_h$. 

\paragraph{Model-based Representation Learning} In the model-based setting, we assume the access to a \textit{realizable} model class $\mathcal{M}_h=\{(w_h,\phi_h):w_h\in\Psi_h,\phi_h\in\Phi_h\},h\in[H]$ such that the true model is included in this class, i.e., $w_h^\star\in\Psi_h,\phi_h^\star\in\Phi_h,\ \forall h\in[H]$. Following the norm bounds on $\phi_h^\star,w_h^\star$, we assume that the same norm bounds hold for our function approximator, i.e., for any $\phi_h\in\Phi_h,w_h\in\Psi_h$, we have $\Vert\phi_h(s,\bm{a})\Vert_2\leq 1$ and $\Vert w_h(s^\prime)\Vert_2\leq\sqrt{d}$ for all $(s,\bm{a},s^\prime)$, and $\int\phi_h(s,\bm{a})^\top w_h(s^\prime)\mathrm{d}s^\prime=1$. Given the dataset $\mathcal{D}:=\mathcal{D}_h^{(n)}\cup\tilde{\mathcal{D}}_h^{(n)}$, \textsc{MBRepLearn} learns the features and transition probability using maximum likelihood estimation (MLE):
\begin{align*}
    \left(\hat{w}^{(n)}_h,\hat{\phi}^{(n)}_h\right)=\argmax_{(w,\phi)\in\mathcal{M}_h}\mathbb{E}_{\mathcal{D}}\left[\log\left(\phi(s,\bm{a})^\top w(s^\prime)\right)\right],\quad\hat{P}^{(n)}_h(s^\prime\vert s,\bm{a})=\hat{\phi}^{(n)}_h(s,\bm{a})^\top\hat{w}^{(n)}_h(s^\prime).
\end{align*}

\paragraph{Model-free Representation Learning} In the model-free setting, we are only given the function class of the feature vectors, $\Phi_h$, which we assume also includes the true feature $\phi^\star_h$. Given the dataset $\mathcal{D}:=\mathcal{D}_h^{(n)}\cup\tilde{\mathcal{D}}_h^{(n)}$, \textsc{MFRepLearn} aims to learn a feature vector that is able to linearly fit the Bellman backup of any function $f(s)$ in an appropriately chosen discriminator function class $\mathcal{F}_h$. To be precise, we aim to optimize the following objective:
\begin{align*}
    \min_{\phi\in\Phi_h}\max_{f\in\mathcal{F}_h}\left\{\min_\theta\mathbb{E}_\mathcal{D}\left[\left(\phi(s,\bm{a})^\top\theta-f(s^\prime)\right)^2\right]-\min_{\tilde{\theta},\tilde{\phi}\in\Phi_h}\mathbb{E}_\mathcal{D}\left[\left(\tilde{\phi}(s,\bm{a})^\top\tilde{\theta}-f(s^\prime)\right)^2\right]\right\},
\end{align*}
where the first term is the empirical squared loss and the second term is the conditional expectation of $f(s^\prime)$ given $(s,\bm{a})$, subtracted for the purpose of bias reduction. 

In practice, for applications where the raw observation states are high-dimensional, e.g. images, estimating the transition is often much harder than estimating the one-directional feature function. In such cases, we expect the $\Psi$ class to be much larger than the $\Phi$ class and the model-free approach to be more efficient.

\subsection{Planning}
Based on the feature vector and transition probability computed from the representation learning phase, a new policy $\pi^{(n)}$ is computed using the planning module. The planning phase is conducted with a Upper-Confidence-Bound (UCB) style approach, where a bonus $\beta^{(n)}$ is added to the reward function when computing the policy. For the model-based planning, we simply let the policy of the planning oracle be the NE (or CE or CCE) of the corresponding Markov game, and let $\Delta^{(n)}=\max_{i\in[M]}\left\{v^{\pi^{(n)}}_{\hat{P}^{(n)},r_i+\beta^{(n)}}-v^{\pi^{(n)}}_{\hat{P}^{(n)},r_i-\beta^{(n)}}\right\}+2H\sqrt{A\zeta^{(n)}}$, where $v^\pi_{P,r}$ is defined to be the value of an Markov game with transition probability $P$, reward function $r$ and policy $\pi$. For the model-free setting, the policy and optimality gap are computed using an LSVI-style algorithm. To be specific, we maintain both an optimistic and a pessimistic estimation of the value functions and the Q-value functions $\overline{V}_{h,i}^{(n)},\underline{V}_{h,i}^{(n)},\overline{Q}_{h,i}^{(n)},\underline{Q}_{h,i}^{(n)}$, which are updated according to Line 11-13 of Algorithm \ref{alg:mf}, where $\pi^{(n)}_h$ is the policy computed from $M$ induced Q-value functions $\tilde{Q}_{h,i}^{(n)}$, which are defined to be a close neighbor of $\overline{Q}_{h,i}^{(n)}$ in $\mathcal{N}_h$ with respect to the $\Vert\cdot\Vert_\infty$ metric, where $\mathcal{N}_h\subseteq\mathbb{R}^{\mathcal{S}\times\mathcal{A}}$ is a properly designed set of functions. The construction of $\mathcal{N}_h$ and the choice of $\tilde{Q}_{h,i}^{(n)}$ are deferred to the appendix. 

Depending on the problem settings, the policy $\pi^{(n)}_h$ takes either one of the following formulations:
\begin{itemize}
    \item For the NE, we compute $\pi^{(n)}_h=\left(\pi^{(n)}_{h,1},\pi^{(n)}_{h,2},\ldots,\pi^{(n)}_{h,M}\right)$ such that $\forall s\in\mathcal{S},i\in[M]$, 
    \begin{align}
    \label{eq:nash}
        \pi_{h,i}^{(n)}(\cdot\vert s)=\argmax_{\pi_{h,i}} \left(\mathbb{D}_{\pi_{h,i},\pi^{(n)}_{h,-i}}\tilde{Q}_{h,i}^{(n)}\right)(s).
    \end{align}
    \item For the CCE, we compute $\pi^{(n)}_h$ such that $\forall s\in\mathcal{S},i\in[M]$,
    \begin{align}
    \label{eq:cce}
        \max_{\pi_{h,i}}\left(\mathbb{D}_{\pi_{h,i},\pi^{(n)}_{h,-i}}\tilde{Q}_{h,i}^{(n)}\right)(s)\leq\left(\mathbb{D}_{\pi^{(n)}}\tilde{Q}_{h,i}^{(n)}\right)(s).
    \end{align}
    \item For CE, we compute $\pi^{(n)}_h$ such that $\forall s\in\mathcal{S},i\in[M]$,
    \begin{align}
    \label{eq:ce}
    \max_{\omega_{h,i}\in\Omega_{h,i}}\left(\mathbb{D}_{\omega_{h,i}\circ\pi^{(n)}_h}\tilde{Q}_{h,i}^{(n)}\right)(s)\leq\left(\mathbb{D}_{\pi^{(n)}}\tilde{Q}_{h,i}^{(n)}\right)(s).
    \end{align}
\end{itemize}
Without loss of generality we assume the solution to the above formulations is unique, if there exist multiple solutions, one can always adopt a deterministic selection rule such that it always outputs the same policy given the same inputs. We then define the optimality gap to be $ \Delta^{(n)}=\max_{i\in[M]}\left\{\overline{v}^{(n)}_i-\underline{v}^{(n)}_i\right\}+2H\sqrt{A\zeta^{(n)}}$, where $\overline{v}^{(n)}_i=\int_{\mathcal{S}}\overline{V}^{(n)}_{1,i}(s)d_1(s)\mathrm{d}s$, and $ \underline{v}^{(n)}_i{=}{}\int_{\mathcal{S}}\underline{V}^{(n)}_{1,i}(s)d_1(s)\mathrm{d}s$. 
\begin{remark}
    For the model-free algorithm, though in the algorithm description, the equilibrium policy $\pi_h(\cdot\vert s)$ needs to be computed for each state $s$, we actually only need to compute the policy for the states included in the history. And $\tilde{Q}_{h,i}^{(n)}$ can be found only using the linear weights and bonus function. Therefore, the complexity of the planning phase is only related with the size of the dataset, instead of the size of the whole state space. 
\end{remark}

The bonus term $\hat{\beta}_h^{(n)}$ is a linear bandit style bonus computed using the learned feature $\hat{\phi}^{(n)}_h$:
\begin{align}
\label{eq:bonus}
    \hat{\beta}^{(n)}_h(s,\bm{a}):=\min\left\{\alpha^{(n)}\left\Vert\hat{\phi}^{(n)}_h(s,\bm{a})\right\Vert_{\left(\hat{\Sigma}^{(n)}_h\right)^{-1}},H\right\}.
\end{align}
where $\hat{\Sigma}^{(n)}_h:=\sum_{(s,\bm{a})\in\mathcal{D}^{(n)}_h}\hat{\phi}^{(n)}_h(s,\bm{a})\hat{\phi}^{(n)}_h(s,\bm{a})^\top+\lambda I_d$ is the empirical covariance matrix.

\section{Theoretical Results}
In this section, we provide the theoretical guarantees of the proposed algorithm for both the model-based and model-free approaches. We denote $\vert\mathcal{M}\vert:=\max_{h\in[H]}\vert\mathcal{M}_h\vert$ and $\vert\Phi\vert:=\max_{h\in[H]}\vert\Phi_h\vert$. The first theorem provides a guarantee of the sample complexity for the model-based method. 

\begin{theorem}[PAC guarantee of Algorithm \ref{alg:mb}]
\label{thm:mb}
When Alg. \ref{alg:mb} is applied with parameters 
\begin{align*}
    \lambda=\Theta\left(d\log\frac{NH\vert\Phi\vert}{\delta}\right),\quad\alpha^{(n)}=\Theta\left(Hd\sqrt{A\log\frac{\vert\mathcal{M}\vert HN}{\delta}}\right),\quad\zeta^{(n)}=\Theta\left(\frac{1}{n}\log\frac{\vert\mathcal{M}\vert HN}{\delta}\right),
\end{align*}
by setting the number of episodes $N$ to be at most
\begin{align*}
    O\left(\frac{H^6d^4A^2}{\varepsilon^2}\log^2\left(\frac{HdA\vert\mathcal{M}\vert}{\delta\varepsilon}\right)\right),
\end{align*}
with probability $1-\delta$, the output policy $\hat{\pi}$ is an $\varepsilon$-approximate $\{\textsc{NE},\textsc{CCE},\textsc{CE}\}$. 
\end{theorem}
Theorem \ref{thm:mb} shows that \ouralg can find an $\varepsilon$-approximate $\{\textsc{NE},\textsc{CCE},\textsc{CE}\}$ by running the algorithm for at most $\tilde{O}\left(H^6d^4A^2\varepsilon^{-2}\right)$ episodes, which depends polynomially on the parameters $H,d,A,\varepsilon^{-1}$ and only has a logarithmic dependency on the cardinality of the model class $\vert\mathcal{M}\vert$. In particular, when reducing the Markov game to the single-agent MDP setting, the sample complexity of the model-based approach matches the result provided in \citep{uehara2021representation}, which is known to have the best sample complexity among all oracle efficient algorithms for low-rank MDPs. 

For model-free representation learning, we have the following guarantee:
\begin{theorem}[PAC guarantee of Algorithm \ref{alg:mf}]
\label{thm:mf}
When Alg. \ref{alg:mf} is applied with parameters
\begin{align*}
    \lambda=\Theta\left(d\log\frac{NH\vert\Phi\vert}{\delta}\right),\quad\alpha^{(n)}=\Theta\left(HAd\sqrt{M\log\frac{dNHAM\vert\Phi\vert}{\delta}}\right),\quad\zeta^{(n)}=\Theta\left(\frac{d^2A}{n}\log\frac{dNHAM\vert\Phi\vert}{\delta}\right),
\end{align*}
and the Markov game is a Block Markov game. When we set the number of episodes $N$ to be at most
\begin{align*}
    O\left(\frac{H^6d^4A^3M}{\varepsilon^2}\log^2\left(\frac{HdAM\vert\Phi\vert}{\delta\varepsilon}\right)\right),
\end{align*}
for an appropriately designed function class $\{\mathcal{N}_h\}_{h=1}^H$ and discriminator class $\{\mathcal{F}_h\}_{h=1}^H$, with probability $1-\delta$, the output policy $\hat{\pi}$ is an $\varepsilon$-approximate $\{\textsc{NE},\textsc{CCE},\textsc{CE}\}$. 
\end{theorem}
For the model-free block Markov game setting, the number of episodes required to find an $\varepsilon$-approximate $\{\textsc{NE},\textsc{CCE},\textsc{CE}\}$ becomes $\tilde{O}\left(H^6d^4A^3M\varepsilon^{-2}\right)$. While it has a worse dependency compared with the model-based approach, the advantage of the model-free approach is it doesn't require the full model class of the transition probability but only the model class of the feature vector, which applies to a wider range of RL problems. 

The proofs of Theorem \ref{thm:mb} and Theorem \ref{thm:mf} are deferred to Appendix \ref{sec:mb_analysis} and \ref{sec:mf_analysis}. Theorem \ref{thm:mb} and Theorem \ref{thm:mf} show that \ouralg learns low-rank Markov games in a statistically efficient and oracle-efficient manner. We also remark that our modular analysis can be of independent theoretical interest. 

The result in Theorem \ref{thm:mb} is tractable in games with a moderate number of players. However, in applications with a large number of players, such as the scenario of autonomous traffic control, the total number of players in the game can be so large that the joint action space size $A = \tilde A^M$ dominates all other factors in the sample complexity bound. This exponential scaling with the number of players is sometimes referred to as the \textit{curse of multi-player}. The only known class of algorithms that 
overcomes this challenge in Markov games is V-learning \cite[see, e.g.,][]{bai2020near,jin2021v}, a value-based method that fits the V-function rather than the Q-function, thus removing the dependency on the action space size. However, V-learning only works for tabular Markov games with finite state and action spaces. Extending V-learning to the function approximation setting is extremely non-trivial, because even in the single agent setting, no known algorithm can achieve sample efficient learning in MDPs while only performing function approximation on the V-function.

In this section we take a different approach that relies on the following observation. In a setting where the number of agents is large, there is often a spatial correlation among the agents, such that each agent's local state is only immediately affected by the agent's own action and the states of agents in its adjacency. For example, in smart traffic control, a vehicle's local environment is only immediately affected by the states of the vehicles around it. On the other hand, it takes time for the course of actions of a vehicle from afar to propagate its influence on the vehicle of reference. Such spatial structure motivates the definition of a \textit{factored} Markov Game.

In a factored Markov Game, each agent $i$ has its local state $s_i$, whose transition is affected by agent $i$'s action $\bm{a}_i$ and the state of the agents in its neighborhood $Z_i$. We remark that the factored Markov Game structure still allows an agent to be affected by all other agents in the long run, as long as the directed graph defined by the neighborhood sets $Z_i$ is connected. In particular, we have

\begin{definition}[Low-Rank Factored Markov Game]\label{def:factor_mg}
    We call a Markov game a low-rank factored Markov game if for any $s,s^\prime\in\mathcal{S},\bm{a}\in\mathcal{A},h\in[H],i\in[M]$, we have
\begin{align*}
    P_h^\star(s^\prime\vert s,\bm{a})=\prod_{i=1}^M\left[\phi_{h,i}^\star(s[Z_i],\bm{a}_i)^\top w_{h,i}^\star(s_i^\prime)\right].
\end{align*}
where $Z_i\subseteq[M]$, $\phi_{h,i}^\star(s[Z_i],\bm{a}_i),w_{h,i}^\star(s_i^\prime)\in\mathbb{R}^d$, $\Vert\phi_{h,i}^\star(s[Z_i],\bm{a}_i)\Vert_2\leq 1$ and $\Vert w_{h,i}^\star(s_i^\prime)\Vert_2\leq\sqrt{d}$ for all $(s[Z_i],\bm{a}_i,s^\prime_i)$. We assume $\vert Z_i\vert\leq L,\forall i\in[M]$. And we are given a group of model classes $\mathcal{M}_{h,i},h\in[H],i\in[M]$ such that $(\phi^\star_{h,i},w_{h,i}^\star)\in\mathcal{M}_{h,i}$.
\end{definition}

We are now ready to present our algorithm and result in the low-rank factored Markov Game setting. For simplicity, we focus on the model-based version. Surprisingly, the same algorithm \ouralg works in this setting, with the representation learning module replaced by solving the following MLE problem:
\begin{align}
    \label{eq:rep}
    (\hat{w}_{h,i}^{(n)},\hat{\phi}_{h,i}^{(n)}):=\argmax_{(w,\phi)\in\mathcal{M}_{h,i}}\mathbb{E}_{\mathcal{D}}\left[\log w(s_i^\prime)^\top\phi(s[Z_i],\bm{a}_i)\right],\hat{P}_h^{(n)}(s^\prime\vert s,\bm{a})=\prod_{i=1}^M\left(\hat{w}_{h,i}^{(n)}(s_i^\prime)^\top\hat{\phi}_{h,i}^{(n)}(s[Z_i],\bm{a}_i)\right),
\end{align}
as well as a few changes of variables. Define $\bar{\phi}_{h,i}^{(n)}(s,\bm{a})=\bigotimes_{j\in Z_i}\hat{\phi}^{(n)}_{h,j}(s[Z_j],\bm{a}_j)\in\mathbb{R}^{d^{\vert Z_i\vert}}$ where $\otimes$ means the Kronecker product. Let
\begin{align}
\label{eq:param}
    \hat{\beta}^{(n)}_h(s,\bm{a})&:=\sum_{i=1}^M\min\left\{\alpha^{(n)}\left\Vert\bar{\phi}^{(n)}_{h,i}(s,\bm{a})\right\Vert_{\left(\bar{\Sigma}^{(n)}_{h,i}\right)^{-1}},H\right\},\quad\Delta^{(n)}:=\max_{i\in[M]}\left\{\overline{v}^{(n)}_i-\underline{v}^{(n)}_i\right\}+2HM\sqrt{\tilde{A}\zeta^{(n)}}
\end{align}
where $\bar{\Sigma}^{(n)}_{h,i}:=\sum_{(s,\bm{a})\in\mathcal{D}^{(n)}_h}\bar{\phi}_{h,i}^{(n)}(s,\bm{a})\bar{\phi}_{h,i}^{(n)}(s,\bm{a})^\top+\lambda I_{d^{\vert Z_i\vert}}$. Then, \ouralg with $\bar{\phi}$ and the newly defined $\hat{\beta}^{(n)}_h,\Delta^{(n)}$ achieves the following guarantee: Denote $M=\max_{h\in[H],i\in[M]}\vert\mathcal{M}_{h,i}\vert$,
\begin{theorem}[PAC guarantee of \ouralg in Low-Rank Factored Markov Game]
\label{thm:mb_factor}
Suppose Alg. \ref{alg:mb} is applied with representation learning module \eqref{eq:rep} and $\beta^{(n)}_h$ and $\Delta^{(n)}_h$ are chosen according to \eqref{eq:param}. When we have $L=O(1)$ and parameters 
\begin{align*}
    \lambda=\Theta\left(Ld^L\log\frac{NHM\vert\Phi\vert}{\delta}\right),\quad\alpha^{(n)}=\Theta\left(H\tilde{A}d^L\sqrt{L\log\frac{\vert\mathcal{M}\vert HNM}{\delta}}\right),\quad\zeta^{(n)}=\Theta\left(\frac{1}{n}\log\frac{\vert\mathcal{M}\vert HNM}{\delta}\right),
\end{align*}
by setting the number of episodes $N$ to be at most
\begin{align*}
    O\left(\frac{M^4H^6d^{2(L+1)^2}\tilde{A}^{2(L+1)}}{\varepsilon^2}\log^2\left(\frac{HdALM\vert\mathcal{M}\vert}{\delta\varepsilon}\right)\right),
\end{align*}
with probability $1-\delta$, the output policy $\hat{\pi}$ is an $\varepsilon$-approximate $\{\textsc{NE},\textsc{CCE},\textsc{CE}\}$. 
\end{theorem}

\begin{remark}
This sample complexity only scales with $\exp(L)$ where $L$ is the degree of the connection graph, which is assumed to be $O(1)$ in Definition \ref{def:factor_mg} and in general much smaller than the total number of agents in practice. We remark that the factored structure is also previously studied in single-agent tabular MDPs (examples include \cite{chen2020efficient, kearns1999efficient, guestrin2002algorithm, guestrin2003efficient, strehl2007efficient}). \cite{chen2020efficient} provided a lower-bound showing that the exponential dependency on $L$ is unimprovable in the worst case. Therefore, our bound here is also nearly tight, upto polynomial factors.
\end{remark}

\section{Experiment}

\begin{table}[t] 
\centering
\caption{ \textbf{Top:} Short Horizon (H=3) exploitability of the final policy of DQN and \ouralg. \textbf{Bottom:} Long Horizon (H=10) exploitability of the final policy of DQN and \ouralg. Note that lower exploitability implies that the policy is closer to the NE policy.}
\begin{tabular}{p{0.22\linewidth}p{0.22\linewidth}p{0.22\linewidth}p{0.22\linewidth}} 
\toprule
                    & H=3 Environment 1     & H=3 Environment 2     & H=3 Environment 3         \\ 
\toprule
DQN                 & 0.0851 (0.1152)       & 0.0877 (0.1961)       & 0.0090 (0.0200)        \\ 
\ouralg             & 0.0013 (0.0018)       & 0.0032 (0.0032)       & 0.0004 (0.0009)        \\
\toprule
                    & H=10 Environment 1     & H=10 Environment 2     & H=10 Environment 3     \\ 
\toprule
DQN                 & 0.2730 (0.3270)        & 0.0340 (0.0760)        & 0.0320 (0.0170)     \\ 
\ouralg             & 0.0780 (0.1560)        & 0.0070 (0.0160)        & 0.0060 (0.0130)     \\
\toprule
\end{tabular}
\label{exp:table:exploitability_0sum}
\end{table}

In this section we investigate our algorithm with proof-of-concept empirical studies. We design our testing bed using rich observation Markov game with arbitrary latent transitions and rewards. To solve the rich observation Markov game, an algorithm must correctly decode the latent structure (thus learning the dynamics) as well as solve the latent Markov game to find the NE/CE/CCE strategies concurrently.
Below, we first introduce the setup of the experiments and then make comparisons with prior baselines in the two-player zero-sum setting. We then follow by showing the efficiency of  \ouralg in the general-sum setting. All further experiment details can be found in Appendix.~\ref{app:sec:exp}.
Here we focus on the model-free version of \ouralg. Specifically, we implement Algorithm.~\ref{alg:mf} with deep learning libraries \citep{paszke2017automatic}. We defer more details to Appendix.~\ref{app:exp:implementation}.

\mypar{Block Markov game} Block Markov game is a multi-agent extension of single agent Block MDP, as defined in Def.~\ref{def:Markov_game}. We design our Block Markov game by first randomly generating a tabular Markov game with horizon $H$, 3 states, 2 players each with 3 actions, and random reward matrix $R_h \in (0,1)^{3\times 3^2 \times H}$ and random transition matrix $T_h(s_h,a_h) \in \Delta_{\mathcal{S}_{h+1}}$. We provide more details (e.g., generation of rich observation) in Appendix ~\ref{app:exp:environment}. 

\mypar{Zero-sum Markov game} In this section we first show the empirical evaluations under the two-player zero-sum Markov game setting. For an environment with horizon $H$, the randomly generated matrix $R$ denotes the reward for player 1 and $-R^{\top}$ denotes the reward for player 2, respectively. For the zero-sum game setting, we designed two variants of Block Markov games: one with short horizon ($H=3$) and one with long horizon ($H=10$). We show in the following that \ouralg works in both settings where the other baseline could only work in the short horizon setting.

\mypar{Baseline} We adopt one open-sourced implementation of DQN~\citep{silver2016mastering} with fictitious self-play~\citep{heinrich2015fictitious}.

We keep track of the \textit{exploitability of the returned strategy} to evaluate the practical performances of the baselines. In the zero-sum setting, we only need to fix one agent (e.g., agent 2), train the other single agent (the exploiter) to maximize its corresponding return until convergence, and report the difference between the returns of the exploiter and the final return of the final policies. We include the exploitability in Table.~\ref{exp:table:exploitability_0sum}. We provide training curves in Appendix.~\ref{app:zero} for completeness. We note that compared with the Deep RL baseline, \ouralg shows a faster and more stable convergence in both environments, where the baseline is unstable during training and has a much larger exploitability. 

\paragraph{General-sum Markov game.}
In this section we move on to the general-sum setting. To our best knowledge, our algorithm is the only principled algorithm that can be implemented on scale under the general-sum setting. For the general sum setting, we can not just compare our returned value to the oracle NE values, because multiple NE/CCE values may exist. Instead, we keep track of the exploitability of the policy and plot the training curve on the exploitability in Fig.~\ref{app:exp:fig:gensum} (deferred to Appendix.~\ref{app:sec:exp}).  Note that in this case we need to test both policies since their reward matrices are independently sampled. 


\section{Discussion and Future Works}
In this paper, we present the first algorithm that solves general-sum Markov games under function approximation. We provide both a model-based and a model-free variant of the algorithm and present the theoretical guarantees. Empirically, we show that our algorithm outperforms existing deep RL baselines in a general benchmark with rich observation. Future work includes evaluating more challenging benchmarks and extending beyond the low-rank Markov game structure.

\bibliography{ref}

\newpage
\appendix
\section{Additional Notations}
Given a (possibly not normalized) transition probability $P:\mathcal{S}\times\mathcal{A}\times\mathcal{S}\times[H]\rightarrow[0,1]$ and a policy $\pi:\mathcal{S}\times[H]\rightarrow\Delta_{\mathcal{A}}$, we define the density function of the state-action pair $(s,\bm{a})$ at step $h$ under $P$ and $\pi$ by
\begin{align*}
    d^\pi_{P,1}(s,\bm{a}):=d_1(s)\pi_1(\bm{a}\vert s),\quad d^\pi_{P,h+1}(s,\bm{a}):=\sum_{\tilde{s}\in\mathcal{S},\tilde{\bm{a}}\in\mathcal{A}}d^\pi_{P,h}(\tilde{s},\tilde{\bm{a}})P_h(s\vert\tilde{s},\tilde{\bm{a}})\pi_{h+1}(\bm{a}\vert s),\forall h\geq 1.
\end{align*}
We abuse the notations a bit and denote $d^{\pi}_{P,h}(s)$ as the marginalized state distribution, i.e., $d^{\pi}_{P,h}(s)=\sum_{\bm{a}\in\mathcal{A}}d^{\pi}_{P,h}(s,\bm{a})$. For any $n\in[N],h\in[H]$, define
\begin{align*}
    \rho^{(n)}_h(s,\bm{a})&=\frac{1}{n}\sum_{i=1}^nd^{\pi^{(i)}}_{P^\star,h}(s)u_{\mathcal{A}}(\bm{a}),\\
    \tilde{\rho}^{(n)}_h(s,\bm{a})&=\frac{1}{n}\sum_{i=1}^n\mathbb{E}_{\tilde{s}\sim d^{\pi^{(i)}}_{P^\star,h-1},\tilde{\bm{a}}\sim U(\mathcal{A})}\left[P_h^\star(s\vert\tilde{s},\tilde{\bm{a}})u_{\mathcal{A}}(\bm{a})\right],\\
    \gamma^{(n)}_h(s,\bm{a})&=\frac{1}{n}\sum_{i=1}^nd^{\pi^{(i)}}_{P^\star,h}(s,\bm{a}).
\end{align*}
When we use the expectation $\mathbb{E}_{(s,\bm{a})\sim\rho}[f(s,\bm{a})]$ (or $\mathbb{E}_{s\sim\rho}[f(s)]$) for some (possibly not normalized) distribution $\rho$ and function $f$, we simply mean $\sum_{s\in\mathcal{S},\bm{a}\in\mathcal{A}}\rho(s,\bm{a})f(s,\bm{a})$ (or $\sum_{s\in\mathcal{S}}\rho(s)f(s)$) so that the expectation can be naturally extended to the unnormalized distributions. For an iteration $n$, a distribution $\rho$ and a feature $\phi$, we denote the expected feature covariance as
\begin{align*}
    \Sigma_{n,\rho,\phi}=n\mathbb{E}_{(s,\bm{a})\sim\rho}\left[\phi(s,\bm{a})\phi(s,\bm{a})^\top\right]+\lambda I_d.
\end{align*}
Meanwhile, define the empirical covariance by
\begin{align*}
    \hat{\Sigma}_{h,\phi}^{(n)}:=\sum_{(s,\bm{a})\in\mathcal{D}_h^{(n)}}\phi(s,\bm{a})\phi(s,\bm{a})^\top+\lambda I_d.
\end{align*}

\section{Analysis of the Model-Based Method}
\label{sec:mb_analysis}
\subsection{Bellman's Equation}
Define $\overline{V}^{(n)}_{H+1,i}\leftarrow 0,\ \underline{V}^{(n)}_{H+1,i}\leftarrow 0,\ \forall i\in[m]$, and we recursively define the Q-value and values by 
\begin{align*} 
    &\overline{Q}_{h,i}^{(n)}(s,\bm{a})\leftarrow r_{h,i}(s,\bm{a})+\left(\hat{P}^{(n)}_h\overline{V}_{h+1,i}^{(n)}\right)(s,\bm{a})+\hat{\beta}^{(n)}_h(s,\bm{a}),\\
    &\underline{Q}_{h,i}^{(n)}(s,\bm{a})\leftarrow r_{h,i}(s,\bm{a})+\left(\hat{P}^{(n)}_h\underline{V}_{h+1,i}^{(n)}\right)(s,\bm{a})-\hat{\beta}^{(n)}_h(s,\bm{a}),\\
    &\overline{V}_{h,i}^{(n)}(s)\leftarrow\left(\mathbb{D}_{\pi^{(n)}_h}\overline{Q}_{h,i}^{(n)}\right)(s),\quad\underline{V}_{h,i}^{(n)}(s)\leftarrow\left(\mathbb{D}_{\pi^{(n)}_h}\underline{Q}_{h,i}^{(n)}\right)(s),\quad\forall s\in\mathcal{S},\bm{a}\in\mathcal{A},i\in[M],h\in[H].
\end{align*}
where $\pi_h^{(n)}$ is the policy computed by the planning oracle. One may verify that $\pi_h^{(n)}$ satisfies the following relation:
\begin{itemize}
    \item For the NE, we have $\pi^{(n)}_h=\left(\pi^{(n)}_{h,1},\pi^{(n)}_{h,2},\ldots,\pi^{(n)}_{h,M}\right)$ and $\forall s\in\mathcal{S},i\in[M]$, 
    \begin{align*}
        \pi_{h,i}^{(n)}(\cdot\vert s)=\argmax_{\pi_{h,i}} \left(\mathbb{D}_{\pi_{h,i},\pi^{(n)}_{h,-i}}\overline{Q}_{h,i}^{(n)}\right)(s).
    \end{align*}
    \item For the CCE, we have $\forall s\in\mathcal{S},i\in[M]$,
    \begin{align*}
        \max_{\pi_{h,i}}\left(\mathbb{D}_{\pi_{h,i},\pi^{(n)}_{h,-i}}\tilde{Q}_{h,i}^{(n)}\right)(s)\leq\left(\mathbb{D}_{\pi^{(n)}}\overline{Q}_{h,i}^{(n)}\right)(s).
    \end{align*}
    \item For CE, we have $\forall s\in\mathcal{S},i\in[M]$,
    \begin{align*}
    \max_{\omega_{h,i}\in\Omega_{h,i}}\left(\mathbb{D}_{\omega_{h,i}\circ\pi^{(n)}_h}\tilde{Q}_{h,i}^{(n)}\right)(s)\leq\left(\mathbb{D}_{\pi^{(n)}}\overline{Q}_{h,i}^{(n)}\right)(s).
    \end{align*}
\end{itemize}
Furthermore, we can verify that $v^{\pi^{(n)}}_{\hat{P}^{(n)},r_i+\beta^{(n)}}=\int_{\mathcal{S}}\overline{V}^{(n)}_{1,i}(s)d_1(s)\mathrm{d}s$, and $v^{\pi^{(n)}}_{\hat{P}^{(n)},r_i-\beta^{(n)}}=\int_{\mathcal{S}}\underline{V}^{(n)}_{1,i}(s)d_1(s)\mathrm{d}s$. We also denote $\overline{v}^{(n)}_i:=v^{\pi^{(n)}}_{\hat{P}^{(n)},r_i+\beta^{(n)}}$ and $\underline{v}^{(n)}_i:=v^{\pi^{(n)}}_{\hat{P}^{(n)},r_i-\beta^{(n)}}$.

\subsection{High Probability Events}
\label{subsec:hp}
We define the following event
\begin{align*}
    \mathcal{E}_1&:\ \forall n\in[N],h\in[H],\rho\in\left\{\rho^{(n)}_h,\tilde{\rho}^{(n)}_h\right\},\quad\mathbb{E}_{(s,\bm{a})\sim\rho}\left[\left\Vert\hat{P}^{(n)}_h(\cdot\vert s,\bm{a})-P^\star_h(\cdot\vert s,\bm{a})\right\Vert_1^2\right]\leq\zeta^{(n)},\\
    \mathcal{E}_2&:\ \forall n\in[N],h\in[H],\phi_h\in\Phi_h,s\in\mathcal{S},\bm{a}\in\mathcal{A},\quad\Vert\phi_h(s,\bm{a})\Vert_{\left(\hat{\Sigma}^{(n)}_{h,\phi_h}\right)^{-1}}=\Theta\left(\Vert\phi_h(s,\bm{a})\Vert_{\Sigma^{-1}_{n,\rho^{(n)}_h,\phi_h}}\right)\\
    \mathcal{E}&:=\mathcal{E}_1\cap\mathcal{E}_2. 
\end{align*}

To prove $\mathcal{E}$ holds with a high probability, we first introduce the following MLE guarantee, whose original version can be found in \citep{agarwal2020flambe}:
\begin{lemma}[MLE guarantee]\label{lem:mle_mb}
For a fixed episode $n$ and any step $h$, with probability $1-\delta$, 
\begin{align*}
   \mathbb{E}_{(s,\bm{a})\sim\{0.5\rho^{(n)}_h+0.5\tilde{\rho}^{(n)}_h\}}\left[\left\Vert\hat{P}^{(n)}_h(\cdot\vert s,\bm{a})-P^\star_h(\cdot\vert s,\bm{a})\right\Vert^2_1\right]\lesssim\frac{1}{n}\log\frac{\vert\mathcal{M}\vert}{\delta}. 
\end{align*}
As a straightforward corollary, with probability $1-\delta$,
\begin{align}
 \forall n\in\mathbb{N}^+,\forall h\in[H],\quad\mathbb{E}_{(s,\bm{a})\sim\{0.5\rho^{(n)}_h+0.5\tilde{\rho}^{(n)}_h\}}\left[\left\Vert\hat{P}^{(n)}_h(\cdot\vert s,\bm{a})-P_h^\star(\cdot\vert s,\bm{a})\right\Vert^2_1\right]\lesssim  \frac{1}{n}\log\frac{nH\vert\mathcal{M}\vert}{\delta}.\label{eq:MLE}
\end{align}
\end{lemma}
\begin{proof}
See Agarwal et al.\citep{agarwal2020flambe} (Theorem 21).
\end{proof}

Based on Lemma \ref{lem:mle_mb} and Lemma \ref{lem:con} in Appendix \ref{sec:aux}, we directly get the following guarantee: 
\begin{lemma}\label{lem:model_based_hp_mb}
When $\hat{P}_h^{(n)}$ is computed using Alg. \ref{alg:mb}, if we set  
\begin{align*}
\lambda=\Theta\left(d\log\frac{NH\vert\Phi\vert}{\delta}\right),\ \zeta^{(n)}=\Theta\left(\frac{1}{n}\log\frac{\vert\mathcal{M}\vert HN}{\delta}\right),    
\end{align*}
then $\mathcal{E}$ holds with probability at least $1-\delta$. 
\end{lemma}

\subsection{Statistical Guarantees}
\begin{lemma}[One-step back inequality for the learned model]
\label{lem:useful2_mb}
Suppose the event $\mathcal{E}$ holds. Consider a set of functions $\{g_h\}^H_{h=1}$ that satisfies $g_h\in\mathcal{S}\times\mathcal{A}\rightarrow\mathbb{R}_+$, s.t. $\Vert g_h\Vert_\infty\leq B$. For any given policy $\pi$, we have
\begin{align*}
    &\mathbb{E}_{(s,\bm{a})\sim d^\pi_{\hat{P}^{(n)},h}}\left[g_h(s,\bm{a})\right]\\
    \leq&\left\{
    \begin{aligned}
        &\sqrt{A\mathbb{E}_{(s,\bm{a})\sim\rho^{(n)}_1}\left[g_1^2(s,\bm{a})\right]},\quad h=1\\
        &\mathbb{E}_{(\tilde{s},\tilde{\bm{a}})\sim d^\pi_{\hat{P}^{(n)},h-1}}\left[\min\left\{\left\Vert\hat{\phi}^{(n)}_{h-1}(\tilde{s},\tilde{\bm{a}})\right\Vert_{\Sigma_{n,\rho^{(n)}_{h-1},\hat{\phi}^{(n)}_{h-1}}^{-1}}\sqrt{nA\mathbb{E}_{(s,\bm{a})\sim\tilde{\rho}^{(n)}_h}\left[g_h^2(s,\bm{a})\right]+B^2\lambda d+B^2 n\zeta^{(n)}}, B\right\}\right],\quad h\geq 2
    \end{aligned}
    \right.
\end{align*}
\end{lemma}
Recall $\Sigma_{n,\rho^{(n)}_h,\hat{\phi}^{(n)}_h}=n\mathbb{E}_{(s,\bm{a})\sim\rho^{(n)}_h}\left[\hat{\phi}^{(n)}_h(s,\bm{a})\hat{\phi}^{(n)}_h(s,\bm{a})^\top\right]+\lambda I_d$. 
\begin{proof}
For step $h=1$, we have
\begin{align*}
    \mathbb{E}_{(s,\bm{a})\sim d^{\pi}_{\hat{P}^{(n)},1}}\left[g_1(s,\bm{a})\right]=&\mathbb{E}_{s\sim d_1,\bm{a}\sim\pi_1(s)}\left[g_1(s,\bm{a})\right]\\
    \leq&\sqrt{\max_{(s,\bm{a})}\frac{d_1(s)\pi_1(\bm{a}\vert s)}{\rho^{(n)}_1(s,\bm{a})}\mathbb{E}_{(s^\prime,\bm{a}^\prime)\sim\rho^{(n)}_1}\left[g_1^2(s^\prime,\bm{a}^\prime)\right]}\\
    =&\sqrt{\max_{(s,\bm{a})}\frac{d_1(s)\pi_1(\bm{a}\vert s)}{d_1(s)u_{\mathcal{A}}(\bm{a})}\mathbb{E}_{(s^\prime,\bm{a}^\prime)\sim\rho^{(n)}_1}\left[g_1^2(s^\prime,\bm{a}^\prime)\right]}\\
    \leq&\sqrt{A\mathbb{E}_{(s,\bm{a})\sim\rho^{(n)}_1}\left[g_1^2(s,\bm{a})\right]}.
\end{align*}
For step $h=2,\ldots,H-1$, we observe the following one-step-back decomposition:
\begin{align*}
    &\mathbb{E}_{(s,\bm{a})\sim d^\pi_{\hat{P}^{(n)},h}}\left[g_h(s,\bm{a})\right]\\
    =&\mathbb{E}_{(\tilde{s},\tilde{\bm{a}})\sim d^\pi_{\hat{P}^{(n)},h-1},s\sim\hat{P}^{(n)}_{h-1}(\tilde{s},\tilde{\bm{a}}),\bm{a}\sim\pi_h(s)}\left[g_h(s,\bm{a})\right]\\
    =&\mathbb{E}_{(\tilde{s},\tilde{\bm{a}})\sim d^\pi_{\hat{P}^{(n)},h-1}}\left[\hat{\phi}^{(n)}_{h-1}(\tilde{s},\tilde{\bm{a}})^\top\int_{\mathcal{S}}\sum_{\bm{a}\in\mathcal{A}}\hat{w}^{(n)}_{h-1}(s)\pi_h(\bm{a}\vert s)g_h(s,\bm{a})\mathrm{d}s\right]\\ 
    =&\mathbb{E}_{(\tilde{s},\tilde{\bm{a}})\sim d^\pi_{\hat{P}^{(n)},h-1}}\left[\min\left\{\hat{\phi}^{(n)}_{h-1}(\tilde{s},\tilde{\bm{a}})^\top\int_{\mathcal{S}}\sum_{\bm{a}\in\mathcal{A}}\hat{w}^{(n)}_{h-1}(s)\pi_h(\bm{a}\vert s)g_h(s,\bm{a})\mathrm{d}s,B\right\}\right]\\ 
    \leq&\mathbb{E}_{(\tilde{s},\tilde{\bm{a}})\sim d^\pi_{\hat{P}^{(n)},h-1}}\left[\min\left\{\left\Vert\hat{\phi}^{(n)}_{h-1}(\tilde{s},\tilde{\bm{a}})\right\Vert_{\Sigma_{n,\rho^{(n)}_{h-1},\hat{\phi}^{(n)}_{h-1}}^{-1}}\left\Vert\int_{\mathcal{S}}\sum_{\bm{a}\in\mathcal{A}}\hat{w}^{(n)}_{h-1}(s)\pi_h(\bm{a}\vert s)g_h(s,\bm{a})\mathrm{d}s\right\Vert_{\Sigma_{n,\rho^{(n)}_{h-1},\hat{\phi}^{(n)}_{h-1}}}, B\right\}\right].
\end{align*} 
where we use the fact that $g_h$ is bounded by $B$. Then, 
{\small
\begin{align*}
    &\left\Vert\int_{\mathcal{S}}\sum_{\bm{a}\in\mathcal{A}}\hat{w}^{(n)}_{h-1}(s)\pi_h(\bm{a}\vert s)g_h(s,\bm{a})\mathrm{d}s\right\Vert^2_{\Sigma_{n,\rho^{(n)}_{h-1},\hat{\phi}^{(n)}_{h-1}}}\\
    \leq&\left(\int_{\mathcal{S}}\sum_{\bm{a}\in\mathcal{A}}\hat{w}^{(n)}_{h-1}(s)\pi_h(\bm{a}\vert s)g_h(s,\bm{a})\mathrm{d}s\right)^\top\left(n\mathbb{E}_{(s,\bm{a})\sim\rho^{(n)}_{h-1}}\left[\hat{\phi}^{(n)}_{h-1}(s,\bm{a})\hat{\phi}^{(n)}_{h-1}(s,\bm{a})^\top\right]+\lambda I_d\right)\left(\int_{\mathcal{S}}\sum_{\bm{a}\in\mathcal{A}}\hat{w}^{(n)}_{h-1}(s)\pi_h(\bm{a}\vert s)g_h(s,\bm{a})\mathrm{d}s\right)\\
    \leq&n\mathbb{E}_{(\tilde{s},\tilde{\bm{a}})\sim\rho^{(n)}_{h-1}}\left[\left(\int_{\mathcal{S}}\sum_{\bm{a}\in\mathcal{A}}\hat{w}^{(n)}_{h-1}(s)^\top\hat{\phi}^{(n)}_{h-1}(\tilde{s},\tilde{\bm{a}})\pi_h(\bm{a}\vert s)g_h(s,\bm{a})\mathrm{d}s\right)^2\right]+B^2\lambda d \tag{$\left\Vert\sum_{a\in\mathcal{A}}\pi_h(\bm{a}\vert s)g_h(s,\bm{a})\right\Vert_\infty\leq B$ and by assumption $\left\Vert\hat{w}^{(n)}_{h-1}(s)\right\Vert_2\leq\sqrt{d}$.}\\
    =&n\mathbb{E}_{(\tilde{s},\tilde{\bm{a}})\sim\rho^{(n)}_{h-1}}\left[\left(\mathbb{E}_{s\sim\hat{P}^{(n)}_{h-1}(\tilde{s},\tilde{\bm{a}}),\bm{a}\sim\pi_h(s)}\left[g_h(s,\bm{a})\right]\right)^2\right]+B^2\lambda d\\
    \leq&n\mathbb{E}_{(\tilde{s},\tilde{\bm{a}})\sim\rho^{(n)}_{h-1}}\left[\left(\mathbb{E}_{s\sim P_{h-1}^\star(\tilde{s},\tilde{\bm{a}}),\bm{a}\sim\pi_h(s)}\left[g_h(s,\bm{a})\right]\right)^2\right]+ B^2\lambda d+nB^2\xi^{(n)}\tag{Event $\mathcal{E}$}\\
    \leq&n\mathbb{E}_{(\tilde{s},\tilde{\bm{a}})\sim\rho^{(n)}_{h-1},s\sim P_{h-1}^\star(\tilde{s},\tilde{\bm{a}}),\bm{a}\sim\pi_h(s)}\left[g_h^2(s,\bm{a})\right]+B^2\lambda d+B^2n\xi^{(n)}.\tag{Jensen}\\
    \leq&nA\mathbb{E}_{(\tilde{s},\tilde{\bm{a}})\sim\rho^{(n)}_{h-1},s\sim P_{h-1}^\star(\tilde{s},\tilde{\bm{a}}),\bm{a}\sim U(\mathcal{A})}\left[g_h^2(s,\bm{a})\right]+ B^2\lambda d+B^2n\zeta^{(n)} \tag{Importance sampling}\\
    \leq&nA\mathbb{E}_{(s,\bm{a})\sim\tilde{\rho}^{(n)}_h}\left[g_h^2(s,\bm{a})\right]+ B^2\lambda d+B^2 n\zeta^{(n)}. \tag{Definition of $\tilde{\rho}^{(n)}_h$}
\end{align*}}
Combing the above results together, we get
\begin{align*}
    &\mathbb{E}_{(s,\bm{a})\sim d^\pi_{\hat{P}^{(n)},h}}\left[g_h(s,\bm{a})\right]\\
    \leq&\mathbb{E}_{(\tilde{s},\tilde{\bm{a}})\sim d^\pi_{\hat{P}^{(n)},h-1}}\left[\min\left\{\left\Vert\hat{\phi}^{(n)}_{h-1}(\tilde{s},\tilde{\bm{a}})\right\Vert_{\Sigma_{n,\rho^{(n)}_{h-1},\hat{\phi}^{(n)}_{h-1}}^{-1}}\left\Vert\int_{\mathcal{S}}\sum_{\bm{a}\in\mathcal{A}}\hat{w}^{(n)}_{h-1}(s)\pi_h(\bm{a}\vert s)g_h(s,\bm{a})\mathrm{d}s\right\Vert_{\Sigma_{n,\rho^{(n)}_{h-1},\hat{\phi}^{(n)}_{h-1}}}, B\right\}\right]\\
    \leq&\mathbb{E}_{(\tilde{s},\tilde{\bm{a}})\sim d^\pi_{\hat{P}^{(n)},h-1}}\left[\min\left\{\left\Vert\hat{\phi}^{(n)}_{h-1}(\tilde{s},\tilde{\bm{a}})\right\Vert_{\Sigma_{n,\rho^{(n)}_{h-1},\hat{\phi}^{(n)}_{h-1}}^{-1}}\sqrt{nA\mathbb{E}_{(s,\bm{a})\sim\tilde{\rho}^{(n)}_h}\left[g_h^2(s,\bm{a})\right]+B^2\lambda d+B^2 n\zeta^{(n)}}, B\right\}\right],
\end{align*}
which has finished the proof.
\end{proof}

\begin{lemma}[One-step back inequality for the true model]\label{lem:useful_mb} 
Consider a set of functions $\{g_h\}^H_{h=1}$ that satisfies $g_h\in\mathcal{S}\times\mathcal{A}\rightarrow\mathbb{R}_+$, s.t. $\Vert g_h\Vert_\infty\leq B$. Then for any given policy $\pi$, we have
\begin{align*}
    &\mathbb{E}_{(s,\bm{a})\sim d^\pi_{P^\star,h}}\left[g_h(s,\bm{a})\right]\\
    \leq&\left\{
    \begin{aligned}
        &\sqrt{A\mathbb{E}_{(s,\bm{a})\sim\rho^{(n)}_1}\left[g_1^2(s,\bm{a})\right]},\quad h=1\\
        &\mathbb{E}_{(\tilde{s},\tilde{\bm{a}})\sim d^\pi_{P^\star,h-1}}\left[\left\Vert\phi^\star_{h-1}(\tilde{s},\tilde{\bm{a}})\right\Vert_{\Sigma_{n,\gamma^{(n)}_{h-1},\phi^\star_{h-1}}^{-1}}\right]\sqrt{nA\mathbb{E}_{(s,\bm{a})\sim\rho^{(n)}_h}\left[g_h^2(s,\bm{a})\right]+B^2\lambda d},\quad h\geq 2
    \end{aligned}
    \right.
\end{align*}
\end{lemma}
Recall $\Sigma_{n,\gamma^{(n)}_h,\phi_h^\star}=n\mathbb{E}_{(s,\bm{a})\sim\gamma^{(n)}_h}\left[\phi_h^\star(s,\bm{a})\phi_h^\star(s,\bm{a})^\top\right]+\lambda I_d$.
\begin{proof}
For step $h=1$, we have
\begin{align*}
    \mathbb{E}_{(s,\bm{a})\sim d^{\pi}_{P^\star,1}}\left[g_1(s,\bm{a})\right]=&\mathbb{E}_{s\sim d_1,\bm{a}\sim\pi_1(s)}\left[g_1(s,\bm{a})\right]\\
    \leq&\sqrt{\max_{(s,\bm{a})}\frac{d_1(s)\pi_1(\bm{a}\vert s)}{\rho^{(n)}_1(s,\bm{a})}\mathbb{E}_{(s^\prime,\bm{a}^\prime)\sim\rho^{(n)}_1}\left[g_1^2(s^\prime,\bm{a}^\prime)\right]}\\
    =&\sqrt{\max_{(s,\bm{a})}\frac{d_1(s)\pi_1(\bm{a}\vert s)}{d_1(s)u_{\mathcal{A}}(\bm{a})}\mathbb{E}_{(s^\prime,\bm{a}^\prime)\sim\rho^{(n)}_1}\left[g_1^2(s^\prime,\bm{a}^\prime)\right]}\\
    \leq&\sqrt{A\mathbb{E}_{(s,\bm{a})\sim\rho^{(n)}_1}\left[g_1^2(s,\bm{a})\right]}.
\end{align*}
For step $h=2,\ldots,H-1$, we observe the following one-step-back decomposition:
\begin{align*}
    &\mathbb{E}_{(s,\bm{a})\sim d^\pi_{P^\star,h}}\left[g_h(s,\bm{a})\right]\\
    =&\mathbb{E}_{(\tilde{s},\tilde{\bm{a}})\sim d^\pi_{P^\star,h-1},s\sim P^\star_{h-1}(\tilde{s},\tilde{\bm{a}}),\bm{a}\sim\pi_h(s)}\left[g_h(s,\bm{a})\right]\\
    =&\mathbb{E}_{(\tilde{s},\tilde{\bm{a}})\sim d^\pi_{P^\star,h-1}}\left[\phi^\star_{h-1}(\tilde{s},\tilde{\bm{a}})^\top\int_{\mathcal{S}}\sum_{\bm{a}\in\mathcal{A}}w^\star_{h-1}(s)\pi_h(\bm{a}\vert s)g_h(s,\bm{a})\mathrm{d}s\right]\\ 
    \leq&\mathbb{E}_{(\tilde{s},\tilde{\bm{a}})\sim d^\pi_{P^\star,h-1}}\left[\left\Vert\phi^\star_{h-1}(\tilde{s},\tilde{\bm{a}})\right\Vert_{\Sigma_{n,\gamma^{(n)}_{h-1},\phi^\star_{h-1}}^{-1}}\right]\left\Vert\int_{\mathcal{S}}\sum_{\bm{a}\in\mathcal{A}}w^\star_{h-1}(s)\pi_h(\bm{a}\vert s)g_h(s,\bm{a})\mathrm{d}s\right\Vert_{\Sigma_{n,\gamma^{(n)}_{h-1},\phi^\star_{h-1}}}.
\end{align*} 
Then, 
{\small
\begin{align*}
    &\left\Vert\int_{\mathcal{S}}\sum_{\bm{a}\in\mathcal{A}}w^\star_{h-1}(s)\pi_h(\bm{a}\vert s)g_h(s,\bm{a})\mathrm{d}s\right\Vert^2_{\Sigma_{n,\gamma^{(n)}_{h-1},\phi^\star_{h-1}}}\\
    \leq&\left(\int_{\mathcal{S}}\sum_{\bm{a}\in\mathcal{A}}w^\star_{h-1}(s)\pi_h(\bm{a}\vert s)g_h(s,\bm{a})\mathrm{d}s\right)^\top\left(n\mathbb{E}_{(s,\bm{a})\sim\gamma^{(n)}_{h-1}}\left[\phi^\star_{h-1}(s,\bm{a})\phi^\star_{h-1}(s,\bm{a})^\top\right]+\lambda I_d\right)\left(\int_{\mathcal{S}}\sum_{\bm{a}\in\mathcal{A}}w^\star_{h-1}(s)\pi_h(\bm{a}\vert s)g_h(s,\bm{a})\mathrm{d}s\right)\\
    \leq&n\mathbb{E}_{(\tilde{s},\tilde{\bm{a}})\sim\gamma^{(n)}_{h-1}}\left[\left(\int_{\mathcal{S}}\sum_{\bm{a}\in\mathcal{A}}w^\star_{h-1}(s)^\top\phi^{\star}_{h-1}(\tilde{s},\tilde{\bm{a}})\pi_h(\bm{a}\vert s)g_h(s,\bm{a})\mathrm{d}s\right)^2\right]+B^2\lambda d \tag{Use the assumption $\left\Vert\sum_{\bm{a}\in\mathcal{A}}\pi_h(\bm{a}\vert s)g_h(s,\bm{a})\right\Vert_\infty\leq B$ and $\left\Vert w^\star_{h-1}(s)\right\Vert_2\leq\sqrt{d}$.}\\
    =&n\mathbb{E}_{(\tilde{s},\tilde{\bm{a}})\sim\gamma^{(n)}_{h-1}}\left[\left(\mathbb{E}_{s\sim P^\star_{h-1}(\tilde{s},\tilde{\bm{a}}),\bm{a}\sim\pi_h(s)}\left[g_h(s,\bm{a})\right]\right)^2\right]+B^2\lambda d\\
    \leq&n\mathbb{E}_{(\tilde{s},\tilde{\bm{a}})\sim\gamma^{(n)}_{h-1},s\sim P_{h-1}^\star(\tilde{s},\tilde{\bm{a}}),\bm{a}\sim\pi_h(s)}\left[g_h^2(s,\bm{a})\right]+ B^2\lambda d\tag{Jensen}\\
    \leq&nA\mathbb{E}_{(\tilde{s},\tilde{\bm{a}})\sim\gamma^{(n)}_{h-1},s\sim P_{h-1}^\star(\tilde{s},\tilde{\bm{a}}),\bm{a}\sim U(\mathcal{A})}\left[g_h^2(s,\bm{a})\right]+B^2\lambda d\tag{Importance sampling}\\
    \leq&nA\mathbb{E}_{(s,\bm{a})\sim\rho^{(n)}_h}\left[g_h^2(s,\bm{a})\right]+ B^2\lambda d, \tag{Definition of $\rho^{(n)}_h$}
\end{align*}}
Combing the above results together, we get
\begin{align*}
    &\mathbb{E}_{(s,\bm{a})\sim d^\pi_{P^\star,h}}\left[g_h(s,\bm{a})\right]\\
    =&\mathbb{E}_{(\tilde{s},\tilde{\bm{a}})\sim d^\pi_{P^\star,h-1},s\sim P^\star_{h-1}(\tilde{s},\tilde{\bm{a}}),\bm{a}\sim\pi_h(s)}\left[g_h(s,\bm{a})\right]\\
    \leq&\mathbb{E}_{(\tilde{s},\tilde{\bm{a}})\sim d^\pi_{P^\star,h-1}}\left[\left\Vert\phi^\star_{h-1}(\tilde{s},\tilde{\bm{a}})\right\Vert_{\Sigma_{n,\gamma^{(n)}_{h-1},\phi^\star_{h-1}}^{-1}}\right]\left\Vert\int_{\mathcal{S}}\sum_{\bm{a}\in\mathcal{A}}w^\star_{h-1}(s)\pi_h(\bm{a}\vert s)g_h(s,\bm{a})\mathrm{d}s\right\Vert_{\Sigma_{n,\gamma^{(n)}_{h-1},\phi^\star_{h-1}}}\\
    \leq&\mathbb{E}_{(\tilde{s},\tilde{\bm{a}})\sim d^\pi_{P^\star,h-1}}\left[\left\Vert\phi^\star_{h-1}(\tilde{s},\tilde{\bm{a}})\right\Vert_{\Sigma_{n,\gamma^{(n)}_{h-1},\phi^\star_{h-1}}^{-1}}\right]\sqrt{nA\mathbb{E}_{(s,\bm{a})\sim\rho^{(n)}_h}\left[g_h^2(s,\bm{a})\right]+B^2\lambda d},
\end{align*}
which has finished the proof.
\end{proof}

\begin{lemma}[Optimism for NE and CCE]
\label{lem:optimism_NE_CCE_mb}
Consider an episode $n\in[N]$ and set $\alpha^{(n)}=\Theta\left(H\sqrt{nA\zeta^{(n)}+d\lambda}\right)$. When the event $\mathcal{E}$ holds and the policy $\pi^{(n)}$ is computed by solving NE or CCE, we have
\begin{align*}
    \overline{v}_i^{(n)}(s)-v^{\dagger,\pi^{(n)}_{-i}}_i(s)\geq-H\sqrt{A\zeta^{(n)}},\quad\forall n\in[N],i\in[M].
\end{align*}
\end{lemma}
\begin{proof}
Define $\tilde{\mu}_{h,i}^{(n)}(\cdot\vert s):=\argmax_{\mu}\left(\mathbb{D}_{\mu,\pi_{h,-i}^{(n)}}Q^{\dagger,\pi^{(n)}_{-i}}_{h,i}\right)(s)$ as the best response policy for player $i$ at step $h$, and let $\tilde{\pi}_h^{(n)}=\tilde{\mu}^{(n)}_{h,i}\times\pi_{h,-i}^{(n)}$. Let $f^{(n)}_h(s,\bm{a})=\left\Vert\hat{P}^{(n)}_h(\cdot\vert s,\bm{a})-P^\star_h(\cdot\vert s,\bm{a})\right\Vert_1$, then according to the event $\mathcal{E}$, we have
\begin{align*}
    &\mathbb{E}_{(s,\bm{a})\sim\rho^{(n)}_h}\left[\left(f^{(n)}_h(s,\bm{a})\right)^2\right]\leq\zeta^{(n)},\quad\mathbb{E}_{(s,\bm{a})\sim\tilde{\rho}^{(n)}_h}\left[\left(f^{(n)}_h(s,\bm{a})\right)^2\right]\leq\zeta^{(n)},\quad\forall n\in[N], h\in[H],\\
    &\Vert\phi_h(s,\bm{a})\Vert_{\left(\hat{\Sigma}^{(n)}_{h,\phi_h}\right)^{-1}}=\Theta\left(\Vert\phi_h(s,\bm{a})\Vert_{\Sigma^{-1}_{n,\rho^{(n)}_h,\phi_h}}\right),\quad\forall n\in[N],h\in[H],\phi_h\in\Phi_h. 
\end{align*}
A direct conclusion of the event $\mathcal{E}$ is we can find an absolute constant $c$, such that
\begin{align*}
    \beta_h^{(n)}(s,\bm{a})=&\min\left\{\alpha^{(n)}\left\Vert\hat{\phi}^{(n)}_h(\tilde{s},\tilde{\bm{a}})\right\Vert_{\left(\Sigma_{h,\hat{\phi}^{(n)}_h}^{(n)}\right)^{-1}},H\right\}\\
    \geq&\min\left\{c\alpha^{(n)}\left\Vert\hat{\phi}^{(n)}_h(\tilde{s},\tilde{\bm{a}})\right\Vert_{\Sigma_{n,\rho^{(n)}_h,\hat{\phi}^{(n)}_h}^{-1}},H\right\},\quad\forall n\in[N],h\in[H]. 
\end{align*}
Next, we prove by induction that
\begin{align}
    \mathbb{E}_{s\sim d_{\hat{P}^{(n)},h}^{\tilde{\pi}^{(n)}}}\left[\overline{V}_{h,i}^{(n)}(s)-V^{\dagger,\pi^{(n)}_{-i}}_{h,i}(s)\right]\geq&\sum_{h^\prime=h}^H\mathbb{E}_{(s,\bm{a})\sim d_{\hat{P}^{(n)},h^\prime}^{\tilde{\pi}^{(n)}}}\left[\hat{\beta}_{h^\prime}^{(n)}(s,\bm{a})-H\min\left\{f^{(n)}_{h^\prime}(s,\bm{a}),1\right\}\right],\quad\forall h\in[H].\label{eq:induction1_mb}
\end{align}
First, notice that $\forall h\in[H]$,
\begin{align*}
    \mathbb{E}_{s\sim d_{\hat{P}^{(n)},h}^{\tilde{\pi}^{(n)}}}\left[\overline{V}_{h,i}^{(n)}(s)-V^{\dagger,\pi^{(n)}_{-i}}_{h,i}(s)\right]=&\mathbb{E}_{s\sim d_{\hat{P}^{(n)},h}^{\tilde{\pi}^{(n)}}}\left[\left(\mathbb{D}_{\pi^{(n)}_h}\overline{Q}_{h,i}^{(n)}\right)(s)-\left(\mathbb{D}_{\tilde{\pi}^{(n)}_h}Q^{\dagger,\pi^{(n)}_{-i}}_{h,i}\right)(s)\right]\\
    \geq&\mathbb{E}_{s\sim d_{\hat{P}^{(n)},h}^{\tilde{\pi}^{(n)}}}\left[\left(\mathbb{D}_{\tilde{\pi}^{(n)}_h}\overline{Q}_{h,i}^{(n)}\right)(s)-\left(\mathbb{D}_{\tilde{\pi}^{(n)}_h}Q^{\dagger,\pi^{(n)}_{-i}}_{h,i}\right)(s)\right]\\
    =&\mathbb{E}_{(s,\bm{a})\sim d_{\hat{P}^{(n)},h}^{\tilde{\pi}^{(n)}}}\left[\overline{Q}_{h,i}^{(n)}(s,\bm{a})-Q^{\dagger,\pi^{(n)}_{-i}}_{h,i}(s,\bm{a})\right],
\end{align*}
where the inequality uses the fact that $\pi_h^{(n)}$ is the NE (or CCE) solution for $\left\{\overline{Q}_{h,i}^{(n)}\right\}_{i=1}^M$. Now we are ready to prove \eqref{eq:induction1_mb}:
\begin{itemize}
    \item When $h=H$, we have
    \begin{align*}
        \mathbb{E}_{s\sim d_{\hat{P}^{(n)},H}^{\tilde{\pi}^{(n)}}}\left[\overline{V}_{H,i}^{(n)}(s)-V^{\dagger,\pi^{(n)}_{-i}}_{H,i}(s)\right]\geq&\mathbb{E}_{(s,\bm{a})\sim d_{\hat{P}^{(n)},H}^{\tilde{\pi}^{(n)}}}\left[\overline{Q}_{H,i}^{(n)}(s,\bm{a})-Q^{\dagger,\pi^{(n)}_{-i}}_{H,i}(s,\bm{a})\right]\\
        =&\mathbb{E}_{(s,\bm{a})\sim d_{\hat{P}^{(n)},H}^{\tilde{\pi}^{(n)}}}\left[\hat{\beta}_h^{(n)}(s,\bm{a})\right]\\
        \geq&\mathbb{E}_{(s,\bm{a})\sim d_{\hat{P}^{(n)},H}^{\tilde{\pi}^{(n)}}}\left[\hat{\beta}_h^{(n)}(s,\bm{a})-H\min\left\{f_H^{(n)}(s,\bm{a}),1\right\}\right].
    \end{align*}
    \item Suppose the statement is true for step $h+1$, then for step $h$, we have
    \begin{align*}
        &\mathbb{E}_{s\sim d_{\hat{P}^{(n)},h}^{\tilde{\pi}^{(n)}}}\left[\overline{V}_{h,i}^{(n)}(s)-V^{\dagger,\pi^{(n)}_{-i}}_{h,i}(s)\right]\\
        \geq&\mathbb{E}_{(s,\bm{a})\sim d_{\hat{P}^{(n)},h}^{\tilde{\pi}^{(n)}}}\left[\overline{Q}_{h,i}^{(n)}(s,\bm{a})-Q^{\dagger,\pi^{(n)}_{-i}}_{h,i}(s,\bm{a})\right]\\
        =&\mathbb{E}_{(s,\bm{a})\sim d_{\hat{P}^{(n)},h}^{\tilde{\pi}^{(n)}}}\left[\hat{\beta}_h^{(n)}(s,\bm{a})+\left(\hat{P}^{(n)}_h\overline{V}_{h+1,i}^{(n)}\right)(s,\bm{a})-\left(P^\star_hV^{\dagger,\pi^{(n)}_{-i}}_{h+1,i}\right)(s,\bm{a})\right]\\
        =&\mathbb{E}_{(s,\bm{a})\sim d_{\hat{P}^{(n)},h}^{\tilde{\pi}^{(n)}}}\left[\hat{\beta}_h^{(n)}(s,\bm{a})+\left(\hat{P}^{(n)}_h\left(\overline{V}_{h+1,i}^{(n)}-V^{\dagger,\pi^{(n)}_{-i}}_{h+1,i}\right)\right)(s,\bm{a})+\left(\left(\hat{P}^{(n)}_h-P^\star_h\right)V^{\dagger,\pi^{(n)}_{-i}}_{h+1,i}\right)(s,\bm{a})\right]\\
        =&\mathbb{E}_{(s,\bm{a})\sim d_{\hat{P}^{(n)},h}^{\tilde{\pi}^{(n)}}}\left[\hat{\beta}_h^{(n)}(s,\bm{a})+\left(\left(\hat{P}^{(n)}_h-P^\star_h\right)V^{\dagger,\pi^{(n)}_{-i}}_{h+1,i}\right)(s,\bm{a})\right]+\mathbb{E}_{s\sim d_{\hat{P}^{(n)},h+1}^{\tilde{\pi}^{(n)}}}\left[\overline{V}_{h+1,i}^{(n)}(s)-V^{\dagger,\pi^{(n)}_{-i}}_{h+1,i}(s)\right]\\
        \geq&\mathbb{E}_{(s,\bm{a})\sim d_{\hat{P}^{(n)},h}^{\tilde{\pi}^{(n)}}}\left[\hat{\beta}_h^{(n)}(s,\bm{a})-H\min\left\{f^{(n)}_h(s,\bm{a}),1\right\}\right]+\mathbb{E}_{s\sim d_{\hat{P}^{(n)},h+1}^{\tilde{\pi}^{(n)}}}\left[\overline{V}_{h+1,i}^{(n)}(s)-V^{\dagger,\pi^{(n)}_{-i}}_{h+1,i}(s)\right]\\
        \geq&\sum_{h^\prime=h}^H\mathbb{E}_{(s,\bm{a})\sim d_{\hat{P}^{(n)},h^\prime}^{\tilde{\pi}^{(n)}}}\left[\hat{\beta}_{h^\prime}^{(n)}(s,\bm{a})-H\min\left\{f^{(n)}_{h^\prime}(s,\bm{a}),1\right\}\right],
    \end{align*}
    where we use the fact 
    \begin{align*}
        \left\vert\left(\hat{P}^{(n)}_h-P^\star_h\right)V^{\dagger,\pi^{(n)}_{-i}}_{h+1,i}\right\vert(s,\bm{a})\leq&\min\left\{H,\left\Vert\hat{P}^{(n)}_h(\cdot\vert s,\bm{a})-P^\star_h(\cdot\vert s,\bm{a})\right\Vert_1\left\Vert V^{\dagger,\pi^{(n)}_{-i}}_{h+1,i}\right\Vert_\infty\right\}\\
        \leq&H\min\left\{1,\left\Vert\hat{P}^{(n)}_h(\cdot\vert s,\bm{a})-P^\star_h(\cdot\vert s,\bm{a})\right\Vert_1\right\}\\
        =&H\min\left\{1,f^{(n)}_{h^\prime}(s,\bm{a})\right\}
    \end{align*}
    and the last row uses the induction assumption. 
\end{itemize}
Therefore, we have proved \eqref{eq:induction1_mb}. We then apply $h=1$ to \eqref{eq:induction1_mb}, and get
\begin{align*}
    \mathbb{E}_{s\sim d_1}\left[\overline{V}_{1,i}^{(n)}(s)-V^{\dagger,\pi^{(n)}_{-i}}_{1,i}(s)\right]=&\mathbb{E}_{s\sim d_{\hat{P}^{(n)},1}^{\tilde{\pi}^{(n)}}}\left[\overline{V}_{1,i}^{(n)}(s)-V^{\dagger,\pi^{(n)}_{-i}}_{1,i}(s)\right]\\
    \geq&\sum_{h=1}^H\mathbb{E}_{(s,\bm{a})\sim d_{\hat{P}^{(n)},h}^{\tilde{\pi}^{(n)}}}\left[\hat{\beta}_h^{(n)}(s,\bm{a})-H\min\left\{f^{(n)}_h(s,\bm{a}),1\right\}\right]\\
    =&\sum_{h=1}^H\mathbb{E}_{(s,\bm{a})\sim d_{\hat{P}^{(n)},h}^{\tilde{\pi}^{(n)}}}\left[\hat{\beta}_h^{(n)}(s,\bm{a})\right]-H\sum_{h=1}^H\mathbb{E}_{(s,\bm{a})\sim d_{\hat{P}^{(n)},h}^{\tilde{\pi}^{(n)}}}\left[\min\left\{f^{(n)}_h(s,\bm{a}),1\right\}\right].
\end{align*}
Next we are going to bound the second term, let $g_h(s,\bm{a})=\min\left\{f_h^{(n)}(s,\bm{a}),1\right\}$ and apply Lemma \ref{lem:useful2_mb} to $g_h$, we have for $h=1$,
\begin{align*}
    \mathbb{E}_{(s,\bm{a})\sim d^{\tilde{\pi}^{(n)}}_{\hat{P}^{(n)},1}}\left[\min\left\{f_1^{(n)}(s,\bm{a}),1\right\}\right]\leq\sqrt{A\mathbb{E}_{(s,\bm{a})\sim\rho_1^{(n)}}\left[\left(f_1^{(n)}(s,\bm{a})\right)^2\right]}\leq\sqrt{A\zeta^{(n)}}.
\end{align*}
And $\forall h\geq 2$, we have
\begin{align*}
    &\mathbb{E}_{(s,\bm{a})\sim d^{\tilde{\pi}^{(n)}}_{\hat{P}^{(n)},h}}\left[\min\left\{f_h^{(n)}(s,\bm{a}),1\right\}\right]\\
    \leq&\mathbb{E}_{(\tilde{s},\tilde{\bm{a}})\sim d^{\tilde{\pi}^{(n)}}_{\hat{P}^{(n)},h-1}}\left[\min\left\{\left\Vert\hat{\phi}^{(n)}_{h-1}(\tilde{s},\tilde{\bm{a}})\right\Vert_{\Sigma_{n,\rho^{(n)}_{h-1},\hat{\phi}^{(n)}_{h-1}}^{-1}}\sqrt{nA\mathbb{E}_{(s,\bm{a})\sim\tilde{\rho}^{(n)}_h}\left[\left(f_h^{(n)}(s,\bm{a})\right)^2\right]+d\lambda+n\zeta^{(n)}},1\right\}\right]\\
    \lesssim&\mathbb{E}_{(\tilde{s},\tilde{\bm{a}})\sim d^{\tilde{\pi}^{(n)}}_{\hat{P}^{(n)},h-1}}\left[\min\left\{\left\Vert\hat{\phi}^{(n)}_{h-1}(\tilde{s},\tilde{\bm{a}})\right\Vert_{\Sigma_{n,\rho^{(n)}_{h-1},\hat{\phi}^{(n)}_{h-1}}^{-1}}\sqrt{nA\zeta^{(n)}+d\lambda+n\zeta^{(n)}},1\right\}\right]. 
\end{align*}
Note that we here use the fact $\min\left\{f^{(n)}_h(s,\bm{a}),1\right\}\leq 1,\ \mathbb{E}_{(s,\bm{a})\sim\rho^{(n)}_h}\left[\left(f^{(n)}_h(s,\bm{a})\right)^2\right]\leq\zeta^{(n)}$ and $\mathbb{E}_{(s,\bm{a})\sim \tilde{\rho}^{(n)}_h}\left[\left(f^{(n)}_h(s,\bm{a})\right)^2\right]\leq\zeta^{(n)}$. Then according to our choice of $\alpha^{(n)}$, we get
\begin{align*}
    \mathbb{E}_{(s,\bm{a})\sim d^{\tilde{\pi}^{(n)}}_{\hat{P}^{(n)},h}}\left[\min\left\{f_h^{(n)}(s,\bm{a}),1\right\}\right]\leq\mathbb{E}_{(\tilde{s},\tilde{\bm{a}})\sim d^{\tilde{\pi}^{(n)}}_{\hat{P}^{(n)},h-1}}\left[\min\left\{\frac{c\alpha^{(n)}}{H}\left\Vert\hat{\phi}^{(n)}_{h-1}(\tilde{s},\tilde{\bm{a}})\right\Vert_{\Sigma_{n,\rho^{(n)}_{h-1},\hat{\phi}^{(n)}_{h-1}}^{-1}},1\right\}\right]. 
\end{align*}
Combining all things together,
\begin{align*}
    \overline{v}_i^{(n)}-v^{\dagger,\pi^{(n)}_{-i}}_i=&\mathbb{E}_{s\sim d_1}\left[\overline{V}_{1,i}^{(n)}(s)-V^{\dagger,\pi^{(n)}_{-i}}_{1,i}(s)\right]\\
    \geq&\sum_{h=1}^H\mathbb{E}_{(s,\bm{a})\sim d_{\hat{P}^{(n)},h}^{\tilde{\pi}^{(n)}}}\left[\hat{\beta}_h^{(n)}(s,\bm{a})\right]-H\sum_{h=1}^H\mathbb{E}_{(s,\bm{a})\sim d_{\hat{P}^{(n)},h}^{\tilde{\pi}^{(n)}}}\left[\min\left\{f^{(n)}_h(s,\bm{a}),1\right\}\right]\\
    \geq&\sum_{h=1}^{H-1}\mathbb{E}_{(\tilde{s},\tilde{\bm{a}})\sim d^{\tilde{\pi}^{(n)}}_{\hat{P}^{(n)},h}}\left[\hat{\beta}_h^{(n)}(s,\bm{a})-\min\left\{c\alpha^{(n)}\left\Vert\hat{\phi}^{(n)}_h(\tilde{s},\tilde{\bm{a}})\right\Vert_{\Sigma_{n,\rho^{(n)}_h,\hat{\phi}^{(n)}_h}^{-1}},H\right\}\right]-H\sqrt{A\zeta^{(n)}}\\
    \geq&-H\sqrt{A\zeta^{(n)}},
\end{align*}
which proves the inequality. 
\end{proof}

\begin{lemma}[Optimism for CE]
\label{lem:optimism_CE_mb}
Consider an episode $n\in[N]$ and set $\alpha^{(n)}=\Theta\left(H\sqrt{nA\zeta^{(n)}+d\lambda}\right)$. When the event $\mathcal{E}$ holds, we have
\begin{align*}
    \overline{v}_i^{(n)}(s)-\max_{\omega\in\Omega_i}v^{\omega\circ\pi^{(n)}}_i(s)\geq-H\sqrt{A\zeta^{(n)}},\quad\forall n\in[N],i\in[M].
\end{align*}
\end{lemma}
\begin{proof}
Denote $\tilde{\omega}_{h,i}^{(n)}=\argmax_{\omega_h\in\Omega_{h,i}}\left(\mathbb{D}_{\omega_h\circ\pi_h^{(n)}}\max_{\omega\in\Omega_i}Q_{h,i}^{\omega\circ\pi^{(n)}}\right)(s)$ and let $\tilde{\pi}_h^{(n)}=\tilde{\omega}_{h,i}\circ\pi^{(n)}_h$. Let $f^{(n)}_h(s,\bm{a})=\left\Vert\hat{P}^{(n)}_h(\cdot\vert s,\bm{a})-P^\star_h(\cdot\vert s,\bm{a})\right\Vert_1$, then according to the event $\mathcal{E}$, we have
\begin{align*}
    &\mathbb{E}_{(s,\bm{a})\sim\rho^{(n)}_h}\left[\left(f^{(n)}_h(s,\bm{a})\right)^2\right]\leq\zeta^{(n)},\quad\mathbb{E}_{(s,\bm{a})\sim\tilde{\rho}^{(n)}_h}\left[\left(f^{(n)}_h(s,\bm{a})\right)^2\right]\leq\zeta^{(n)},\quad\forall n\in[N], h\in[H],\\
    &\Vert\phi_h(s,\bm{a})\Vert_{\left(\hat{\Sigma}^{(n)}_{h,\phi_h}\right)^{-1}}=\Theta\left(\Vert\phi_h(s,\bm{a})\Vert_{\Sigma^{-1}_{n,\rho^{(n)}_h,\phi_h}}\right),\quad\forall n\in[N],h\in[H],\phi_h\in\Phi_h. 
\end{align*}
A direct conclusion of the event $\mathcal{E}$ is we can find an absolute constant $c$, such that
\begin{align*}
    \beta_h^{(n)}(s,\bm{a})=&\min\left\{\alpha^{(n)}\left\Vert\hat{\phi}^{(n)}_h(\tilde{s},\tilde{\bm{a}})\right\Vert_{\left(\Sigma_{h,\hat{\phi}^{(n)}_h}^{(n)}\right)^{-1}},H\right\}\\
    \geq&\min\left\{c\alpha^{(n)}\left\Vert\hat{\phi}^{(n)}_h(\tilde{s},\tilde{\bm{a}})\right\Vert_{\Sigma_{n,\rho^{(n)}_h,\hat{\phi}^{(n)}_h}^{-1}},H\right\},\quad\forall n\in[N],h\in[H]. 
\end{align*}
Next, we prove by induction that
\begin{align}
    \mathbb{E}_{s\sim d_{\hat{P}^{(n)},h}^{\tilde{\pi}^{(n)}}}\left[\overline{V}_{h,i}^{(n)}(s)-\max_{\omega\in\Omega_i}V_{h,i}^{\omega\circ\pi^{(n)}}(s)\right]\geq&\sum_{h^\prime=h}^H\mathbb{E}_{(s,\bm{a})\sim d_{\hat{P}^{(n)},h^\prime}^{\tilde{\pi}^{(n)}}}\left[\hat{\beta}_{h^\prime}^{(n)}(s,\bm{a})-H\min\left\{f^{(n)}_{h^\prime}(s,\bm{a}),1\right\}\right],\quad\forall h\in[H].\label{eq:induction1_CE_mb}
\end{align}
First, notice that $\forall h\in[H]$,
\begin{align*}
    \mathbb{E}_{s\sim d_{\hat{P}^{(n)},h}^{\tilde{\pi}^{(n)}}}\left[\overline{V}_{h,i}^{(n)}(s)-\max_{\omega\in\Omega_i}V_{h,i}^{\omega\circ\pi^{(n)}}(s)\right]=&\mathbb{E}_{s\sim d_{\hat{P}^{(n)},h}^{\tilde{\pi}^{(n)}}}\left[\left(\mathbb{D}_{\pi^{(n)}_h}\overline{Q}_{h,i}^{(n)}\right)(s)-\left(\mathbb{D}_{\tilde{\pi}^{(n)}_h}\max_{\omega\in\Omega_i}Q^{\omega\circ\pi^{(n)}}_{h,i}\right)(s)\right]\\
    \geq&\mathbb{E}_{s\sim d_{\hat{P}^{(n)},h}^{\tilde{\pi}^{(n)}}}\left[\left(\mathbb{D}_{\tilde{\pi}^{(n)}_h}\overline{Q}_{h,i}^{(n)}\right)(s)-\left(\mathbb{D}_{\tilde{\pi}^{(n)}_h}\max_{\omega\in\Omega_i}Q^{\omega\circ\pi^{(n)}}_{h,i}\right)(s)\right]\\
    =&\mathbb{E}_{(s,\bm{a})\sim d_{\hat{P}^{(n)},h}^{\tilde{\pi}^{(n)}}}\left[\overline{Q}_{h,i}^{(n)}(s,\bm{a})-\max_{\omega\in\Omega_i}Q^{\omega\circ\pi^{(n)}}_{h,i}(s,\bm{a})\right].
\end{align*}
where the inequality uses the fact that $\pi_h^{(n)}$ is the CE solution for $\left\{\overline{Q}_{h,i}^{(n)}\right\}_{i=1}^M$. Now we are ready to prove \eqref{eq:induction1_CE_mb}:
\begin{itemize}
    \item When $h=H$, we have
    \begin{align*}
        \mathbb{E}_{s\sim d_{\hat{P}^{(n)},H}^{\tilde{\pi}^{(n)}}}\left[\overline{V}_{H,i}^{(n)}(s)-\max_{\omega\in\Omega_i}V^{\omega\circ\pi^{(n)}}_{H,i}(s)\right]\geq&\mathbb{E}_{(s,\bm{a})\sim d_{\hat{P}^{(n)},H}^{\tilde{\pi}^{(n)}}}\left[\overline{Q}_{H,i}^{(n)}(s,\bm{a})-\max_{\omega\in\Omega_i}Q^{\omega\circ\pi^{(n)}}_{H,i}(s,\bm{a})\right]\\
        =&\mathbb{E}_{(s,\bm{a})\sim d_{\hat{P}^{(n)},H}^{\tilde{\pi}^{(n)}}}\left[\hat{\beta}_h^{(n)}(s,\bm{a})\right]\\
        \geq&\mathbb{E}_{(s,\bm{a})\sim d_{\hat{P}^{(n)},H}^{\tilde{\pi}^{(n)}}}\left[\hat{\beta}_h^{(n)}(s,\bm{a})-H\min\left\{f_H^{(n)}(s,\bm{a}),1\right\}\right].
    \end{align*}
    \item Suppose the statement is true for $h+1$, then for step $h$, we have
    \begin{align*}
        &\mathbb{E}_{s\sim d_{\hat{P}^{(n)},h}^{\tilde{\pi}^{(n)}}}\left[\overline{V}_{h,i}^{(n)}(s)-\max_{\omega\in\Omega_i}V^{\omega\circ\pi^{(n)}}_{h,i}(s)\right]\\
        \geq&\mathbb{E}_{(s,\bm{a})\sim d_{\hat{P}^{(n)},h}^{\tilde{\pi}^{(n)}}}\left[\overline{Q}_{h,i}^{(n)}(s,\bm{a})-\max_{\omega\in\Omega_i}Q^{\omega\circ\pi^{(n)}}_{h,i}(s,\bm{a})\right]\\
        =&\mathbb{E}_{(s,\bm{a})\sim d_{\hat{P}^{(n)},h}^{\tilde{\pi}^{(n)}}}\left[\hat{\beta}_h^{(n)}(s,\bm{a})+\left(\hat{P}^{(n)}_h\overline{V}_{h+1,i}^{(n)}\right)(s,\bm{a})-\left(P^\star_h\max_{\omega\in\Omega_i}V^{\omega\circ\pi^{(n)}}_{h+1,i}\right)(s,\bm{a})\right]\\
        =&\mathbb{E}_{(s,\bm{a})\sim d_{\hat{P}^{(n)},h}^{\tilde{\pi}^{(n)}}}\left[\hat{\beta}_h^{(n)}(s,\bm{a})+\left(\hat{P}^{(n)}_h\left(\overline{V}_{h+1,i}^{(n)}-\max_{\omega\in\Omega_i}V^{\omega\circ\pi^{(n)}}_{h+1,i}\right)\right)(s,\bm{a})+\left(\left(\hat{P}^{(n)}_h-P^\star_h\right)\max_{\omega\in\Omega_i}V^{\omega\circ\pi^{(n)}}_{h+1,i}\right)(s,\bm{a})\right]\\
        =&\mathbb{E}_{(s,\bm{a})\sim d_{\hat{P}^{(n)},h}^{\tilde{\pi}^{(n)}}}\left[\hat{\beta}_h^{(n)}(s,\bm{a})+\left(\left(\hat{P}^{(n)}_h-P^\star_h\right)\max_{\omega\in\Omega_i}V^{\omega\circ\pi^{(n)}}_{h+1,i}\right)(s,\bm{a})\right]\\
        &+\mathbb{E}_{s\sim d_{\hat{P}^{(n)},h+1}^{\tilde{\pi}^{(n)}}}\left[\overline{V}_{h+1,i}^{(n)}(s)-\max_{\omega\in\Omega_i}V^{\omega\circ\pi^{(n)}}_{h+1,i}(s)\right]\\
        \geq&\mathbb{E}_{(s,\bm{a})\sim d_{\hat{P}^{(n)},h}^{\tilde{\pi}^{(n)}}}\left[\hat{\beta}_h^{(n)}(s,\bm{a})-H\min\left\{f^{(n)}_h(s,\bm{a}),1\right\}\right]+\mathbb{E}_{s\sim d_{\hat{P}^{(n)},h+1}^{\tilde{\pi}^{(n)}}}\left[\overline{V}_{h+1,i}^{(n)}(s)-\max_{\omega\in\Omega_i}V^{\omega\circ\pi^{(n)}}_{h+1,i}(s)\right]\\
        \geq&\sum_{h^\prime=h}^H\mathbb{E}_{(s,\bm{a})\sim d_{\hat{P}^{(n)},h^\prime}^{\tilde{\pi}^{(n)}}}\left[\hat{\beta}_{h^\prime}^{(n)}(s,\bm{a})-H\min\left\{f^{(n)}_{h^\prime}(s,\bm{a}),1\right\}\right],
    \end{align*}
    where we use the fact 
    \begin{align*}
        \left\vert\left(\hat{P}^{(n)}_h-P^\star_h\right)\max_{\omega\in\Omega_i}V^{\omega\circ\pi^{(n)}}_{h+1,i}\right\vert(s,\bm{a})\leq&\min\left\{H,\left\Vert\hat{P}^{(n)}_h(\cdot\vert s,\bm{a})-P^\star_h(\cdot\vert s,\bm{a})\right\Vert_1\left\Vert\max_{\omega\in\Omega_i}V^{\omega\circ\pi^{(n)}}_{h+1,i}\right\Vert_\infty\right\}\\
        \leq&H\min\left\{1,\left\Vert\hat{P}^{(n)}_h(\cdot\vert s,\bm{a})-P^\star_h(\cdot\vert s,\bm{a})\right\Vert_1\right\}\\
        =&H\min\left\{1,f^{(n)}_{h^\prime}(s,\bm{a})\right\}
    \end{align*}
    and the last row uses the induction assumption. 
\end{itemize}
Therefore, we have proved \eqref{eq:induction1_CE_mb}. We then apply $h=1$ to \eqref{eq:induction1_CE_mb}, and get
\begin{align*}
    \mathbb{E}_{s\sim d_1}\left[\overline{V}_{1,i}^{(n)}(s)-\max_{\omega\in\Omega_i}V^{\omega\circ\pi^{(n)}}_{1,i}(s)\right]=&\mathbb{E}_{s\sim d_{\hat{P}^{(n)},1}^{\tilde{\pi}^{(n)}}}\left[\overline{V}_{1,i}^{(n)}(s)-\max_{\omega\in\Omega_i}V^{\omega\circ\pi^{(n)}}_{1,i}(s)\right]\\
    \geq&\sum_{h=1}^H\mathbb{E}_{(s,\bm{a})\sim d_{\hat{P}^{(n)},h}^{\tilde{\pi}^{(n)}}}\left[\hat{\beta}_h^{(n)}(s,\bm{a})-H\min\left\{f^{(n)}_h(s,\bm{a}),1\right\}\right]\\
    =&\sum_{h=1}^H\mathbb{E}_{(s,\bm{a})\sim d_{\hat{P}^{(n)},h}^{\tilde{\pi}^{(n)}}}\left[\hat{\beta}_h^{(n)}(s,\bm{a})\right]-H\sum_{h=1}^H\mathbb{E}_{(s,\bm{a})\sim d_{\hat{P}^{(n)},h}^{\tilde{\pi}^{(n)}}}\left[\min\left\{f^{(n)}_h(s,\bm{a}),1\right\}\right].
\end{align*}
Next we are going to bound the second term, let $g_h(s,\bm{a})=\min\left\{f_h^{(n)}(s,\bm{a}),1\right\}$ and apply Lemma \ref{lem:useful2_mb} to $g_h$, we have for $h=1$,
\begin{align*}
    \mathbb{E}_{(s,\bm{a})\sim d^{\tilde{\pi}^{(n)}}_{\hat{P}^{(n)},1}}\left[\min\left\{f_1^{(n)}(s,\bm{a}),1\right\}\right]\leq\sqrt{A\mathbb{E}_{(s,\bm{a})\sim\rho_1^{(n)}}\left[\left(f_1^{(n)}(s,\bm{a})\right)^2\right]}\leq\sqrt{A\zeta^{(n)}}.
\end{align*}
And $\forall h\geq 2$, we have
\begin{align*}
    &\mathbb{E}_{(s,\bm{a})\sim d^{\tilde{\pi}^{(n)}}_{\hat{P}^{(n)},h}}\left[\min\left\{f_h^{(n)}(s,\bm{a}),1\right\}\right]\\
    \leq&\mathbb{E}_{(\tilde{s},\tilde{\bm{a}})\sim d^{\tilde{\pi}^{(n)}}_{\hat{P}^{(n)},h-1}}\left[\min\left\{\left\Vert\hat{\phi}^{(n)}_{h-1}(\tilde{s},\tilde{\bm{a}})\right\Vert_{\Sigma_{n,\rho^{(n)}_{h-1},\hat{\phi}^{(n)}_{h-1}}^{-1}}\sqrt{nA\mathbb{E}_{(s,\bm{a})\sim\tilde{\rho}^{(n)}_h}\left[\left(f_h^{(n)}(s,\bm{a})\right)^2\right]+d\lambda+n\zeta^{(n)}},1\right\}\right]\\
    \lesssim&\mathbb{E}_{(\tilde{s},\tilde{\bm{a}})\sim d^{\tilde{\pi}^{(n)}}_{\hat{P}^{(n)},h-1}}\left[\min\left\{\left\Vert\hat{\phi}^{(n)}_{h-1}(\tilde{s},\tilde{\bm{a}})\right\Vert_{\Sigma_{n,\rho^{(n)}_{h-1},\hat{\phi}^{(n)}_{h-1}}^{-1}}\sqrt{nA\zeta^{(n)}+d\lambda+n\zeta^{(n)}},1\right\}\right]. 
\end{align*}
Note that we here use the fact $\min\left\{f^{(n)}_h(s,\bm{a}),1\right\}\leq 1,\ \mathbb{E}_{(s,\bm{a})\sim\rho^{(n)}_h}\left[\left(f^{(n)}_h(s,\bm{a})\right)^2\right]\leq\zeta^{(n)}$ and $\mathbb{E}_{(s,\bm{a})\sim \tilde{\rho}^{(n)}_h}\left[\left(f^{(n)}_h(s,\bm{a})\right)^2\right]\leq\zeta^{(n)}$. Then according to our choice of $\alpha^{(n)}$, we get
\begin{align*}
    \mathbb{E}_{(s,\bm{a})\sim d^{\tilde{\pi}^{(n)}}_{\hat{P}^{(n)},h}}\left[f_h^{(n)}(s,\bm{a})\right]\leq\mathbb{E}_{(\tilde{s},\tilde{\bm{a}})\sim d^{\tilde{\pi}^{(n)}}_{\hat{P}^{(n)},h-1}}\left[\min\left\{\frac{c\alpha^{(n)}}{H}\left\Vert\hat{\phi}^{(n)}_{h-1}(\tilde{s},\tilde{\bm{a}})\right\Vert_{\Sigma_{n,\rho^{(n)}_{h-1},\hat{\phi}^{(n)}_{h-1}}^{-1}},1\right\}\right]. 
\end{align*}
Combining all things together,
\begin{align*}
    \overline{v}_i^{(n)}-\max_{\omega\in\Omega_i}v^{\omega\circ\pi^{(n)}}_i=&\mathbb{E}_{s\sim d_1}\left[\overline{V}_{1,i}^{(n)}(s)-\max_{\omega\in\Omega_i}V^{\omega\circ\pi^{(n)}}_{1,i}(s)\right]\\
    \geq&\sum_{h=1}^H\mathbb{E}_{(s,\bm{a})\sim d_{\hat{P}^{(n)},h}^{\tilde{\pi}^{(n)}}}\left[\hat{\beta}_h^{(n)}(s,\bm{a})\right]-H\sum_{h=1}^H\mathbb{E}_{(s,\bm{a})\sim d_{\hat{P}^{(n)},h}^{\tilde{\pi}^{(n)}}}\left[\min\left\{f^{(n)}_h(s,\bm{a}),1\right\}\right]\\
    \geq&\sum_{h=1}^{H-1}\mathbb{E}_{(\tilde{s},\tilde{\bm{a}})\sim d^{\tilde{\pi}^{(n)}}_{\hat{P}^{(n)},h}}\left[\hat{\beta}_h^{(n)}(s,\bm{a})-\min\left\{c\alpha^{(n)}\left\Vert\hat{\phi}^{(n)}_h(\tilde{s},\tilde{\bm{a}})\right\Vert_{\Sigma_{n,\rho^{(n)}_h,\hat{\phi}^{(n)}_h}^{-1}},H\right\}\right]-H\sqrt{A\zeta^{(n)}}\\
    \geq&-H\sqrt{A\zeta^{(n)}},
\end{align*}
which proves the inequality. 
\end{proof}

\begin{lemma}[Pessimism]\label{lem:pessimism_mb}
    Consider an episode $n\in[N]$ and set $\alpha^{(n)}=\Theta\left(H\sqrt{nA\zeta^{(n)}+d\lambda}\right)$. When the event $\mathcal{E}$ holds, we have
    \begin{align*}
        \underline{v}_i^{(n)}(s)-v^{\pi^{(n)}}_i(s)\leq H\sqrt{A\zeta^{(n)}},\quad\forall n\in[N],i\in[M].
    \end{align*}
\end{lemma}
\begin{proof}
Let $f^{(n)}_h(s,\bm{a})=\left\Vert\hat{P}^{(n)}_h(\cdot\vert s,\bm{a})-P^\star_h(\cdot\vert s,\bm{a})\right\Vert_1$, then according to the event $\mathcal{E}$, we have
\begin{align*}
    &\mathbb{E}_{(s,\bm{a})\sim\rho^{(n)}_h}\left[\left(f^{(n)}_h(s,\bm{a})\right)^2\right]\leq\zeta^{(n)},\quad\mathbb{E}_{(s,\bm{a})\sim\tilde{\rho}^{(n)}_h}\left[\left(f^{(n)}_h(s,\bm{a})\right)^2\right]\leq\zeta^{(n)},\quad\forall n\in[N], h\in[H],\\
    &\Vert\phi_h(s,\bm{a})\Vert_{\left(\hat{\Sigma}^{(n)}_{h,\phi_h}\right)^{-1}}=\Theta\left(\Vert\phi_h(s,\bm{a})\Vert_{\Sigma^{-1}_{n,\rho^{(n)}_h,\phi_h}}\right),\quad\forall n\in[N],h\in[H],\phi_h\in\Phi_h. 
\end{align*}
A direct conclusion of the event $\mathcal{E}$ is we can find an absolute constant $c$, such that
\begin{align*}
    \beta_h^{(n)}(s,\bm{a})=&\min\left\{\alpha^{(n)}\left\Vert\hat{\phi}^{(n)}_h(\tilde{s},\tilde{\bm{a}})\right\Vert_{\left(\Sigma_{h,\hat{\phi}^{(n)}_h}^{(n)}\right)^{-1}},H\right\}\\
    \geq&\min\left\{c\alpha^{(n)}\left\Vert\hat{\phi}^{(n)}_h(\tilde{s},\tilde{\bm{a}})\right\Vert_{\Sigma_{n,\rho^{(n)}_h,\hat{\phi}^{(n)}_h}^{-1}},H\right\},\quad\forall n\in[N],h\in[H]. 
\end{align*}
Again, we prove the following inequality by induction:
\begin{align}
    \mathbb{E}_{s\sim d_{\hat{P}^{(n)},h}^{\pi^{(n)}}}\left[\underline{V}_{h,i}^{(n)}(s)-V^{\pi^{(n)}}_{h,i}(s)\right]\leq&\sum_{h^\prime=h}^H\mathbb{E}_{(s,\bm{a})\sim d_{\hat{P}^{(n)},h^\prime}^{\pi^{(n)}}}\left[-\hat{\beta}_{h^\prime}^{(n)}(s,\bm{a})+H\min\left\{f^{(n)}_{h^\prime}(s,\bm{a}),1\right\}\right],\quad\forall h\in[H].
\end{align}
\begin{itemize}
    \item When $h=H$, we have
    \begin{align*}
        \mathbb{E}_{s\sim d_{\hat{P}^{(n)},H}^{\pi^{(n)}}}\left[\underline{V}_{H,i}^{(n)}(s)-V^{\pi^{(n)}}_{H,i}(s)\right]=&\mathbb{E}_{(s,\bm{a})\sim d_{\hat{P}^{(n)},H}^{\pi^{(n)}}}\left[\underline{Q}_{H,i}^{(n)}(s,\bm{a})-Q^{\pi^{(n)}}_{H,i}(s,\bm{a})\right]\\
        =&\mathbb{E}_{(s,\bm{a})\sim d_{\hat{P}^{(n)},H}^{\pi^{(n)}}}\left[-\hat{\beta}_H^{(n)}(s,\bm{a})\right]\\
        \leq&\mathbb{E}_{(s,\bm{a})\sim d_{\hat{P}^{(n)},H}^{\pi^{(n)}}}\left[-\hat{\beta}_H^{(n)}(s,\bm{a})+H\min\left\{f_H^{(n)}(s,\bm{a}),1\right\}\right]
    \end{align*}
    \item Suppose the statement is true for $h+1$, then for step $h$, we have
    \begin{align*}
        &\mathbb{E}_{s\sim d_{\hat{P}^{(n)},h}^{\pi^{(n)}}}\left[\underline{V}_{h,i}^{(n)}(s)-V^{\pi^{(n)}}_{h,i}(s)\right]\\
        =&\mathbb{E}_{(s,\bm{a})\sim d_{\hat{P}^{(n)},h}^{\pi^{(n)}}}\left[\underline{Q}_{h,i}^{(n)}(s,\bm{a})-Q^{\pi^{(n)}}_{h,i}(s,\bm{a})\right]\\
        =&\mathbb{E}_{(s,\bm{a})\sim d_{\hat{P}^{(n)},h}^{\pi^{(n)}}}\left[-\hat{\beta}_h^{(n)}(s,\bm{a})+\left(\hat{P}^{(n)}_h\underline{V}_{h+1,i}^{(n)}\right)(s,\bm{a})-\left(P^\star_hV^{\pi^{(n)}}_{h+1,i}\right)(s,\bm{a})\right]\\
        =&\mathbb{E}_{(s,\bm{a})\sim d_{\hat{P}^{(n)},h}^{\pi^{(n)}}}\left[-\hat{\beta}_h^{(n)}(s,\bm{a})+\left(\hat{P}^{(n)}_h\left(\underline{V}_{h+1,i}^{(n)}-V^{\pi^{(n)}}_{h+1,i}\right)\right)(s,\bm{a})+\left(\left(\hat{P}^{(n)}_h-P^\star_h\right)V^{\pi^{(n)}}_{h+1,i}\right)(s,\bm{a})\right]\\
        =&\mathbb{E}_{(s,\bm{a})\sim d_{\hat{P}^{(n)},h}^{\pi^{(n)}}}\left[-\hat{\beta}_h^{(n)}(s,\bm{a})+\left(\left(\hat{P}^{(n)}_h-P^\star_h\right)V^{\pi^{(n)}}_{h+1,i}\right)(s,\bm{a})\right]+\mathbb{E}_{s\sim d_{\hat{P}^{(n)},h+1}^{\pi^{(n)}}}\left[\left(\underline{V}_{h+1,i}^{(n)}-V^{\pi^{(n)}}_{h+1,i}\right)(s)\right]\\
        \leq&\mathbb{E}_{(s,\bm{a})\sim d_{\hat{P}^{(n)},h}^{\pi^{(n)}}}\left[-\hat{\beta}_h^{(n)}(s,\bm{a})+H\min\left\{f^{(n)}_h(s,\bm{a}),1\right\}\right]+\mathbb{E}_{s\sim d_{\hat{P}^{(n)},h+1}^{\pi^{(n)}}}\left[\left(\underline{V}_{h+1,i}^{(n)}-V^{\pi^{(n)}}_{h+1,i}\right)(s)\right]\\
        \leq&\sum_{h^\prime=h}^H\mathbb{E}_{(s,\bm{a})\sim d_{\hat{P}^{(n)},h^\prime}^{\pi^{(n)}}}\left[-\hat{\beta}_{h^\prime}^{(n)}(s,\bm{a})+H\min\left\{f^{(n)}_{h^\prime}(s,\bm{a}),1\right\}\right].
    \end{align*}
    where we use the fact 
    \begin{align*}
        \left\vert\left(\hat{P}^{(n)}_h-P^\star_h\right)V^{\pi^{(n)}}_{h+1,i}\right\vert(s,\bm{a})\leq&\min\left\{H,\left\Vert\hat{P}^{(n)}_h(\cdot\vert s,\bm{a})-P^\star_h(\cdot\vert s,\bm{a})\right\Vert_1\left\Vert V^{\pi^{(n)}}_{h+1,i}\right\Vert_\infty\right\}\\
        \leq&H\min\left\{1,\left\Vert\hat{P}^{(n)}_h(\cdot\vert s,\bm{a})-P^\star_h(\cdot\vert s,\bm{a})\right\Vert_1\right\}\\
        =&H\min\left\{1,f^{(n)}_{h^\prime}(s,\bm{a})\right\}
    \end{align*}
    and the last row uses the induction assumption. 
\end{itemize}
The remaining steps are exactly the same as the proof in Lemma \ref{lem:optimism_NE_CCE_mb} or Lemma \ref{lem:optimism_CE_mb}, we may prove 
\begin{align*}
    \mathbb{E}_{(s,\bm{a})\sim d^{\pi^{(n)}}_{\hat{P}^{(n)},1}}\left[\min\left\{f_1^{(n)}(s,\bm{a}),1\right\}\right]\leq\sqrt{A\zeta^{(n)}},
\end{align*}
and
\begin{align*}
    \mathbb{E}_{(s,\bm{a})\sim d^{\pi^{(n)}}_{\hat{P}^{(n)},h}}\left[f_h^{(n)}(s,\bm{a})\right]\leq\mathbb{E}_{(\tilde{s},\tilde{\bm{a}})\sim d^{\pi^{(n)}}_{\hat{P}^{(n)},h-1}}\left[\min\left\{\frac{c\alpha^{(n)}}{H}\left\Vert\hat{\phi}^{(n)}_{h-1}(\tilde{s},\tilde{\bm{a}})\right\Vert_{\Sigma_{n,\rho^{(n)}_{h-1},\hat{\phi}^{(n)}_{h-1}}^{-1}},1\right\}\right],\quad\forall h\geq 2. 
\end{align*}
Combining all things together, we get
\begin{align*}
    \underline{v}_i^{(n)}-v^{\pi^{(n)}}_i=&\mathbb{E}_{s\sim d_1}\left[\underline{V}_{1,i}^{(n)}(s)-V^{\pi^{(n)}}_{1,i}(s)\right]\\
    \leq&\sum_{h=1}^H\mathbb{E}_{(s,\bm{a})\sim d_{\hat{P}^{(n)},h}^{\pi^{(n)}}}\left[-\hat{\beta}_h^{(n)}(s,\bm{a})+H\min\left\{f^{(n)}_h(s,\bm{a}),1\right\}\right]\\
    \leq&\sum_{h=1}^{H-1}\mathbb{E}_{(s,\bm{a})\sim d_{\hat{P}^{(n)},h}^{\pi^{(n)}}}\left[-\hat{\beta}_h^{(n)}(s,\bm{a})+\min\left\{c\alpha^{(n)}\left\Vert\hat{\phi}^{(n)}_h(\tilde{s},\tilde{\bm{a}})\right\Vert_{\Sigma_{n,\rho^{(n)}_h,\hat{\phi}^{(n)}_h}^{-1}},H\right\}\right]+H\sqrt{A\zeta^{(n)}}\\
    \leq&H\sqrt{A\zeta^{(n)}},
\end{align*}
which has finished the proof. 
\end{proof}

\begin{lemma}\label{lem:pseudo_regret_mb}
For the model-based algorithm, when we pick $\lambda=\Theta\left(d\log\frac{NH\vert\Phi\vert}{\delta}\right)$, $\zeta^{(n)}=\Theta\left(\frac{1}{n}\log\frac{\vert\mathcal{M}\vert HN}{\delta}\right)$ and $\alpha^{(n)}=\Theta\left(H\sqrt{nA\zeta^{(n)}+d\lambda}\right)$, with probability $1-\delta$, we have 
\begin{align*}
    \sum_{n=1}^N\Delta^{(n)}\lesssim H^3d^2A N^{\frac{1}{2}}\log\frac{\vert\mathcal{M}\vert HN}{\delta}.
\end{align*}
\end{lemma}
\begin{proof}
With our choice of $\lambda$ and $\zeta^{(n)}$, according to Lemma \ref{lem:model_based_hp_mb}, we know $\mathcal{E}$ holds with probability $1-\delta$. Furthermore, we have
\begin{align*}
    \alpha^{(n)}=\Theta\left(H\sqrt{A\log\frac{\vert\mathcal{M}\vert HN}{\delta}+d^2\log\frac{NH\vert\Phi\vert}{\delta}}\right)=O\left(dH\sqrt{A\log\frac{\vert\mathcal{M}\vert HN}{\delta}}\right)
\end{align*} 

Let $f^{(n)}_h(s,\bm{a})=\left\Vert\hat{P}^{(n)}_h(\cdot\vert s,\bm{a})-P^\star_h(\cdot\vert s,\bm{a})\right\Vert_1$. According to the definition of the event $\mathcal{E}$, we have
\begin{align}\label{eq:conditioning_mb}
    \mathbb{E}_{s\sim\rho^{(n)}_h}\left[\left(f^{(n)}_h(s,\bm{a})\right)^2\right]\leq\zeta^{(n)},\Vert\phi_h(s,\bm{a})\Vert_{\left(\hat{\Sigma}^{(n)}_{h,\phi_h}\right)^{-1}}=\Theta\left(\Vert\phi_h(s,\bm{a})\Vert_{\Sigma^{-1}_{n,\rho^{(n)}_h,\phi_h}}\right),\quad\forall n\in[N],h\in[H],\phi_h\in\Phi_h.
\end{align}
By definition, we have
\begin{align*}
    \Delta^{(n)}=\max_{i\in[M]}\left\{\overline{v}^{(n)}_i-\underline{v}^{(n)}_i\right\}+2H\sqrt{A\zeta^{(n)}}.
\end{align*}
For each fixed $i\in[M],h\in[H]$ and $n\in[N]$, we have
\begin{align*}
    &\mathbb{E}_{s\sim d^{\pi^{(n)}}_{P^\star,h}}\left[\overline{V}^{(n)}_{h,i}(s)-\underline{V}^{(n)}_{h,i}(s)\right]\\
    =&\mathbb{E}_{s\sim d_{P^\star,h}^{\pi^{(n)}}}\left[\left(\mathbb{D}_{\pi^{(n)}_h}\overline{Q}_{h,i}^{(n)}\right)(s)-\left(\mathbb{D}_{\pi^{(n)}_h}\underline{Q}_{h,i}^{(n)}\right)(s)\right]\\
    =&\mathbb{E}_{(s,\bm{a})\sim d_{P^\star,h}^{\pi^{(n)}}}\left[\overline{Q}_{h,i}^{(n)}(s,\bm{a})-\underline{Q}_{h,i}^{(n)}(s,\bm{a})\right]\\
    =&\mathbb{E}_{(s,\bm{a})\sim d_{P^\star,h}^{\pi^{(n)}}}\left[2\hat{\beta}^{(n)}_h(s,\bm{a})+\left(\hat{P}^{(n)}_h\left(\overline{V}_{h+1,i}^{(n)}-\underline{V}_{h+1,i}^{(n)}\right)\right)(s,\bm{a})\right]\\
    =&\mathbb{E}_{(s,\bm{a})\sim d_{P^\star,h}^{\pi^{(n)}}}\left[2\hat{\beta}^{(n)}_h(s,\bm{a})+\left(\left(\hat{P}^{(n)}_h-P^\star_h\right)\left(\overline{V}_{h+1,i}^{(n)}-\underline{V}_{h+1,i}^{(n)}\right)\right)(s,\bm{a})\right]+\mathbb{E}_{s\sim d^{\pi^{(n)}}_{P^\star,h+1}}\left[\overline{V}^{(n)}_{h+1,i}(s)-\underline{V}^{(n)}_{h+1,i}(s)\right]\\
    \leq&\mathbb{E}_{(s,\bm{a})\sim d_{P^\star,h}^{\pi^{(n)}}}\left[2\hat{\beta}^{(n)}_h(s,\bm{a})+2H^2f_h^{(n)}(s,\bm{a})\right]+\mathbb{E}_{s\sim d^{\pi^{(n)}}_{P^\star,h+1}}\left[\overline{V}^{(n)}_{h+1,i}(s)-\underline{V}^{(n)}_{h+1,i}(s)\right].
\end{align*}
Note that we use the fact $\overline{V}^{(n)}_{h+1,i}(s)-\underline{V}^{(n)}_{h+1,i}(s)$ is upper bounded by $2H^2$, which can be proved easily using induction using the fact that $\hat{\beta}_h^{(n)}(s,\bm{a})\leq H$. Applying the above formula recursively to $\mathbb{E}_{s\sim d^{\pi^{(n)}}_{P^\star,h+1}}\left[\overline{V}^{(n)}_{h+1,i}(s)-\underline{V}^{(n)}_{h+1,i}(s)\right]$, one gets the following result (or more formally, one can prove by induction, just like what we did in Lemma \ref{lem:optimism_NE_CCE_mb}, Lemma \ref{lem:optimism_CE_mb} and Lemma \ref{lem:pessimism_mb}):
\begin{align}
    \mathbb{E}_{s\sim d^{\pi^{(n)}}_{P^\star,1}}\left[\overline{V}^{(n)}_{1,i}(s)-\underline{V}^{(n)}_{1,i}(s)\right]\leq 2\underbrace{\sum_{h=1}^H\mathbb{E}_{(s,\bm{a})\sim d_{P^\star,h}^{\pi^{(n)}}}\left[\hat{\beta}^{(n)}_h(s,\bm{a})\right]}_{(a)}+2H^2\underbrace{\sum_{h=1}^H\mathbb{E}_{(s,\bm{a})\sim d_{P^\star,h}^{\pi^{(n)}}}\left[f_h^{(n)}(s,\bm{a})\right]}_{(b)}\label{eq:regret_middle_mb}. 
\end{align}
First, we calculate the first term (a) in Inequality \eqref{eq:regret_middle_mb}. Following Lemma \ref{lem:useful_mb} and noting the bonus $\hat{\beta}^{(n)}_h$ is $O(H)$, we have
\begin{align*}
    &\sum_{h=1}^H\mathbb{E}_{(s,\bm{a})\sim d^{\pi^{(n)}}_{P^\star,h}}\left[\hat{\beta}^{(n)}_h(s,\bm{a})\right]\\
    \lesssim&\sum_{h=1}^H\mathbb{E}_{(s,\bm{a})\sim d^{\pi^{(n)}}_{P^\star,h}}\left[\min\left\{\alpha^{(n)}\left\Vert\hat{\phi}^{(n)}_h(s,\bm{a})\right\Vert_{\Sigma^{-1}_{n,\rho^{(n)}_h,\hat{\phi}^{(n)}_h}},H\right\}\right]\tag{From \eqref{eq:conditioning_mb} }\\ 
    \lesssim&\sum_{h=1}^{H-1}\mathbb{E}_{(\tilde{s},\tilde{\bm{a}})\sim d^{\pi^{(n)}}_{P^\star,h}}\left[\left\Vert\phi^\star_h(\tilde{s},\tilde{\bm{a}})\right\Vert_{\Sigma_{n,\gamma^{(n)}_h,\phi^\star_h}^{-1}}\right]\sqrt{{nA\left(\alpha^{(n)}\right)^2}\mathbb{E}_{(s,\bm{a})\sim\rho^{(n)}_h}\left[\left\Vert\hat{\phi}^{(n)}_h(s,\bm{a})\right\Vert^2_{\Sigma^{-1}_{n,\rho^{(n)}_h,\hat{\phi}^{(n)}_h}}\right]+H^2d\lambda}\\
    +&\sqrt{A\left(\alpha^{(n)}\right)^2\mathbb{E}_{(s,\bm{a})\sim\rho^{(n)}_1}\left[\left\Vert\hat{\phi}_1^{(n)}(s,\bm{a})\right\Vert^2_{\Sigma^{-1}_{n,\rho^{(n)}_1,\hat{\phi}_1^{(n)}}}\right]}.
\end{align*}
Note that we use the fact that $B=H$ when applying Lemma \ref{lem:useful_mb}. In addition, we have 
\begin{align*}
     &n\mathbb{E}_{(s,\bm{a})\sim\rho^{(n)}_h}\left[\left\Vert\hat \phi_h^{(n)}(s,\bm{a})\right\Vert^2_{\Sigma^{-1}_{n,\rho^{(n)}_h,\hat{\phi}^{(n)}_h}}\right]\\
     =&n\textrm{Tr}\left(\mathbb{E}_{(s,\bm{a})\sim\rho^{(n)}_h}\left[\hat{\phi}^{(n)}_h(s,\bm{a})\hat{\phi}^{(n)}_h(s,\bm{a})^\top\right]\left(n\mathbb{E}_{(s,\bm{a})\sim\rho^{(n)}_h}\left[\hat{\phi}^{(n)}_h(s,\bm{a})\hat{\phi}^{(n)}_h(s,\bm{a})^\top\right]+\lambda I_d\right)^{-1}\right)\\
     \leq& d.
\end{align*}
Then,
\begin{align*}
    \sum_{h=1}^H\mathbb{E}_{(s,\bm{a})\sim d^{\pi^{(n)}}_{P^\star,h}}\left[\hat{\beta}^{(n)}_h(s,\bm{a})\right]
    \leq\mathbb{E}_{(\tilde{s},\tilde{\bm{a}})\sim d^{\pi^{(n)}}_{P^\star,h}}\left[\left\Vert\phi_h^\star(\tilde{s},\tilde{\bm{a}})\right\Vert_{\Sigma_{n,\gamma^{(n)}_h,\phi_h^\star}^{-1}}\right]\sqrt{dA\left(\alpha^{(n)}\right)^2+H^2d\lambda}+\sqrt{{dA\left(\alpha^{(n)}\right)^2}/n}. 
\end{align*}
Second, we  calculate the term (b) in inequality \eqref{eq:regret_middle_mb}. Following Lemma \ref{lem:useful_mb} and noting that $f^{(n)}_h(s,\bm{a}$ is upper-bounded by $2$ (i.e., $B=2$ in Lemma \ref{lem:useful_mb}), we have 
\begin{align*}
    &\sum_{h=1}^H\mathbb{E}_{(s,\bm{a})\sim d^{\pi^{(n)}}_{P^\star,h}}[f_h^{(n)}(s,\bm{a})]\\
    \leq&\sum_{h=1}^{H-1}\mathbb{E}_{(\tilde{s},\tilde{\bm{a}})\sim d^{\pi^{(n)}}_{P^\star,h}}\left[\left\Vert\phi_h^\star(\tilde{s},\tilde{\bm{a}})\right\Vert_{\Sigma_{n,\gamma^{(n)}_h,\phi^\star_h}^{-1}}\right]\sqrt{nA\mathbb{E}_{(s,\bm{a})\sim\rho^{(n)}_h}\left[\left(f^{(n)}_h(s,\bm{a})\right)^2\right]+d\lambda}+\sqrt{A\mathbb{E}_{(s,\bm{a})\sim\rho^{(n)}_h}\left[\left(f^{(n)}_1(s,\bm{a})\right)^2\right]}\\
    \leq&\sum_{h=1}^{H-1}\mathbb{E}_{(\tilde{s},\tilde{\bm{a}})\sim d^{\pi^{(n)}}_{P^\star,h}}\left[\left\Vert\phi_h^\star(\tilde{s},\tilde{\bm{a}})\right\Vert_{\Sigma_{n,\gamma^{(n)}_h,\phi^\star_h}^{-1}}\right]\sqrt{nA\zeta^{(n)}+d\lambda}+\sqrt{A\zeta^{(n)}}\\
    \lesssim&\frac{\alpha^{(n)}}{H}\sum_{h=1}^{H-1}\mathbb{E}_{(\tilde{s},\tilde{\bm{a}})\sim d^{\pi_n}_{P^\star,h}}\left[\left\Vert\phi_h^\star(\tilde{s},\tilde{\bm{a}})\right\Vert_{\Sigma_{n,\gamma^{(n)}_h,\phi^\star_h}^{-1}}\right]+\sqrt{A\zeta^{(n)}}, 
\end{align*}
where in the second inequality, we use $\mathbb{E}_{(s,\bm{a})\sim\rho_h^{(n)}}\left[\left(f_h^{(n)}(s,\bm{a})\right)^2\right]\leq\zeta^{(n)}$, and in the last line, recall $\sqrt{nA\zeta^{(n)}+d\lambda}\lesssim\alpha^{(n)}/H$. Then, by combining the above calculation of the term (a) and term (b) in inequality \eqref{eq:regret_middle_mb}, we have:
\begin{align*}
      \overline{v}^{(n)}_i-\underline{v}^{(n)}_i=&\mathbb{E}_{s\sim d^{\pi^{(n)}}_{P^\star,1}}\left[\overline{V}^{(n)}_{1,i}(s)-\underline{V}^{(n)}_{1,i}(s)\right]\\
      \lesssim&\sum_{h=1}^{H-1}\left(\mathbb{E}_{(\tilde{s},\tilde{\bm{a}})\sim d^{\pi^{(n)}}_{P^\star,h}}\left[\Vert\phi_h^\star(\tilde{s},\tilde{\bm{a}})\Vert_{\Sigma_{n,\gamma^{(n)}_h,\phi_h^\star}^{-1}}\right]\sqrt{dA\left(\alpha^{(n)}\right)^2+H^2d\lambda}+\sqrt{\frac{dA\left(\alpha^{(n)}\right)^2}{n}}\right)\\
      &+H^2\sum_{h=1}^{H-1}\left(\frac{\alpha^{(n)}}{H}\mathbb{E}_{(\tilde{s},\tilde{\bm{a}})\sim d^{\pi^{(n)}}_{P^\star,h}}\left[\Vert\phi^\star_h(\tilde{s},\tilde{\bm{a}})\Vert_{\Sigma_{n,\gamma^{(n)}_h,\phi^\star_h}^{-1}}\right]+\sqrt{A\zeta^{(n)}}\right). 
\end{align*}
Taking maximum over $i$ on both sides and using the definition of $\Delta^{(n)}$, we get
\begin{align*}
    \Delta^{(n)}=&\max_{i\in[M]}\left\{\overline{v}^{(n)}_i-\underline{v}^{(n)}_i\right\}+2H\sqrt{A\zeta^{(n)}}\\
    \lesssim&\sum_{h=1}^{H-1}\left(\mathbb{E}_{(\tilde{s},\tilde{\bm{a}})\sim d^{\pi^{(n)}}_{P^\star,h}}\left[\Vert\phi_h^\star(\tilde{s},\tilde{\bm{a}})\Vert_{\Sigma_{n,\gamma^{(n)}_h,\phi_h^\star}^{-1}}\right]\sqrt{dA\left(\alpha^{(n)}\right)^2+H^2d\lambda}+\sqrt{\frac{dA\left(\alpha^{(n)}\right)^2}{n}}\right)\\
    &+H^2\sum_{h=1}^{H-1}\left(\frac{\alpha^{(n)}}{H}\mathbb{E}_{(\tilde{s},\tilde{\bm{a}})\sim d^{\pi^{(n)}}_{P^\star,h}}\left[\Vert\phi^\star_h(\tilde{s},\tilde{\bm{a}})\Vert_{\Sigma_{n,\gamma^{(n)}_h,\phi^\star_h}^{-1}}\right]+\sqrt{A\zeta^{(n)}}\right). 
\end{align*}
Hereafter, we take the dominating term out. Note that
\begin{align*}
    &\sum_{n=1}^N\mathbb{E}_{(\tilde{s},\tilde{\bm{a}})\sim d^{\pi^{(n)}}_{P^\star,h}}\left[\left\Vert\phi_h^\star(\tilde{s},\tilde{\bm{a}})\right\Vert_{\Sigma_{n,\gamma^{(n)}_h,\phi_h^\star}^{-1}}\right]\leq\sqrt{N\sum_{n=1}^N\mathbb{E}_{(\tilde{s},\tilde{\bm{a}})\sim d^{\pi^{(n)}}_{P^\star,h}}\left[\phi_h^\star(\tilde{s},\tilde{\bm{a}})^\top\Sigma^{-1}_{n,\gamma^{(n)}_h,\phi_h^\star}\phi_h^\star(\tilde{s},\tilde{\bm{a}})\right]}\tag{CS inequality}\\
    \lesssim&\sqrt{N\left(\log\det\left(\sum_{n=1}^N\mathbb{E}_{(\tilde{s},\tilde{\bm{a}})\sim d^{ \pi^{(n)}}_{P^\star,h}}[\phi^\star_h(\tilde{s},\tilde{\bm{a}})\phi^\star_h(\tilde{s},\tilde{\bm{a}})^\top]\right)-\log\det(\lambda I_d)\right)}\tag{Lemma \ref{lem:reduction}}\\ 
    \leq&\sqrt{dN \log\left(1+\frac{N}{d\lambda}\right)}.\tag{Potential function bound, Lemma \ref{lem:potential} noting $\Vert\phi_h^\star(s,\bm{a})\Vert_2\leq 1$ for any $(s,\bm{a})$.}
\end{align*}
Finally, 
\begin{align*}
    \sum_{n=1}^N\Delta^{(n)}\lesssim&H\left(\sqrt{dN\log\left(1+\frac{N}{d}\right)}\sqrt{{dA\left(\alpha^{(N)}\right)^2}+H^2d\lambda}+\sum_{n=1}^{N}\sqrt{\frac{dA\left(\alpha^{(n)}\right)^2}{n}}\right)\\
    &+H^3\left(\frac{1}{H}\sqrt{dN\log\left(1+\frac{N}{d\lambda}\right)}\alpha^{(N)}+\sum_{n=1}^N\sqrt{A\zeta^{(n)}}\right)\\ 
    \lesssim&H^2d\sqrt{NA\log\left(1+\frac{N}{d\lambda}\right)}\alpha^{(N)}\tag{Some algebra. We take the dominating term out. Note that $\alpha^{(n)}$ is increasing in $n$}\\
    \lesssim&H^3d^2A N^{\frac{1}{2}}\log\frac{\vert\mathcal{M}\vert HN}{\delta}.
\end{align*}
This concludes the proof.
\end{proof}

\paragraph{Proof of Theorem \ref{thm:mb}}
\begin{proof}
For any fixed episode $n$ and agent $i$, by Lemma \ref{lem:optimism_NE_CCE_mb}, Lemma \ref{lem:optimism_CE_mb} and Lemma \ref{lem:pessimism_mb}, we have
\begin{align*}
    v^{\dagger,\pi^{(n)}_{-i}}_i-v^{\pi^{(n)}}_i \left(\textrm{or }\max_{\omega\in\Omega_i}v^{\omega\circ\pi^{(n)}}_i-v^{\pi^{(n)}}_i\right)\leq\overline{v}^{(n)}_i-\underline{v}^{(n)}_i+2H\sqrt{A\zeta^{(n)}}\leq\Delta^{(n)}. 
\end{align*}
Taking maximum over $i$ on both sides, we have
\begin{align}
    \max_{i\in[M]}\left\{v^{\dagger,\pi^{(n)}_{-i}}_i-v^{\pi^{(n)}}_i\right\}\left(\textrm{or } \max_{i\in[M]}\left\{\max_{\omega\in\Omega_i}v^{\omega\circ\pi^{(n)}}_i-v^{\pi^{(n)}}_i\right\}\right)\leq\Delta^{(n)}.\label{eq:opt_mb}
\end{align}
From Lemma \ref{lem:pseudo_regret_mb}, with probability $1-\delta$, we can ensure 
\begin{align*}
    \sum_{n=1}^N\Delta^{(n)}\lesssim H^3d^2A N^{\frac{1}{2}}\log\frac{\vert\mathcal{M}\vert HN}{\delta}.
\end{align*}
Therefore, according to Lemma \ref{lem:convert_1}, when we pick $N$ to be 
\begin{align*}
    O\left(\frac{H^6d^4A^2}{\varepsilon^2}\log^2\left(\frac{HdA\vert\mathcal{M}\vert}{\delta\varepsilon}\right)\right),
\end{align*}
we have 
\begin{align*}
    \frac{1}{N}\sum_{n=1}^N\Delta^{(n)}\leq\varepsilon.
\end{align*}
On the other hand, from \eqref{eq:opt_mb}, we have
\begin{align*}
    &\max_{i\in[M]}\left\{v^{\dagger,\hat{\pi}_{-i}}_i-v^{\hat{\pi}}_i\right\}\left(\textrm{or }\max_{i\in[M]}\left\{\max_{\omega\in\Omega_i}v^{\omega\circ\hat{\pi}}_i-v^{\hat{\pi}}_i\right\}\right)\\
    =&\max_{i\in[M]}\left\{v^{\dagger,\pi^{(n^\star)}_{-i}}_i-v^{\pi^{(n^\star)}}_i\right\}\left(\textrm{or }\max_{i\in[M]}\left\{\max_{\omega\in\Omega_i}v^{\omega\circ\pi^{(n^\star)}}_i-v^{\pi^{(n^\star)}}_i\right\}\right)\\
    \leq&\Delta^{(n^\star)}=\min_{n\in[N]}\Delta^{(n)}\leq\frac{1}{N}\sum_{n=1}^N\Delta^{(n)}\leq\varepsilon,
\end{align*}
which has finished the proof. 
\end{proof}

\section{Analysis of the Model-Free Method}
\label{sec:mf_analysis}
For the model-free method, throughout this section we assume the Markov game is a block Markov game. 

\subsection{Non-Parametric Transition Model}
Define 
\begin{align*}
    \hat{P}^{(n)}_h(s^\prime\vert s,\bm{a})=\hat{\phi}_h^{(n)}(s,\bm{a})^\top\left(\sum_{(\tilde{s},\tilde{\bm{a}})\in\mathcal{D}}\hat{\phi}_h^{(n)}(\tilde{s},\tilde{\bm{a}})\hat{\phi}_h^{(n)}(\tilde{s},\tilde{\bm{a}})^\top+\lambda I_d\right)^{-1}\sum_{(\tilde{s},\tilde{\bm{a}},\tilde{s}^\prime)\in\mathcal{D}}\hat{\phi}_h^{(n)}(\tilde{s},\tilde{\bm{a}})\bm{1}_{\tilde{s}^\prime=s^\prime},
\end{align*}
where $\mathcal{D}=\mathcal{D}^{(n)}_h\cup\tilde{\mathcal{D}}_h^{(n)}$. The formulation of $\hat{P}^{(n)}_h$ makes the transition operator $\hat{P}^{(n)}_hf$ be exactly equal to the least square estimator, i.e., 
\begin{align*}
    \left(\hat{P}^{(n)}_hf\right)(s,\bm{a})=\hat{\phi}_h^{(n)}(s,\bm{a})^\top\left(\sum_{(\tilde{s},\tilde{\bm{a}})\in\mathcal{D}}\hat{\phi}_h^{(n)}(\tilde{s},\tilde{\bm{a}})\hat{\phi}_h^{(n)}(\tilde{s},\tilde{\bm{a}})^\top+\lambda I_d\right)^{-1}\sum_{(\tilde{s},\tilde{\bm{a}},\tilde{s}^\prime)\in\mathcal{D}}\hat{\phi}_h^{(n)}(\tilde{s},\tilde{\bm{a}})f(s^\prime).
\end{align*}

Furthermore, when $\phi^{(n)}_h$ belongs to the set of one-hot vectors, one can verify that we always have $\hat{P}^{(n)}_h(s^\prime\vert s,\bm{a})\geq 0$ and $\sum_{s^\prime\in\mathcal{S}}\hat{P}^{(n)}_h(s^\prime\vert s,\bm{a})\leq 1$. 

\subsection{Construction of $\mathcal{N}_h$ and $\mathcal{F}_h$}
Let $\mathcal{C}_h=\{\Sigma_h:\Sigma_h=\lambda I_d+\sum_{k=1}^l\phi_h(s_k,\bm{a}_k)\phi_h(s_k,\bm{a}_k)^\top\vert\phi_h\in\Phi_h, l\in[N],s_k\in\mathcal{S},\bm{a}_k\in\mathcal{A},\forall k\in[l]\}$ be the set of all possible covariance matrix generated in the algorithm. Fix a variable $L$, for each $h\in[H]$, define a function class $\tilde{\mathcal{F}}_h\in\mathbb{R}^{\mathcal{S}\times\mathcal{A}}$ by
\begin{align*}
    \tilde{\mathcal{F}}_h=\{&\left.f(s,\bm{a}):=r_{h,i}(s, \bm{a})+\phi_h(s,\bm{a})^\top\theta+\min\left\{c\Vert\phi_h(s,\bm{a})\Vert_{\Sigma_h^{-1}},H\right\}\right\vert\\
    &i\in[M],\phi_h\in\Phi_h,\Vert\theta\Vert_2\leq 2H^2\sqrt{d},c\in[0,L],\Sigma_h\in\mathcal{C}_h\}
\end{align*}
For a given parameter $\tilde{\varepsilon}$, let $\mathcal{N}_h$ be a $\tilde{\varepsilon}$-net of $\tilde{\mathcal{F}}_h$ under the $\Vert\cdot\Vert_\infty$ metric. Define $\Pi_h$ as the set of all possible policies produced by \eqref{eq:nash} (or \eqref{eq:cce} or \eqref{eq:ce}, according to the problem setting). We then define the discriminator function class $\mathcal{F}_h$ as followings: 
\begin{align*}
    \mathcal{F}_{1,h}&:=\left\{\left.f(s):=\mathbb{E}_{\bm{a}\sim U(\mathcal{A})}\left[\left\vert\phi_h(s,\bm{a})^\top\theta-\phi_h^\prime(s,\bm{a})^\top\theta^\prime\right\vert\right]\right\vert\phi_h,\phi_h^\prime\in\Phi_h,\max\{\Vert\theta\Vert_2,\Vert\theta^\prime\Vert_2\}\leq\sqrt{d}\right\},\\
    \mathcal{F}_{2,h}&:=\left\{\left.f(s):=\mathbb{E}_{\bm{a}\sim\pi_{h+1}(s)}\left[\frac{r_{h+1,i}(s,\bm{a})}{H}+\phi_{h+1}(s,\bm{a})^\top\theta\right]\right\vert i\in[M],\pi_{h+1}\in\Pi_{h+1},\phi_{h+1}\in\Phi_{h+1},\Vert\theta\Vert_2\leq\sqrt{d}\right\},\\
    \mathcal{F}_{3,h}&:=\bigg\{\left.f(s):=\max_{\tilde{\mu}_{h+1,i}}\mathbb{E}_{\bm{a}\sim(\tilde{\mu}_{h+1,i}\times\pi_{h+1,-i})(s)}\left[\frac{r_{h+1,i}(s,\bm{a})}{H}+\phi_{h+1}(s,\bm{a})^\top\theta\right]\right\vert\\
    &\quad\quad\quad\quad i\in[M],\pi_{h+1}\in\Pi_{h+1},\phi_{h+1}\in\Phi_{h+1},\Vert\theta\Vert_2\leq\sqrt{d}\bigg\},\tag{For NE and CCE}\\
    \mathcal{F}_{3,h}&:=\bigg\{\left.f(s):=\max_{\omega_{h+1,i}\in\Omega_{h+1,i}}\mathbb{E}_{\bm{a}\sim(\omega_{h+1,i}\circ\pi_{h+1})(s)}\left[\frac{r_{h+1,i}(s,\bm{a})}{H}+\phi_{h+1}(s,\bm{a})^\top\theta\right]\right\vert\\
    &\quad\quad\quad\quad i\in[M],\pi_{h+1}\in\Pi_{h+1},\phi_{h+1}\in\Phi_{h+1},\Vert\theta\Vert_2\leq\sqrt{d}\bigg\},\tag{For CE}\\
    \mathcal{F}_{4,h}&:=\bigg\{\left.f(s):=\mathbb{E}_{\bm{a}\sim\pi_{h+1}(s)}\left[\frac{\min\left\{c\Vert\phi_{h+1}(s,\bm{a})\Vert_{\Sigma_{h+1}^{-1}},H\right\}}{H^2}+\phi_{h+1}(s,\bm{a})^\top\theta\right]\right\vert \\
    &\quad\quad\quad\quad c\in[0,L],\pi_{h+1}\in\Pi_{h+1},\Sigma_{h+1}\in\mathcal{C}_{h+1},\phi_{h+1}\in\Phi_{h+1},\Vert\theta\Vert_2\leq\sqrt{d}\bigg\},\\
    \mathcal{G}&:=\{f:\mathcal{S}\rightarrow[0,1]\},\\
    \mathcal{F}_h&:=\left(\mathcal{F}_{1,h}\cup\mathcal{F}_{2,h}\cup\mathcal{F}_{3,h}\cup\mathcal{F}_{4,h}\right)\cap\mathcal{G}.
\end{align*}

\subsection{High Probability Events}
We define the following event
\begin{align*}
    \mathcal{E}_1&:\ \forall n\in[N],h\in[H],\rho\in\left\{\rho^{(n)}_h,\tilde{\rho}^{(n)}_h\right\},f\in\mathcal{F}_h,\quad\mathbb{E}_\rho\left[\left(\left(\left(\hat{P}^{(n)}_h-P^\star_h\right)f\right)(s,\bm{a})\right)^2\right]\leq\zeta^{(n)},\\
    \mathcal{E}_2&:\ \forall n\in[N],h\in[H],\phi_h\in\Phi_h,\quad\Vert\phi_h(s,\bm{a})\Vert_{\left(\hat{\Sigma}^{(n)}_{h,\phi_h}\right)^{-1}}=\Theta\left(\Vert\phi_h(s,\bm{a})\Vert_{\Sigma^{-1}_{n,\rho^{(n)}_h,\phi_h}}\right)\\
    \mathcal{E}&:=\mathcal{E}_1\cap\mathcal{E}_2. 
\end{align*}
Similar to the procedure of the model-based case, we first prove a few lemmas which lead to the conclusion that $\mathcal{E}$ holds with a high probability. 

\begin{lemma}\label{lem:v_formula}
For any $n\in[N],h\in[H]$, we have $\hat{P}^{(n)}_h(s^\prime\vert s,\bm{a})=\hat{\phi}_h^{(n)}(s,\bm{a})^\top\hat{w}_h^{(n)}(s^\prime)$ for some $\hat{w}_h^{(n)}:\mathcal{S}\rightarrow\mathbb{R}^d$. For any function $f:\mathcal{S}\rightarrow[0,1]$ and $n\in[N],h\in[H]$, we have $\left\Vert\int_{\mathcal{S}}\hat{w}_h^{(n)}(s^\prime)f(s^\prime)\mathrm{d}s^\prime\right\Vert_2\leq\sqrt{d}$, and there exist $\theta,\tilde{\theta}\in\mathbb{R}^d$ such that $\left(P_h^\star f\right)(s,\bm{a})=\phi_h^\star(s,\bm{a})^\top\theta,\ \left(\hat{P}_h^{(n)}f\right)(s,\bm{a})=\hat{\phi}_h^{(n)}(s,\bm{a})^\top\tilde{\theta}$ and $\max\{\Vert\theta\Vert_2,\Vert\tilde{\theta}\Vert_2\}\leq \sqrt{d}$. Furthermore, we have $\Vert\tilde{\theta}\Vert_\infty\leq 1$. 
\end{lemma}
\begin{proof}
By definition, we have 
\begin{align*}
    \left(P_h^\star f\right)(s,\bm{a})=&\int_{\mathcal{S}}P_h^\star(s^\prime\vert s,\bm{a})f(s^\prime)\mathrm{d}s^\prime\\
    =&\phi^\star_h(s,\bm{a})^\top\int_{\mathcal{S}}w_h^\star(s^\prime)f(s^\prime)\mathrm{d}s^\prime\\
    =&\phi^\star_h(s,\bm{a})^\top\theta,
\end{align*}
where $\theta=\int_{\mathcal{S}}w_h^\star(s^\prime)f(s^\prime)\mathrm{d}s^\prime$. Furthermore, note that $\Vert f\Vert_\infty\leq 1$, according to the assumption on $w_h^\star$, we have
\begin{align*}
    \left\Vert\int_{\mathcal{S}}w^\star_h(s^\prime)f(s^\prime)\mathrm{d}s^\prime\right\Vert_2\leq\sqrt{d}, 
\end{align*}
which implies $\Vert\theta\Vert_2\leq\sqrt{d}$. For $\left(\hat{P}_h^{(n)}f\right)(s,\bm{a})$, let
\begin{align*} 
    \hat{w}^{(n)}_h(s^\prime):=\left(\sum_{(\tilde{s},\tilde{\bm{a}})\in\mathcal{D}_h^{(n)}\cup\tilde{\mathcal{D}}_h^{(n)}}\phi_h^{(n)}(\tilde{s},\tilde{\bm{a}})\phi_h^{(n)}(\tilde{s},\tilde{\bm{a}})^\top+\lambda I_d\right)^{-1}\sum_{(\tilde{s},\tilde{\bm{a}},\tilde{s}^\prime)\in\mathcal{D}_h^{(n)}\cup\tilde{\mathcal{D}}_h^{(n)}}\phi_h^{(n)}(\tilde{s},\tilde{\bm{a}})\bm{1}_{\tilde{s}^\prime=s^\prime}.
\end{align*}
Since $\phi_h^{(n)}(s,\bm{a})$ is an one-hot vector, one has $\left\Vert\hat{w}_h^{(n)}(s^\prime)\right\Vert_\infty\leq 1, \forall s^\prime\in\mathcal{S}$. It follows that $\left\Vert\int_{\mathcal{S}}\hat{w}_h^{(n)}(s^\prime)f(s^\prime)\mathrm{d}s^\prime\right\Vert_\infty\leq 1$, and therefore, $\left\Vert\int_{\mathcal{S}}\hat{w}_h^{(n)}(s^\prime)f(s^\prime)\mathrm{d}s^\prime\right\Vert_2\leq\sqrt{d}$.  By definition, we have
\begin{align*}
    \left(\hat{P}_h^{(n)}f\right)(s,\bm{a})=&\int_{\mathcal{S}}\hat{P}_h^{(n)}(s^\prime\vert s,\bm{a})f(s^\prime)\mathrm{d}s^\prime\\
    =&\phi_h^{(n)}(s,\bm{a})^\top\int_{\mathcal{S}}\hat{w}^{(n)}_h(s^\prime)f(s^\prime)\mathrm{d}s^\prime\\
    =&\phi_h^{(n)}(s,\bm{a})^\top\tilde{\theta},
\end{align*}
where $\tilde{\theta}=\int_{\mathcal{S}}\hat{w}_h^{(n)}(s^\prime)f(s^\prime)\mathrm{d}s^\prime$. Due to the property we just derived for $\hat{w}^{(n)}_h$, similar to the proof of the true model, we also have $\Vert\tilde{\theta}\Vert_2\leq\sqrt{d}$. Meanwhile, one can easily see that $\Vert\tilde{\theta}\Vert_\infty\leq 1$, using the fact $\left\Vert\int_{\mathcal{S}}\hat{w}_h^{(n)}(s^\prime)f(s^\prime)\mathrm{d}s^\prime\right\Vert_\infty\leq 1$. 
\end{proof}

\begin{lemma}[Covering Number of $\tilde{\mathcal{F}}_h$]
\label{lem:cov_num_orig}
When $\Phi_h$ is the set of one-hot vectors and $\lambda\geq 1$, it's possible to construct the $\tilde{\varepsilon}$-net $\mathcal{N}_h$ such that $\vert\mathcal{N}_h\vert\leq M\left(\frac{12H^2L^2d}{\tilde{\varepsilon}}\right)^{3d}\vert\Phi\vert,\forall h\in[H]$. Furthermore, we have $\vert\Pi_h\vert\leq \vert\mathcal{N}_h\vert^M\leq M^M\left(\frac{12H^2L^2d}{\tilde{\varepsilon}}\right)^{3Md}\vert\Phi\vert^M$. 
\end{lemma}
\begin{proof}
Recall that 
\begin{align*}
    \tilde{\mathcal{F}}_h=\{&\left.f(s,\bm{a}):=r_{h,i}(s,\bm{a})+\phi_h(s,\bm{a})^\top\theta+\min\{c\Vert\phi_h(s,\bm{a})\Vert_{\Sigma_h^{-1}},H\}\right\vert \\
    &i\in[M],\phi_h\in\Phi_h,\Vert\theta\Vert_2\leq 2H^2\sqrt{d},c\in[0,L],\Sigma\in\mathcal{C}_h\}.
\end{align*}
Note that when $\Phi_h$ is the set of one-hot vectors, $\Sigma_h$ will be a diagonal matrix. In this case, $\tilde{\mathcal{F}}_h$ is the subset of the following function class:
\begin{align*}
    \tilde{\mathcal{F}}_h^\prime:=\{&\left.f(s,\bm{a}):=r_{h,i}(s,\bm{a})+\min\{c\phi_h(s,\bm{a})^\top\theta^\prime,H\}+\phi_h(s,\bm{a})^\top\theta\right\vert\\
    &i\in[M],\phi_h\in\Phi_h,0\leq c\leq L,\max\{\Vert\theta\Vert_2,\Vert\theta^\prime\Vert_2\}\leq 2H^2\sqrt{d}\}.
\end{align*}
Let $\Theta$ be an $\ell_2$-cover of the set $\{\theta\in\mathbb{R}^d:\Vert\theta\Vert_2\leq 2H^2\sqrt{d}\}$ at scale $\tilde{\varepsilon}$. Then we know $\vert\Theta\vert\leq\left(\frac{4H^2\sqrt{d}}{\tilde{\varepsilon}}\right)^d$. Let $\mathcal{W}$ be an $\ell_\infty$-cover of the set $[0,L]$ at scale $\tilde{\varepsilon}^\prime:=\frac{\tilde{\varepsilon}}{2H^2\sqrt{d}}$, we have $\vert\mathcal{W}\vert\leq \frac{2H^2L\sqrt{d}}{\tilde{\varepsilon}}$. Define the covering set by 
\begin{align*}
    \bar{\mathcal{F}}_h:=\left\{\left.\bar{f}(s,\bm{a}):=r_{h,i}(s,\bm{a})+\min\{\tilde{c}\phi_h(s,\bm{a})^\top\tilde{\theta}^\prime,H\}+\phi_h(s,\bm{a})^\top\tilde{\theta}\right\vert i\in[M],\phi_h\in\Phi_h,\tilde{c}\in\mathcal{W},\tilde{\theta},\tilde{\theta}^\prime\in\Theta\right\}.
\end{align*}
Then, for any $f\in\tilde{\mathcal{F}}_h^\prime$, by definition, suppose $f$ takes the following form:
\begin{align*}
    f(s,\bm{a}):=r_{h,i}(s,\bm{a})+\min\{c\phi_h(s,\bm{a})^\top\theta^\prime,H\}+\phi_h(s,\bm{a})^\top\theta,\quad 0\leq c\leq L,\max\{\Vert\theta\Vert_2,\Vert\tilde{\theta}\Vert_2\}\leq 2H^2\sqrt{d}.
\end{align*}
Then we can find $\tilde{\theta},\tilde{\theta}^\prime\in\Theta$, $\tilde{c}\in\mathcal{W}$ such that $\Vert\theta-\tilde{\theta}\Vert_2\leq\tilde{\varepsilon},\Vert\theta^\prime-\tilde{\theta}^\prime\Vert_2\leq\tilde{\varepsilon}$ and $\vert c-\tilde{c}\vert\leq\tilde{\varepsilon}^\prime$. Let
\begin{align*}
    \bar{f}(s,\bm{a}):=r_{h,i}(s,\bm{a})+\min\{\tilde{c}\phi_h(s,\bm{a})^\top\tilde{\theta}^\prime,H\}+\phi_h(s,\bm{a})^\top\tilde{\theta},
\end{align*}
then we have
\begin{align*}
    &\vert f(s,\bm{a})-\bar{f}(s,\bm{a})\vert\\
    \leq&\Vert\phi_h(s,\bm{a})\Vert_2\left\Vert\theta-\tilde{\theta}\right\Vert_2+\Vert\phi_h(s,\bm{a})\Vert_2\left\Vert\tilde{c}\tilde{\theta}^\prime-c\theta^\prime\right\Vert_2\\
    \leq&\tilde{\varepsilon}+\vert\tilde{c}-c\vert\left\Vert\tilde{\theta}^\prime\right\Vert_2+c\left\Vert\theta^\prime-\tilde{\theta}^\prime\right\Vert_2\\
    \leq&\tilde{\varepsilon}+2H^2\sqrt{d}\tilde{\varepsilon}^\prime+L\tilde{\varepsilon}\\
    \leq&3L\tilde{\varepsilon},
\end{align*}
which implies $\bar{\mathcal{F}}_h$ is a $3L\tilde{\varepsilon}$-covering of $\tilde{F}_h^\prime$ (therefore, is a $3L\tilde{\varepsilon}$-covering of $\tilde{F}_h$), and we have
\begin{align*}
    \left\vert\bar{\mathcal{F}}_h\right\vert\leq M\left(\frac{4H^2Ld}{\tilde{\varepsilon}}\right)^{3d}\vert\Phi\vert.
\end{align*}
Replacing $\tilde{\varepsilon}$ by $\frac{\tilde{\varepsilon}}{3L}$, we get an $\tilde{\varepsilon}$-covering of $\tilde{\mathcal{F}}_h$ whose size is no larger than $M\left(\frac{12H^2L^2d}{\tilde{\varepsilon}}\right)^{3d}\vert\Phi\vert$. For $\Pi_h$, since each policy is determined by $M$ members from $\mathcal{N}_h$, we have $\vert\Pi_h\vert\leq\vert\mathcal{N}_h\vert^M$, which has finished the proof. 
\end{proof}

\begin{lemma}[Covering Number of $\mathcal{F}_h$]
\label{lem:cov_num}
When $\Phi_h$ is the set of one-hot vectors and $\lambda\geq 1$. The $\gamma$-covering number of $\mathcal{F}_h$ is at most $4M\vert\Pi_{h+1}\vert\left(\frac{6L^2d}{\gamma}\right)^{3d}\vert\Phi\vert^2$.
\end{lemma}
\begin{proof}
We cover $\mathcal{F}_{1,h},\mathcal{F}_{2,h},\mathcal{F}_{3,h},\mathcal{F}_{4,h}$ separately. For $\mathcal{F}_{1,h}$, let $\Theta$ be an $\ell_2$-cover of the set $\{\theta\in\mathbb{R}^d:\Vert\theta\Vert_2\leq\sqrt{d}\}$ at scale $\gamma$. Then we know $\vert\Theta\vert\leq\left(\frac{2\sqrt{d}}{\gamma}\right)^d$. Define the covering set of $\mathcal{F}_{1,h}$ as
\begin{align*}
    \tilde{\mathcal{F}}_{1,h}:=\left\{\left.\tilde{f}(s):=\mathbb{E}_{\bm{a}\sim U(\mathcal{A})}\left[\left\vert\phi_h(s,\bm{a})^\top\tilde{\theta}-\phi_h^\prime(s,\bm{a})^\top\tilde{\theta}^\prime\right\vert\right]\right\vert\phi_h,\phi_h^\prime\in\Phi_h,\tilde{\theta},\tilde{\theta}^\prime\in\Theta\right\}.
\end{align*}
For any $f\in\mathcal{F}_{1,h}$, suppose 
\begin{align*}
    f(s)=\mathbb{E}_{\bm{a}\sim U(\mathcal{A})}\left[\left\vert\phi_h(s,\bm{a})^\top\theta-\phi_h^\prime(s,\bm{a})^\top\theta^\prime\right\vert\right],\quad\phi_h,\phi_h^\prime\in\Phi_h,\max\{\Vert\theta\Vert_2,\Vert\theta^\prime\Vert_2\}\leq\sqrt{d},
\end{align*}
Then we can find $\tilde{\theta},\tilde{\theta}^\prime\in\Theta$ such that $\Vert\theta-\tilde{\theta}\Vert_2\leq\gamma,\Vert\theta^\prime-\tilde{\theta}^\prime\Vert_2\leq\gamma$. Let
\begin{align*}
    \tilde{f}(s):=\mathbb{E}_{\bm{a}\sim U(\mathcal{A})}\left[\left\vert\phi_h(s,\bm{a})^\top\tilde{\theta}-\phi_h^\prime(s,\bm{a})^\top\tilde{\theta}^\prime\right\vert\right].
\end{align*}
Then we have
\begin{align*}
    \vert f(s)-\tilde{f}(s)\vert\leq\frac{1}{A}\sum_{\bm{a}\in\mathcal{A}}\left\Vert\phi_h(s,\bm{a})\right\Vert_2\left\Vert\theta-\tilde{\theta}\right\Vert_2+\frac{1}{A}\sum_{\bm{a}\in\mathcal{A}}\left\Vert\phi_h^\prime(s,\bm{a})\right\Vert_2\left\Vert\theta^\prime-\tilde{\theta}^\prime\right\Vert_2\leq 2\gamma,
\end{align*}
which implies $\tilde{\mathcal{F}}_{1,h}$ is a $2\gamma$ covering of $\mathcal{F}_{1,h}$. Furthermore, we have
\begin{align*}
    \vert\tilde{\mathcal{F}}_{1,h}\vert\leq \left(\frac{2d}{\gamma}\right)^{2d}\vert\Phi\vert^2.
\end{align*}
For $\mathcal{F}_{2,h}$, we construct
\begin{align*}
    \tilde{\mathcal{F}}_{2,h}&:=\left\{\left.\tilde{f}(s):=\mathbb{E}_{\bm{a}\sim\pi_{h+1}(s)}\left[\frac{r_{h+1,i}(s,\bm{a})}{H}+\phi_{h+1}(s,\bm{a})^\top\tilde{\theta}\right]\right\vert i\in[M],\phi_{h+1}\in\Phi_{h+1},\tilde{\theta}\in\Theta,\pi_{h+1}\in\Pi_{h+1}\right\}.
\end{align*}
Similar to the proof of $\mathcal{F}_{1,h}$, we may verify $\tilde{\mathcal{F}}_{2,h}$ is a $\gamma$-covering of $\mathcal{F}_{2,h}$, and 
\begin{align*}
    \vert\tilde{\mathcal{F}}_{2,h}\vert\leq M\vert\Pi_{h+1}\vert\left(\frac{2d}{\gamma}\right)^d\vert\Phi\vert.
\end{align*}
For $\mathcal{F}_{3,h}$, we only prove the case of NE or CCE, the case of CE can be proved in a similar way. We construct
\begin{align*}
    \tilde{\mathcal{F}}_{3,h}&:=\bigg\{\left.\tilde{f}(s):=\max_{\tilde{\mu}_{h+1,i}}\mathbb{E}_{\bm{a}\sim(\tilde{\mu}_{h+1,i}\times\pi_{h+1,-i})(s)}\left[\frac{r_{h+1,i}(s,\bm{a})}{H}+\phi_{h+1}(s,\bm{a})^\top\tilde{\theta}\right]\right\vert\\
    &i\in[M],\phi_{h+1}\in\Phi_{h+1},\tilde{\theta}\in\Theta,\pi_{h+1}\in\Pi_{h+1}\bigg\}.
\end{align*}
For any $f\in\mathcal{F}_{3,h}$, suppose
\begin{align*}
    f(s)=\max_{\tilde{\mu}_{h+1,i}}\mathbb{E}_{\bm{a}\sim(\tilde{\mu}_{h+1,i}\times\pi_{h+1,-i})(s)}\left[\frac{r_{h+1,i}(s,\bm{a})}{H}+\phi_{h+1}(s,\bm{a})^\top\theta\right],\quad i\in[M],\pi_{h+1}\in\Pi_{h+1},\phi_{h+1}\in\Phi_{h+1},\Vert\theta\Vert_2\leq\sqrt{d}.
\end{align*}
Then we can find $\tilde{\theta}\in\Theta$ such that $\Vert\theta-\tilde{\theta}\Vert_2\leq\gamma$. Let
\begin{align*}
    \tilde{f}(s)=\max_{\tilde{\mu}_{h+1,i}}\mathbb{E}_{\bm{a}\sim(\tilde{\mu}_{h+1,i}\times\pi_{h+1,-i})(s)}\left[\frac{r_{h+1,i}(s,\bm{a})}{H}+\phi_{h+1}(s,\bm{a})^\top\tilde{\theta}\right],
\end{align*}
we have
\begin{align*}
    f(s)-\tilde{f}(s)=&\max_{\tilde{\mu}_{h+1,i}}\mathbb{E}_{\bm{a}\sim(\tilde{\mu}_{h+1,i}\times\pi_{h+1,-i})(s)}\left[\frac{r_{h+1,i}(s,\bm{a})}{H}+\phi_{h+1}(s,\bm{a})^\top\theta\right]\\
    &-\max_{\tilde{\mu}_{h+1,i}}\mathbb{E}_{\bm{a}\sim(\tilde{\mu}_{h+1,i}\times\pi_{h+1,-i})(s)}\left[\frac{r_{h+1,i}(s,\bm{a})}{H}+\phi_{h+1}(s,\bm{a})^\top\tilde{\theta}\right]\\
    \leq&\max_{\tilde{\mu}_{h+1,i}}\bigg(\mathbb{E}_{\bm{a}\sim(\tilde{\mu}_{h+1,i}\times\pi_{h+1,-i})(s)}\left[\frac{r_{h+1,i}(s,\bm{a})}{H}+\phi_{h+1}(s,\bm{a})^\top\theta\right]\\
    &- \mathbb{E}_{\bm{a}\sim(\tilde{\mu}_{h+1,i}\times\pi_{h+1,-i})(s)}\left[\frac{r_{h+1,i}(s,\bm{a})}{H}+\phi_{h+1}(s,\bm{a})^\top\tilde{\theta}\right]\bigg)\\
    =&\max_{\tilde{\mu}_{h+1,i}}\left(\mathbb{E}_{\bm{a}\sim(\tilde{\mu}_{h+1,i}\times\pi_{h+1,-i})(s)}\left[\phi_{h+1}(s,\bm{a})^\top\theta-\phi_{h+1}(s,\bm{a})^\top\tilde{\theta}\right]\right)\\
    \leq&\Vert\theta-\tilde{\theta}\Vert_2\\
    \leq&\gamma,
\end{align*}
and
\begin{align*}
    \tilde{f}(s)-f(s)=&\max_{\tilde{\mu}_{h+1,i}}\mathbb{E}_{\bm{a}\sim(\tilde{\mu}_{h+1,i}\times\pi_{h+1,-i})(s)}\left[\frac{r_{h+1,i}(s,\bm{a})}{H}+\phi_{h+1}(s,\bm{a})^\top\tilde{\theta}\right]\\
    &-\max_{\tilde{\mu}_{h+1,i}}\mathbb{E}_{\bm{a}\sim(\tilde{\mu}_{h+1,i}\times\pi_{h+1,-i})(s)}\left[\frac{r_{h+1,i}(s,\bm{a})}{H}+\phi_{h+1}(s,\bm{a})^\top\theta\right]\\
    \leq&\max_{\tilde{\mu}_{h+1,i}}\bigg(\mathbb{E}_{\bm{a}\sim(\tilde{\mu}_{h+1,i}\times\pi_{h+1,-i})(s)}\left[\frac{r_{h+1,i}(s,\bm{a})}{H}+\phi_{h+1}(s,\bm{a})^\top\tilde{\theta}\right]\\
    &- \mathbb{E}_{\bm{a}\sim(\tilde{\mu}_{h+1,i}\times\pi_{h+1,-i})(s)}\left[\frac{r_{h+1,i}(s,\bm{a})}{H}+\phi_{h+1}(s,\bm{a})^\top\theta\right]\bigg)\\
    =&\max_{\tilde{\mu}_{h+1,i}}\left(\mathbb{E}_{\bm{a}\sim(\tilde{\mu}_{h+1,i}\times\pi_{h+1,-i})(s)}\left[\phi_{h+1}(s,\bm{a})^\top\tilde{\theta}-\phi_{h+1}(s,\bm{a})^\top\theta\right]\right)\\
    \leq&\Vert\theta-\tilde{\theta}\Vert_2\\
    \leq&\gamma,
\end{align*}
which implies
\begin{align*}
    \left\vert\tilde{f}(s)-f(s)\right\vert \leq \gamma.
\end{align*}
Therefore, we conclude $\tilde{\mathcal{F}}_{3,h}$ is a $\gamma$-covering of $\mathcal{F}_{3,h}$, and
\begin{align*}
    \vert\tilde{\mathcal{F}}_{3,h}\vert\leq M\vert\Pi_{h+1}\vert\left(\frac{2d}{\gamma}\right)^d\vert\Phi\vert.
\end{align*}

For $\mathcal{F}_{4,h}$, note that when $\Phi_h$ is the set of one-hot vectors, $\Sigma_h$ will be a diagonal matrix. In this case, $\mathcal{F}_{4,h}$ is the subset of the following function class:
\begin{align*}
    \mathcal{F}_{4,h}^\prime:=\bigg\{&\left.f(s):=\mathbb{E}_{\bm{a}\sim\pi_{h+1}(s)}\left[\frac{\min\{c\phi_{h+1}(s,\bm{a})^\top\theta^\prime,H\}}{H^2}+\phi_{h+1}(s,\bm{a})^\top\theta\right]\right\vert\\
    &0\leq c\leq L,\pi_{h+1}\in\Pi_{h+1},\max\{\Vert\theta\Vert_2,\Vert\theta^\prime\Vert_2\}\leq\sqrt{d},\phi_{h+1}\in\Phi_{h+1}\bigg\}.
\end{align*}
In this case, let $\mathcal{W}$ be an $\ell_\infty$ cover of the set $[0,L]$ at scale $\tilde{\gamma}:=\frac{\gamma}{\sqrt{d}}$, we have $\vert\mathcal{W}\vert\leq \frac{L\sqrt{d}}{\gamma}$. Let
\begin{align*}
    \tilde{\mathcal{F}}_{4,h}:=\bigg\{&\left.\tilde{f}(s):=\mathbb{E}_{\bm{a}\sim\pi_{h+1}(s)}\left[\frac{\min\{\tilde{c}\phi_{h+1}(s,\bm{a})^\top\tilde{\theta}^\prime,H\}}{H^2}+\phi_{h+1}(s,\bm{a})^\top\tilde{\theta}\right]\right\vert\\
    &\tilde{c}\in\mathcal{W},\pi_{h+1}\in\Pi_{h+1},\tilde{\theta},\tilde{\theta}^\prime\in\Theta,\phi_{h+1}\in\Phi_{h+1}\bigg\}.
\end{align*}
Then, for any $f\in\mathcal{F}_{4,h}$, suppose 
\begin{align*}
    f(s):=\mathbb{E}_{\bm{a}\sim\pi_{h+1}(s)}&\left[\frac{\min\{c\phi_{h+1}(s,\bm{a})^\top\theta^\prime,H\}}{H^2}+\phi_{h+1}(s,\bm{a})^\top\theta\right],\\
    &0\leq c\leq L,\pi_{h+1}\in\Pi_{h+1},\max\{\Vert\theta\Vert_2,\Vert\theta^\prime\Vert_2\}\leq\sqrt{d},\phi_{h+1}\in\Phi_{h+1}.
\end{align*}
Then we can find $\tilde{\theta},\tilde{\theta}^\prime\in\Theta$, $\tilde{c}\in\mathcal{W}$ such that $\Vert\theta-\tilde{\theta}\Vert_2\leq\gamma,\Vert\theta^\prime-\tilde{\theta}^\prime\Vert_2\leq\gamma$ and $\vert c-\tilde{c}\vert\leq\tilde{\gamma}$. Let
\begin{align*}
    \tilde{f}(s):=\mathbb{E}_{\bm{a}\sim\pi_{h+1}(s)}\left[\frac{\min\{\tilde{c}\phi_{h+1}(s,\bm{a})^\top\tilde{\theta}^\prime,H\}}{H^2}+\phi_{h+1}(s,\bm{a})^\top\tilde{\theta}\right],
\end{align*}
then we have
\begin{align*}
    &\vert f(s)-\tilde{f}(s)\vert\\
    \leq&\mathbb{E}_{\bm{a}\sim\pi_{h+1}(s)}\left[\Vert\phi_{h+1}(s,\bm{a})\Vert_2\left\Vert\theta-\tilde{\theta}\right\Vert_2\right]+\frac{1}{H^2}\mathbb{E}_{\bm{a}\sim\pi_{h+1}(s)}\left[\Vert\phi_{h+1}(s,\bm{a})\Vert_2\left\Vert\tilde{c}\tilde{\theta}^\prime-c\theta^\prime\right\Vert_2\right]\\
    \leq&\gamma+\frac{1}{H^2}\left(\vert\tilde{c}-c\vert\left\Vert\tilde{\theta}^\prime\right\Vert_2+c\left\Vert\theta^\prime-\tilde{\theta}^\prime\right\Vert_2\right)\\
    \leq&\gamma+\frac{\sqrt{d}}{H^2}\tilde{\gamma}+\frac{L}{H^2}\gamma\\
    \leq&3L\gamma,
\end{align*}
which implies $\tilde{\mathcal{F}}_{4,h}$ is a $3L\gamma$-covering of $\mathcal{F}_{4,h}$, and we have
\begin{align*}
    \left\vert\tilde{\mathcal{F}}_{4,h}\right\vert\leq\vert\Pi_{h+1}\vert\left(\frac{2Ld}{\gamma}\right)^{3d}\vert\Phi\vert.
\end{align*}
In summary, we know $\tilde{\mathcal{F}}_h:=\tilde{\mathcal{F}}_{1,h}\cup\tilde{\mathcal{F}}_{2,h}\cup\tilde{\mathcal{F}}_{3,h}\cup\tilde{\mathcal{F}}_{4,h}$ is a $3L\gamma$-covering of $\mathcal{F}_h$. And 
\begin{align*}
    \left\vert\mathcal{F}_h\right\vert\leq 4M\vert\Pi_{h+1}\vert\left(\frac{2Ld}{\gamma}\right)^{3d}\vert\Phi\vert^2.
\end{align*}
Replacing $\gamma$ by $\frac{\gamma}{3L}$, we get an $\gamma$-covering of $\mathcal{F}_h$ whose size is no larger than $4M\vert\Pi_{h+1}\vert\left(\frac{6L^2d}{\gamma}\right)^{3d}\vert\Phi\vert^2$, which has finished the proof. 
\end{proof}

Below we omit the superscript $n$ and subscript $h$ when clear from the context. Denote
\begin{align}
    \mathcal{L}_{\lambda,\mathcal{D}}(\phi,\theta,f) &= \frac{1}{|\mathcal{D}|}\sum_{(s,\bm{a},s')\in \mathcal{D}}\left(\phi(s,\bm{a})^\top\theta - f(s^\prime)\right)^2+\frac{\lambda}{|\mathcal{D}|}\Vert \theta\Vert^2_2\\
    \mathcal{L}_{\mathcal{D}}(\phi,\theta,f) &= \frac{1}{|\mathcal{D}|}\sum_{(s,\bm{a},s^\prime)\in \mathcal{D}}\left(\phi(s,\bm{a})^\top \theta - f(s')\right)^2\\
    \mathcal{L}_{\rho}(\phi,\theta,f)&=\mathbb{E}_{(s,\bm{a})\sim\rho,s^\prime\sim P^\star(s,\bm{a})}\left[\left(\phi(s,\bm{a})^\top\theta - f(s^\prime)\right)^2\right].
\end{align}

\begin{lemma}[Uniform Convergence for Square Loss]\label{lem:fastrate_sqloss}
Let there be a dataset $\mathcal{D}:=\{(s_i,\bm{a}_i,s^\prime_i)\}_{i=1}^n$ collected in $n$ episodes. Denote the data generating distribution in iteration $i$ by $d_i$, and $\rho=\frac{1}{n}\sum_{i=1}^n d_i$. Note that $d_i$ can depend on the randomness in episodes $1,\ldots,i-1$. For a finite feature class $\Phi$ and a discriminator class $\mathcal{F}:\mathcal{S}\rightarrow [0,1]$ with $\gamma$-covering number $\Vert\mathcal{F}\Vert_\gamma$, with probability at least $1-\delta$,
\begin{align*}
    &\left|\left[\mathcal{L}_\rho(\phi,\theta,f)-\mathcal{L}_\rho(\phi^\star,\theta^\star_f,f)\right]-\left[\mathcal{L}_\mathcal{D}(\phi,\theta,f)-\mathcal{L}_\mathcal{D}(\phi^\star,\theta^\star_f,f)\right]\right|\\
    \leq& \frac{1}{2}\left[\mathcal{L}_\rho(\phi,\theta,f)-\mathcal{L}_\rho(\phi^\star,\theta^\star_f,f)\right]+\frac{64\log(\frac{2(4n)^d\cdot|\Phi|\cdot\Vert\mathcal{F}\Vert_{1/2n}}{\delta})}{n}
\end{align*}
for all $\phi\in\Phi$, $\Vert\theta\Vert_\infty\leq 1$ and $f\in\mathcal{F}$. Recall that $\phi^\star$ is the true feature and $\theta^\star_f$ is defined as $\mathbb{E}_{s'\sim P^\star(s,\bm{a})}[f(s^\prime)]=\langle\phi^\star(s,\bm{a}),\theta_f^\star\rangle$. 
\end{lemma}
\begin{proof}
To start, we focus on a given $f\in\mathcal{F}$.
We first give a high probability bound on the following deviation term:
\begin{align*}
    \left\vert\mathcal{L}_{\rho}(\phi,\theta,f)-\mathcal{L}_{\rho}(\phi^*,\theta^*_f,f)-\left(\mathcal{L}_{\mathcal{D}}(\phi,\theta,f)-\mathcal{L}_{\mathcal{D}}(\phi^*,\theta^*_f,f)\right)\right\vert.
\end{align*}
Denote $g(s_i,\bm{a}_i)=\phi(s_i,\bm{a}_i)^\top\theta$ and $g^\star(s_i,\bm{a}_i)=\phi^\star(s_i,\bm{a}_i)^\top\theta^\star_f$. At episode $i$, let $\mathcal{F}_{i-1}$ be the $\sigma$-field generated by all the random variables over the first $i-1$ episodes, for the random variable $Y_i:=\left(g(s_i,\bm{a}_i)-f(s^\prime_i)\right)^2-\left(g^\star(s_i,\bm{a}_i)-f(s^\prime_i)\right)^2$, we have
\begin{align*}
    \mathbb{E}[Y_i\vert\mathcal{F}_{i-1}]=&\mathbb{E}\left[\left(g(s_i,\bm{a}_i)-f(s^\prime_i)\right)^2-\left(g^\star(s_i,\bm{a}_i)-f(s^\prime_i)\right)^2\right] \\
    =&\mathbb{E}\left[\left(g(s_i,\bm{a}_i)+g^\star(s_i,\bm{a}_i)-2f(s^\prime_i)\right)\left(g(s_i,\bm{a}_i)-g^\star(s_i,\bm{a}_i)\right)\right]\\
    =&\mathbb{E}\left[\left(g(s_i,\bm{a}_i)-g^\star(s_i,\bm{a}_i)\right)^2\right].
\end{align*}
Here the conditional expectation is taken according to the distribution $d_i\vert\mathcal{F}_{i-1}$. The last equality is due to the fact that
\begin{align*}
    &\mathbb{E}\left[\left(g^\star(s_i,\bm{a}_i)-f(s^\prime_i)\right)\left(g(s_i,\bm{a}_i)-g^\star(s_i,\bm{a}_i)\right)\right]\\
    =&\mathbb{E}_{s_i,\bm{a}_i}\left[\mathbb{E}_{s^\prime_i}\left[\left(g^\star(s_i,\bm{a}_i)-f(s^\prime_i)\right)\left(g(s_i,\bm{a}_i)-g^\star(s_i,\bm{a}_i)\right)\vert s_i,\bm{a}_i\right]\right]\\
    =&0.
\end{align*}
Next, for the conditional variance of the random variable, we have:
\begin{align*}
    \mathbb{V}[Y_i\vert\mathcal{F}_{i-1}]\leq&\mathbb{E}\left[Y_i^2\vert\mathcal{F}_{i-1}\right]=\mathbb{E}\left[\left(g(s_i,\bm{a}_i)+g^\star(s_i,\bm{a}_i)-2f(s^\prime_i)\right)^2\left(g(s_i,\bm{a}_i)-g^\star(s_i,\bm{a}_i)\right)^2\vert\mathcal{F}_{i-1}\right]\\
    \leq&16\mathbb{E}\left[\left(g(s_i,\bm{a}_i)-g^\star(s_i,\bm{a}_i)\right)^2\vert\mathcal{F}_{i-1}\right]\\
    \leq&16\mathbb{E}[Y_i\vert\mathcal{F}_{i-1}].
\end{align*}
Noticing $Y_i\in[-4, 4]$. Applying Lemma 1 in \citep{foster2020beyond}, we get with probability at least $1-\delta^\prime$, we can bound the deviation term above as:
\begin{align*}
    &\left\vert\mathcal{L}_\rho(\phi,\theta,f)-\mathcal{L}_\rho(\phi^\star,\theta^\star_f,f)-\left(\mathcal{L}_{\mathcal{D}}(\phi,\theta,f)-\mathcal{L}_{\mathcal{D}}(\phi^\star,\theta^\star_f,f)\right)\right\vert\\  
    \leq&\sqrt{\frac{2\sum_{i=1}^n\mathbb{V}[Y_i\vert\mathcal{F}_{i-1}]\log\frac{2}{\delta^\prime}}{n^2}}+\frac{16\log\frac{2}{\delta^\prime}}{3n}\\
    \leq&\sqrt{\frac{32\sum_{i=1}^n\mathbb{E}[Y_i\vert\mathcal{F}_{i-1}]\log\frac{2}{\delta^\prime}}{n^2}}+\frac{16\log\frac{2}{\delta^\prime}}{3n},
\end{align*}
Further, consider a finite point-wise cover of the function class $\mathcal{G}:=\{g(s,\bm{a})=\phi(s,\bm{a})^\top\theta:\phi\in\Phi,\Vert\theta\Vert_\infty\leq 1\}$. Note that, with a $\ell_\infty$-cover $\overline{\mathcal{W}}$ of $\mathcal{W}=\{\Vert\theta\Vert_\infty\leq 1\}$ at scale $\gamma$, we have for all $(s,\bm{a})$ and $\phi\in\Phi$, there exists $\bar{\theta}\in\overline{\mathcal{W}}$, $|\langle\phi(s,\bm{a}),\theta-\bar{\theta}\rangle|\leq\gamma$, and we have $|\mathcal{W}|=\left(\frac{2}{\gamma}\right)^d$.
Let $\tilde{\mathcal{F}}$ be a $\gamma$-covering set of $\mathcal{F}$. For any $f\in\mathcal{F}$, there exists $\bar{f}\in\tilde{\mathcal{F}}$ such that $\Vert f-\bar{f}\Vert_\infty\leq\gamma$. Then, applying a union bound over elements in $\Phi\times\overline{\mathcal{W}}\times\tilde{\mathcal{F}}$, with probability $1-\vert\Phi\vert\vert\overline{\mathcal{W}}\vert\vert\tilde{\mathcal{F}}\vert\delta^\prime$, for all $\theta\in \mathcal{W}$, $f\in\mathcal{F}$, we have:
\begin{align*}
    &\left\vert\mathcal{L}_\rho(\phi,\theta,f)-\mathcal{L}_\rho(\phi^\star,\theta^\star_f,f)-\left(\mathcal{L}_{\mathcal{D}}(\phi,\theta,f)-\mathcal{L}_{\mathcal{D}}(\phi^\star,\theta^\star_f,f)\right)\right\vert\\
    \leq&\left\vert\mathcal{L}_\rho(\phi,\bar{\theta},\bar{f})-\mathcal{L}_\rho(\phi^\star,\theta^\star_{\bar{f}},\bar{f})-\left(\mathcal{L}_{\mathcal{D}}(\phi,\bar{\theta},\bar{f})-\mathcal{L}_{\mathcal{D}}(\phi^\star,\theta^\star_{\bar{f}},\bar{f})\right)\right\vert+16\gamma\\
    \leq&\sqrt{\frac{32\sum_{i=1}^n\mathbb{E}[\bar{Y}_i\vert\mathcal{F}_{i-1}]\log\frac{2}{\delta^\prime}}{n^2}}+\frac{16\log\frac{2}{\delta^\prime}}{3n}+16\gamma\\
    \leq&\frac{1}{2n}\sum_{i=1}^n\mathbb{E}[\bar{Y}_i\vert\mathcal{F}_{i-1}]+\frac{16\log\frac{2}{\delta^\prime}}{n}+\frac{16\log\frac{2}{\delta^\prime}}{3n}+16\gamma\\
    \leq&\frac{1}{2n}\sum_{i=1}^n\mathbb{E}[Y_i\vert\mathcal{F}_{i-1}]+\frac{16\log\frac{2}{\delta^\prime}}{n}+\frac{16\log\frac{2}{\delta^\prime}}{3n}+32\gamma\\
    \leq&\frac{1}{2}\left(\mathcal{L}_\rho(\phi,\theta,f)-\mathcal{L}_\rho(\phi^\star,\theta^\star_f,f)\right)+\frac{32\log\frac{2}{\delta^\prime}}{n}+32\gamma\\
    \leq&\frac{1}{2}\left(\mathcal{L}_\rho(\phi,\theta,f)-\mathcal{L}_\rho(\phi^\star,\theta^\star_f,f)\right)+\frac{64\log\frac{2}{\delta^\prime}}{n}\tag{setting $\gamma=1/n$}
\end{align*}
where $\bar{Y}_i:=\left(\phi(s_i,\bm{a}_i)^\top\bar{\theta}-\bar{f}(s^\prime)\right)^2-\left(\phi(s_i,\bm{a}_i)^\top\theta^\star_{\bar{f}}-\bar{f}(s^\prime)\right)^2$. Finally, setting $\delta=\delta^\prime/\left(|\Phi||\overline{\mathcal{W}}||\tilde{\mathcal{F}}|\right)$, we get $\log\frac{2}{\delta^\prime}\leq\log\frac{2(4n)^d|\Phi||\tilde{\mathcal{F}}|}{\delta}$. This completes the proof.
\end{proof} 

\begin{lemma}[Deviation Bounds for Representation Learning in Alg. \ref{alg:mf}]\label{lem:dev_bound} Let $\varepsilon^\prime=\frac{128\log(\frac{2(4n)^d\cdot|\Phi|\cdot\|\mathcal{F}\|_{1/2n}}{\delta})}{n}$. If the representation learning module in Alg. \ref{alg:mf} is called with a dataset $\mathcal{D}$ of size $n$, then with probability at least $1-\delta$, for any $f\in\mathcal{F}\subset[0,1]^{\mathcal{S}}$, we have
\begin{align*}
    \mathbb{E}_\rho\left[\left(\hat{\phi}(s,\bm{a})^\top\hat{\theta}_f-\phi^\star(s,\bm{a})^\top\theta^\star_f\right)^2\right]&\leq\varepsilon^\prime+\frac{2\lambda d}{n}.
\end{align*}
\end{lemma}
\begin{proof}
We begin by using the result in Lemma \ref{lem:fastrate_sqloss} such that, with probability at least $1-\delta$, for all $\|\theta\|_\infty\leq 1$, $\phi\in\Phi$ and $f\in\mathcal{F}$, we have
\begin{align*}
    \left|\left[\mathcal{L}_\rho(\phi,\theta,f)-\mathcal{L}_\rho(\phi^\star,\theta^\star_f,f)\right]-\left[\mathcal{L}_\mathcal{D}(\phi,\theta,f)-\mathcal{L}_\mathcal{D}(\phi^\star,\theta^\star_f,f)\right]\right|\leq \frac{1}{2}\left[\mathcal{L}_\rho(\phi,\theta,f)-\mathcal{L}_\rho(\phi^\star,\theta^\star_f,f)\right]+\varepsilon^\prime/2.
\end{align*}
Thus, with probability at least $1-\delta$ we have: 
\begin{align*}
    &\mathbb{E}_\rho\left[\left(\hat{\phi}(s,\bm{a})^\top\hat{\theta}_f-\phi^\star(s,\bm{a})^\top\theta^\star_f\right)^2\right]\\
    =&\mathcal{L}_\rho(\hat{\phi},\hat{\theta}_f, f)-\mathcal{L}_\rho(\phi^\star, \theta^\star_f, f)\tag{since $\mathbb{E}_{s^\prime\sim P^\star(s,\bm{a})}\left[f(s^\prime)\right]=\phi^\star(s,\bm{a})^\top\theta^\star_f$}\\
    \leq&2\left(\mathcal{L}_{\mathcal{D}}(\hat{\phi},\hat{\theta}_f,f)-\mathcal{L}_{\mathcal{D}}(\phi^\star,\theta^\star_f,f)\right)+ \varepsilon^\prime \tag{Lemma \ref{lem:fastrate_sqloss}, and $\|\hat{\theta}_f\|_\infty\leq 1$ according to the proof in Lemma \ref{lem:v_formula}}\\
    \leq&2\left(\mathcal{L}_{\lambda,\mathcal{D}}(\hat{\phi},\hat{\theta}_f,f)-\mathcal{L}_{\lambda,\mathcal{D}}(\phi^\star,\theta^\star_f,f)+\frac{\lambda}{n}\|\theta^\star_f\|^2_2\right)+\varepsilon^\prime\\
    \leq&\varepsilon^\prime+\frac{2\lambda d}{n}\tag{by the optimality of $\hat{\phi},\hat{\theta}_f$ under $\mathcal{L}_{\lambda,\mathcal{D}}(\cdot,\cdot,f)$},
\end{align*}
which means the inequality in the lemma statement holds. Here, we use $\|\theta^\star_f\|^2_2\leq d$. 
\end{proof}

\begin{lemma}\label{lem:model_free_hp}
When $\hat{P}_h^{(n)}$ is computed using Alg. \ref{alg:mf} and the Markov games is a block Markov game, if we set 
\begin{align*}
    \lambda&=\Theta\left(d\log\frac{NH\vert\Phi\vert}{\delta}\right),\ \zeta^{(n)}=\Theta\left(\frac{d^2M\log\frac{dNHML\vert\Phi\vert}{\delta\tilde{\varepsilon}}}{n}\right).
\end{align*}
then $\mathcal{E}$ holds with probability at least $1-\delta$. 
\end{lemma}
\begin{proof}
Combining Lemma \ref{lem:dev_bound} and Lemma \ref{lem:cov_num}, we have that
\begin{align*}
    \max_{f\in\mathcal{F}_h}\mathbb{E}_\rho\left[(\hat{\phi}(s,\bm{a})^\top\hat{\theta}_f-\phi^\star(s,\bm{a})^\top\theta^\star_f)^2\right]\leq\varepsilon^\prime+\frac{2\lambda d}{n}\leq\zeta^{(n)}:=\Theta\left(d^2M\frac{\log\left(\frac{dNHML\vert\Phi\vert}{\delta\tilde{\varepsilon}}\right)}{n}\right),
\end{align*}
which shows $\mathcal{E}_1$ holds with a high probability. Combining this result with Lemma \ref{lem:con}, we have proved Lemma \ref{lem:model_free_hp}. 
\end{proof}

\subsection{Statistical Guarantees}
To ensure the algorithm is well-defined, we first prove the following lemma which implies the optimistic Q-value estimators always belong to the function class $\tilde{\mathcal{F}}_h$. 
\begin{lemma}
When $\alpha^{(n)}\leq L$, we have $\overline{Q}_{h,i}^{(n)}\in\tilde{\mathcal{F}}_h, \forall h\in[H],i\in[M],n\in[N]$. 
\end{lemma}
\begin{proof}
    Because $\hat{\beta}_h^{(n)}$ is upper bounded by $H$, by induction one can easily get $\overline{V}_{h+1,i}^{(n)}\leq 2H^2$. Then according to the result of Lemma \ref{lem:v_formula}, we know $(\hat{P}_h^{(n)}\overline{V}^{(n)}_{h+1,i})(s,a)=\phi_h^{(n)}(s,\bm{a})^\top\theta$ with $\Vert\theta\Vert_2\leq 2H^2\sqrt{d}$. We conclude $\overline{Q}_{h,i}^{(n)}\in\tilde{\mathcal{F}}_h$. 
\end{proof}
We will show later that our choice of $\alpha^{(n)}$ and $L$ always satisfies the condition $\alpha^{(n)}\leq L$. 

\begin{lemma}
\label{lem:appro}
    We have
    \begin{itemize}
        \item For NE and CCE, 
        \begin{align*}
            \max_{\pi_{h,i}}\left(\mathbb{D}_{\pi_{h,i},\pi_{h,-i}^{(n)}}\overline{Q}_{h,i}^{(n)}\right)(s)\leq\left(\mathbb{D}_{\pi_h^{(n)}}\overline{Q}_{h,i}^{(n)}\right)(s) + 2\tilde{\varepsilon};
        \end{align*}
        \item For CE, 
        \begin{align*}
            \max_{\omega_{h,i}\in\Omega_{h,i}}\left(\mathbb{D}_{\omega_{h,i}\circ\pi_h^{(n)}}\overline{Q}_{h,i}^{(n)}\right)(s)\leq\left(\mathbb{D}_{\pi_h^{(n)}}\overline{Q}_{h,i}^{(n)}\right)(s) + 2\tilde{\varepsilon}.
        \end{align*}
    \end{itemize}
\end{lemma}
\begin{proof}
    We only prove the case of NE and CCE, the case of CE can be proved similarly. Let $\tilde{Q}_{h,i}^{(n)}$ be the nearest neighbour of $\overline{Q}_{h,i}^{(n)}$ in $\mathcal{N}_h$, we have
    \begin{align*}
        \max_{\pi_{h,i}}\left(\mathbb{D}_{\pi_{h,i},\pi_{h,-i}^{(n)}}\overline{Q}_{h,i}^{(n)}\right)(s)\leq&\max_{\pi_{h,i}}\left(\mathbb{D}_{\pi_{h,i},\pi_{h,-i}^{(n)}}\tilde{Q}_{h,i}^{(n)}\right)(s)+\tilde{\varepsilon}\\
        \leq&\left(\mathbb{D}_{\pi_h^{(n)}}\tilde{Q}_{h,i}^{(n)}\right)(s)+\tilde{\varepsilon}\tag{Definition of $\pi_h^{(n)}$}\\
        \leq&\left(\mathbb{D}_{\pi_h^{(n)}}\overline{Q}_{h,i}^{(n)}\right)(s)+2\tilde{\varepsilon},
    \end{align*}
    which has finished the proof. 
\end{proof}

\begin{lemma}[One-step back inequality for the learned model]
\label{lem:useful2_mf}
Suppose the event $\mathcal{E}$ holds. Consider a set of functions $\{g_h\}^H_{h=1}$ that satisfies $g_h\in\mathcal{S}\times\mathcal{A}\rightarrow\mathbb{R}_+$, s.t. $\Vert g_h\Vert_\infty\leq B$. For a given policy $\pi$, suppose $\mathbb{E}_{\bm{a}\sim U(\mathcal{A})}\left[g_h(\cdot,\bm{a})\right]\in\mathcal{F}_{1,h}$, then we have
\begin{align*}
    &\left\vert\mathbb{E}_{(s,\bm{a})\sim d^\pi_{\hat{P}^{(n)},h}}\left[g_h(s,\bm{a})\right]\right\vert\\
    \leq&\left\{
    \begin{aligned}
        &\sqrt{A\mathbb{E}_{(s,\bm{a})\sim\rho^{(n)}_1}\left[g_1^2(s,\bm{a})\right]},\quad h=1\\
        &\mathbb{E}_{(\tilde{s},\tilde{\bm{a}})\sim d^\pi_{\hat{P}^{(n)},h-1}}\left[\min\left\{\left\Vert\hat{\phi}^{(n)}_{h-1}(\tilde{s},\tilde{\bm{a}})\right\Vert_{\Sigma_{n,\rho^{(n)}_{h-1},\hat{\phi}^{(n)}_{h-1}}^{-1}}\sqrt{nA^2\mathbb{E}_{(s,\bm{a})\sim\tilde{\rho}^{(n)}_h}\left[g_h^2(s,\bm{a})\right]+B^2\lambda d+nA^2\zeta^{(n)}}, B\right\}\right],\quad h\geq 2
    \end{aligned}
    \right.
\end{align*}
\end{lemma}
Recall $\Sigma_{n,\rho^{(n)}_h,\hat{\phi}^{(n)}_h}=n\mathbb{E}_{(s,\bm{a})\sim\rho^{(n)}_h}\left[\hat{\phi}^{(n)}_h(s,\bm{a})\hat{\phi}^{(n)}_h(s,\bm{a})^\top\right]+\lambda I_d$. 
\begin{proof}
For step $h=1$, we have
\begin{align*}
    \mathbb{E}_{(s,\bm{a})\sim d^{\pi}_{\hat{P}^{(n)},1}}\left[g_1(s,\bm{a})\right]=&\mathbb{E}_{s\sim d_1,\bm{a}\sim\pi_1(s)}\left[g_1(s,\bm{a})\right]\\
    \leq&\sqrt{\max_{(s,\bm{a})}\frac{d_1(s)\pi_1(\bm{a}\vert s)}{\rho^{(n)}_1(s,\bm{a})}\mathbb{E}_{(s^\prime,\bm{a}^\prime)\sim\rho^{(n)}_1}\left[g_1^2(s^\prime,\bm{a}^\prime)\right]}\\
    =&\sqrt{\max_{(s,\bm{a})}\frac{d_1(s)\pi_1(\bm{a}\vert s)}{d_1(s)u_{\mathcal{A}}(\bm{a})}\mathbb{E}_{(s^\prime,\bm{a}^\prime)\sim\rho^{(n)}_1}\left[g_1^2(s^\prime,\bm{a}^\prime)\right]}\\
    \leq&\sqrt{A\mathbb{E}_{(s,\bm{a})\sim\rho^{(n)}_1}\left[g_1^2(s,\bm{a})\right]}.
\end{align*}
For step $h=2,\ldots,H-1$, we observe the following one-step-back decomposition:
\begin{align*}
    &\mathbb{E}_{(s,\bm{a})\sim d^\pi_{\hat{P}^{(n)},h}}\left[g_h(s,\bm{a})\right]\\
    =&\mathbb{E}_{(\tilde{s},\tilde{\bm{a}})\sim d^\pi_{\hat{P}^{(n)},h-1},s\sim\hat{P}^{(n)}_{h-1}(\tilde{s},\tilde{\bm{a}}),\bm{a}\sim\pi_h(s)}\left[g_h(s,\bm{a})\right]\\
    =&\mathbb{E}_{(\tilde{s},\tilde{\bm{a}})\sim d^\pi_{\hat{P}^{(n)},h-1}}\left[\hat{\phi}^{(n)}_{h-1}(\tilde{s},\tilde{\bm{a}})^\top\int_{\mathcal{S}}\sum_{\bm{a}\in\mathcal{A}}\hat{w}^{(n)}_{h-1}(s)\pi_h(\bm{a}\vert s)g_h(s,\bm{a})\mathrm{d}s\right]\\ 
    =&\mathbb{E}_{(\tilde{s},\tilde{\bm{a}})\sim d^\pi_{\hat{P}^{(n)},h-1}}\left[\min\left\{\hat{\phi}^{(n)}_{h-1}(\tilde{s},\tilde{\bm{a}})^\top\int_{\mathcal{S}}\sum_{\bm{a}\in\mathcal{A}}\hat{w}^{(n)}_{h-1}(s)\pi_h(\bm{a}\vert s)g_h(s,\bm{a})\mathrm{d}s,B\right\}\right]\\ 
    \leq&\mathbb{E}_{(\tilde{s},\tilde{\bm{a}})\sim d^\pi_{\hat{P}^{(n)},h-1}}\left[\min\left\{\left\Vert\hat{\phi}^{(n)}_{h-1}(\tilde{s},\tilde{\bm{a}})\right\Vert_{\Sigma_{n,\rho^{(n)}_{h-1},\hat{\phi}^{(n)}_{h-1}}^{-1}}\left\Vert\int_{\mathcal{S}}\sum_{\bm{a}\in\mathcal{A}}\hat{w}^{(n)}_{h-1}(s)\pi_h(\bm{a}\vert s)g_h(s,\bm{a})\mathrm{d}s\right\Vert_{\Sigma_{n,\rho^{(n)}_{h-1},\hat{\phi}^{(n)}_{h-1}}}, B\right\}\right].
\end{align*} 
where we use the fact that $g_h$ is bounded by $B$. Then, 
{\small
\begin{align*}
    &\left\Vert\int_{\mathcal{S}}\sum_{\bm{a}\in\mathcal{A}}\hat{w}^{(n)}_{h-1}(s)\pi_h(\bm{a}\vert s)g_h(s,\bm{a})\mathrm{d}s\right\Vert^2_{\Sigma_{n,\rho^{(n)}_{h-1},\hat{\phi}^{(n)}_{h-1}}}\\
    \leq&\left(\int_{\mathcal{S}}\sum_{\bm{a}\in\mathcal{A}}\hat{w}^{(n)}_{h-1}(s)\pi_h(\bm{a}\vert s)g_h(s,\bm{a})\mathrm{d}s\right)^\top\left(n\mathbb{E}_{(s,\bm{a})\sim\rho^{(n)}_{h-1}}\left[\hat{\phi}^{(n)}_{h-1}(s,\bm{a})\hat{\phi}^{(n)}_{h-1}(s,\bm{a})^\top\right]+\lambda I_d\right)\left(\int_{\mathcal{S}}\sum_{\bm{a}\in\mathcal{A}}\hat{w}^{(n)}_{h-1}(s)\pi_h(\bm{a}\vert s)g_h(s,\bm{a})\mathrm{d}s\right)\\
    \leq&n\mathbb{E}_{(\tilde{s},\tilde{\bm{a}})\sim\rho^{(n)}_{h-1}}\left[\left(\int_{\mathcal{S}}\sum_{\bm{a}\in\mathcal{A}}\hat{w}^{(n)}_{h-1}(s)^\top\hat{\phi}^{(n)}_{h-1}(\tilde{s},\tilde{\bm{a}})\pi_h(\bm{a}\vert s)g_h(s,\bm{a})\mathrm{d}s\right)^2\right]+ B^2\lambda d \tag{$\left\Vert\sum_{\bm{a}\in\mathcal{A}}\pi_h(\bm{a}\vert s)g_h(s,\bm{a})\right\Vert_\infty\leq B$ and by Lemma \ref{lem:v_formula} $\left\Vert\int_{\mathcal{S}}\hat{w}^{(n)}_{h-1}(s)l(s)\mathrm{d}s\right\Vert_2\leq\sqrt{d}$ for any $l:\mathcal{S}\to[0,1]$.}\\
    =&n\mathbb{E}_{(\tilde{s},\tilde{\bm{a}})\sim\rho^{(n)}_{h-1}}\left[\left(\mathbb{E}_{s\sim\hat{P}^{(n)}_{h-1}(\tilde{s},\tilde{\bm{a}}),\bm{a}\sim\pi_h(s)}\left[g_h(s,\bm{a})\right]\right)^2\right]+B^2\lambda d\\
    \leq&nA^2\mathbb{E}_{(\tilde{s},\tilde{\bm{a}})\sim\rho^{(n)}_{h-1}}\left[\left(\mathbb{E}_{s\sim \hat{P}^{(n)}_{h-1}(\tilde{s},\tilde{\bm{a}}),\bm{a}\sim U(\mathcal{A})}\left[g_h(s,\bm{a})\right]\right)^2\right]+ B^2\lambda d\tag{Importance sampling}\\
    \leq&nA^2\mathbb{E}_{(\tilde{s},\tilde{\bm{a}})\sim\rho^{(n)}_{h-1}}\left[\left(\mathbb{E}_{s\sim P_{h-1}^\star(\tilde{s},\tilde{\bm{a}}),\bm{a}\sim U(\mathcal{A})}\left[g_h(s,\bm{a})\right]\right)^2\right]+ B^2\lambda d+nA^2\xi^{(n)}\tag{Assumption on $g_h$}\\
    \leq&nA^2\mathbb{E}_{(\tilde{s},\tilde{\bm{a}})\sim\rho^{(n)}_{h-1},s\sim P_{h-1}^\star(\tilde{s},\tilde{\bm{a}}),\bm{a}\sim U(\mathcal{A})}\left[g_h^2(s,\bm{a})\right]+B^2\lambda d+nA^2\xi^{(n)}.\tag{Jensen}\\
    \leq&nA^2\mathbb{E}_{(s,\bm{a})\sim\tilde{\rho}^{(n)}_h}\left[g_h^2(s,\bm{a})\right]+ B^2\lambda d+nA^2\zeta^{(n)}. \tag{Definition of $\tilde{\rho}^{(n)}_h$}
\end{align*}}
Combing the above results together, we get
\begin{align*}
    &\mathbb{E}_{(s,\bm{a})\sim d^\pi_{\hat{P}^{(n)},h}}\left[g_h(s,\bm{a})\right]\\
    \leq&\mathbb{E}_{(\tilde{s},\tilde{\bm{a}})\sim d^\pi_{\hat{P}^{(n)},h-1}}\left[\min\left\{\left\Vert\hat{\phi}^{(n)}_{h-1}(\tilde{s},\tilde{\bm{a}})\right\Vert_{\Sigma_{n,\rho^{(n)}_{h-1},\hat{\phi}^{(n)}_{h-1}}^{-1}}\left\Vert\int_{\mathcal{S}}\sum_{\bm{a}\in\mathcal{A}}\hat{w}^{(n)}_{h-1}(s)\pi_h(\bm{a}\vert s)g_h(s,\bm{a})\mathrm{d}s\right\Vert_{\Sigma_{n,\rho^{(n)}_{h-1},\hat{\phi}^{(n)}_{h-1}}}, B\right\}\right]\\
    \leq&\mathbb{E}_{(\tilde{s},\tilde{\bm{a}})\sim d^\pi_{\hat{P}^{(n)},h-1}}\left[\min\left\{\left\Vert\hat{\phi}^{(n)}_{h-1}(\tilde{s},\tilde{\bm{a}})\right\Vert_{\Sigma_{n,\rho^{(n)}_{h-1},\hat{\phi}^{(n)}_{h-1}}^{-1}}\sqrt{nA^2\mathbb{E}_{(s,\bm{a})\sim\tilde{\rho}^{(n)}_h}\left[g_h^2(s,\bm{a})\right]+B^2\lambda d+nA^2\zeta^{(n)}}, B\right\}\right],
\end{align*}
which has finished the proof.
\end{proof}
The following lemma is an exact copy of Lemma \ref{lem:useful_mb}, and here we state it again just for completeness. 
\begin{lemma}[One-step back inequality for the true model]\label{lem:useful_mf} 
Consider a set of functions $\{g_h\}^H_{h=1}$ that satisfies $g_h\in\mathcal{S}\times\mathcal{A}\rightarrow\mathbb{R}_+$, s.t. $\Vert g_h\Vert_\infty\leq B$. Then for any given policy $\pi$, we have
\begin{align*}
    &\left\vert\mathbb{E}_{(s,\bm{a})\sim d^\pi_{P^\star,h}}\left[g_h(s,\bm{a})\right]\right\vert\\
    \leq&\left\{
    \begin{aligned}
        &\sqrt{A\mathbb{E}_{(s,\bm{a})\sim\rho^{(n)}_1}\left[g_1^2(s,\bm{a})\right]},\quad h=1\\
        &\mathbb{E}_{(\tilde{s},\tilde{\bm{a}})\sim d^\pi_{P^\star,h-1}}\left[\left\Vert\phi^\star_{h-1}(\tilde{s},\tilde{\bm{a}})\right\Vert_{\Sigma_{n,\gamma^{(n)}_{h-1},\phi^\star_{h-1}}^{-1}}\right]\sqrt{nA\mathbb{E}_{(s,\bm{a})\sim\tilde{\rho}^{(n)}_h}\left[g_h^2(s,\bm{a})\right]+B^2\lambda d},\quad h\geq 2
    \end{aligned}
    \right.
\end{align*}
\end{lemma}
Recall $\Sigma_{n,\gamma^{(n)}_h,\phi^\star_h}=n\mathbb{E}_{(s,\bm{a})\sim\gamma^{(n)}_h}\left[\phi_h^\star(s,\bm{a})\phi_h^\star(s,\bm{a})^\top\right]+\lambda I_d$.

\begin{lemma}[Optimism for NE and CCE]
\label{lem:optimism_NE_CCE_mf}
Consider an episode $n\in[N]$ and set $\alpha^{(n)}=\Theta\left(H\sqrt{nA^2\zeta^{(n)}+d\lambda}\right)$. When the event $\mathcal{E}$ holds and the policy $\pi^{(n)}$ is computed by solving NE or CCE, we have
\begin{align*}
    \overline{v}_i^{(n)}(s)-v^{\dagger,\pi^{(n)}_{-i}}_i(s)\geq-H\sqrt{A\zeta^{(n)}}-2H\tilde{\varepsilon},\quad\forall n\in[N],i\in[M].
\end{align*}
\end{lemma}
\begin{proof}
Denote $\tilde{\mu}_{h,i}^{(n)}(\cdot\vert s):=\argmax_{\mu}\left(\mathbb{D}_{\mu,\pi_{h,-i}^{(n)}}Q^{\dagger,\pi^{(n)}_{-i}}_{h,i}\right)(s)$ and let $\tilde{\pi}_h^{(n)}=\tilde{\mu}^{(n)}_{h,i}\times\pi_{h,-i}^{(n)}$. Let $f^{(n)}_h(s,\bm{a})=\left\vert\frac{1}{H}\left(\hat{P}^{(n)}_h-P^\star_h\right)V_{h+1,i}^{\dagger,\pi_{-i}^{(n)}}\right\vert(s,\bm{a})$, note that by definition, we have $\frac{1}{H}V_{h+1,i}^{\dagger,\pi_{-i}^{(n)}}(s)$ is bounded by $1$, and 
\begin{align*}
    \frac{1}{H}V_{h+1,i}^{\dagger,\pi_{-i}^{(n)}}(s)=&\mathbb{E}_{\bm{a}\sim\tilde{\pi}_h^{(n)}(s)}\left[\frac{r_{h+1,i}(s,\bm{a})}{H}+\frac{1}{H}\left(P_{h+1}^\star V_{h+2,i}^{\dagger,\pi_{-i}^{(n)}}\right)(s,\bm{a})\right]\\
    =&\max_{\mu_{h+1,i}}\mathbb{E}_{\bm{a}\sim(\mu_{h+1,i}\times\pi_{h+1,-i}^{(n)})(s)}\left[\frac{r_{h+1,i}(s,\bm{a})}{H}+\frac{1}{H}\left(P_{h+1}^\star V_{h+2,i}^{\dagger,\pi_{-i}^{(n)}}\right)(s,\bm{a})\right]\in\mathcal{F}_{3,h}.
\end{align*}
where we use the result of Lemma \ref{lem:v_formula} and get $\frac{1}{H}\left(P_{h+1}^\star V_{h+2,i}^{\dagger,\pi_{-i}^{(n)}}\right)(s,\bm{a})$ is a linear function in $\phi^\star_{h+1}$ and the 2-norm of the weight is upper bounded by $\sqrt{d}$. Then according to the event $\mathcal{E}$, we have
\begin{align*}
    &\mathbb{E}_{(s,\bm{a})\sim\rho^{(n)}_h}\left[\left(f^{(n)}_h(s,\bm{a})\right)^2\right]\leq\zeta^{(n)},\quad\mathbb{E}_{(s,\bm{a})\sim\tilde{\rho}^{(n)}_h}\left[\left(f^{(n)}_h(s,\bm{a})\right)^2\right]\leq\zeta^{(n)},\quad\forall n\in[N], h\in[H]\\
    &\Vert\phi_h(s,\bm{a})\Vert_{\left(\hat{\Sigma}^{(n)}_{h,\phi_h}\right)^{-1}}=\Theta\left(\Vert\phi_h(s,\bm{a})\Vert_{\Sigma^{-1}_{n,\rho^{(n)}_h,\phi_h}}\right),\quad\forall n\in[N],h\in[H],\phi_h\in\Phi_h. 
\end{align*}
A direct conclusion of the event $\mathcal{E}$ is we can find an absolute constant $c$, such that
\begin{align*}
    \beta_h^{(n)}(s,\bm{a})=&\min\left\{\alpha^{(n)}\left\Vert\hat{\phi}^{(n)}_h(\tilde{s},\tilde{\bm{a}})\right\Vert_{\left(\Sigma_{h,\hat{\phi}^{(n)}_h}^{(n)}\right)^{-1}},H\right\}\\
    \geq&\min\left\{c\alpha^{(n)}\left\Vert\hat{\phi}^{(n)}_h(\tilde{s},\tilde{\bm{a}})\right\Vert_{\Sigma_{n,\rho^{(n)}_h,\hat{\phi}^{(n)}_h}^{-1}},H\right\},\quad\forall n\in[N],h\in[H]. 
\end{align*}
Next, we prove by induction that
\begin{align}
    &\mathbb{E}_{s\sim d_{\hat{P}^{(n)},h}^{\tilde{\pi}^{(n)}}}\left[\overline{V}_{h,i}^{(n)}(s)-V^{\dagger,\pi^{(n)}_{-i}}_{h,i}(s)\right]\nonumber\\
    \geq&\sum_{h^\prime=h}^H\mathbb{E}_{(s,\bm{a})\sim d_{\hat{P}^{(n)},h^\prime}^{\tilde{\pi}^{(n)}}}\left[\hat{\beta}_{h^\prime}^{(n)}(s,\bm{a})-Hf^{(n)}_{h^\prime}(s,\bm{a})\right]-2(H-h+1)\tilde{\varepsilon},\quad\forall h\in[H].\label{eq:induction1_mf}
\end{align}
First, notice that $\forall h\in[H]$,
\begin{align*}
    \mathbb{E}_{s\sim d_{\hat{P}^{(n)},h}^{\tilde{\pi}^{(n)}}}\left[\overline{V}_{h,i}^{(n)}(s)-V^{\dagger,\pi^{(n)}_{-i}}_{h,i}(s)\right]=&\mathbb{E}_{s\sim d_{\hat{P}^{(n)},h}^{\tilde{\pi}^{(n)}}}\left[\left(\mathbb{D}_{\pi^{(n)}_h}\overline{Q}_{h,i}^{(n)}\right)(s)-\left(\mathbb{D}_{\tilde{\pi}^{(n)}_h}Q^{\dagger,\pi^{(n)}_{-i}}_{h,i}\right)(s)\right]\\
    \geq&\mathbb{E}_{s\sim d_{\hat{P}^{(n)},h}^{\tilde{\pi}^{(n)}}}\left[\left(\mathbb{D}_{\tilde{\pi}^{(n)}_h}\overline{Q}_{h,i}^{(n)}\right)(s)-\left(\mathbb{D}_{\tilde{\pi}^{(n)}_h}Q^{\dagger,\pi^{(n)}_{-i}}_{h,i}\right)(s)\right]-2\tilde{\varepsilon}\\
    =&\mathbb{E}_{(s,\bm{a})\sim d_{\hat{P}^{(n)},h}^{\tilde{\pi}^{(n)}}}\left[\overline{Q}_{h,i}^{(n)}(s,\bm{a})-Q^{\dagger,\pi^{(n)}_{-i}}_{h,i}(s,\bm{a})\right]-2\tilde{\varepsilon},
\end{align*}
where the inequality uses the result of Lemma \ref{lem:appro}. Now we are ready to prove \eqref{eq:induction1_mf},
\begin{itemize}
    \item When $h=H$, we have
    \begin{align*}
        \mathbb{E}_{s\sim d_{\hat{P}^{(n)},H}^{\tilde{\pi}^{(n)}}}\left[\overline{V}_{H,i}^{(n)}(s)-V^{\dagger,\pi^{(n)}_{-i}}_{H,i}(s)\right]\geq&\mathbb{E}_{(s,\bm{a})\sim d_{\hat{P}^{(n)},H}^{\tilde{\pi}^{(n)}}}\left[\overline{Q}_{H,i}^{(n)}(s,\bm{a})-Q^{\dagger,\pi^{(n)}_{-i}}_{H,i}(s,\bm{a})\right]-2\tilde{\varepsilon}\\
        =&\mathbb{E}_{(s,\bm{a})\sim d_{\hat{P}^{(n)},H}^{\tilde{\pi}^{(n)}}}\left[\hat{\beta}_h^{(n)}(s,\bm{a})\right]-2\tilde{\varepsilon}\\
        \geq&\mathbb{E}_{(s,\bm{a})\sim d_{\hat{P}^{(n)},H}^{\tilde{\pi}^{(n)}}}\left[\hat{\beta}_h^{(n)}(s,\bm{a})-Hf_H^{(n)}(s,\bm{a})\right]-2\tilde{\varepsilon}.
    \end{align*}
    \item Suppose the statement is true for $h+1$, then for step $h$, we have
    \begin{align*}
        &\mathbb{E}_{s\sim d_{\hat{P}^{(n)},h}^{\tilde{\pi}^{(n)}}}\left[\overline{V}_{h,i}^{(n)}(s)-V^{\dagger,\pi^{(n)}_{-i}}_{h,i}(s)\right]\\
        \geq&\mathbb{E}_{(s,\bm{a})\sim d_{\hat{P}^{(n)},h}^{\tilde{\pi}^{(n)}}}\left[\overline{Q}_{h,i}^{(n)}(s,\bm{a})-Q^{\dagger,\pi^{(n)}_{-i}}_{h,i}(s,\bm{a})\right]-2\tilde{\varepsilon}\\
        =&\mathbb{E}_{(s,\bm{a})\sim d_{\hat{P}^{(n)},h}^{\tilde{\pi}^{(n)}}}\left[\hat{\beta}_h^{(n)}(s,\bm{a})+\left(\hat{P}^{(n)}_h\overline{V}_{h+1,i}^{(n)}\right)(s,\bm{a})-\left(P^\star_hV^{\dagger,\pi^{(n)}_{-i}}_{h+1,i}\right)(s,\bm{a})\right]-2\tilde{\varepsilon}\\
        =&\mathbb{E}_{(s,\bm{a})\sim d_{\hat{P}^{(n)},h}^{\tilde{\pi}^{(n)}}}\left[\hat{\beta}_h^{(n)}(s,\bm{a})+\left(\hat{P}^{(n)}_h\left(\overline{V}_{h+1,i}^{(n)}-V^{\dagger,\pi^{(n)}_{-i}}_{h+1,i}\right)\right)(s,\bm{a})+\left(\left(\hat{P}^{(n)}_h-P^\star_h\right)V^{\dagger,\pi^{(n)}_{-i}}_{h+1,i}\right)(s,\bm{a})\right]-2\tilde{\varepsilon}\\
        =&\mathbb{E}_{(s,\bm{a})\sim d_{\hat{P}^{(n)},h}^{\tilde{\pi}^{(n)}}}\left[\hat{\beta}_h^{(n)}(s,\bm{a})+\left(\left(\hat{P}^{(n)}_h-P^\star_h\right)V^{\dagger,\pi^{(n)}_{-i}}_{h+1,i}\right)(s,\bm{a})\right]+\mathbb{E}_{s\sim d_{\hat{P}^{(n)},h+1}^{\tilde{\pi}^{(n)}}}\left[\overline{V}_{h+1,i}^{(n)}(s)-V^{\dagger,\pi^{(n)}_{-i}}_{h+1,i}(s)\right]-2\tilde{\varepsilon}\\
        \geq&\mathbb{E}_{(s,\bm{a})\sim d_{\hat{P}^{(n)},h}^{\tilde{\pi}^{(n)}}}\left[\hat{\beta}_h^{(n)}(s,\bm{a})-Hf^{(n)}_h(s,\bm{a})\right]+\mathbb{E}_{s\sim d_{\hat{P}^{(n)},h+1}^{\tilde{\pi}^{(n)}}}\left[\overline{V}_{h+1,i}^{(n)}(s)-V^{\dagger,\pi^{(n)}_{-i}}_{h+1,i}(s)\right]-2\tilde{\varepsilon}\\
        \geq&\sum_{h^\prime=h}^H\mathbb{E}_{(s,\bm{a})\sim d_{\hat{P}^{(n)},h^\prime}^{\tilde{\pi}^{(n)}}}\left[\hat{\beta}_{h^\prime}^{(n)}(s,\bm{a})-Hf^{(n)}_{h^\prime}(s,\bm{a})\right]-2(H-h+1)\tilde{\varepsilon},
    \end{align*}
    where the last row uses the induction assumption. 
\end{itemize}
Therefore, we have proved \eqref{eq:induction1_mf}. We then apply $h=1$ to \eqref{eq:induction1_mf}, and get
\begin{align*}
    &\mathbb{E}_{s\sim d_1}\left[\overline{V}_{1,i}^{(n)}(s)-V^{\dagger,\pi^{(n)}_{-i}}_{1,i}(s)\right]\\
    =&\mathbb{E}_{s\sim d_{\hat{P}^{(n)},1}^{\tilde{\pi}^{(n)}}}\left[\overline{V}_{1,i}^{(n)}(s)-V^{\dagger,\pi^{(n)}_{-i}}_{1,i}(s)\right]\\
    \geq&\sum_{h=1}^H\mathbb{E}_{(s,\bm{a})\sim d_{\hat{P}^{(n)},h}^{\tilde{\pi}^{(n)}}}\left[\hat{\beta}_h^{(n)}(s,\bm{a})-Hf^{(n)}_h(s,\bm{a})\right]-2H\tilde{\varepsilon}\\
    =&\sum_{h=1}^H\mathbb{E}_{(s,\bm{a})\sim d_{\hat{P}^{(n)},h}^{\tilde{\pi}^{(n)}}}\left[\hat{\beta}_h^{(n)}(s,\bm{a})\right]-H\sum_{h=1}^H\mathbb{E}_{(s,\bm{a})\sim d_{\hat{P}^{(n)},h}^{\tilde{\pi}^{(n)}}}\left[f^{(n)}_h(s,\bm{a})\right]-2H\tilde{\varepsilon}.
\end{align*}
For the second term, since $\frac{1}{H}\hat{P}_h^{(n)}V^{\dagger,\pi^{(n)}_{-i}}_{h+1,i}$ is linear in $\hat{\phi}_h^{(n)}$ and $\frac{1}{H}P_h^\star V^{\dagger,\pi^{(n)}_{-i}}_{h+1,i}$ is linear in $\phi_h^\star$, and according to the result of Lemma \ref{lem:v_formula}, the 2-norm of their weights are both upper bounded by $\sqrt{d}$. Therefore, we have $\mathbb{E}_{\bm{a}\sim U(\mathcal{A})}\left[f_h^{(n)}(\cdot,\bm{a})\right]\in\mathcal{F}_{1,h}$. By Lemma \ref{lem:useful2_mf}, we have for $h=1$,
\begin{align*}
    \mathbb{E}_{(s,\bm{a})\sim d^{\tilde{\pi}^{(n)}}_{\hat{P}^{(n)},1}}\left[f_1^{(n)}(s,\bm{a})\right]\leq\sqrt{A\mathbb{E}_{(s,\bm{a})\sim\rho_1^{(n)}}\left[\left(f_1^{(n)}(s,\bm{a})\right)^2\right]}\leq\sqrt{A\zeta^{(n)}}.
\end{align*}
And $\forall h\geq 2$, we have
\begin{align*}
    &\mathbb{E}_{(s,\bm{a})\sim d^{\tilde{\pi}^{(n)}}_{\hat{P}^{(n)},h}}\left[f_h^{(n)}(s,\bm{a})\right]\\
    \leq&\mathbb{E}_{(\tilde{s},\tilde{\bm{a}})\sim d^{\tilde{\pi}^{(n)}}_{\hat{P}^{(n)},h-1}}\left[\min\left\{\left\Vert\hat{\phi}^{(n)}_{h-1}(\tilde{s},\tilde{\bm{a}})\right\Vert_{\Sigma_{n,\rho^{(n)}_{h-1},\hat{\phi}^{(n)}_{h-1}}^{-1}}\sqrt{nA^2\mathbb{E}_{(s,\bm{a})\sim\tilde{\rho}^{(n)}_h}\left[\left(f_h^{(n)}(s,\bm{a})\right)^2\right]+d\lambda+nA^2\zeta^{(n)}},1\right\}\right]\\
    \lesssim&\mathbb{E}_{(\tilde{s},\tilde{\bm{a}})\sim d^{\tilde{\pi}^{(n)}}_{\hat{P}^{(n)},h-1}}\left[\min\left\{\left\Vert\hat{\phi}^{(n)}_{h-1}(\tilde{s},\tilde{\bm{a}})\right\Vert_{\Sigma_{n,\rho^{(n)}_{h-1},\hat{\phi}^{(n)}_{h-1}}^{-1}}\sqrt{nA^2\zeta^{(n)}+d\lambda},1\right\}\right]. 
\end{align*}
Note that we here use $f^{(n)}_h(s,\bm{a})\leq 1,\ \mathbb{E}_{(s,\bm{a})\sim\rho^{(n)}_h}\left[\left(f^{(n)}_h(s,\bm{a})\right)^2\right]\leq\zeta^{(n)}$ and $\mathbb{E}_{(s,\bm{a})\sim \tilde{\rho}^{(n)}_h}\left[\left(f^{(n)}_h(s,\bm{a})\right)^2\right]\leq\zeta^{(n)}$. Then according to our choice of $\alpha^{(n)}$, we get
\begin{align*}
    \mathbb{E}_{(s,\bm{a})\sim d^{\tilde{\pi}^{(n)}}_{\hat{P}^{(n)},h}}\left[f_h^{(n)}(s,\bm{a})\right]\leq\mathbb{E}_{(\tilde{s},\tilde{\bm{a}})\sim d^{\tilde{\pi}^{(n)}}_{\hat{P}^{(n)},h-1}}\left[\min\left\{\frac{c\alpha^{(n)}}{H}\left\Vert\hat{\phi}^{(n)}_{h-1}(\tilde{s},\tilde{\bm{a}})\right\Vert_{\Sigma_{n,\rho^{(n)}_{h-1},\hat{\phi}^{(n)}_{h-1}}^{-1}},1\right\}\right]. 
\end{align*}
Combining all things together,
\begin{align*}
    \overline{v}_i^{(n)}-v^{\dagger,\pi^{(n)}_{-i}}_i=&\mathbb{E}_{s\sim d_1}\left[\overline{V}_{1,i}^{(n)}(s)-V^{\dagger,\pi^{(n)}_{-i}}_{1,i}(s)\right]\\
    \geq&\sum_{h=1}^H\mathbb{E}_{(s,\bm{a})\sim d_{\hat{P}^{(n)},h}^{\tilde{\pi}^{(n)}}}\left[\hat{\beta}_h^{(n)}(s,\bm{a})\right]-H\sum_{h=1}^H\mathbb{E}_{(s,\bm{a})\sim d_{\hat{P}^{(n)},h}^{\tilde{\pi}^{(n)}}}\left[f^{(n)}_h(s,\bm{a})\right]-2H\tilde{\varepsilon}\\
    \geq&\sum_{h=1}^{H-1}\mathbb{E}_{(\tilde{s},\tilde{\bm{a}})\sim d^{\tilde{\pi}^{(n)}}_{\hat{P}^{(n)},h}}\left[\hat{\beta}_h^{(n)}(s,\bm{a})-\min\left\{c\alpha^{(n)}\left\Vert\hat{\phi}^{(n)}_h(\tilde{s},\tilde{\bm{a}})\right\Vert_{\Sigma_{n,\rho^{(n)}_h,\hat{\phi}^{(n)}_h}^{-1}},H\right\}\right]-H\sqrt{A\zeta^{(n)}}-2H\tilde{\varepsilon}\\
    =&-H\sqrt{A\zeta^{(n)}}-2H\tilde{\varepsilon},
\end{align*}
which proves the inequality. 
\end{proof}

\begin{lemma}[Optimism for CE]
\label{lem:optimism_CE_mf}
Consider an episode $n\in[N]$ and set $\alpha^{(n)}=\Theta\left(H\sqrt{nA^2\zeta^{(n)}+d\lambda}\right)$. When the event $\mathcal{E}$ holds, we have
\begin{align*}
    \overline{v}_i^{(n)}(s)-\max_{\omega\in\Omega_i}v^{\omega\circ\pi^{(n)}}_i(s)\geq-H\sqrt{A\zeta^{(n)}}-2H\tilde{\varepsilon},\quad\forall n\in[N],i\in[M].
\end{align*}
\end{lemma}
\begin{proof}
Denote $\tilde{\omega}_{h,i}^{(n)}=\argmax_{\omega_h\in\Omega_{h,i}}\left(\mathbb{D}_{\omega_h\circ\pi_h^{(n)}}\max_{\omega\in\Omega_i}Q_{h,i}^{\omega\circ\pi^{(n)}}\right)(s)$ and let $\tilde{\pi}_h^{(n)}=\tilde{\omega}_{h,i}\circ\pi^{(n)}_h$. Let $f^{(n)}_h(s,\bm{a})=\left\vert\frac{1}{H}\left(\hat{P}^{(n)}_h-P^\star_h\right)\max_{\omega\in\Omega_i}V_{h+1,i}^{\omega\circ\pi^{(n)}}\right\vert(s,\bm{a})$, note that by definition, we have $\frac{1}{H}\max_{\omega\in\Omega_i}V_{h+1,i}^{\omega\circ\pi^{(n)}}(s)$ is bounded by $1$, and 
\begin{align*}
    \frac{1}{H}\max_{\omega\in\Omega_i}V_{h+1,i}^{\omega\circ\pi^{(n)}}(s)=\max_{\omega_{h+1,i}\in\Omega_{h+1,i}}\mathbb{E}_{\bm{a}\sim(\omega_{h+1,i}\circ\pi_h)(s)}\left[\frac{r_{h+1,i}(s,\bm{a})}{H}+\frac{1}{H}\left(P_{h+1}^\star \max_{\omega\in\Omega_i}V_{h+2,i}^{\omega\circ\pi^{(n)}}\right)(s,\bm{a})\right]\in\mathcal{F}_{3,h}.
\end{align*}
where we use the result of Lemma \ref{lem:v_formula} and get $\frac{1}{H}\left(P_{h+1}^\star \max_{\omega\in\Omega_i}V_{h+2,i}^{\omega\circ\pi^{(n)}}\right)(s,\bm{a})$ is a linear function in $\phi^\star_h$ and the 2-norm of the weight is upper bounded by $\sqrt{d}$. Then according to the event $\mathcal{E}$, we have
\begin{align*}
    &\mathbb{E}_{(s,\bm{a})\sim\rho^{(n)}_h}\left[\left(f^{(n)}_h(s,\bm{a})\right)^2\right]\leq\zeta^{(n)},\quad\mathbb{E}_{(s,\bm{a})\sim\tilde{\rho}^{(n)}_h}\left[\left(f^{(n)}_h(s,\bm{a})\right)^2\right]\leq\zeta^{(n)},\quad\forall n\in[N], h\in[H]\\
    &\Vert\phi_h(s,\bm{a})\Vert_{\left(\hat{\Sigma}^{(n)}_{h,\phi_h}\right)^{-1}}=\Theta\left(\Vert\phi_h(s,\bm{a})\Vert_{\Sigma^{-1}_{n,\rho^{(n)}_h,\phi_h}}\right),\quad\forall n\in[N],h\in[H],\phi_h\in\Phi_h. 
\end{align*}
A direct conclusion of the event $\mathcal{E}$ is we can find an absolute constant $c$, such that
\begin{align*}
    \beta_h^{(n)}(s,\bm{a})=&\min\left\{\alpha^{(n)}\left\Vert\hat{\phi}^{(n)}_h(\tilde{s},\tilde{\bm{a}})\right\Vert_{\left(\Sigma_{h,\hat{\phi}^{(n)}_h}^{(n)}\right)^{-1}},H\right\}\\
    \geq&\min\left\{c\alpha^{(n)}\left\Vert\hat{\phi}^{(n)}_h(\tilde{s},\tilde{\bm{a}})\right\Vert_{\Sigma_{n,\rho^{(n)}_h,\hat{\phi}^{(n)}_h}^{-1}},H\right\},\quad\forall n\in[N],h\in[H]. 
\end{align*}
Next, we prove by induction that
\begin{align}
    &\mathbb{E}_{s\sim d_{\hat{P}^{(n)},h}^{\tilde{\pi}^{(n)}}}\left[\overline{V}_{h,i}^{(n)}(s)-\max_{\omega\in\Omega_i}V_{h,i}^{\omega\circ\pi^{(n)}}(s)\right]\nonumber\\
    \geq&\sum_{h^\prime=h}^H\mathbb{E}_{(s,\bm{a})\sim d_{\hat{P}^{(n)},h^\prime}^{\tilde{\pi}^{(n)}}}\left[\hat{\beta}_{h^\prime}^{(n)}(s,\bm{a})-Hf^{(n)}_{h^\prime}(s,\bm{a})\right]-2(H-h+1)\tilde{\varepsilon},\quad\forall h\in[H].\label{eq:induction11_mf}
\end{align}
First, notice that $\forall h\in[H]$,
\begin{align*}
    \mathbb{E}_{s\sim d_{\hat{P}^{(n)},h}^{\tilde{\pi}^{(n)}}}\left[\overline{V}_{h,i}^{(n)}(s)-\max_{\omega\in\Omega_i}V_{h,i}^{\omega\circ\pi^{(n)}}(s)\right]=&\mathbb{E}_{s\sim d_{\hat{P}^{(n)},h}^{\tilde{\pi}^{(n)}}}\left[\left(\mathbb{D}_{\pi^{(n)}_h}\overline{Q}_{h,i}^{(n)}\right)(s)-\left(\mathbb{D}_{\tilde{\pi}^{(n)}_h}\max_{\omega\in\Omega_i}Q^{\omega\circ\pi^{(n)}}_{h,i}\right)(s)\right]\\
    \geq&\mathbb{E}_{s\sim d_{\hat{P}^{(n)},h}^{\tilde{\pi}^{(n)}}}\left[\left(\mathbb{D}_{\tilde{\pi}^{(n)}_h}\overline{Q}_{h,i}^{(n)}\right)(s)-\left(\mathbb{D}_{\tilde{\pi}^{(n)}_h}\max_{\omega\in\Omega_i}Q^{\omega\circ\pi^{(n)}}_{h,i}\right)(s)\right]-2\tilde{\varepsilon}\\
    =&\mathbb{E}_{(s,\bm{a})\sim d_{\hat{P}^{(n)},h}^{\tilde{\pi}^{(n)}}}\left[\overline{Q}_{h,i}^{(n)}(s,\bm{a})-\max_{\omega\in\Omega_i}Q^{\omega\circ\pi^{(n)}}_{h,i}(s,\bm{a})\right]-2\tilde{\varepsilon}.
\end{align*}
where the inequality uses the result of Lemma \ref{lem:appro}. Now we are ready to prove \eqref{eq:induction11_mf},
\begin{itemize}
    \item When $h=H$, we have
    \begin{align*}
        \mathbb{E}_{s\sim d_{\hat{P}^{(n)},H}^{\tilde{\pi}^{(n)}}}\left[\overline{V}_{H,i}^{(n)}(s)-\max_{\omega\in\Omega_i}V^{\omega\circ\pi^{(n)}}_{H,i}(s)\right]\geq&\mathbb{E}_{(s,\bm{a})\sim d_{\hat{P}^{(n)},H}^{\tilde{\pi}^{(n)}}}\left[\overline{Q}_{H,i}^{(n)}(s,\bm{a})-\max_{\omega\in\Omega_i}Q^{\omega\circ\pi^{(n)}}_{H,i}(s,\bm{a})\right]\\
        =&\mathbb{E}_{(s,\bm{a})\sim d_{\hat{P}^{(n)},H}^{\tilde{\pi}^{(n)}}}\left[\hat{\beta}_h^{(n)}(s,\bm{a})\right]\\
        \geq&\mathbb{E}_{(s,\bm{a})\sim d_{\hat{P}^{(n)},H}^{\tilde{\pi}^{(n)}}}\left[\hat{\beta}_h^{(n)}(s,\bm{a})-Hf_H^{(n)}(s,\bm{a})\right]-2\tilde{\varepsilon}.
    \end{align*}
    \item Suppose the statement is true for $h+1$, then for step $h$, we have
    \begin{align*}
        &\mathbb{E}_{s\sim d_{\hat{P}^{(n)},h}^{\tilde{\pi}^{(n)}}}\left[\overline{V}_{h,i}^{(n)}(s)-\max_{\omega\in\Omega_i}V^{\omega\circ\pi^{(n)}}_{h,i}(s)\right]\\
        \geq&\mathbb{E}_{(s,\bm{a})\sim d_{\hat{P}^{(n)},h}^{\tilde{\pi}^{(n)}}}\left[\overline{Q}_{h,i}^{(n)}(s,\bm{a})-\max_{\omega\in\Omega_i}Q^{\omega\circ\pi^{(n)}}_{h,i}(s,\bm{a})\right]-2\tilde{\varepsilon}\\
        =&\mathbb{E}_{(s,\bm{a})\sim d_{\hat{P}^{(n)},h}^{\tilde{\pi}^{(n)}}}\left[\hat{\beta}_h^{(n)}(s,\bm{a})+\left(\hat{P}^{(n)}_h\overline{V}_{h+1,i}^{(n)}\right)(s,\bm{a})-\left(P^\star_h\max_{\omega\in\Omega_i}V^{\omega\circ\pi^{(n)}_{-i}}_{h+1,i}\right)(s,\bm{a})\right]-2\tilde{\varepsilon}\\
        =&\mathbb{E}_{(s,\bm{a})\sim d_{\hat{P}^{(n)},h}^{\tilde{\pi}^{(n)}}}\left[\hat{\beta}_h^{(n)}(s,\bm{a})+\left(\hat{P}^{(n)}_h\left(\overline{V}_{h+1,i}^{(n)}-\max_{\omega\in\Omega_i}V^{\omega\circ\pi^{(n)}}_{h+1,i}\right)\right)(s,\bm{a})-\left(\left(\hat{P}^{(n)}_h-P^\star_h\right)\max_{\omega\in\Omega_i}V^{\omega\circ\pi^{(n)}}_{h+1,i}\right)(s,\bm{a})\right]\\
        &-2\tilde{\varepsilon}\\
        =&\mathbb{E}_{(s,\bm{a})\sim d_{\hat{P}^{(n)},h}^{\tilde{\pi}^{(n)}}}\left[\hat{\beta}_h^{(n)}(s,\bm{a})-\left(\left(\hat{P}^{(n)}_h-P^\star_h\right)\max_{\omega\in\Omega_i}V^{\omega\circ\pi^{(n)}_{-i}}_{h+1,i}\right)(s,\bm{a})\right]\\
        &+\mathbb{E}_{s\sim d_{\hat{P}^{(n)},h+1}^{\tilde{\pi}^{(n)}}}\left[\overline{V}_{h+1,i}^{(n)}(s)-\max_{\omega\in\Omega_i}V^{\omega\circ\pi^{(n)}}_{h+1,i}(s)\right]-2\tilde{\varepsilon}\\
        \geq&\mathbb{E}_{(s,\bm{a})\sim d_{\hat{P}^{(n)},h}^{\tilde{\pi}^{(n)}}}\left[\hat{\beta}_h^{(n)}(s,\bm{a})-Hf^{(n)}_h(s,\bm{a})\right]+\mathbb{E}_{s\sim d_{\hat{P}^{(n)},h+1}^{\tilde{\pi}^{(n)}}}\left[\overline{V}_{h+1,i}^{(n)}(s)-\max_{\omega\in\Omega_i}V^{\omega\circ\pi^{(n)}}_{h+1,i}(s)\right]-2\tilde{\varepsilon}\\
        \geq&\sum_{h^\prime=h}^H\mathbb{E}_{(s,\bm{a})\sim d_{\hat{P}^{(n)},h^\prime}^{\tilde{\pi}^{(n)}}}\left[\hat{\beta}_{h^\prime}^{(n)}(s,\bm{a})-Hf^{(n)}_{h^\prime}(s,\bm{a})\right]-2(H-h+1)\tilde{\varepsilon},
    \end{align*}
    where the last row uses the induction assumption. 
\end{itemize}
Therefore, we have proved \eqref{eq:induction11_mf}. We then apply $h=1$ to \eqref{eq:induction11_mf}, and get
\begin{align*}
    &\mathbb{E}_{s\sim d_1}\left[\overline{V}_{1,i}^{(n)}(s)-\max_{\omega\in\Omega_i}V^{\omega\circ\pi^{(n)}}_{1,i}(s)\right]\\
    =&\mathbb{E}_{s\sim d_{\hat{P}^{(n)},1}^{\tilde{\pi}^{(n)}}}\left[\overline{V}_{1,i}^{(n)}(s)-\max_{\omega\in\Omega_i}V^{\omega\circ\pi^{(n)}}_{1,i}(s)\right]\\
    \geq&\sum_{h=1}^H\mathbb{E}_{(s,\bm{a})\sim d_{\hat{P}^{(n)},h}^{\tilde{\pi}^{(n)}}}\left[\hat{\beta}_h^{(n)}(s,\bm{a})-Hf^{(n)}_h(s,\bm{a})\right]-2H\tilde{\varepsilon}\\
    =&\sum_{h=1}^H\mathbb{E}_{(s,\bm{a})\sim d_{\hat{P}^{(n)},h}^{\tilde{\pi}^{(n)}}}\left[\hat{\beta}_h^{(n)}(s,\bm{a})\right]-H\sum_{h=1}^H\mathbb{E}_{(s,\bm{a})\sim d_{\hat{P}^{(n)},h}^{\tilde{\pi}^{(n)}}}\left[f^{(n)}_h(s,\bm{a})\right]-2H\tilde{\varepsilon}.
\end{align*}
For the second term, since $\frac{1}{H}\hat{P}_h^{(n)}\max_{\omega\in\Omega_i}V^{\omega\circ\pi^{(n)}}_{h+1,i}$ is linear in $\hat{\phi}_h^{(n)}$ and $\frac{1}{H}P_h^\star\max_{\omega\in\Omega_i}V^{\omega\circ\pi^{(n)}}_{h+1,i}$ is linear in $\phi_h^\star$, and according to the result of Lemma \ref{lem:v_formula}, the 2-norm of their weights are both upper bounded by $\sqrt{d}$. Therefore, we have $\mathbb{E}_{\bm{a}\sim U(\mathcal{A})}\left[f_h^{(n)}(\cdot,\bm{a})\right]\in\mathcal{F}_{1,h}$. By Lemma \ref{lem:useful2_mf}, we have for $h=1$,
\begin{align*}
    \mathbb{E}_{(s,\bm{a})\sim d^{\tilde{\pi}^{(n)}}_{\hat{P}^{(n)},1}}\left[f_1^{(n)}(s,\bm{a})\right]\leq\sqrt{A\mathbb{E}_{(s,\bm{a})\sim\rho_1^{(n)}}\left[\left(f_1^{(n)}(s,\bm{a})\right)^2\right]}\leq\sqrt{A\zeta^{(n)}}.
\end{align*}
And $\forall h\geq 2$, we have
\begin{align*}
    &\mathbb{E}_{(s,\bm{a})\sim d^{\tilde{\pi}^{(n)}}_{\hat{P}^{(n)},h}}\left[f_h^{(n)}(s,\bm{a})\right]\\
    \leq&\mathbb{E}_{(\tilde{s},\tilde{\bm{a}})\sim d^{\tilde{\pi}^{(n)}}_{\hat{P}^{(n)},h-1}}\left[\min\left\{\left\Vert\hat{\phi}^{(n)}_{h-1}(\tilde{s},\tilde{\bm{a}})\right\Vert_{\Sigma_{n,\rho^{(n)}_{h-1},\hat{\phi}^{(n)}_{h-1}}^{-1}}\sqrt{nA^2\mathbb{E}_{(s,\bm{a})\sim\tilde{\rho}^{(n)}_h}\left[\left(f_h^{(n)}(s,\bm{a})\right)^2\right]+d\lambda+nA^2\zeta^{(n)}},1\right\}\right]\\
    \lesssim&\mathbb{E}_{(\tilde{s},\tilde{\bm{a}})\sim d^{\tilde{\pi}^{(n)}}_{\hat{P}^{(n)},h-1}}\left[\min\left\{\left\Vert\hat{\phi}^{(n)}_{h-1}(\tilde{s},\tilde{\bm{a}})\right\Vert_{\Sigma_{n,\rho^{(n)}_{h-1},\hat{\phi}^{(n)}_{h-1}}^{-1}}\sqrt{nA^2\zeta^{(n)}+d\lambda},1\right\}\right]. 
\end{align*}
Note that we here use $f^{(n)}_h(s,\bm{a})\leq 1,\ \mathbb{E}_{(s,\bm{a})\sim\rho^{(n)}_h}\left[\left(f^{(n)}_h(s,\bm{a})\right)^2\right]\leq\zeta^{(n)}$ and $\mathbb{E}_{(s,\bm{a})\sim \tilde{\rho}^{(n)}_h}\left[\left(f^{(n)}_h(s,\bm{a})\right)^2\right]\leq\zeta^{(n)}$. Then according to our choice of $\alpha^{(n)}$, we get
\begin{align*}
    \mathbb{E}_{(s,\bm{a})\sim d^{\tilde{\pi}^{(n)}}_{\hat{P}^{(n)},h}}\left[f_h^{(n)}(s,\bm{a})\right]\leq\mathbb{E}_{(\tilde{s},\tilde{\bm{a}})\sim d^{\tilde{\pi}^{(n)}}_{\hat{P}^{(n)},h-1}}\left[\min\left\{\frac{c\alpha^{(n)}}{H}\left\Vert\hat{\phi}^{(n)}_{h-1}(\tilde{s},\tilde{\bm{a}})\right\Vert_{\Sigma_{n,\rho^{(n)}_{h-1},\hat{\phi}^{(n)}_{h-1}}^{-1}},1\right\}\right]. 
\end{align*}
Combining all things together,
\begin{align*}
    &\overline{v}_i^{(n)}-\max_{\omega\in\Omega_i}v^{\omega\circ\pi^{(n)}}_i\\
    =&\mathbb{E}_{s\sim d_1}\left[\overline{V}_{1,i}^{(n)}(s)-\max_{\omega\in\Omega_i}V^{\omega\circ\pi^{(n)}}_{1,i}(s)\right]\\
    \geq&\sum_{h=1}^H\mathbb{E}_{(s,\bm{a})\sim d_{\hat{P}^{(n)},h}^{\tilde{\pi}^{(n)}}}\left[\hat{\beta}_h^{(n)}(s,\bm{a})\right]-H\sum_{h=1}^H\mathbb{E}_{(s,\bm{a})\sim d_{\hat{P}^{(n)},h}^{\tilde{\pi}^{(n)}}}\left[f^{(n)}_h(s,\bm{a})\right]-2H\tilde{\varepsilon}\\
    \geq&\sum_{h=1}^{H-1}\mathbb{E}_{(\tilde{s},\tilde{\bm{a}})\sim d^{\tilde{\pi}^{(n)}}_{\hat{P}^{(n)},h}}\left[\hat{\beta}_h^{(n)}(s,\bm{a})-\min\left\{c\alpha^{(n)}\left\Vert\hat{\phi}^{(n)}_h(\tilde{s},\tilde{\bm{a}})\right\Vert_{\Sigma_{n,\rho^{(n)}_h,\hat{\phi}^{(n)}_h}^{-1}},H\right\}\right]-H\sqrt{A\zeta^{(n)}}-2H\tilde{\varepsilon}\\
    =&-H\sqrt{A\zeta^{(n)}}-2H\tilde{\varepsilon},
\end{align*}
which proves the inequality. 
\end{proof}

\begin{lemma}[pessimism]\label{lem:pessimism_mf}
    Consider an episode $n\in[N]$ and set $\alpha^{(n)}=\Theta\left(H\sqrt{nA^2\zeta^{(n)}+d\lambda}\right)$. When the event $\mathcal{E}$ holds, we have
    \begin{align*}
        \underline{v}_i^{(n)}(s)-v^{\pi^{(n)}}_i(s)\leq H\sqrt{A\zeta^{(n)}},\quad\forall n\in[N],i\in[M].
    \end{align*}
\end{lemma}
\begin{proof}
Let $\tilde{f}^{(n)}_h(s,\bm{a})=\left\vert\frac{1}{H}\left(\hat{P}^{(n)}_h-P^\star_h\right)V_{h+1,i}^{\pi^{(n)}}\right\vert(s,\bm{a})$, note that by definition, we have $\frac{1}{H}V_{h+1,i}^{\pi^{(n)}}(s)$ is bounded by $1$, and 
\begin{align*}
    \frac{1}{H}V_{h+1,i}^{\pi^{(n)}}(s)=&\mathbb{E}_{\bm{a}\sim\pi_h^{(n)}(s)}\left[\frac{r_{h+1,i}(s,\bm{a})}{H}+\frac{1}{H}\left(P_{h+1}^\star V_{h+2,i}^{\pi^{(n)}}\right)(s,\bm{a})\right]\in\mathcal{F}_{2,h}.
\end{align*}
where we use the result of Lemma \ref{lem:v_formula} and get $\frac{1}{H}\left(P_{h+1}^\star V_{h+2,i}^{\pi^{(n)}}\right)(s,\bm{a})$ is a linear function in $\phi^\star_{h+1}$ and the 2-norm of the weight is upper bounded by $\sqrt{d}$. Then according to the event $\mathcal{E}$, we have
\begin{align*}
    &\mathbb{E}_{(s,\bm{a})\sim\rho^{(n)}_h}\left[\left(f^{(n)}_h(s,\bm{a})\right)^2\right]\leq\zeta^{(n)},\quad\mathbb{E}_{(s,\bm{a})\sim\tilde{\rho}^{(n)}_h}\left[\left(f^{(n)}_h(s,\bm{a})\right)^2\right]\leq\zeta^{(n)},\quad\forall n\in[N], h\in[H]\\
    &\Vert\phi_h(s,\bm{a})\Vert_{\left(\hat{\Sigma}^{(n)}_{h,\phi_h}\right)^{-1}}=\Theta\left(\Vert\phi_h(s,\bm{a})\Vert_{\Sigma^{-1}_{n,\rho^{(n)}_h,\phi_h}}\right),\quad\forall n\in[N],h\in[H],\phi_h\in\Phi_h. 
\end{align*}
A direct conclusion of the event $\mathcal{E}$ is we can find an absolute constant $c$, such that
\begin{align*}
    \beta_h^{(n)}(s,\bm{a})=&\min\left\{\alpha^{(n)}\left\Vert\hat{\phi}^{(n)}_h(\tilde{s},\tilde{\bm{a}})\right\Vert_{\left(\Sigma_{h,\hat{\phi}^{(n)}_h}^{(n)}\right)^{-1}},H\right\}\\
    \geq&\min\left\{c\alpha^{(n)}\left\Vert\hat{\phi}^{(n)}_h(\tilde{s},\tilde{\bm{a}})\right\Vert_{\Sigma_{n,\rho^{(n)}_h,\hat{\phi}^{(n)}_h}^{-1}},H\right\},\quad\forall n\in[N],h\in[H]. 
\end{align*}
Again, we prove the following inequality by induction:
\begin{align}
    \mathbb{E}_{s\sim d_{\hat{P}^{(n)},h}^{\pi^{(n)}}}\left[\underline{V}_{h,i}^{(n)}(s)-V^{\pi^{(n)}}_{h,i}(s)\right]\leq&\sum_{h^\prime=h}^H\mathbb{E}_{(s,\bm{a})\sim d_{\hat{P}^{(n)},h^\prime}^{\pi^{(n)}}}\left[-\hat{\beta}_{h^\prime}^{(n)}(s,\bm{a})+Hf^{(n)}_{h^\prime}(s,\bm{a})\right],\quad\forall h\in[H].\label{eq:induction2_mf}
\end{align}
\begin{itemize}
    \item When $h=H$, we have
    \begin{align*}
        \mathbb{E}_{s\sim d_{\hat{P}^{(n)},H}^{\pi^{(n)}}}\left[\underline{V}_{H,i}^{(n)}(s)-V^{\pi^{(n)}}_{H,i}(s)\right]=&\mathbb{E}_{(s,\bm{a})\sim d_{\hat{P}^{(n)},H}^{\pi^{(n)}}}\left[\underline{Q}_{H,i}^{(n)}(s,\bm{a})-Q^{\pi^{(n)}}_{H,i}(s,\bm{a})\right]\\
        =&\mathbb{E}_{(s,\bm{a})\sim d_{\hat{P}^{(n)},H}^{\pi^{(n)}}}\left[-\hat{\beta}_H^{(n)}(s,\bm{a})\right]\\
        \leq&\mathbb{E}_{(s,\bm{a})\sim d_{\hat{P}^{(n)},H}^{\pi^{(n)}}}\left[-\hat{\beta}_H^{(n)}(s,\bm{a})+Hf_H^{(n)}(s,\bm{a})\right]
    \end{align*}
    \item Suppose the statement is true for $h+1$, then for step $h$, we have
    \begin{align*}
        &\mathbb{E}_{s\sim d_{\hat{P}^{(n)},h}^{\pi^{(n)}}}\left[\underline{V}_{h,i}^{(n)}(s)-V^{\pi^{(n)}}_{h,i}(s)\right]\\
        =&\mathbb{E}_{(s,\bm{a})\sim d_{\hat{P}^{(n)},h}^{\pi^{(n)}}}\left[\underline{Q}_{h,i}^{(n)}(s,\bm{a})-Q^{\pi^{(n)}}_{h,i}(s,\bm{a})\right]\\
        =&\mathbb{E}_{(s,\bm{a})\sim d_{\hat{P}^{(n)},h}^{\pi^{(n)}}}\left[-\hat{\beta}_h^{(n)}(s,\bm{a})+\left(\hat{P}^{(n)}_h\underline{V}_{h+1,i}^{(n)}\right)(s,\bm{a})-\left(P^\star_hV^{\pi^{(n)}}_{h+1,i}\right)(s,\bm{a})\right]\\
        =&\mathbb{E}_{(s,\bm{a})\sim d_{\hat{P}^{(n)},h}^{\pi^{(n)}}}\left[-\hat{\beta}_h^{(n)}(s,\bm{a})+\left(\hat{P}^{(n)}_h\left(\underline{V}_{h+1,i}^{(n)}-V^{\pi^{(n)}}_{h+1,i}\right)\right)(s,\bm{a})+\left(\left(\hat{P}^{(n)}_h-P^\star_h\right)V^{\pi^{(n)}}_{h+1,i}\right)(s,\bm{a})\right]\\
        =&\mathbb{E}_{(s,\bm{a})\sim d_{\hat{P}^{(n)},h}^{\pi^{(n)}}}\left[-\hat{\beta}_h^{(n)}(s,\bm{a})+\left(\left(\hat{P}^{(n)}_h-P^\star_h\right)V^{\pi^{(n)}}_{h+1,i}\right)(s,\bm{a})\right]+\mathbb{E}_{s\sim d_{\hat{P}^{(n)},h+1}^{\pi^{(n)}}}\left[\left(\underline{V}_{h+1,i}^{(n)}-V^{\pi^{(n)}}_{h+1,i}\right)(s)\right]\\
        \leq&\mathbb{E}_{(s,\bm{a})\sim d_{\hat{P}^{(n)},h}^{\pi^{(n)}}}\left[-\hat{\beta}_h^{(n)}(s,\bm{a})+Hf^{(n)}_h(s,\bm{a})\right]+\mathbb{E}_{s\sim d_{\hat{P}^{(n)},h+1}^{\pi^{(n)}}}\left[\left(\underline{V}_{h+1,i}^{(n)}-V^{\pi^{(n)}}_{h+1,i}\right)(s)\right]\\
        \leq&\sum_{h^\prime=h}^H\mathbb{E}_{(s,\bm{a})\sim d_{\hat{P}^{(n)},h^\prime}^{\pi^{(n)}}}\left[-\hat{\beta}_{h^\prime}^{(n)}(s,\bm{a})+Hf^{(n)}_{h^\prime}(s,\bm{a})\right].
    \end{align*}
    where the last row uses the induction assumption. 
\end{itemize}
The remaining steps are exactly the same as the proof in Lemma \ref{lem:optimism_NE_CCE_mf} or Lemma \ref{lem:optimism_CE_mf}, we may prove 
\begin{align*}
    \mathbb{E}_{(s,\bm{a})\sim d^{\pi^{(n)}}_{\hat{P}^{(n)},1}}\left[\min\left\{f_1^{(n)}(s,\bm{a}),1\right\}\right]\leq\sqrt{A\zeta^{(n)}},
\end{align*}
and
\begin{align*}
    \mathbb{E}_{(s,\bm{a})\sim d^{\pi^{(n)}}_{\hat{P}^{(n)},h}}\left[f_h^{(n)}(s,\bm{a})\right]\leq\mathbb{E}_{(\tilde{s},\tilde{\bm{a}})\sim d^{\pi^{(n)}}_{\hat{P}^{(n)},h-1}}\left[\min\left\{\frac{c\alpha^{(n)}}{H}\left\Vert\hat{\phi}^{(n)}_{h-1}(\tilde{s},\tilde{\bm{a}})\right\Vert_{\Sigma_{n,\rho^{(n)}_{h-1},\hat{\phi}^{(n)}_{h-1}}^{-1}},1\right\}\right],\quad\forall h\geq 2. 
\end{align*}
Combining all things together, we get
\begin{align*}
    \underline{v}_i^{(n)}-v^{\pi^{(n)}}_i=&\mathbb{E}_{s\sim d_1}\left[\underline{V}_{1,i}^{(n)}(s)-V^{\pi^{(n)}}_{1,i}(s)\right]\\
    \leq&\sum_{h=1}^H\mathbb{E}_{(s,\bm{a})\sim d_{\hat{P}^{(n)},h}^{\pi^{(n)}}}\left[-\hat{\beta}_h^{(n)}(s,\bm{a})+Hf^{(n)}_h(s,\bm{a})\right]\\
    \leq&\sum_{h=1}^{H-1}\mathbb{E}_{(s,\bm{a})\sim d_{\hat{P}^{(n)},h}^{\pi^{(n)}}}\left[-\hat{\beta}_h^{(n)}(s,\bm{a})+\min\left\{c\alpha^{(n)}\left\Vert\hat{\phi}^{(n)}_h(\tilde{s},\tilde{\bm{a}})\right\Vert_{\Sigma_{n,\rho^{(n)}_h,\hat{\phi}^{(n)}_h}^{-1}},H\right\}\right]+H\sqrt{A\zeta^{(n)}}\\
    \leq&H\sqrt{A\zeta^{(n)}},
\end{align*}
which has finished the proof. 
\end{proof}

\begin{lemma}\label{lem:pseudo_regret_mf}
For the model-free algorithm, suppose $N$ is large enough, when we pick $\lambda=\Theta\left(d\log\frac{NH\vert\Phi\vert}{\delta}\right)$, $\zeta^{(n)}=\Theta\left(\frac{d^2M}{n}\log\frac{dNHML\vert\Phi\vert}{\tilde{\varepsilon}\delta}\right)$, $L=\Theta(NHAMd)$, $\tilde{\varepsilon}=\frac{1}{2HN}$ and $\alpha^{(n)}=\Theta\left(H\sqrt{nA^2\zeta^{(n)}+d\lambda}\right)$, with probability $1-\delta$, we have 
\begin{align*}
    \sum_{n=1}^N\Delta^{(n)}\lesssim H^3d^2A^{\frac{3}{2}}N^{\frac{1}{2}}M^{\frac{1}{2}}\log\frac{dNHAM\vert\Phi\vert}{\delta}.
\end{align*}
\end{lemma}
\begin{proof}
With our choice of $\lambda$ and $\zeta^{(n)}$, according to Lemma \ref{lem:model_free_hp}, we know $\mathcal{E}$ holds with probability $1-\delta$. Furthermore, with a proper choice of the absolute constants, we have
\begin{align*}
    \alpha^{(n)}=&\Theta\left(H\sqrt{d^2A^2M\log\frac{dNHML\vert\Phi\vert}{\delta}+d^2\log\frac{NH\vert\Phi\vert}{\delta}}\right)\\
    \leq&O\left(HdA\sqrt{M\log\frac{dNHMA\vert\Phi\vert}{\delta}}\right)\\
    \leq&O\left(NHAMd\right)\leq L. 
\end{align*}
Let $f^{(n)}_h(s,\bm{a})=\frac{1}{2H^2}\left\vert\left(\hat{P}^{(n)}_h-P^\star_h\right)\left(\overline{V}_{h+1,i}^{(n)}-\underline{V}_{h+1,i}^{(n)}\right)\right\vert(s,\bm{a})$. We first verify $\frac{1}{2H^2}\left(\overline{V}_{h+1,i}^{(n)}-\underline{V}_{h+1,i}^{(n)}\right)\in\mathcal{F}_{4,h}$. By definition, we have
\begin{align*}
    \frac{1}{2H^2}\left(\overline{V}_{h+1,i}^{(n)}-\underline{V}_{h+1,i}^{(n)}\right)=\mathbb{E}_{\bm{a}\sim\pi_h^{(n)}(s)}\left[\frac{1}{H^2}\hat{\beta}_{h+1}^{(n)}(s,\bm{a})+\frac{1}{2H^2}P^\star_{h+1}\left(\overline{V}_{h+2,i}^{(n)}-\underline{V}_{h+2,i}^{(n)}\right)(s,\bm{a})\right]
\end{align*}
The first term is equal to $\frac{1}{H^2}\min\left(\alpha^{(n)}\sqrt{\hat{\phi}^{(n)}_h(s,\bm{a})^\top\left(\hat{\Sigma}^{(n)}_h\right)^{-1}\hat{\phi}^{(n)}_h(s,\bm{a})},H\right)$, which is exactly the same as that in the definition of $\mathcal{F}_{4,h}$ (note that we use the property $\alpha^{(n)}\leq L,\forall n\in[N]$). For the second term, note that we have $0\leq\frac{1}{2H^2}\left(\overline{V}_{h,i}^{(n)}-\underline{V}_{h,i}^{(n)}\right)\leq 1,\forall h$. Therefore, by Lemma \ref{lem:v_formula}, $\frac{1}{2H^2}P^\star_{h+1}\left(\overline{V}_{h+2,i}^{(n)}-\underline{V}_{h+2,i}^{(n)}\right)(s,\bm{a})$ is a linear function in $\phi^\star_{h+1}$ whose weight's 2-norm is upper bounded by $\sqrt{d}$. Combing the above arguments, we conclude $\frac{1}{2H^2}\left(\overline{V}_{h+1,i}^{(n)}-\underline{V}_{h+1,i}^{(n)}\right)\in\mathcal{F}_{4,h}$. According to the definition of the event $\mathcal{E}$, we have
\begin{align}\label{eq:conditioning_mf}
    \mathbb{E}_{s\sim\rho^{(n)}_h}\left[\left(f^{(n)}_h(s,\bm{a})\right)^2\right]\leq\zeta^{(n)},\quad\Vert\phi_h(s,\bm{a})\Vert_{\left(\hat{\Sigma}^{(n)}_{h,\phi_h}\right)^{-1}}=\Theta\left(\Vert\phi_h(s,\bm{a})\Vert_{\Sigma^{-1}_{n,\rho^{(n)}_h,\phi_h}}\right),\quad\forall n\in[N],\phi_h\in\Phi_h,h\in[H].
\end{align}
By definition, we have
\begin{align*}
    \Delta^{(n)}=\max_{i\in[M]}\left\{\overline{v}^{(n)}_i-\underline{v}^{(n)}_i\right\}+2H\sqrt{A\zeta^{(n)}}.
\end{align*}
For each fixed $i\in[M],h\in[H]$ and $n\in[N]$, we have
\begin{align*}
    &\mathbb{E}_{s\sim d^{\pi^{(n)}}_{P^\star,h}}\left[\overline{V}^{(n)}_{h,i}(s)-\underline{V}^{(n)}_{h,i}(s)\right]\\
    =&\mathbb{E}_{s\sim d_{P^\star,h}^{\pi^{(n)}}}\left[\left(\mathbb{D}_{\pi^{(n)}_h}\overline{Q}_{h,i}^{(n)}\right)(s)-\left(\mathbb{D}_{\pi^{(n)}_h}\underline{Q}_{h,i}^{(n)}\right)(s)\right]\\
    =&\mathbb{E}_{(s,\bm{a})\sim d_{P^\star,h}^{\pi^{(n)}}}\left[\overline{Q}_{h,i}^{(n)}(s,\bm{a})-\underline{Q}_{h,i}^{(n)}(s,\bm{a})\right]\\
    =&\mathbb{E}_{(s,\bm{a})\sim d_{P^\star,h}^{\pi^{(n)}}}\left[2\hat{\beta}^{(n)}_h(s,\bm{a})+\left(\hat{P}^{(n)}_h\left(\overline{V}_{h+1,i}^{(n)}-\underline{V}_{h+1,i}^{(n)}\right)\right)(s,\bm{a})\right]\\
    =&\mathbb{E}_{(s,\bm{a})\sim d_{P^\star,h}^{\pi^{(n)}}}\left[2\hat{\beta}^{(n)}_h(s,\bm{a})+\left(\left(\hat{P}^{(n)}_h-P^\star_h\right)\left(\overline{V}_{h+1,i}^{(n)}-\underline{V}_{h+1,i}^{(n)}\right)\right)(s,\bm{a})\right]+\mathbb{E}_{s\sim d^{\pi^{(n)}}_{P^\star,h+1}}\left[\overline{V}^{(n)}_{h+1,i}(s)-\underline{V}^{(n)}_{h+1,i}(s)\right]\\
    \leq&\mathbb{E}_{(s,\bm{a})\sim d_{P^\star,h}^{\pi^{(n)}}}\left[2\hat{\beta}^{(n)}_h(s,\bm{a})+2H^2f_h^{(n)}(s,\bm{a})\right]+\mathbb{E}_{s\sim d^{\pi^{(n)}}_{P^\star,h+1}}\left[\overline{V}^{(n)}_{h+1,i}(s)-\underline{V}^{(n)}_{h+1,i}(s)\right]\\
    \leq&\ldots\\
    \leq&2\sum_{h^\prime=h}^H\mathbb{E}_{(s,\bm{a})\sim d_{P^\star,h^\prime}^{\pi^{(n)}}}\left[\hat{\beta}^{(n)}_{h^\prime}(s,\bm{a})+H^2f_{h^\prime}^{(n)}(s,\bm{a})\right],
\end{align*}
where the last inequality is calculated using induction. In particular, 
\begin{align}
    \mathbb{E}_{s\sim d^{\pi^{(n)}}_{P^\star,1}}\left[\overline{V}^{(n)}_{1,i}(s)-\underline{V}^{(n)}_{1,i}(s)\right]\leq 2\underbrace{\sum_{h=1}^H\mathbb{E}_{(s,\bm{a})\sim d_{P^\star,h}^{\pi^{(n)}}}\left[\hat{\beta}^{(n)}_h(s,\bm{a})\right]}_{(a)}+2H^2\underbrace{\sum_{h=1}^H\mathbb{E}_{(s,\bm{a})\sim d_{P^\star,h}^{\pi^{(n)}}}\left[f_h^{(n)}(s,\bm{a})\right]}_{(b)}\label{eq:regret_middle_mf}. 
\end{align}
First, we calculate the first term (a) in Inequality \eqref{eq:regret_middle_mf}. Following Lemma \ref{lem:useful_mf} and noting the bonus $\hat{\beta}^{(n)}_h$ is $O(H)$, we have
\begin{align*}
    &\sum_{h=1}^H\mathbb{E}_{(s,\bm{a})\sim d^{\pi^{(n)}}_{P^\star,h}}\left[\hat{\beta}^{(n)}_h(s,\bm{a})\right]\\
    \lesssim&\sum_{h=1}^H\mathbb{E}_{(s,\bm{a})\sim d^{\pi^{(n)}}_{P^\star,h}}\left[\min\left(\alpha^{(n)}\left\Vert\hat{\phi}^{(n)}_h(s,\bm{a})\right\Vert_{\Sigma^{-1}_{n,\rho^{(n)}_h,\hat{\phi}^{(n)}_h}},H\right)\right]\tag{From \eqref{eq:conditioning_mf} }\\ 
    \lesssim&\sum_{h=1}^{H-1}\mathbb{E}_{(\tilde{s},\tilde{\bm{a}})\sim d^{\pi^{(n)}}_{P^\star,h}}\left[\left\Vert\phi^\star_h(\tilde{s},\tilde{\bm{a}})\right\Vert_{\Sigma_{n,\gamma^{(n)}_h,\phi^\star_h}^{-1}}\right]\sqrt{{nA\left(\alpha^{(n)}\right)^2}\mathbb{E}_{(s,\bm{a})\sim\rho^{(n)}_h}\left[\left\Vert\hat{\phi}^{(n)}_h(s,\bm{a})\right\Vert^2_{\Sigma^{-1}_{n,\rho^{(n)}_h,\hat{\phi}^{(n)}_h}}\right]+H^2d\lambda}\\
    +&\sqrt{A\left(\alpha^{(n)}\right)^2\mathbb{E}_{(s,\bm{a})\sim\rho^{(n)}_1}\left[\left\Vert\hat{\phi}_1^{(n)}(s,\bm{a})\right\Vert^2_{\Sigma^{-1}_{n,\rho^{(n)}_1,\hat{\phi}_1^{(n)}}}\right]}.
\end{align*}
Note that we use the fact that $B=H$ when applying Lemma \ref{lem:useful}. In addition, we have 
\begin{align*}
     &n\mathbb{E}_{(s,\bm{a})\sim\rho^{(n)}_h}\left[\left\Vert\hat \phi_h^{(n)}(s,\bm{a})\right\Vert^2_{\Sigma^{-1}_{n,\rho^{(n)}_h,\hat{\phi}^{(n)}_h}}\right]\\
     =&n\textrm{Tr}\left(\mathbb{E}_{(s,\bm{a})\sim\rho^{(n)}_h}\left[\hat{\phi}^{(n)}_h(s,\bm{a})\hat{\phi}^{(n)}_h(s,\bm{a})^\top\right]\left(n\mathbb{E}_{(s,\bm{a})\sim\rho^{(n)}_h}\left[\hat{\phi}^{(n)}_h(s,\bm{a})\hat{\phi}^{(n)}_h(s,\bm{a})^\top\right]+\lambda I_d\right)^{-1}\right)\\
     \leq& d.
\end{align*}
Then,
\begin{align*}
    \sum_{h=1}^H\mathbb{E}_{(s,\bm{a})\sim d^{\pi^{(n)}}_{P^\star,h}}\left[\hat{\beta}^{(n)}_h(s,\bm{a})\right]
    \leq\mathbb{E}_{(\tilde{s},\tilde{\bm{a}})\sim d^{\pi^{(n)}}_{P^\star,h}}\left[\left\Vert\phi_h^\star(\tilde{s},\tilde{\bm{a}})\right\Vert_{\Sigma_{n,\gamma^{(n)}_h,\phi_h^\star}^{-1}}\right]\sqrt{dA\left(\alpha^{(n)}\right)^2+H^2d\lambda}+\sqrt{{dA\left(\alpha^{(n)}\right)^2}/n}. 
\end{align*}
Second, we  calculate the term (b) in inequality \eqref{eq:regret_middle}. Following Lemma \ref{lem:useful} and noting $\left(f^{(n)}_h(s,\bm{a})\right)^2$ is upper-bounded by $1$ (i.e., $B=1$ in Lemma \ref{lem:useful}), we have 
\begin{align*}
    &\sum_{h=1}^H\mathbb{E}_{(s,\bm{a})\sim d^{\pi^{(n)}}_{P^\star,h}}[f_h^{(n)}(s,\bm{a})]\\
    \leq&\sum_{h=1}^{H-1}\mathbb{E}_{(\tilde{s},\tilde{\bm{a}})\sim d^{\pi^{(n)}}_{P^\star,h}}\left[\left\Vert\phi_h^\star(\tilde{s},\tilde{\bm{a}})\right\Vert_{\Sigma_{n,\gamma^{(n)}_h,\phi^\star_h}^{-1}}\right]\sqrt{nA\mathbb{E}_{(s,\bm{a})\sim\rho^{(n)}_h}\left[\left(f^{(n)}_h(s,\bm{a})\right)^2\right]+d\lambda}+\sqrt{A\mathbb{E}_{(s,\bm{a})\sim\rho^{(n)}_h}\left[\left(f^{(n)}_1(s,\bm{a})\right)^2\right]}\\
    \leq&\sum_{h=1}^{H-1}\mathbb{E}_{(\tilde{s},\tilde{\bm{a}})\sim d^{\pi^{(n)}}_{P^\star,h}}\left[\left\Vert\phi_h^\star(\tilde{s},\tilde{\bm{a}})\right\Vert_{\Sigma_{n,\gamma^{(n)}_h,\phi^\star_h}^{-1}}\right]\sqrt{nA\zeta^{(n)}+d\lambda}+\sqrt{A\zeta^{(n)}}\\
    \lesssim&\frac{\alpha^{(n)}}{H}\sum_{h=1}^{H-1}\mathbb{E}_{(\tilde{s},\tilde{\bm{a}})\sim d^{\pi_n}_{P^\star,h}}\left[\left\Vert\phi_h^\star(\tilde{s},\tilde{\bm{a}})\right\Vert_{\Sigma_{n,\gamma^{(n)}_h,\phi^\star_h}^{-1}}\right]+\sqrt{A\zeta^{(n)}}, 
\end{align*}
where in the second inequality, we use $\mathbb{E}_{(s,\bm{a})\sim\rho_h^{(n)}}\left[\left(f_h^{(n)}(s,\bm{a})\right)^2\right]\leq\zeta^{(n)}$, and in the last line, recall $\sqrt{nA\zeta^{(n)}+d\lambda}\lesssim\alpha^{(n)}/H$. Then, by combining the above calculation of the term (a) and term (b) in inequality \eqref{eq:regret_middle}, we have:
\begin{align*}
      \overline{v}^{(n)}_i-\underline{v}^{(n)}_i=&\mathbb{E}_{s\sim d^{\pi^{(n)}}_{P^\star,1}}\left[\overline{V}^{(n)}_{1,i}(s)-\underline{V}^{(n)}_{1,i}(s)\right]\\
      \lesssim&\sum_{h=1}^{H-1}\left(\mathbb{E}_{(\tilde{s},\tilde{\bm{a}})\sim d^{\pi^{(n)}}_{P^\star,h}}\left[\Vert\phi_h^\star(\tilde{s},\tilde{\bm{a}})\Vert_{\Sigma_{n,\gamma^{(n)}_h,\phi_h^\star}^{-1}}\right]\sqrt{dA\left(\alpha^{(n)}\right)^2+H^2d\lambda}+\sqrt{\frac{dA\left(\alpha^{(n)}\right)^2}{n}}\right)\\
      &+H^2\sum_{h=1}^{H-1}\left(\frac{\alpha^{(n)}}{H}\mathbb{E}_{(\tilde{s},\tilde{\bm{a}})\sim d^{\pi^{(n)}}_{P^\star,h}}\left[\Vert\phi^\star_h(\tilde{s},\tilde{\bm{a}})\Vert_{\Sigma_{n,\gamma^{(n)}_h,\phi^\star_h}^{-1}}\right]+\sqrt{A\zeta^{(n)}}\right). 
\end{align*}
Taking maximum over $i$ on both sides and use the definition of $\Delta^{(n)}$, we get
\begin{align*}
    \Delta^{(n)}=&\max_{i\in[M]}\left\{\overline{v}^{(n)}_i-\underline{v}^{(n)}_i\right\}+2H\sqrt{A\zeta^{(n)}}\\
    \lesssim&\sum_{h=1}^{H-1}\left(\mathbb{E}_{(\tilde{s},\tilde{\bm{a}})\sim d^{\pi^{(n)}}_{P^\star,h}}\left[\Vert\phi_h^\star(\tilde{s},\tilde{\bm{a}})\Vert_{\Sigma_{n,\gamma^{(n)}_h,\phi_h^\star}^{-1}}\right]\sqrt{dA\left(\alpha^{(n)}\right)^2+H^2d\lambda}+\sqrt{\frac{dA\left(\alpha^{(n)}\right)^2}{n}}\right)\\
    &+H^2\sum_{h=1}^{H-1}\left(\frac{\alpha^{(n)}}{H}\mathbb{E}_{(\tilde{s},\tilde{\bm{a}})\sim d^{\pi^{(n)}}_{P^\star,h}}\left[\Vert\phi^\star_h(\tilde{s},\tilde{\bm{a}})\Vert_{\Sigma_{n,\gamma^{(n)}_h,\phi^\star_h}^{-1}}\right]+\sqrt{A\zeta^{(n)}}\right). 
\end{align*}
Hereafter, we take the dominating term out. Note that
\begin{align*}
    &\sum_{n=1}^N\mathbb{E}_{(\tilde{s},\tilde{\bm{a}})\sim d^{\pi^{(n)}}_{P^\star,h}}\left[\left\Vert\phi_h^\star(\tilde{s},\tilde{\bm{a}})\right\Vert_{\Sigma_{n,\gamma^{(n)}_h,\phi_h^\star}^{-1}}\right]\leq\sqrt{N\sum_{n=1}^N\mathbb{E}_{(\tilde{s},\tilde{\bm{a}})\sim d^{\pi^{(n)}}_{P^\star,h}}\left[\phi_h^\star(\tilde{s},\tilde{\bm{a}})^\top\Sigma^{-1}_{n,\gamma^{(n)}_h,\phi_h^\star}\phi_h^\star(\tilde{s},\tilde{\bm{a}})\right]}\tag{CS inequality}\\
    \lesssim&\sqrt{N\left(\log\det\left(\sum_{n=1}^N\mathbb{E}_{(\tilde{s},\tilde{\bm{a}})\sim d^{ \pi^{(n)}}_{P^\star,h}}[\phi^\star_h(\tilde{s},\tilde{\bm{a}})\phi^\star_h(\tilde{s},\tilde{\bm{a}})^\top]\right)-\log\det(\lambda I_d)\right)}\tag{Lemma \ref{lem:reduction}}\\ 
    \leq&\sqrt{dN \log\left(1+\frac{N}{d\lambda}\right)}.\tag{Potential function bound, Lemma \ref{lem:potential} noting $\Vert\phi_h^\star(s,\bm{a})\Vert_2\leq 1$ for any $(s,\bm{a})$.}
\end{align*}
Finally, 
\begin{align*}
    \sum_{n=1}^N\Delta^{(n)}\lesssim&H\left(\sqrt{dN\log\left(1+\frac{N}{d}\right)}\sqrt{{dA\left(\alpha^{(N)}\right)^2}+H^2d\lambda}+\sum_{n=1}^{N}\sqrt{\frac{dA\left(\alpha^{(n)}\right)^2}{n}}\right)\\
    &+H^3\left(\frac{1}{H}\sqrt{dN\log\left(1+\frac{N}{d\lambda}\right)}\alpha^{(N)}+\sum_{n=1}^N\sqrt{A\zeta^{(n)}}\right) + 2HN\tilde{\varepsilon}\\ 
    \lesssim&H^2d\sqrt{NA\log\left(1+\frac{N}{d\lambda}\right)}\alpha^{(N)}\tag{Some algebra. We take the dominating term out. Note that $\alpha^{(n)}$ is increasing in $n$}\\
    \lesssim&H^3d^2A^{\frac{3}{2}}N^{\frac{1}{2}}M^{\frac{1}{2}}\log\frac{dNHAM\vert\Phi\vert}{\delta}.
\end{align*}
This concludes the proof.
\end{proof}

\paragraph{Proof of Theorem \ref{thm:mf}}
\begin{proof}
For any fixed episode $n$ and agent $i$, by Lemma \ref{lem:optimism_NE_CCE_mf}, Lemma \ref{lem:optimism_CE_mf} and Lemma \ref{lem:pessimism_mf}, we have
\begin{align*}
    v^{\dagger,\pi^{(n)}_{-i}}_i-v^{\pi^{(n)}}_i \left(\textrm{or }\max_{\omega\in\Omega_i}v^{\omega\circ\pi^{(n)}}_i-v^{\pi^{(n)}}_i\right)\leq\overline{v}^{(n)}_i-\underline{v}^{(n)}_i+2\sqrt{A\zeta^{(n)}}+2H\tilde{\varepsilon}\leq\Delta^{(n)}+2H\tilde{\varepsilon}. 
\end{align*}
Taking maximum over $i$ on both sides, we have
\begin{align}
    \max_{i\in[M]}\left\{v^{\dagger,\pi^{(n)}_{-i}}_i-v^{\pi^{(n)}}_i\right\}\left(\textrm{or } \max_{i\in[M]}\left\{\max_{\omega\in\Omega_i}v^{\omega\circ\pi^{(n)}}_i-v^{\pi^{(n)}}_i\right\}\right)\leq\Delta^{(n)}+2H\tilde{\varepsilon}.\label{eq:optim}
\end{align}
From Lemma \ref{lem:pseudo_regret_mf}, with probability $1-\delta$, we can ensure 
\begin{align*}
    \sum_{n=1}^N(\Delta^{(n)}+2H\tilde{\varepsilon})\lesssim H^3d^2A^{\frac{3}{2}}N^{\frac{1}{2}}M^{\frac{1}{2}}\log\frac{dNHAM\vert\Phi\vert}{\delta}.
\end{align*}
Therefore, according to Lemma \ref{lem:convert_1}, when we pick $N$ to be 
\begin{align*}
    O\left(\frac{H^6d^4A^3M}{\varepsilon^2}\log^2\left(\frac{HdAM\vert\Phi\vert}{\delta\varepsilon}\right)\right),
\end{align*}
we have 
\begin{align*}
    \frac{1}{N}\sum_{n=1}^N(\Delta^{(n)}+2H\tilde{\varepsilon})\leq\varepsilon.
\end{align*}
On the other hand, from \eqref{eq:optim}, we have
\begin{align*}
    &\max_{i\in[M]}\left\{v^{\dagger,\hat{\pi}_{-i}}_i-v^{\hat{\pi}}_i\right\}\left(\textrm{or }\max_{i\in[M]}\left\{\max_{\omega\in\Omega_i}v^{\omega\circ\hat{\pi}}_i-v^{\hat{\pi}}_i\right\}\right)\\
    =&\max_{i\in[M]}\left\{v^{\dagger,\pi^{(n^\star)}_{-i}}_i-v^{\pi^{(n^\star)}}_i\right\}\left(\textrm{or }\max_{i\in[M]}\left\{\max_{\omega\in\Omega_i}v^{\omega\circ\pi^{(n^\star)}}_i-v^{\pi^{(n^\star)}}_i\right\}\right)\\
    \leq&\Delta^{(n^\star)}+2H\tilde{\varepsilon}=\min_{n\in[N]}\Delta^{(n)}+2H\tilde{\varepsilon}\leq\frac{1}{N}\sum_{n=1}^N(\Delta^{(n)}+2H\tilde{\varepsilon})\leq\varepsilon,
\end{align*}
which has finished the proof. 
\end{proof}

\section{Analysis of the Factored Markov Games}
In this part, we adopt the same notations as in the proof of the model-based case, including $\overline{V}^{(n)}_h,\underline{V}^{(n)}_h,\overline{Q}^{(n)}_h,\underline{Q}^{(n)}_h$. 

\subsection{High Probability Events}
\label{subsec:hp_factor}
Define the set $\bar{\Phi}_{h,i}=\{\bar{\phi}_{h,i}(s,\bm{a}):=\bigotimes_{j\in Z_i}\phi_{h,j}(s[Z_j],\bm{a}_j)\vert \phi_{h,j}\in\Phi_{h,j}\}$. Let $\vert\Phi\vert = \max_{h,j}\vert\Phi_{h,j}\vert$ and $\vert\bar{\Phi}\vert=\max_{h,i}\vert\bar{\Phi}_{h,i}\vert$. Clearly, we have $\vert\bar{\Phi}\vert\leq\vert\Phi\vert^L$. 
Define the following event
\begin{align*}
    \mathcal{E}_1&:\ \forall n\in[N],h\in[H],i\in[M],\rho\in\left\{\rho^{(n)}_h,\tilde{\rho}^{(n)}_h\right\},\quad\mathbb{E}_\rho\left[\left\Vert\hat{P}^{(n)}_{h,i}(\cdot\vert s[Z_i],\bm{a}_i)-P^\star_{h,i}(\cdot\vert s[Z_i],\bm{a}_i)\right\Vert_1^2\right]\leq\zeta^{(n)},\\
    \mathcal{E}_2&:\ \forall n\in[N],h\in[H],i\in[M],\bar{\phi}_{h,i}\in\bar{\Phi}_{h,i},\quad\Vert\bar{\phi}_{h,i}(s,\bm{a})\Vert_{\left(\hat{\Sigma}^{(n)}_{h,\bar{\phi}_{h,i}}\right)^{-1}}=\Theta\left(\Vert\bar{\phi}_{h,i}(s,\bm{a})\Vert_{\Sigma^{-1}_{n,\rho^{(n)}_h,\bar{\phi}_{h,i}}}\right)\\
    \mathcal{E}&:=\mathcal{E}_1\cap\mathcal{E}_2. 
\end{align*}
The following lemma shows that the event $\mathcal{E}$ holds with a high probability with proper choices of the parameters. 
\begin{lemma}\label{lem:model_based_hp_factor}
When $\hat{P}_h^{(n)}$ is computed using Alg. \ref{alg:mb} with the factored setting, if we set  
\begin{align*}
\lambda=\Theta\left(Ld^L\log\frac{NHM\vert\Phi\vert}{\delta}\right),\ \zeta^{(n)}=\Theta\left(\frac{1}{n}\log\frac{\vert\mathcal{M}\vert HNM}{\delta}\right),    
\end{align*}
then $\mathcal{E}$ holds with probability at least $1-\delta$. 
\end{lemma}
The proof of Lemma \ref{lem:model_based_hp_factor} is follows a similar procedure as that of Lemma \ref{lem:model_based_hp_mb}, with minor changes on the notations as well as some modifications on the union bound. 

\subsection{Statistical Guarantees}
\begin{lemma}[One-step back inequality for the learned model]
\label{lem:useful2}
Suppose the event $\mathcal{E}$ holds. Consider a set of functions $\{g_h\}^H_{h=1}$ that satisfies $g_h\in\mathcal{S}[Z_i]\times\mathcal{A}_i\rightarrow\mathbb{R}_+$, s.t. $\Vert g_h\Vert_\infty\leq B$. For a given policy $\pi$, we have
{\small
\begin{align*}
    &\left\vert\mathbb{E}_{(s,\bm{a})\sim d^\pi_{\hat{P}^{(n)},h}}\left[g_h(s[Z_i],\bm{a}_i)\right]\right\vert\\
    \leq&\left\{
    \begin{aligned}
        &\sqrt{\tilde{A}\mathbb{E}_{(s,\bm{a})\sim\rho^{(n)}_1}\left[g_1^2(s[Z_i],\bm{a}_i)\right]},\quad h=1\\
        &\mathbb{E}_{(\tilde{s},\tilde{\bm{a}})\sim d^\pi_{\hat{P}^{(n)},h-1}}\left[\min\left\{\tilde{A}\left\Vert\bar{\phi}^{(n)}_{h-1,i}(\tilde{s},\tilde{\bm{a}})\right\Vert_{\Sigma_{n,\rho^{(n)}_{h-1},\bar{\phi}^{(n)}_{h-1,i}}^{-1}}\sqrt{n\mathbb{E}_{(s,\bm{a})\sim\tilde{\rho}^{(n)}_h}\left[g_h^2(s[Z_i],\bm{a}_i)\right]+B^2\lambda d^L+nB^2\zeta^{(n)}},B\right\}\right],\quad h\geq 2
    \end{aligned}
    \right.
\end{align*}}
\end{lemma}
where $\bar{\phi}_{h,i}^{(n)}(s,\bm{a}):=\bigotimes_{j\in Z_i}\hat{\phi}^{(n)}_{h-1,j}(s[Z_j],\bm{a}_j)$, and $\Sigma_{n,\rho^{(n)}_h,\bar{\phi}^{(n)}_{h,i}}=n\mathbb{E}_{(s,\bm{a})\sim\rho^{(n)}_h}\left[\bar{\phi}^{(n)}_{h,i}(s,\bm{a})\bar{\phi}^{(n)}_{h,i}(s,\bm{a})^\top\right]+\lambda I_{d^{\vert Z_i\vert}}$. 
\begin{proof}
For step $h=1$, we have
\begin{align*}
    \mathbb{E}_{(s,\bm{a}_i)\sim d^{\pi}_{\hat{P}^{(n)},1}}\left[g_1(s[Z_i],\bm{a}_i)\right]=&\mathbb{E}_{s\sim d_1,\bm{a}_i\sim\pi_1(s)}\left[g_1(s[Z_i],\bm{a}_i)\right]\\
    \leq&\sqrt{\max_{(s,\bm{a}_i)}\frac{d_1(s)\pi_1(\bm{a}_i\vert s)}{\rho^{(n)}_1(s,\bm{a}_i)}\mathbb{E}_{(s^\prime,\bm{a}_i^\prime)\sim\rho^{(n)}_1}\left[g_1^2(s^\prime[Z_i],\bm{a}_i^\prime)\right]}\\
    =&\sqrt{\max_{(s,\bm{a}_i)}\frac{d_1(s)\pi_1(\bm{a}_i\vert s)}{d_1(s)u_{\mathcal{A}}(\bm{a}_i)}\mathbb{E}_{(s^\prime,\bm{a}^\prime_i)\sim\rho^{(n)}_1}\left[g_1^2(s^\prime[Z_i],\bm{a}_i^\prime)\right]}\\
    \leq&\sqrt{\tilde{A}\mathbb{E}_{(s,\bm{a}_i)\sim\rho^{(n)}_1}\left[g_1^2(s[Z_i],\bm{a}_i)\right]}.
\end{align*}
For $h\geq 2$, we observe the following one-step-back decomposition:
{\small
\begin{align*}
    &\mathbb{E}_{(\tilde{s},\tilde{\bm{a}}_i)\sim d^\pi_{\hat{P}^{(n)},h}}\left[g_h(s[Z_i],\bm{a}_i)\right]\\
    =&\mathbb{E}_{(\tilde{s},\tilde{\bm{a}})\sim d^\pi_{\hat{P}^{(n)},h-1},s\sim\hat{P}^{(n)}_{h-1}(\tilde{s},\tilde{\bm{a}}),\bm{a}_i\sim\pi_h(s)}\left[g_h(s[Z_i],\bm{a}_i)\right]\\
    =&\mathbb{E}_{(\tilde{s},\tilde{\bm{a}})\sim d^\pi_{\hat{P}^{(n)},h-1}}\left[\int_{\mathcal{S}}\prod_{j=1}^M\left[\hat{\phi}^{(n)}_{h-1,j}(\tilde{s}[Z_j],\tilde{\bm{a}}_j)^\top\hat{w}^{(n)}_{h-1,j}(s_j)\right]\sum_{\bm{a}_i\in\mathcal{A}_i}\pi_h(\bm{a}_i\vert s)g_h(s[Z_i],\bm{a}_i)\mathrm{d}s\right]\\
    =&\mathbb{E}_{(\tilde{s},\tilde{\bm{a}})\sim d^\pi_{\hat{P}^{(n)},h-1}}\left[\min\left\{\int_{\mathcal{S}}\prod_{j=1}^M\left[\hat{\phi}^{(n)}_{h-1,j}(\tilde{s}[Z_j],\tilde{\bm{a}}_j)^\top\hat{w}^{(n)}_{h-1,j}(s_j)\right]\sum_{\bm{a}_i\in\mathcal{A}_i}\pi_h(\bm{a}_i\vert s)g_h(s[Z_i],\bm{a}_i)\mathrm{d}s,B\right\}\right]\\
    \leq&\mathbb{E}_{(\tilde{s},\tilde{\bm{a}})\sim d^\pi_{\hat{P}^{(n)},h-1}}\left[\min\left\{\tilde{A}\int_{\mathcal{S}}\prod_{j=1}^M\left[\hat{\phi}^{(n)}_{h-1,j}(\tilde{s}[Z_j],\tilde{\bm{a}}_j)^\top\hat{w}^{(n)}_{h-1,j}(s_j)\right]\frac{1}{\vert\mathcal{A}_i\vert}\sum_{\bm{a}_i\in\mathcal{A}_i}g_h(s[Z_i],\bm{a}_i)\mathrm{d}s,B\right\}\right]\\
    =&\mathbb{E}_{(\tilde{s},\tilde{\bm{a}})\sim d^\pi_{\hat{P}^{(n)},h-1}}\left[\min\left\{\tilde{A}\int_{\mathcal{S}[Z_i]}\prod_{j\in Z_i}\left[\hat{\phi}^{(n)}_{h-1,j}(\tilde{s}[Z_j],\tilde{\bm{a}}_j)^\top\hat{w}^{(n)}_{h-1,j}(s_j)\right]\frac{1}{\vert\mathcal{A}_i\vert}\sum_{\bm{a}_i\in\mathcal{A}_i}g_h(s[Z_i],\bm{a}_i)\mathrm{d}s[Z_i],B\right\}\right]\\
    =&\mathbb{E}_{(\tilde{s},\tilde{\bm{a}})\sim d^\pi_{\hat{P}^{(n)},h-1}}\left[\min\left\{\tilde{A}\int_{\mathcal{S}[Z_i]}\left[\bigotimes_{j\in Z_i}\hat{\phi}^{(n)}_{h-1,j}(\tilde{s}[Z_j],\tilde{\bm{a}}_j)\right]^\top\left[\bigotimes_{j\in Z_i}\hat{w}^{(n)}_{h-1,j}(s_j)\right]\frac{1}{\vert\mathcal{A}_i\vert}\sum_{\bm{a}_i\in\mathcal{A}_i}g_h(s[Z_i],\bm{a}_i)\mathrm{d}s[Z_i],B\right\}\right]\\ 
    =&\mathbb{E}_{(\tilde{s},\tilde{\bm{a}})\sim d^\pi_{\hat{P}^{(n)},h-1}}\left[\min\left\{\tilde{A}\int_{\mathcal{S}[Z_i]}\bar{\phi}^{(n)}_{h-1,i}(\tilde{s},\tilde{\bm{a}})^\top\left[\bigotimes_{j\in Z_i}\hat{w}^{(n)}_{h-1,j}(s_j)\right]\frac{1}{\vert\mathcal{A}_i\vert}\sum_{\bm{a}_i\in\mathcal{A}_i}g_h(s[Z_i],\bm{a}_i)\mathrm{d}s[Z_i],B\right\}\right]\\ 
    \leq&\mathbb{E}_{(\tilde{s},\tilde{\bm{a}})\sim d^\pi_{\hat{P}^{(n)},h-1}}\Bigg[\min\Bigg\{\tilde{A}\left\Vert\bar{\phi}^{(n)}_{h-1,i}(\tilde{s},\tilde{\bm{a}})\right\Vert_{\Sigma_{n,\rho^{(n)}_{h-1},\bar{\phi}^{(n)}_{h-1,i}}^{-1}}\left\Vert\int_{\mathcal{S}[Z_i]}\frac{1}{\vert\mathcal{A}_i\vert}\sum_{\bm{a}_i\in\mathcal{A}_i}\left(\bigotimes_{j\in Z_i}\hat{w}^{(n)}_{h-1,j}(s_j)\right)g_h(s[Z_i],\bm{a}_i)\mathrm{d}s[Z_i]\right\Vert_{\Sigma_{n,\rho^{(n)}_{h-1},\bar{\phi}^{(n)}_{h-1,i}}}\\
    &\quad\quad\quad\quad,B\Bigg\}\Bigg].
\end{align*}}
Then, 
\begin{align*}
    &\left\Vert\int_{\mathcal{S}[Z_i]}\frac{1}{\vert\mathcal{A}_i\vert}\sum_{\bm{a}_i\in\mathcal{A}}\left(\bigotimes_{j\in Z_i}\hat{w}^{(n)}_{h-1,j}(s_j)\right)g_h(s[Z_i],\bm{a}_i)\mathrm{d}s[Z_i]\right\Vert^2_{\Sigma_{n,\rho^{(n)}_{h-1},\hat{\phi}^{(n)}_{h-1},i}}\\
    \leq&n\mathbb{E}_{(\tilde{s},\tilde{\bm{a}})\sim\rho^{(n)}_{h-1}}\left[\left(\int_{\mathcal{S}[Z_i]}\frac{1}{\vert\mathcal{A}_i\vert}\sum_{\bm{a}_i\in\mathcal{A}}\prod_{j\in Z_i}\left(\hat{w}^{(n)}_{h-1,j}(s_j)^\top\hat{\phi}^{(n)}_{h-1,j}(\tilde{s},\tilde{\bm{a}}_j)\right)g_h(s[Z_i],\bm{a}_i)\mathrm{d}s[Z_i]\right)^2\right]+B^2\lambda d^L \tag{$\left\Vert\frac{1}{\vert\mathcal{A}_i\vert}\sum_{\bm{a}_i\in\mathcal{A}_i}g_h(s[Z_i],\bm{a}_i)\right\Vert_\infty\leq B$ and $\left\Vert\hat{w}^{(n)}_{h-1,i}(s_i)\right\Vert_2\leq\sqrt{d}$.}\\
    =&n\mathbb{E}_{(\tilde{s},\tilde{\bm{a}})\sim\rho^{(n)}_{h-1}}\left[\left(\mathbb{E}_{s\sim\hat{P}^{(n)}_{h-1}(\tilde{s},\tilde{\bm{a}}),\bm{a}_i\sim U(\mathcal{A}_i)}\left[g_h(s[Z_i],\bm{a}_i)\right]\right)^2\right]+B^2\lambda d^L\\
    \leq&n\mathbb{E}_{(\tilde{s},\tilde{\bm{a}})\sim\rho^{(n)}_{h-1}}\left[\left(\mathbb{E}_{s\sim P_{h-1}^\star(\tilde{s},\tilde{\bm{a}}),\bm{a}_i\sim U(\mathcal{A}_i)}\left[g_h(s[Z_i],\bm{a}_i)\right]\right)^2\right]+ B^2\lambda d^L+nB^2\xi^{(n)}\tag{Event $\mathcal{E}$}\\
    \leq&n\mathbb{E}_{(\tilde{s},\tilde{\bm{a}})\sim\rho^{(n)}_{h-1},s\sim P_{h-1}^\star(\tilde{s},\tilde{\bm{a}}),\bm{a}_i\sim U(\mathcal{A}_i)}\left[g_h^2(s[Z_i],\bm{a}_i)\right]+B^2\lambda d^L+B^2n\xi^{(n)}.\tag{Jensen}\\
    =&n\mathbb{E}_{(s,\bm{a}_i)\sim\tilde{\rho}^{(n)}_h}\left[g_h^2(s[Z_i],\bm{a}_i)\right]+B^2\lambda d^L+B^2 n\zeta^{(n)}. \tag{Definition of $\tilde{\rho}^{(n)}_h$}
\end{align*}
Combing the above results together, we get
{\small
\begin{align*}
    &\mathbb{E}_{(\tilde{s},\tilde{\bm{a}}_i)\sim d^\pi_{\hat{P}^{(n)},h}}\left[g_h(s[Z_i],\bm{a}_i)\right]\\
    \leq&\mathbb{E}_{(\tilde{s},\tilde{\bm{a}})\sim d^\pi_{\hat{P}^{(n)},h-1}}\Bigg[\min\Bigg\{\tilde{A}\left\Vert\bar{\phi}^{(n)}_{h-1,i}(\tilde{s},\tilde{\bm{a}})\right\Vert_{\Sigma_{n,\rho^{(n)}_{h-1},\bar{\phi}^{(n)}_{h-1,i}}^{-1}}\left\Vert\int_{\mathcal{S}[Z_i]}\frac{1}{\vert\mathcal{A}_i\vert}\sum_{\bm{a}_i\in\mathcal{A}_i}\left(\bigotimes_{j\in Z_i}\hat{w}^{(n)}_{h-1,j}(s_j)\right)g_h(s[Z_i],\bm{a}_i)\mathrm{d}s[Z_i]\right\Vert_{\Sigma_{n,\rho^{(n)}_{h-1},\bar{\phi}^{(n)}_{h-1,i}}}\\
    &\quad\quad\quad\quad,B\Bigg\}\Bigg]\\
    \leq&\mathbb{E}_{(\tilde{s},\tilde{\bm{a}})\sim d^\pi_{\hat{P}^{(n)},h-1}}\left[\min\left\{\tilde{A}\left\Vert\bar{\phi}^{(n)}_{h-1,i}(\tilde{s},\tilde{\bm{a}})\right\Vert_{\Sigma_{n,\rho^{(n)}_{h-1},\bar{\phi}^{(n)}_{h-1,i}}^{-1}}\sqrt{n\mathbb{E}_{(s,\bm{a}_i)\sim\tilde{\rho}^{(n)}_h}\left[g_h^2(s[Z_i],\bm{a}_i)\right]+B^2\lambda d^L+B^2n\zeta^{(n)}},B\right\}\right],
\end{align*}}
which has finished the proof.
\end{proof}

\begin{lemma}[One-step back inequality for the true model]\label{lem:useful} 
Consider a set of functions $\{g_h\}^H_{h=1}$ that satisfies $g_h\in\mathcal{S}[Z_i]\times\mathcal{A}_i\rightarrow\mathbb{R}_+$, s.t. $\Vert g_h\Vert_\infty\leq B$. Then for any policy $\pi$, we have
\begin{align*}
    &\left\vert\mathbb{E}_{(s,\bm{a}_i)\sim d^\pi_{P^\star,h}}\left[g_h(s[Z_i],\bm{a}_i)\right]\right\vert\\
    \leq&\left\{
    \begin{aligned}
        &\sqrt{\tilde{A}\mathbb{E}_{(s,\bm{a}_i)\sim\rho^{(n)}_1}\left[g_1^2(s[Z_i],\bm{a}_i)\right]},\quad h=1,\\
        &\tilde{A}\mathbb{E}_{(\tilde{s},\tilde{\bm{a}})\sim d^\pi_{P^\star,h-1}}\left[\left\Vert\bar{\phi}^\star_{h-1,i}(\tilde{s},\tilde{\bm{a}})\right\Vert_{\Sigma_{n,\gamma^{(n)}_{h-1},\bar{\phi}^\star_{h-1,i}}^{-1}}\right]\sqrt{n\mathbb{E}_{(s,\bm{a})\sim\rho^{(n)}_h}\left[g_h^2(s[Z_i],\bm{a}_i)\right]+B^2\lambda d^L},\quad h\geq 2,
    \end{aligned}
    \right.
\end{align*}
where $\bar{\phi}_{h,i}^\star(s,\bm{a}):=\bigotimes_{j\in Z_i}\phi^\star_{h-1,j}(s[Z_j],\bm{a}_j)$, and $\Sigma_{n,\gamma^{(n)}_h,\bar{\phi}^\star_{h,i}}=n\mathbb{E}_{(s,\bm{a})\sim\gamma^{(n)}_h}\left[\bar{\phi}^\star_{h,i}(s,\bm{a})\bar{\phi}^\star_{h,i}(s,\bm{a})^\top\right]+\lambda I_{d^{\vert Z_i\vert}}$. 
\end{lemma}
\begin{proof}
For step $h=1$, we have
\begin{align*}
    \mathbb{E}_{(s,\bm{a})\sim d^\pi_{P^\star,1}}\left[g_1(s[Z_i],\bm{a}_i)\right]=&\mathbb{E}_{s\sim d_1,\bm{a}_i\sim\pi_1(s)}\left[g_1(s[Z_i],\bm{a}_i)\right]\\
    \leq&\sqrt{\max_{(s,\bm{a}_i)}\frac{d_1(s)\pi_1(\bm{a}_i\vert s)}{\rho^{(n)}_1(s,\bm{a}_i)}\mathbb{E}_{(s^\prime,\bm{a}^\prime_i)\sim\rho^{(n)}_1}\left[g_1^2(s^\prime[Z_i],\bm{a}^\prime_i)\right]}\\
    =&\sqrt{\max_{(s,\bm{a}_i)}\frac{d_1(s)\pi_1(\bm{a}_i\vert s)}{d_1(s)u_{\mathcal{A}_i}(\bm{a}_i)}\mathbb{E}_{(s^\prime,\bm{a}_i^\prime)\sim\rho^{(n)}_1}\left[g_1^2(s^\prime[Z_i],\bm{a}^\prime_i)\right]}\\
    \leq&\sqrt{\tilde{A}\mathbb{E}_{(s,\bm{a}_i)\sim\rho^{(n)}_1}\left[g_1^2(s[Z_i],\bm{a}_i)\right]}.
\end{align*}
For step $h=2,\ldots,H-1$, we observe the following one-step-back decomposition:
\begin{align*}
    &\mathbb{E}_{(\tilde{s},\tilde{\bm{a}}_i)\sim d^\pi_{P^\star,h}}\left[g_h(s[Z_i],\bm{a}_i)\right]\\
    =&\mathbb{E}_{(\tilde{s},\tilde{\bm{a}})\sim d^\pi_{P^\star,h-1},s\sim P^\star_{h-1}(\tilde{s},\tilde{\bm{a}}),\bm{a}_i\sim\pi_h(s)}\left[g_h(s[Z_i],\bm{a}_i)\right]\\
    =&\mathbb{E}_{(\tilde{s},\tilde{\bm{a}})\sim d^\pi_{P^\star,h-1}}\left[\left(\bigotimes_{j\in Z_i}\phi^\star_{h-1,j}(\tilde{s}[Z_j],\tilde{\bm{a}}_j)\right)^\top\int_{\mathcal{S}}\sum_{\bm{a}_i\in\mathcal{A}_i}\left(\bigotimes_{j\in Z_i}w^\star_{h-1,j}(s_j)\right)\pi_h(\bm{a}_i\vert s)g_h(s[Z_i],\bm{a}_i)\mathrm{d}s\right]\\ 
    \leq&\tilde{A}\mathbb{E}_{(\tilde{s},\tilde{\bm{a}})\sim d^\pi_{P^\star,h-1}}\left[\left(\bigotimes_{j\in Z_i}\phi^\star_{h-1,j}(\tilde{s}[Z_j],\tilde{\bm{a}}_j)\right)^\top\int_{\mathcal{S}}\sum_{\bm{a}_i\in\mathcal{A}_i}\frac{1}{\vert\mathcal{A}_i\vert}\left(\bigotimes_{j\in Z_i}w^\star_{h-1,j}(s_j)\right)g_h(s[Z_i],\bm{a}_i)\mathrm{d}s[Z_i]\right]\\ 
    \leq&\tilde{A}\mathbb{E}_{(\tilde{s},\tilde{\bm{a}})\sim d^\pi_{P^\star,h-1}}\left[\left\Vert\bigotimes_{j\in Z_i}\phi^\star_{h-1,j}(\tilde{s}[Z_j],\tilde{\bm{a}}_j)\right\Vert_{\Sigma_{n,\gamma^{(n)}_{h-1},\bar{\phi}^\star_{h-1,i}}^{-1}}\right]\\
    &\quad\quad\quad\quad\cdot\left\Vert\int_{\mathcal{S}}\sum_{\bm{a}_i\in\mathcal{A}_i}\frac{1}{\vert\mathcal{A}_i\vert}\left(\bigotimes_{j\in Z_i}w^\star_{h-1,j}(s_j)\right)g_h(s[Z_i],\bm{a}_i)\mathrm{d}s[Z_i]\right\Vert_{\Sigma_{n,\gamma^{(n)}_{h-1},\bar{\phi}^\star_{h-1},i}}.
\end{align*} 
Then, 
{\small
\begin{align*}
    &\left\Vert\int_{\mathcal{S}}\sum_{\bm{a}_i\in\mathcal{A}_i}\frac{1}{\vert\mathcal{A}_i\vert}\left(\bigotimes_{j\in Z_i}w^\star_{h-1,j}(s_j)\right)g_h(s[Z_i],\bm{a}_i)\mathrm{d}s[Z_i]\right\Vert_{\Sigma_{n,\gamma^{(n)}_{h-1},\bar{\phi}^\star_{h-1,i}}}^2\\
    \leq&n\mathbb{E}_{(\tilde{s},\tilde{\bm{a}})\sim\gamma^{(n)}_{h-1}}\left[\left(\int_{\mathcal{S}}\sum_{\bm{a}_i\in\mathcal{A}_i}\frac{1}{\vert\mathcal{A}_i\vert}\left(\bigotimes_{j\in Z_i}w^\star_{h-1,j}(s_j)\right)^\top\left(\bigotimes_{j\in Z_i}\phi^\star_{h-1,j}(\tilde{s}[Z_j],\tilde{\bm{a}}_j)\right)g_h(s[Z_i],\bm{a}_i)\mathrm{d}s[Z_i]\right)^2\right]+B^2\lambda d^L \tag{Use the assumption $\left\Vert\sum_{\bm{a}_i\in\mathcal{A}_i}\frac{1}{\vert\mathcal{A}_i\vert}g_h(s[Z_i],\bm{a}_i)\right\Vert_\infty\leq B$ and $\left\Vert w^\star_{h-1,i}(s_i)\right\Vert_2\leq\sqrt{d}$.}\\
    =&n\mathbb{E}_{(\tilde{s},\tilde{\bm{a}})\sim\gamma^{(n)}_{h-1}}\left[\left(\mathbb{E}_{s\sim P^\star_{h-1}(\tilde{s},\tilde{\bm{a}}),\bm{a}_i\sim U(\mathcal{A}_i)}\left[g_h(s[Z_i],\bm{a}_i)\right]\right)^2\right]+B^2\lambda d^L\\
    \leq&n\mathbb{E}_{(\tilde{s},\tilde{\bm{a}})\sim\gamma^{(n)}_{h-1},s\sim P_{h-1}^\star(\tilde{s},\tilde{\bm{a}}),\bm{a}_i\sim U(\mathcal{A}_i)}\left[g_h^2(s[Z_i],\bm{a}_i)\right]+ B^2\lambda d^L\tag{Jensen}\\
    \leq&n\mathbb{E}_{(s,\bm{a}_i)\sim\rho^{(n)}_h}\left[g_h^2(s[Z_i],\bm{a}_i)\right]+ B^2\lambda d^L, \tag{Definition of $\rho^{(n)}_h$}
\end{align*}}
which has finished the proof. 
\end{proof}

\begin{lemma}[One-step back inequality for the true model]\label{lem:useful_fac} 
Consider a set of functions $\{g_h\}^H_{h=1}$ that satisfies $g_h\in\mathcal{S}[\cup_{j\in Z_i}Z_j]\times\mathcal{A}[Z_i]\rightarrow\mathbb{R}$, s.t. $\Vert g_h\Vert_\infty\leq B$. Then for any policy $\pi$, we have
\begin{align*}
    &\left\vert\mathbb{E}_{(s,\bm{a})\sim d^\pi_{P^\star,h}}\left[g_h(s[\cup_{j\in Z_i}Z_j],\bm{a}[Z_i])\right]\right\vert\\
    \leq&\left\{
    \begin{aligned}
        &\sqrt{\tilde{A}^L\mathbb{E}_{(s,\bm{a})\sim\rho^{(n)}_1}\left[g_1^2(s[\cup_{j\in Z_i}Z_j],\bm{a}[Z_i])\right]},\quad h=1,\\
        &\tilde{A}^L\mathbb{E}_{(\tilde{s},\tilde{\bm{a}})\sim d^\pi_{P^\star,h-1}}\left[\left\Vert\tilde{\phi}^\star_{h-1,i}(\tilde{s},\tilde{\bm{a}})\right\Vert_{\Sigma_{n,\gamma^{(n)}_{h-1},\tilde{\phi}^\star_{h-1,i}}^{-1}}\right]\sqrt{n\mathbb{E}_{(s,\bm{a})\sim\rho^{(n)}_h}\left[g_h^2(s[\cup_{j\in Z_i}Z_j],\bm{a}[Z_i])\right]+B^2\lambda d^{L^2}},\quad h\geq 2,
    \end{aligned}
    \right.
\end{align*}
where $\tilde{\phi}_{h,i}^\star(s,\bm{a}):=\bigotimes_{k\in \cup_{j\in Z_i}Z_j}\phi^\star_{h-1,j}(s[Z_k],\bm{a}_k)$, and $\Sigma_{n,\gamma^{(n)}_h,\tilde{\phi}^\star_{h,i}}=n\mathbb{E}_{(s,\bm{a})\sim\gamma^{(n)}_h}\left[\tilde{\phi}^\star_{h,i}(s,\bm{a})\tilde{\phi}^\star_{h,i}(s,\bm{a})^\top\right]+\lambda I_{d^{\vert\cup_{j\in Z_i} Z_j\vert}}$. 
\end{lemma}
\begin{proof}
    This Lemma can be proved using similar steps as those in the proof of Lemma \ref{lem:useful}, noting that in this case the dimension of $\tilde{\phi}_{h,i}^\star$ is at most $L^2$. 
\end{proof}

\begin{lemma}[Optimism for NE and CCE]
\label{lem:optimism_NE_CCE}
Consider an episode $n\in[N]$ and set $\alpha^{(n)}=\Theta\left(H\tilde{A}\sqrt{n\zeta^{(n)}+d^L\lambda}\right)$. When the event $\mathcal{E}$ holds and the policy $\pi^{(n)}$ is computed by solving NE or CCE, we have
\begin{align*}
    \overline{v}_i^{(n)}(s)-v^{\dagger,\pi^{(n)}_{-i}}_i(s)\geq-HM\sqrt{\tilde{A}\zeta^{(n)}},\quad\forall n\in[N],i\in[M].
\end{align*}
\end{lemma}
\begin{proof}
Denote $\tilde{\mu}_{h,i}^{(n)}(\cdot\vert s):=\argmax_{\mu}\left(\mathbb{D}_{\mu,\pi_{h,-i}^{(n)}}Q^{\dagger,\pi^{(n)}_{-i}}_{h,i}\right)(s)$ and let $\tilde{\pi}_h^{(n)}=\tilde{\mu}^{(n)}_{h,i}\times\pi_{h,-i}^{(n)}$. Let $f^{(n)}_h(s,\bm{a})=\left\Vert\hat{P}^{(n)}_h(\cdot\vert s,\bm{a})-P^\star_h(\cdot\vert s,\bm{a})\right\Vert_1$ and $f^{(n)}_{h,i}(s[Z_i],\bm{a}_i)=\left\Vert\hat{P}^{(n)}_{h,i}(\cdot\vert s[Z_i],\bm{a}_i)-P^\star_{h,i}(\cdot\vert s[Z_i],\bm{a}_i)\right\Vert_1$. Then according to the event $\mathcal{E}$, we have
\begin{align*}
    &\mathbb{E}_{(s,\bm{a})\sim\rho^{(n)}_h}\left[\left(f^{(n)}_{h,i}(s[Z_i],\bm{a}_i)\right)^2\right]\leq\zeta^{(n)},\quad\mathbb{E}_{(s,\bm{a})\sim\tilde{\rho}^{(n)}_h}\left[\left(f^{(n)}_{h,i}(s[Z_i],\bm{a}_i)\right)^2\right]\leq\zeta^{(n)},\quad\forall n\in[N],h\in[H],i\in[M]\\
    &\left\Vert\bar{\phi}_{h,i}(s,\bm{a})\right\Vert_{\left(\hat{\Sigma}^{(n)}_{h,\bar{\phi}_{h,i}}\right)^{-1}}=\Theta\left(\left\Vert\bar{\phi}_{h,i}(s,\bm{a})\right\Vert_{\Sigma^{-1}_{n,\rho^{(n)}_h,\bar{\phi}_{h,i}}}\right),\quad\forall n\in[N],h\in[H],\bar{\phi}_{h,i}\in\bar{\Phi}_{h,i},i\in[M]. 
\end{align*}
A direct conclusion of the event $\mathcal{E}$ is we can find an absolute constant $c$, such that
\begin{align*}
    \beta_h^{(n)}(s,\bm{a})=&\sum_{i=1}^M\min\left\{\alpha^{(n)}\left\Vert\bar{\phi}^{(n)}_{h,i}(\tilde{s},\tilde{\bm{a}})\right\Vert_{\left(\Sigma_{h,\bar{\phi}^{(n)}_{h,i}}^{(n)}\right)^{-1}},H\right\}\\
    \geq&\sum_{i=1}^M\min\left\{c\alpha^{(n)}\left\Vert\bar{\phi}^{(n)}_{h,i}(\tilde{s},\tilde{\bm{a}})\right\Vert_{\Sigma_{n,\rho^{(n)}_h,\bar{\phi}^{(n)}_{h,i}}^{-1}},H\right\},\quad\forall n\in[N],h\in[H]. 
\end{align*}
Next, similar to the proof in Lemma \ref{lem:optimism_NE_CCE_mb}, we may prove
\begin{align*}
    \mathbb{E}_{s\sim d_1}\left[\overline{V}_{1,i}^{(n)}(s)-V^{\dagger,\pi^{(n)}_{-i}}_{1,i}(s)\right]\geq\sum_{h=1}^H\mathbb{E}_{(s,\bm{a})\sim d_{\hat{P}^{(n)},h}^{\tilde{\pi}^{(n)}}}\left[\hat{\beta}_h^{(n)}(s,\bm{a})\right]-H\sum_{h=1}^H\mathbb{E}_{(s,\bm{a})\sim d_{\hat{P}^{(n)},h}^{\tilde{\pi}^{(n)}}}\left[\min\left\{f^{(n)}_h(s,\bm{a}),1\right\}\right].
\end{align*}
For the second term, note that we have the relation $\min\left\{f^{(n)}_h(s,\bm{a}),1\right\}\leq\sum_{i=1}^M\min\left\{f^{(n)}_{h,i}(s[Z_i],\bm{a}_i),1\right\}$. By Lemma \ref{lem:useful2}, we have for $h=1$,
\begin{align*}
    \mathbb{E}_{(s,\bm{a})\sim d^{\tilde{\pi}^{(n)}}_{\hat{P}^{(n)},1}}\left[\min\left\{f_{1,i}^{(n)}(s[Z_i],\bm{a}_i),1\right\}\right]\leq\sqrt{A\mathbb{E}_{(s,\bm{a})\sim\rho_1^{(n)}}\left[\left(f_{1,i}^{(n)}(s[Z_i],\bm{a}_i)\right)^2\right]}\leq\sqrt{\tilde{A}\zeta^{(n)}}.
\end{align*}
And $\forall h\geq 2$, we have
\begin{align*}
    &\mathbb{E}_{(s,\bm{a})\sim d^{\tilde{\pi}^{(n)}}_{\hat{P}^{(n)},h}}\left[\min\left\{f_{h,i}^{(n)}(s[Z_i],\bm{a}_i),1\right\}\right]\\
    \lesssim&\mathbb{E}_{(\tilde{s},\tilde{\bm{a}})\sim d^{\tilde{\pi}^{(n)}}_{\hat{P}^{(n)},h-1}}\left[\min\left\{\tilde{A}\left\Vert\bar{\phi}^{(n)}_{h-1,i}(\tilde{s},\tilde{\bm{a}})\right\Vert_{\Sigma_{n,\rho^{(n)}_{h-1},\bar{\phi}^{(n)}_{h-1,i}}^{-1}}\sqrt{n\mathbb{E}_{(s,\bm{a})\sim\tilde{\rho}^{(n)}_h}\left[\left(f_{h,i}^{(n)}(s[Z_i],\bm{a}_i)\right)^2\right]+d^L\lambda+n\zeta^{(n)}},1\right\}\right]\\
    \lesssim&\mathbb{E}_{(\tilde{s},\tilde{\bm{a}})\sim d^{\tilde{\pi}^{(n)}}_{\hat{P}^{(n)},h-1}}\left[\min\left\{\tilde{A}\left\Vert\bar{\phi}^{(n)}_{h-1,i}(\tilde{s},\tilde{\bm{a}})\right\Vert_{\Sigma_{n,\rho^{(n)}_{h-1},\bar{\phi}^{(n)}_{h-1,i}}^{-1}}\sqrt{n\zeta^{(n)}+d^L\lambda},1\right\}\right]
\end{align*}
Note that we here use $\min\{f^{(n)}_{h,i}(s[Z_i],\bm{a}_i),1\}\leq 1,\ \mathbb{E}_{(s,\bm{a})\sim\rho^{(n)}_h}\left[\left(f^{(n)}_{h,i}(s[Z_i],\bm{a}_i)\right)^2\right]\leq\zeta^{(n)}$ and $\mathbb{E}_{(s,\bm{a})\sim \tilde{\rho}^{(n)}_h}\left[\left(f^{(n)}_{h,i}(s[Z_i],\bm{a}_i)\right)^2\right]\leq\zeta^{(n)}$. Then according to our choice of $\alpha^{(n)}$, we get
\begin{align*}
    \mathbb{E}_{(s,\bm{a})\sim d^{\tilde{\pi}^{(n)}}_{\hat{P}^{(n)},h}}\left[\min\left\{f_{h,i}^{(n)}(s[Z_i],\bm{a}_i),1\right\}\right]\leq\mathbb{E}_{(\tilde{s},\tilde{\bm{a}})\sim d^{\tilde{\pi}^{(n)}}_{\hat{P}^{(n)},h-1}}\left[\min\left\{\frac{c\alpha^{(n)}}{H}\left\Vert\bar{\phi}^{(n)}_{h-1,i}(\tilde{s},\tilde{\bm{a}})\right\Vert_{\Sigma_{n,\rho^{(n)}_{h-1},\bar{\phi}^{(n)}_{h-1,i}}^{-1}},1\right\}\right]. 
\end{align*}
Combining all things together,
\begin{align*}
    \overline{v}_i^{(n)}-v^{\dagger,\pi^{(n)}_{-i}}_i=&\mathbb{E}_{s\sim d_1}\left[\overline{V}_{1,i}^{(n)}(s)-V^{\dagger,\pi^{(n)}_{-i}}_{1,i}(s)\right]\\
    \geq&\sum_{h=1}^H\mathbb{E}_{(s,\bm{a})\sim d_{\hat{P}^{(n)},h}^{\tilde{\pi}^{(n)}}}\left[\hat{\beta}_h^{(n)}(s,\bm{a})\right]-H\sum_{h=1}^H\mathbb{E}_{(s,\bm{a})\sim d_{\hat{P}^{(n)},h}^{\tilde{\pi}^{(n)}}}\left[\min\left\{f^{(n)}_h(s,\bm{a}),1\right\}\right]\\
    \geq&\sum_{h=1}^{H-1}\mathbb{E}_{(s,\bm{a})\sim d^{\tilde{\pi}^{(n)}}_{\hat{P}^{(n)},h}}\left[\hat{\beta}_h^{(n)}(s,\bm{a})-\sum_{j=1}^M\min\left\{c\alpha^{(n)}\left\Vert\bar{\phi}^{(n)}_{h,j}(s,\bm{a})\right\Vert_{\Sigma_{\rho^{(n)}_h,\bar{\phi}^{(n)}_{h,j}}^{-1}},H\right\}\right]-HM\sqrt{\tilde{A}\zeta^{(n)}}\\
    \geq&-HM\sqrt{\tilde{A}\zeta^{(n)}},
\end{align*}
which proves the inequality. 
\end{proof}

\begin{lemma}[Optimism for CE]
\label{lem:optimism_CE}
Consider an episode $n\in[N]$ and set $\alpha^{(n)}=\Theta\left(H\tilde{A}\sqrt{n\zeta^{(n)}+d^L\lambda}\right)$. When the event $\mathcal{E}$ holds, we have
\begin{align*}
    \overline{v}_i^{(n)}(s)-\max_{\omega\in\Omega_i}v^{\omega\circ\pi^{(n)}}_i(s)\geq-HM\sqrt{A\zeta^{(n)}},\quad\forall n\in[N],i\in[M].
\end{align*}
\end{lemma}
\begin{proof}
Denote $\tilde{\omega}_{h,i}^{(n)}=\argmax_{\omega_h\in\Omega_{h,i}}\left(\mathbb{D}_{\omega_h\circ\pi_h^{(n)}}\max_{\omega\in\Omega_i}Q_{h,i}^{\omega\circ\pi^{(n)}}\right)(s)$ and let $\tilde{\pi}_h^{(n)}=\tilde{\omega}_{h,i}\circ\pi^{(n)}_h$. Let $f^{(n)}_h(s,\bm{a})=\left\Vert\hat{P}^{(n)}_h(\cdot\vert s,\bm{a})-P^\star_h(\cdot\vert s,\bm{a})\right\Vert_1$ and $f^{(n)}_{h,i}(s[Z_i],\bm{a}_i)=\left\Vert\hat{P}^{(n)}_{h,i}(\cdot\vert s[Z_i],\bm{a}_i)-P^\star_{h,i}(\cdot\vert s[Z_i],\bm{a}_i)\right\Vert_1$. Then according to the event $\mathcal{E}$, we have
\begin{align*}
    &\mathbb{E}_{(s,\bm{a})\sim\rho^{(n)}_h}\left[\left(f^{(n)}_{h,i}(s[Z_i],\bm{a}_i)\right)^2\right]\leq\zeta^{(n)},\quad\mathbb{E}_{(s,\bm{a})\sim\tilde{\rho}^{(n)}_h}\left[\left(f^{(n)}_{h,i}(s[Z_i],\bm{a}_i)\right)^2\right]\leq\zeta^{(n)},\quad\forall n\in[N],h\in[H],i\in[M]\\
    &\left\Vert\bar{\phi}_{h,i}(s,\bm{a})\right\Vert_{\left(\hat{\Sigma}^{(n)}_{h,\bar{\phi}_{h,i}}\right)^{-1}}=\Theta\left(\left\Vert\bar{\phi}_{h,i}(s,\bm{a})\right\Vert_{\Sigma^{-1}_{n,\rho^{(n)}_h,\bar{\phi}_{h,i}}}\right),\quad\forall n\in[N],h\in[H],\bar{\phi}_{h,i}\in\bar{\Phi}_{h,i},i\in[M]. 
\end{align*}
A direct conclusion of the event $\mathcal{E}$ is we can find an absolute constant $c$, such that
\begin{align*}
    \beta_h^{(n)}(s,\bm{a})=&\min\left\{\alpha^{(n)}\sum_{i=1}^M\left\Vert\bar{\phi}^{(n)}_{h,i}(\tilde{s},\tilde{\bm{a}})\right\Vert_{\left(\Sigma_{h,\bar{\phi}^{(n)}_{h,i}}^{(n)}\right)^{-1}},H\right\}\\
    \geq&c\min\left\{\alpha^{(n)}\sum_{i=1}^M\left\Vert\bar{\phi}^{(n)}_{h,i}(\tilde{s},\tilde{\bm{a}})\right\Vert_{\Sigma_{n,\rho^{(n)}_h,\bar{\phi}^{(n)}_{h,i}}^{-1}},H\right\},\quad\forall n\in[N],h\in[H]. 
\end{align*}
Next, similar to the proof in Lemma \ref{lem:optimism_CE_mb}, we may prove
\begin{align*}
    \mathbb{E}_{s\sim d_1}\left[\overline{V}_{1,i}^{(n)}(s)-\max_{\omega\in\Omega_i}V^{\omega\circ\pi^{(n)}}_{1,i}(s)\right]\geq\sum_{h=1}^H\mathbb{E}_{(s,\bm{a})\sim d_{\hat{P}^{(n)},h}^{\tilde{\pi}^{(n)}}}\left[\hat{\beta}_h^{(n)}(s,\bm{a})\right]-H\sum_{h=1}^H\mathbb{E}_{(s,\bm{a})\sim d_{\hat{P}^{(n)},h}^{\tilde{\pi}^{(n)}}}\left[\min\left\{f^{(n)}_h(s,\bm{a}),1\right\}\right].
\end{align*}
Note that we can use exactly the same steps in the proof of Lemma \ref{lem:optimism_NE_CCE} to bound the second term, and we get for $h=1$,
\begin{align*}
    \mathbb{E}_{(s,\bm{a})\sim d^{\tilde{\pi}^{(n)}}_{\hat{P}^{(n)},1}}\left[\min\left\{f_{1,i}^{(n)}(s[Z_i],\bm{a}_i),1\right\}\right]\leq \sqrt{\tilde{A}\zeta^{(n)}}.
\end{align*}
And $\forall h\geq 2$, 
\begin{align*}
    \mathbb{E}_{(s,\bm{a})\sim d^{\tilde{\pi}^{(n)}}_{\hat{P}^{(n)},h}}\left[\min\left\{f_{h,i}^{(n)}(s[Z_i],\bm{a}_i),1\right\}\right]\leq\frac{c\alpha^{(n)}}{H}\mathbb{E}_{(\tilde{s},\tilde{\bm{a}})\sim d^{\tilde{\pi}^{(n)}}_{\hat{P}^{(n)},h-1}}\left[\left\Vert\bar{\phi}^{(n)}_{h-1,i}(\tilde{s},\tilde{\bm{a}})\right\Vert_{\Sigma_{n,\rho^{(n)}_{h-1},\bar{\phi}^{(n)}_{h-1,i}}^{-1}}\right]. 
\end{align*}
Combining all things together,
\begin{align*}
    &\overline{v}_i^{(n)}-\max_{\omega\in\Omega_i}v^{\omega\circ\pi^{(n)}}_i\\
    =&\mathbb{E}_{s\sim d_1}\left[\overline{V}_{1,i}^{(n)}(s)-\max_{\omega\in\Omega_i}V^{\omega\circ\pi^{(n)}}_{1,i}(s)\right]\\
    \geq&\sum_{h=1}^H\mathbb{E}_{(s,\bm{a})\sim d_{\hat{P}^{(n)},h}^{\tilde{\pi}^{(n)}}}\left[\hat{\beta}_h^{(n)}(s,\bm{a})\right]-H\sum_{h=1}^H\mathbb{E}_{(s,\bm{a})\sim d_{\hat{P}^{(n)},h}^{\tilde{\pi}^{(n)}}}\left[\min\left\{f^{(n)}_h(s,\bm{a}),1\right\}\right]\\
    \geq&\sum_{h=1}^{H-1}\mathbb{E}_{(s,\bm{a})\sim d^{\tilde{\pi}^{(n)}}_{\hat{P}^{(n)},h}}\left[\hat{\beta}_h^{(n)}(s,\bm{a})-\sum_{j=1}^M\min\left(c\alpha^{(n)}\left\Vert\bar{\phi}^{(n)}_{h,j}(s,\bm{a})\right\Vert_{\Sigma_{\rho^{(n)}_h,\bar{\phi}^{(n)}_{h,j}}^{-1}},H\right)\right]-HM\sqrt{\tilde{A}\zeta^{(n)}}\\
    \geq&-HM\sqrt{\tilde{A}\zeta^{(n)}},
\end{align*}
which proves the inequality. 
\end{proof}

\begin{lemma}[pessimism]\label{lem:pessimism}
    Consider an episode $n\in[N]$ and set $\alpha^{(n)}=\Theta\left(H\tilde{A}\sqrt{n\zeta^{(n)}+d^L\lambda}\right)$. When the event $\mathcal{E}$ holds, we have
    \begin{align*}
        \underline{v}_i^{(n)}(s)-v^{\pi^{(n)}}_i(s)\leq HM\sqrt{\tilde{A}\zeta^{(n)}},\quad\forall n\in[N],i\in[M].
    \end{align*}
\end{lemma}
\begin{proof}
Let $f^{(n)}_h(s,\bm{a})=\left\Vert\hat{P}^{(n)}_h(\cdot\vert s,\bm{a})-P^\star_h(\cdot\vert s,\bm{a})\right\Vert_1$ and $f^{(n)}_{h,i}(s[Z_i],\bm{a}_i)=\left\Vert\hat{P}^{(n)}_{h,i}(\cdot\vert s[Z_i],\bm{a}_i)-P^\star_{h,i}(\cdot\vert s[Z_i],\bm{a}_i)\right\Vert_1$. Then according to the event $\mathcal{E}$, we have
\begin{align*}
    &\mathbb{E}_{(s,\bm{a})\sim\rho^{(n)}_h}\left[\left(f^{(n)}_{h,i}(s[Z_i],\bm{a}_i)\right)^2\right]\leq\zeta^{(n)},\quad\mathbb{E}_{(s,\bm{a})\sim\tilde{\rho}^{(n)}_h}\left[\left(f^{(n)}_{h,i}(s[Z_i],\bm{a}_i)\right)^2\right]\leq\zeta^{(n)},\quad\forall n\in[N],h\in[H],i\in[M]\\
    &\left\Vert\bar{\phi}_{h,i}(s,\bm{a})\right\Vert_{\left(\hat{\Sigma}^{(n)}_{h,\bar{\phi}_{h,i}}\right)^{-1}}=\Theta\left(\left\Vert\bar{\phi}_{h,i}(s,\bm{a})\right\Vert_{\Sigma^{-1}_{n,\rho^{(n)}_h,\bar{\phi}_{h,i}}}\right),\quad\forall n\in[N],h\in[H],\bar{\phi}_{h,i}\in\bar{\Phi}_{h,i},i\in[M]. 
\end{align*}
A direct conclusion of the event $\mathcal{E}$ is we can find an absolute constant $c$, such that
\begin{align*}
    \beta_h^{(n)}(s,\bm{a})=&\sum_{i=1}^M\min\left\{\alpha^{(n)}\left\Vert\bar{\phi}^{(n)}_{h,i}(\tilde{s},\tilde{\bm{a}})\right\Vert_{\left(\Sigma_{h,\bar{\phi}^{(n)}_{h,i}}^{(n)}\right)^{-1}},H\right\}\\
    \geq&\sum_{i=1}^M\min\left\{c\alpha^{(n)}\left\Vert\bar{\phi}^{(n)}_{h,i}(\tilde{s},\tilde{\bm{a}})\right\Vert_{\Sigma_{n,\rho^{(n)}_h,\bar{\phi}^{(n)}_{h,i}}^{-1}},H\right\},\quad\forall n\in[N],h\in[H]. 
\end{align*}
Next, similar to the proof in Lemma \ref{lem:pessimism_mb}, we may prove
\begin{align}
    \mathbb{E}_{s\sim d_{\hat{P}^{(n)},h}^{\pi^{(n)}}}\left[\underline{V}_{h,i}^{(n)}(s)-V^{\pi^{(n)}}_{h,i}(s)\right]\leq&\sum_{h^\prime=h}^H\mathbb{E}_{(s,\bm{a})\sim d_{\hat{P}^{(n)},h^\prime}^{\pi^{(n)}}}\left[-\hat{\beta}_{h^\prime}^{(n)}(s,\bm{a})+H\min\left\{f^{(n)}_{h^\prime}(s,\bm{a}),1\right\}\right],\quad\forall h\in[H].
\end{align}
and we get for $h=1$,
\begin{align*}
    \mathbb{E}_{(s,\bm{a})\sim d^{\tilde{\pi}^{(n)}}_{\hat{P}^{(n)},1}}\left[\min\left\{f_{1,i}^{(n)}(s[Z_i],\bm{a}_i),1\right\}\right]\leq \sqrt{\tilde{A}\zeta^{(n)}}.
\end{align*}
And $\forall h\geq 2$, 
\begin{align*}
    \mathbb{E}_{(s,\bm{a})\sim d^{\tilde{\pi}^{(n)}}_{\hat{P}^{(n)},h}}\left[\min\left\{f_{h,i}^{(n)}(s[Z_i],\bm{a}_i),1\right\}\right]\leq\frac{c\alpha^{(n)}}{H}\mathbb{E}_{(\tilde{s},\tilde{\bm{a}})\sim d^{\tilde{\pi}^{(n)}}_{\hat{P}^{(n)},h-1}}\left[\left\Vert\bar{\phi}^{(n)}_{h-1,i}(\tilde{s},\tilde{\bm{a}})\right\Vert_{\Sigma_{n,\rho^{(n)}_{h-1},\bar{\phi}^{(n)}_{h-1,i}}^{-1}}\right]. 
\end{align*}
Finally, we get
\begin{align*}
    \underline{v}_i^{(n)}-v^{\pi^{(n)}}_i=&\mathbb{E}_{s\sim d_1}\left[\underline{V}_{1,i}^{(n)}(s)-V^{\pi^{(n)}}_{1,i}(s)\right]\\
    \leq&\sum_{h=1}^{H-1}\mathbb{E}_{(s,\bm{a})\sim d_{\hat{P}^{(n)},h}^{\pi^{(n)}}}\left[-\hat{\beta}_h^{(n)}(s,\bm{a})+\sum_{j=1}^M\min\left(c\alpha^{(n)}\left\Vert\bar{\phi}^{(n)}_{h,j}(s,\bm{a})\right\Vert_{\Sigma_{\rho^{(n)}_h,\bar{\phi}^{(n)}_{h,j}}^{-1}},H\right)\right]+HM\sqrt{\tilde{A}\zeta^{(n)}}\\
    \leq&HM\sqrt{\tilde{A}\zeta^{(n)}},
\end{align*}
which has finished the proof. 
\end{proof}

\begin{lemma}\label{lem:pseudo_regret}
When the event $\mathcal{E}$ holds and $\alpha^{(n)}=\Theta\left(H\tilde{A}\sqrt{n\zeta^{(n)}+d^L\lambda}\right)$ satisfies $\alpha^{(1)}\leq\alpha^{(2)}\leq\ldots\leq\alpha^{(N)}$, we have 
\begin{align*}
    \sum_{n=1}^N\Delta^{(n)}\lesssim H^2Md^{L^2}A^L\sqrt{N\log\left(1+\frac{N}{d\lambda}\right)}\alpha^{(N)}.
\end{align*}
\end{lemma}
\begin{proof}
Let $f^{(n)}_h(s,\bm{a})=\left\Vert\hat{P}^{(n)}_h(\cdot\vert s,\bm{a})-P^\star_h(\cdot\vert s,\bm{a})\right\Vert_1$ and $f^{(n)}_{h,i}(s[Z_i],\bm{a}_i)=\left\Vert\hat{P}^{(n)}_{h,i}(\cdot\vert s[Z_i],\bm{a}_i)-P^\star_{h,i}(\cdot\vert s[Z_i],\bm{a}_i)\right\Vert_1$. Then according to the event $\mathcal{E}$, we have
\begin{align*}
    &\mathbb{E}_{(s,\bm{a})\sim\rho^{(n)}_h}\left[\left(f^{(n)}_{h,i}(s[Z_i],\bm{a}_i)\right)^2\right]\leq\zeta^{(n)},\quad\mathbb{E}_{(s,\bm{a})\sim\tilde{\rho}^{(n)}_h}\left[\left(f^{(n)}_{h,i}(s[Z_i],\bm{a}_i)\right)^2\right]\leq\zeta^{(n)},\quad\forall n\in[N],h\in[H],i\in[M]\\
    &\left\Vert\bar{\phi}_{h,i}(s,\bm{a})\right\Vert_{\left(\hat{\Sigma}^{(n)}_{h,\bar{\phi}_{h,i}}\right)^{-1}}=\Theta\left(\left\Vert\bar{\phi}_{h,i}(s,\bm{a})\right\Vert_{\Sigma^{-1}_{n,\rho^{(n)}_h,\bar{\phi}_{h,i}}}\right),\quad\forall n\in[N],h\in[H],\bar{\phi}_{h,i}\in\bar{\Phi}_{h,i},i\in[M]. 
\end{align*}
By definition, we have
\begin{align*}
    \Delta^{(n)}=\max_{i\in[M]}\left\{\overline{v}^{(n)}_i-\underline{v}^{(n)}_i\right\}+2HM\sqrt{\tilde{A}\zeta^{(n)}}.
\end{align*}
With similar steps as those in the proof of Lemma \ref{lem:pseudo_regret_mb} (note that $\overline{V}^{(n)}_{h,i}(s)-\underline{V}^{(n)}_{h,i}(s)$ is upper bounded by $2H^2M$), we have
\begin{align}
    \mathbb{E}_{s\sim d^{\pi^{(n)}}_{P^\star,1}}\left[\overline{V}^{(n)}_{1,i}(s)-\underline{V}^{(n)}_{1,i}(s)\right]\leq 2\underbrace{\sum_{h=1}^H\mathbb{E}_{(s,\bm{a})\sim d_{P^\star,h}^{\pi^{(n)}}}\left[\hat{\beta}^{(n)}_h(s,\bm{a})\right]}_{(a)}+2H^2M\underbrace{\sum_{h=1}^H\mathbb{E}_{(s,\bm{a})\sim d_{P^\star,h}^{\pi^{(n)}}}\left[f_h^{(n)}(s,\bm{a})\right]}_{(b)}\label{eq:regret_middle}. 
\end{align}
First, we calculate the first term (a) in Inequality \eqref{eq:regret_middle}. Following Lemma \ref{lem:useful_fac}, we have
\begin{align*}
    &\sum_{h=1}^H\mathbb{E}_{(s,\bm{a})\sim d^{\pi^{(n)}}_{P^\star,h}}\left[\hat{\beta}^{(n)}_h(s,\bm{a})\right]\\
    \lesssim&\sum_{h=1}^H\sum_{i=1}^M\mathbb{E}_{(s,\bm{a})\sim d^{\pi^{(n)}}_{P^\star,h}}\left[\min\left(\alpha^{(n)}\left\Vert\bar{\phi}^{(n)}_{h,i}(s,\bm{a})\right\Vert_{\Sigma^{-1}_{n,\rho^{(n)}_h,\bar{\phi}^{(n)}_{h,i}}},H\right)\right]\\ 
    \lesssim&\sum_{h=1}^{H-1}\sum_{i=1}^M\tilde{A}^L\mathbb{E}_{(\tilde{s},\tilde{\bm{a}})\sim d^{\pi^{(n)}}_{P^\star,h}}\left[\left\Vert\tilde{\phi}^\star_{h,i}(\tilde{s},\tilde{\bm{a}})\right\Vert_{\Sigma_{n,\gamma^{(n)}_h,\tilde{\phi}^\star_{h,i}}^{-1}}\right]\\
    &\cdot\sqrt{{n\left(\alpha^{(n)}\right)^2}\mathbb{E}_{(s,\bm{a})\sim\rho^{(n)}_h}\left[\left\Vert\bar{\phi}^{(n)}_{h,i}(s,\bm{a})\right\Vert^2_{\Sigma^{-1}_{n,\rho^{(n)}_h,\bar{\phi}^{(n)}_{h,i}}}\right]+H^2d^{L^2}\lambda}\\
    +&\sqrt{\tilde{A}^L\left(\alpha^{(n)}\right)^2\mathbb{E}_{(s,\bm{a})\sim\rho^{(n)}_1}\left[\left\Vert\bar{\phi}^{(n)}_{1,i}(s,\bm{a})\right\Vert^2_{\Sigma^{-1}_{n,\rho^{(n)}_1,\bar{\phi}^{(n)}_{1,i}}}\right]}.
\end{align*}
Note that we use the fact that $B=H$ when applying Lemma \ref{lem:useful}. In addition, we have 
\begin{align*}
     &n\mathbb{E}_{(s,\bm{a})\sim\rho^{(n)}_h}\left[\left\Vert\bar{\phi}_{h,i}^{(n)}(s,\bm{a})\right\Vert^2_{\Sigma^{-1}_{n,\rho^{(n)}_h,\bar{\phi}^{(n)}_{h,i}}}\right]\\
     =&n\textrm{Tr}\left(\mathbb{E}_{(s,\bm{a})\sim\rho^{(n)}_h}\left[\bar{\phi}_{h,i}^{(n)}(s,\bm{a})\bar{\phi}_{h,i}^{(n)}(s,\bm{a})^\top\right]\left(n\mathbb{E}_{(s,\bm{a})\sim\rho^{(n)}_h}\left[\bar{\phi}_{h,i}^{(n)}(s,\bm{a})\bar{\phi}_{h,i}^{(n)}(s,\bm{a})^\top\right]+\lambda I_{d^{\vert Z_i\vert}}\right)^{-1}\right)\\
     \leq& d^L.
\end{align*}
Then,
\begin{align*}
    &\sum_{h=1}^H\mathbb{E}_{(s,\bm{a})\sim d^{\pi^{(n)}}_{P^\star,h}}\left[\hat{\beta}^{(n)}_h(s,\bm{a})\right]\\
    \leq&\sum_{h=1}^{H-1}\sum_{i=1}^M\tilde{A}^L\mathbb{E}_{(\tilde{s},\tilde{\bm{a}})\sim d^{\pi^{(n)}}_{P^\star,h}}\left[\left\Vert\tilde{\phi}^\star_{h,i}(\tilde{s},\tilde{\bm{a}})\right\Vert_{\tilde{\Sigma}_{n,\gamma^{(n)}_h,\tilde{\phi}^\star_{h,i}}^{-1}}\right]\sqrt{d^L\left(\alpha^{(n)}\right)^2+H^2d^{L^2}\lambda}\\
    &+\sqrt{{d^L\tilde{A}^L\left(\alpha^{(n)}\right)^2}/n}. 
\end{align*}
Second, we  calculate the term (b) in inequality \eqref{eq:regret_middle}. Following Lemma \ref{lem:useful} and noting $f^{(n)}_{h,i}(s[Z_i],\bm{a}_i)$ is upper-bounded by $2$ (i.e., $B=2$ in Lemma \ref{lem:useful}), we have 
\begin{align*}
    &\sum_{h=1}^H\mathbb{E}_{(s,\bm{a})\sim d^{\pi^{(n)}}_{P^\star,h}}\left[f_h^{(n)}(s,\bm{a})\right]\\
    \leq&\sum_{i=1}^M\sum_{h=1}^H\mathbb{E}_{(s,\bm{a})\sim d^{\pi^{(n)}}_{P^\star,h}}\left[f_{h,i}^{(n)}(s[Z_i],\bm{a}_i)\right]\\
    \leq&\sum_{i=1}^M\sum_{h=1}^{H-1}\tilde{A}\mathbb{E}_{(\tilde{s},\tilde{\bm{a}})\sim d^{\pi^{(n)}}_{P^\star,h}}\left[\left\Vert\bar{\phi}_{h,i}^\star(\tilde{s},\tilde{\bm{a}})\right\Vert_{\Sigma_{n,\gamma^{(n)}_h,\bar{\phi}^\star_{h,i}}^{-1}}\right]\sqrt{n\mathbb{E}_{(s,\bm{a})\sim\rho^{(n)}_h}\left[\left(f^{(n)}_{h,i}(s[Z_i],\bm{a}_i)\right)^2\right]+d^L\lambda}\\
    &+\sqrt{\tilde{A}\mathbb{E}_{(s,\bm{a})\sim\rho^{(n)}_h}\left[\left(f^{(n)}_1(s[Z_j],\bm{a}_j)\right)^2\right]}\\
    \leq&\sum_{i=1}^M\sum_{h=1}^{H-1}
    \tilde{A}\mathbb{E}_{(\tilde{s},\tilde{\bm{a}})\sim d^{\pi^{(n)}}_{P^\star,h}}\left[\left\Vert\bar{\phi}_{h,i}^\star(\tilde{s},\tilde{\bm{a}})\right\Vert_{\Sigma_{n,\gamma^{(n)}_h,\bar{\phi}^\star_{h,i}}^{-1}}\right]\sqrt{n\zeta^{(n)}+d^L\lambda}+\sqrt{\tilde{A}\zeta^{(n)}}\\
    \lesssim&\frac{\alpha^{(n)}}{H}\sum_{i=1}^M\sum_{h=1}^{H-1}\mathbb{E}_{(\tilde{s},\tilde{\bm{a}})\sim d^{\pi^{(n)}}_{P^\star,h}}\left[\left\Vert\bar{\phi}_{h,i}^\star(\tilde{s},\tilde{\bm{a}})\right\Vert_{\Sigma_{n,\gamma^{(n)}_h,\bar{\phi}^\star_{h,i}}^{-1}}\right]+\sqrt{\tilde{A}\zeta^{(n)}}, 
\end{align*}
where in the second inequality, we use $\mathbb{E}_{(s,\bm{a})\sim\rho_h^{(n)}}\left[\left(f_{h,i}^{(n)}(s[Z_i],\bm{a}_i)\right)^2\right]\leq\zeta^{(n)}$, and in the last line, recall $\tilde{A}\sqrt{n\zeta^{(n)}+d^L\lambda}\lesssim\alpha^{(n)}/H$. Then, by combining the above calculation of the term (a) and term (b) in inequality \eqref{eq:regret_middle}, we have:
\begin{align*}
      &\overline{v}^{(n)}_i-\underline{v}^{(n)}_i\\
      =&\mathbb{E}_{s\sim d^{\pi^{(n)}}_{P^\star,1}}\left[\overline{V}^{(n)}_{1,i}(s)-\underline{V}^{(n)}_{1,i}(s)\right]\\
      \lesssim&\sum_{i=1}^M\sum_{h=1}^{H-1}\left(\tilde{A}^L\mathbb{E}_{(\tilde{s},\tilde{\bm{a}})\sim d^{\pi^{(n)}}_{P^\star,h}}\left[\left\Vert\tilde{\phi}_{h,i}^\star(\tilde{s},\tilde{\bm{a}})\right\Vert_{\Sigma_{n,\gamma^{(n)}_h,\tilde{\phi}_{h,i}^\star}^{-1}}\right]\sqrt{d^L\left(\alpha^{(n)}\right)^2+H^2d^{L^2}\lambda}+\sqrt{\frac{d^L\tilde{A}^L\left(\alpha^{(n)}\right)^2}{n}}\right)\\
      &+H^2M\sum_{i=1}^M\sum_{h=1}^{H-1}\left(\frac{\alpha^{(n)}}{H}\mathbb{E}_{(\tilde{s},\tilde{\bm{a}})\sim d^{\pi^{(n)}}_{P^\star,h}}\left[\left\Vert\bar{\phi}^\star_{h,i}(\tilde{s},\tilde{\bm{a}})\right\Vert_{\Sigma_{n,\gamma^{(n)}_h,\bar{\phi}^\star_{h,i}}^{-1}}\right]+\sqrt{\tilde{A}\zeta^{(n)}}\right). 
\end{align*}
Taking maximum over $i$ on both sides and use the definition of $\Delta^{(n)}$, we get
\begin{align*}
    \Delta^{(n)}=&\max_{i\in[M]}\left\{\overline{v}^{(n)}_i-\underline{v}^{(n)}_i\right\}+2HM\sqrt{\tilde{A}\zeta^{(n)}}\\
    \lesssim&\sum_{i=1}^M\sum_{h=1}^{H-1}\left(\tilde{A}^L\mathbb{E}_{(\tilde{s},\tilde{\bm{a}})\sim d^{\pi^{(n)}}_{P^\star,h}}\left[\left\Vert\tilde{\phi}_{h,i}^\star(\tilde{s},\tilde{\bm{a}})\right\Vert_{\tilde{\Sigma}_{n,\gamma^{(n)}_h,\tilde{\phi}_{h,i}^\star}^{-1}}\right]\sqrt{d^L\left(\alpha^{(n)}\right)^2+H^2d^{L^2}\lambda}+\sqrt{\frac{d^L\tilde{A}^L\left(\alpha^{(n)}\right)^2}{n}}\right)\\
    &+H^2M\sum_{i=1}^M\sum_{h=1}^{H-1}\left(\frac{\alpha^{(n)}}{H}\mathbb{E}_{(\tilde{s},\tilde{\bm{a}})\sim d^{\pi^{(n)}}_{P^\star,h}}\left[\left\Vert\bar{\phi}^\star_{h,i}(\tilde{s},\tilde{\bm{a}})\right\Vert_{\Sigma_{n,\gamma^{(n)}_h,\bar{\phi}^\star_{h,i}}^{-1}}\right]+\sqrt{\tilde{A}\zeta^{(n)}}\right). 
\end{align*}
Hereafter, we take the dominating term out. Note that
\begin{align*}
    &\sum_{n=1}^N\mathbb{E}_{(\tilde{s},\tilde{\bm{a}})\sim d^{\pi^{(n)}}_{P^\star,h}}\left[\left\Vert\tilde{\phi}_{h,i}^\star(\tilde{s},\tilde{\bm{a}})\right\Vert_{\Sigma_{n,\gamma^{(n)}_h,\tilde{\phi}_{h,i}^\star}^{-1}}\right]\\
    \leq&\sqrt{N\sum_{n=1}^N\mathbb{E}_{(\tilde{s},\tilde{\bm{a}})\sim d^{\pi^{(n)}}_{P^\star,h}}\left[\tilde{\phi}_{h,i}^\star(\tilde{s},\tilde{\bm{a}})^\top\Sigma^{-1}_{n,\gamma^{(n)}_h,\tilde{\phi}_{h,i}^\star}\tilde{\phi}_{h,i}^\star(\tilde{s},\tilde{\bm{a}})\right]}\tag{CS inequality}\\
    \lesssim&\sqrt{N\left(\log\det\left(\lambda I_{d^{\vert\cup_{j\in Z_i}Z_j\vert}}+\sum_{n=1}^N\mathbb{E}_{(\tilde{s},\tilde{\bm{a}})\sim d^{ \pi^{(n)}}_{P^\star,h}}\left[\tilde{\phi}_{h,i}^\star(\tilde{s},\tilde{\bm{a}})\tilde{\phi}_{h,i}^\star(\tilde{s},\tilde{\bm{a}})^\top\right]\right)-\log\det(\lambda I_{d^{\vert\cup_{j\in Z_i}Z_j\vert}})\right)}\tag{Lemma \ref{lem:reduction}}\\ 
    \leq&\sqrt{d^{L^2}N \log\left(1+\frac{N}{d\lambda}\right)}.\tag{Potential function bound, Lemma \ref{lem:potential} noting $\Vert\phi_{h,i}^\star(s[Z_i],\bm{a}_i)\Vert_2\leq 1$ for any $(s,\bm{a})$.}
\end{align*}
Similarly, we have
\begin{align*}
    \sum_{n=1}^N\mathbb{E}_{(\tilde{s},\tilde{\bm{a}})\sim d^{\pi^{(n)}}_{P^\star,h}}\left[\left\Vert\bar{\phi}_{h,i}^\star(\tilde{s},\tilde{\bm{a}})\right\Vert_{\Sigma_{n,\gamma^{(n)}_h,\bar{\phi}_{h,i}^\star}^{-1}}\right]\leq\sqrt{d^LN \log\left(1+\frac{N}{d\lambda}\right)}.
\end{align*}
Finally, 
\begin{align*}
    \sum_{n=1}^N\Delta^{(n)}\lesssim&HM\left(\sqrt{d^{L^2}N\log\left(1+\frac{N}{d\lambda}\right)}\tilde{A}^L\sqrt{{d^L\left(\alpha^{(N)}\right)^2}+H^2d^{L^2}\lambda}+\sum_{n=1}^{N}\sqrt{\frac{d^L\tilde{A}^L\left(\alpha^{(n)}\right)^2}{n}}\right)\\
    &+H^3M^2\left(\frac{1}{H}\sqrt{d^LN\log\left(1+\frac{N}{d\lambda}\right)}\alpha^{(N)}+\sum_{n=1}^N\sqrt{\tilde{A}\zeta^{(n)}}\right)\\ 
    \lesssim&H^2M^2d^{L^2}\tilde{A}^L\sqrt{N\log\left(1+\frac{N}{d\lambda}\right)}\alpha^{(N)}\tag{Some algebra. We take the dominating term out. Note that $\alpha^{(n)}$ is increasing in $n$}.
\end{align*}
This concludes the proof.
\end{proof}

\subsection{Proof of the Main Theorems}
\label{sec:main_pf}
\begin{lemma}
For the model-based algorithm, when we pick $\lambda=\Theta\left(Ld^L\log\frac{NHM\vert\Phi\vert}{\delta}\right)$, $\alpha^{(n)}=\Theta\left(H\tilde{A}\sqrt{n\zeta^{(n)}+d^L\lambda}\right)$ and $\zeta^{(n)}=\Theta\left(\frac{1}{n}\log\frac{\vert\mathcal{M}\vert HNM}{\delta}\right)$, with probability $1-\delta$, we have 
\begin{align*}
    \sum_{n=1}^N\Delta^{(n)}\lesssim H^3M^2d^{(L+1)^2}\tilde{A}^{\frac{L+1}{2}} N^{\frac{1}{2}}\log\frac{\vert\mathcal{M}\vert HNM}{\delta}.
\end{align*}
\end{lemma}
\begin{proof}
The result of Lemma \ref{lem:model_based_hp_factor} implies with our choice of $\lambda$ and $\zeta^{(n)}$, the event $\mathcal{E}$ holds with probability at least $1-\delta$. In this case, we have
\begin{align}
    \alpha^{(n)}=\Theta\left(H\tilde{A}\sqrt{\log\frac{\vert\mathcal{M}\vert HNM}{\delta}+Ld^{2L}\log\frac{NHM\vert\Phi\vert}{\delta}}\right),
\end{align}
which is a constant unrelated with $n$. Therefore, using the result of Lemma \ref{lem:pseudo_regret}, we get
\begin{align*}
    \sum_{n=1}^N\Delta^{(n)}\lesssim H^2d^{L^2}\tilde{A}^LM^2\sqrt{N\log\left(1+\frac{N}{d\lambda}\right)}\alpha^{(N)}\lesssim H^3M^2d^{(L+1)^2}\tilde{A}^{L+1}L^{\frac{1}{2}}N^{\frac{1}{2}}\log\frac{\vert\mathcal{M}\vert HNM}{\delta},
\end{align*}
which has finished the proof. 
\end{proof}

\paragraph{Proof of Theorem \ref{thm:mb}}
\begin{proof}
For any fixed episode $n$ and agent $i$, by Lemma \ref{lem:optimism_NE_CCE}, Lemma \ref{lem:optimism_CE} and Lemma \ref{lem:pessimism}, we have
\begin{align*}
    v^{\dagger,\pi^{(n)}_{-i}}_i-v^{\pi^{(n)}}_i \left(\textrm{or }\max_{\omega\in\Omega_i}v^{\omega\circ\pi^{(n)}}_i-v^{\pi^{(n)}}_i\right)\leq\overline{v}^{(n)}_i-\underline{v}^{(n)}_i+2HM\sqrt{\tilde{A}\zeta^{(n)}}\leq\Delta^{(n)}. 
\end{align*}
Taking maximum over $i$ on both sides, we have
\begin{align}
    \max_{i\in[M]}\left\{v^{\dagger,\pi^{(n)}_{-i}}_i-v^{\pi^{(n)}}_i\right\}\left(\textrm{or } \max_{i\in[M]}\left\{\max_{\omega\in\Omega_i}v^{\omega\circ\pi^{(n)}}_i-v^{\pi^{(n)}}_i\right\}\right)\leq\Delta^{(n)}.\label{eq:opt}
\end{align}
From Lemma \ref{lem:pseudo_regret_mb}, with probability $1-\delta$, we can ensure 
\begin{align*}
    \sum_{n=1}^N\Delta^{(n)}\lesssim H^3M^2d^{(L+1)^2}\tilde{A}^{L+1}L^{\frac{1}{2}}N^{\frac{1}{2}}\log\frac{\vert\mathcal{M}\vert HNM}{\delta}.
\end{align*}
Therefore, according to Lemma \ref{lem:convert_1}, when we pick $N$ to be 
\begin{align*}
    O\left(\frac{L^5M^4H^6d^{2(L+1)^2}\tilde{A}^{2(L+1)}}{\varepsilon^2}\log^2\left(\frac{HdALM\vert\mathcal{M}\vert}{\delta\varepsilon}\right)\right),
\end{align*}
we have 
\begin{align*}
    \frac{1}{N}\sum_{n=1}^N\Delta^{(n)}\leq\varepsilon.
\end{align*}
On the other hand, from \eqref{eq:opt}, we have
\begin{align*}
    &\max_{i\in[M]}\left\{v^{\dagger,\hat{\pi}_{-i}}_i-v^{\hat{\pi}}_i\right\}\left(\textrm{or }\max_{i\in[M]}\left\{\max_{\omega\in\Omega_i}v^{\omega\circ\hat{\pi}}_i-v^{\hat{\pi}}_i\right\}\right)\\
    =&\max_{i\in[M]}\left\{v^{\dagger,\pi^{(n^\star)}_{-i}}_i-v^{\pi^{(n^\star)}}_i\right\}\left(\textrm{or }\max_{i\in[M]}\left\{\max_{\omega\in\Omega_i}v^{\omega\circ\pi^{(n^\star)}}_i-v^{\pi^{(n^\star)}}_i\right\}\right)\\
    \leq&\Delta^{(n^\star)}=\min_{n\in[N]}\Delta^{(n)}\leq\frac{1}{N}\sum_{n=1}^N\Delta^{(n)}\leq\varepsilon,
\end{align*}
which has finished the proof, noting our assumption that $L=O(1)$. 
\end{proof}

\section{Auxiliary Lemmas}
\label{sec:aux}
\begin{lemma}[Concentration of the bonus term (\cite{zanette2021cautiously}, Lemma 39)] \label{lem:con}
Set $\lambda^{(n)}\geq\Theta(d\log(nH\vert\Phi\vert/\delta))$ for any $n$. Define
\begin{align*}
    \Sigma_{n,\rho^{(n)}_h,\phi}=n\mathbb{E}_{(s,\bm{a})\sim\rho^{(n)}_h}[\phi(s,\bm{a})\phi^\top(s,\bm{a})]+\lambda^{(n)}I_d,\quad\hat{\Sigma}^{(n)}_{h,\phi}=\sum_{i=1}^{n}\phi(s_h^{(i)},\bm{a}_h^{(i)})\phi^\top(s_h^{(i)},\bm{a}_h^{(i)})+\lambda^{(n)}I_d. 
\end{align*} 
With probability $1-\delta$, we have
\begin{align*}
  \forall n\in\mathbb{N}^+,\forall h\in[H],\forall\phi\in\Phi,\quad c_1 \Vert\phi(s,\bm{a})\Vert_{\Sigma^{-1}_{\rho^{(n)}_h,\phi}}\leq \Vert\phi(s,\bm{a})\Vert_{\left(\hat{\Sigma}^{(n)}_{h,\phi}\right)^{-1}}\leq c_2\Vert\phi(s,\bm{a})\Vert_{\Sigma^{-1}_{\rho^{(n)}_h,\phi}}. 
\end{align*}
\end{lemma}

\begin{lemma}[\cite{agarwal2020pc}, Lemma G.2]
\label{lem:reduction}
Consider the following process. For $n=1,\ldots,N$, $M_n=M_{n-1}+G_n$ with $M_0=\lambda_0 I$ and $G_n$ being a positive semidefinite matrix with eigenvalues upper bounded by $1$. We have
\begin{align*}
    2\log \det (M_N)-2\log \det(\lambda_0 I)\geq \sum_{n=1}^N\mathrm{Tr}(G_n M^{-1}_{n-1}). 
\end{align*}
\end{lemma}

\begin{lemma}[Potential function lemma]
\label{lem:potential}
Suppose $\mathrm{Tr}(G_n)\leq B^2$. 
\begin{align*}
    2\log\det(M_N)-2\log\det(\lambda_0 I)\leq d\log\left(1+\frac{NB^2}{d\lambda_0}\right)
\end{align*}
\end{lemma}
\begin{proof}
Let $\sigma_1,\cdots,\sigma_d$ be the set of singular values of $M_N$ recalling $M_N$ is a positive semidefinite matrix. Then, by the AM-GM inequality, 
\begin{align*}
    \log\det(M_N)/\det(\lambda_0 I)=\log\prod_{i=1}^d(\sigma_i/\lambda_0)\leq\log d\left(\frac{1}{d}\sum_{i=1}^d (\sigma_i/\lambda_0))\right)
\end{align*}
Since we have $\sum_i\sigma_i=\mathrm{Tr}(M_N)\leq d\lambda_0+NB^2$, the statement is concluded. 
\end{proof}

\begin{lemma}
\label{lem:convert_1}
For parameters $A,B,\varepsilon$ such that $\frac{A^2B}{\varepsilon^2}$ is larger than some absolute constant, when we pick $N=\frac{A^2}{\varepsilon^2}\log^2\frac{A^4B^2}{\varepsilon^4}=O\left(\frac{A^2}{\varepsilon^2}\log^2\frac{AB}{\varepsilon}\right)$, we have
\begin{align*}
    \frac{A}{\sqrt{N}}\log(BN)\leq\varepsilon.
\end{align*}
\end{lemma}
\begin{proof}
We have
\begin{align*}
    \frac{A}{\sqrt{N}}\log(BN)=\varepsilon\frac{\log\left(\frac{A^2B}{\varepsilon^2}\log^2\frac{A^4B^2}{\varepsilon^4}\right)}{\log\frac{A^4B^2}{\varepsilon^4}}
\end{align*}
Note that
\begin{align*}
    \frac{A^2B}{\varepsilon^2}\log^2\frac{A^4B^2}{\varepsilon^4} \leq \frac{A^4B^2}{\varepsilon^4}\Leftrightarrow\log^2\frac{A^4B^2}{\varepsilon^4}\leq\frac{A^2B}{\varepsilon^2}
\end{align*}
where the right hand side is always true whenever $\frac{A^2B}{\varepsilon^2}$ is larger than some given constant. Therefore, we get
\begin{align*}
    \frac{A}{\sqrt{N}}\log(BN)\leq\varepsilon. 
\end{align*}
\end{proof}

\section{Experiment Details}
\label{app:sec:exp}

\subsection{Detailed environment setup} \label{app:exp:environment}
In this section we introduce the details of the environment construction of the Block Markov games. For completeness we repeat certain details already introduced in the main text. We design our Block Markov game by first randomly generating a tabular Markov game with horizon $H$, 3 states, 2 players each with 3 actions, and random reward matrix $R_h \in (0,1)^{3\times 3^2 \times H}$ and random transition matrix $T_h(s_h,a_h) \in \Delta(\mathcal{S}_{h+1})$. For the reward generalization, for each $r(s,a,s')$ entry in the reward matrix, we assign it with a random number sampled from a uniform distribution from -1 to 1. For the probability matrix generation, for each conditional distribution $T(\cdot|s,a)$, we randomly sample 3 numbers from a uniform distribution from -1 to 1 and form the probability simplex by normalization. For the generation of rich observation (emission distribution), we follow the experiment design of \citep{misra2020kinematic}: the dimension of the observation is $2^{\ceil{\log(H+|\mathcal{S}|+1)}}$. For an observation $o$ that emitted from state $s$ and time step $h$, we concatenate the one-hot vector of $s$ and $h$, adding i.i.d. Gaussian noise $\mathcal{N}(0,0.1)$ on each entry, pend zero at the end if necessary, and finally multiply with a Hadamard matrix. In our setting, we have variants with different horizons $H$.

\subsection{Implementation Details} \label{app:exp:implementation}
For the implementation of \ouralg, we break down the introduction into two parts: the implementation of representation learning and the implementation of game solving algorithm with current features. For the implementation of representation learning, we follow the same function approximation as (\cite{zhang2022efficient}) and adapt their open-sourced code at \url{https://github.com/yudasong/briee}. We include an overview of the function class for completeness: we adopt a two layer neural network with tanh non-linearity as the function class as the discriminator class. For the decoder, we let $\psi(o) = \text{softmax}(A^{\top}o)$, where $A \in \mathbb{R}^{|\mathcal{O}|\times 3}$, and we let $\phi(o,\textbf{a}) = \psi(o) \otimes \textbf{a}$. Here $\textbf{a}$ denotes the one-hot encoding in the joint action space. 

Different from \cite{zhang2022efficient}, we solve the optimization problem by directly solving the min-max-min problem  instead of using an iterative method. We show the implementation in Algorithm.~\ref{alg:rep_learn_practice}. We first perform minibatch stochastic gradient descent aggressively on the discriminator selection step (line.~5, on $\hat \phi$ and $f$) and the feature selection step (line.~6, on $\phi$), where in each step we first compute the linear weight $w$ and $\hat w$ closed-formly and then perform gradient descent/ascend on the features and discriminators. Note that here the number of iteration $T$ is very small.

For solving the Markov games, in addition to following Algorithm.~\ref{alg:mf}, to solve line.12 (i.e., solving \eqref{eq:nash} or \eqref{eq:cce} or \eqref{eq:ce}), we implement the NE/CCE solvers based on the public repository: \url{https://github.com/quantumiracle/MARS}. Note that the essential difference lies in that \citep{xie2020learning} assumes that the algorithm has the access to the ground-truth feature but our algorithm needs to utilize the different features we learn for each iteration. We also adopt the Deep RL baseline from the same public repository.

\subsection{Zero-sum experiment training curves}
\label{app:zero}
In this section we provide the training curves of \ouralg and Deep RL baseline in the zero-sum setting in Figure.~\ref{app:fig:exp:zero_sum}.
\begin{figure}
    \centering
    \includegraphics[width=0.28\linewidth]{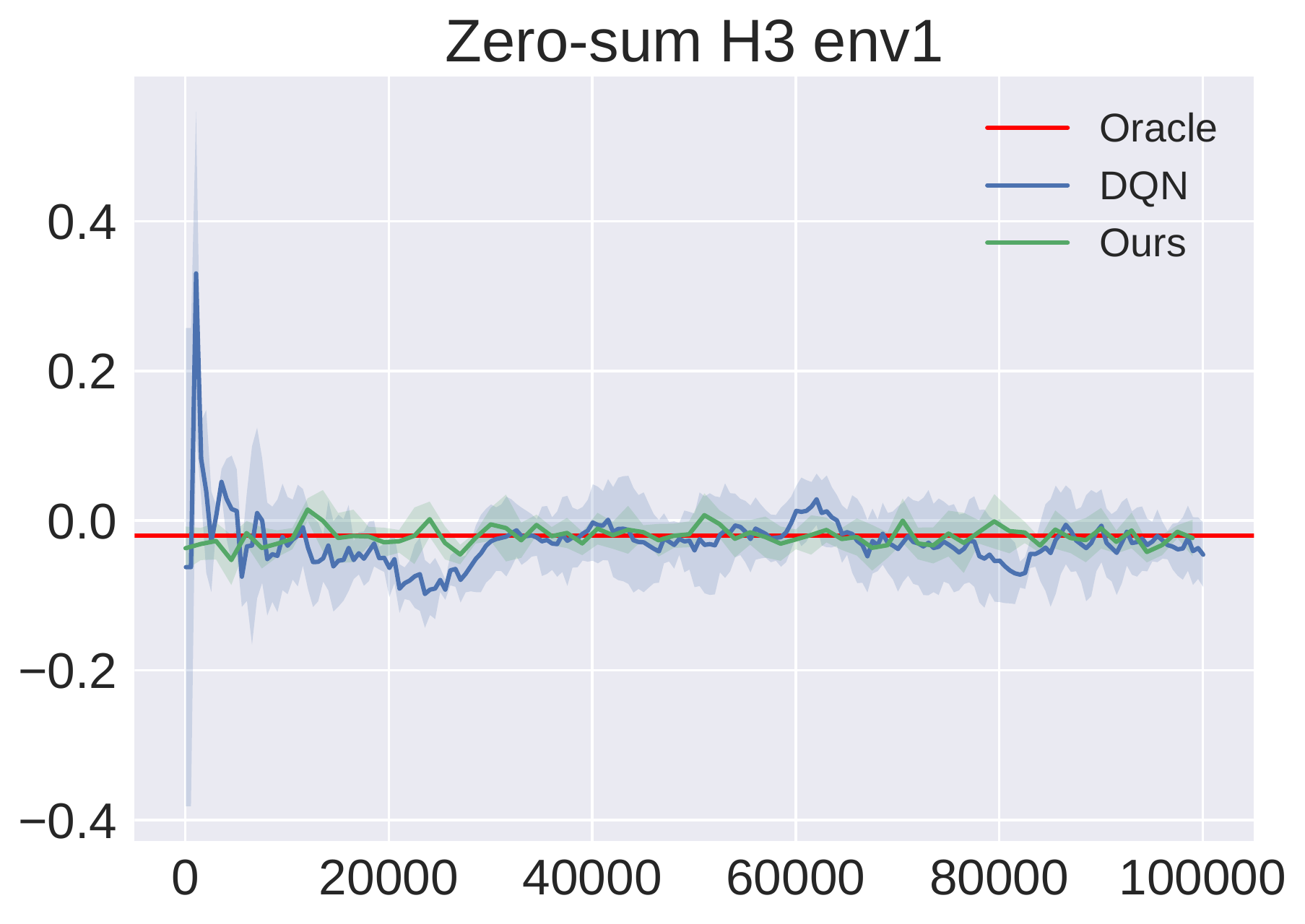}
    \includegraphics[width=0.28\linewidth]{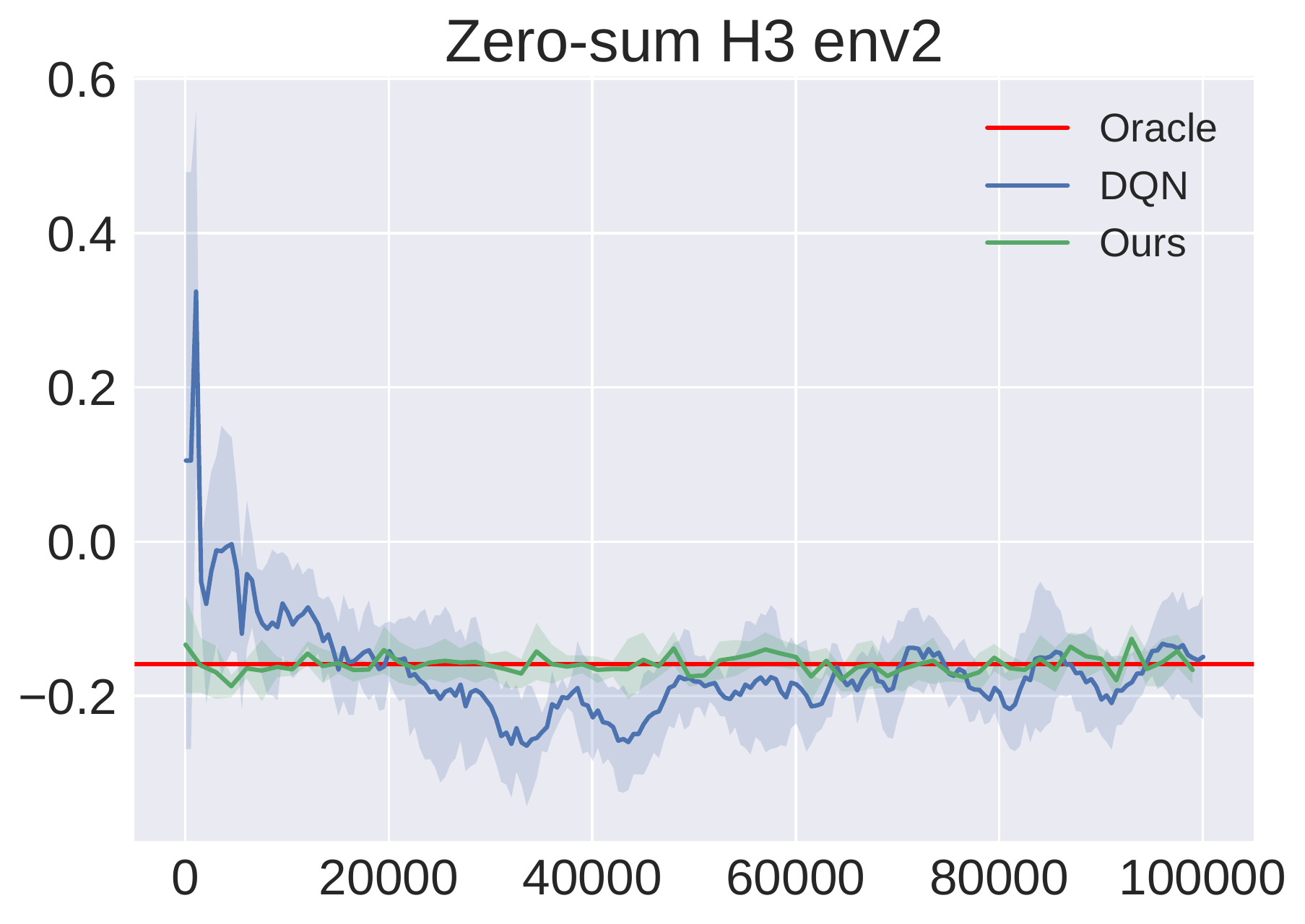}
    \includegraphics[width=0.28\linewidth]{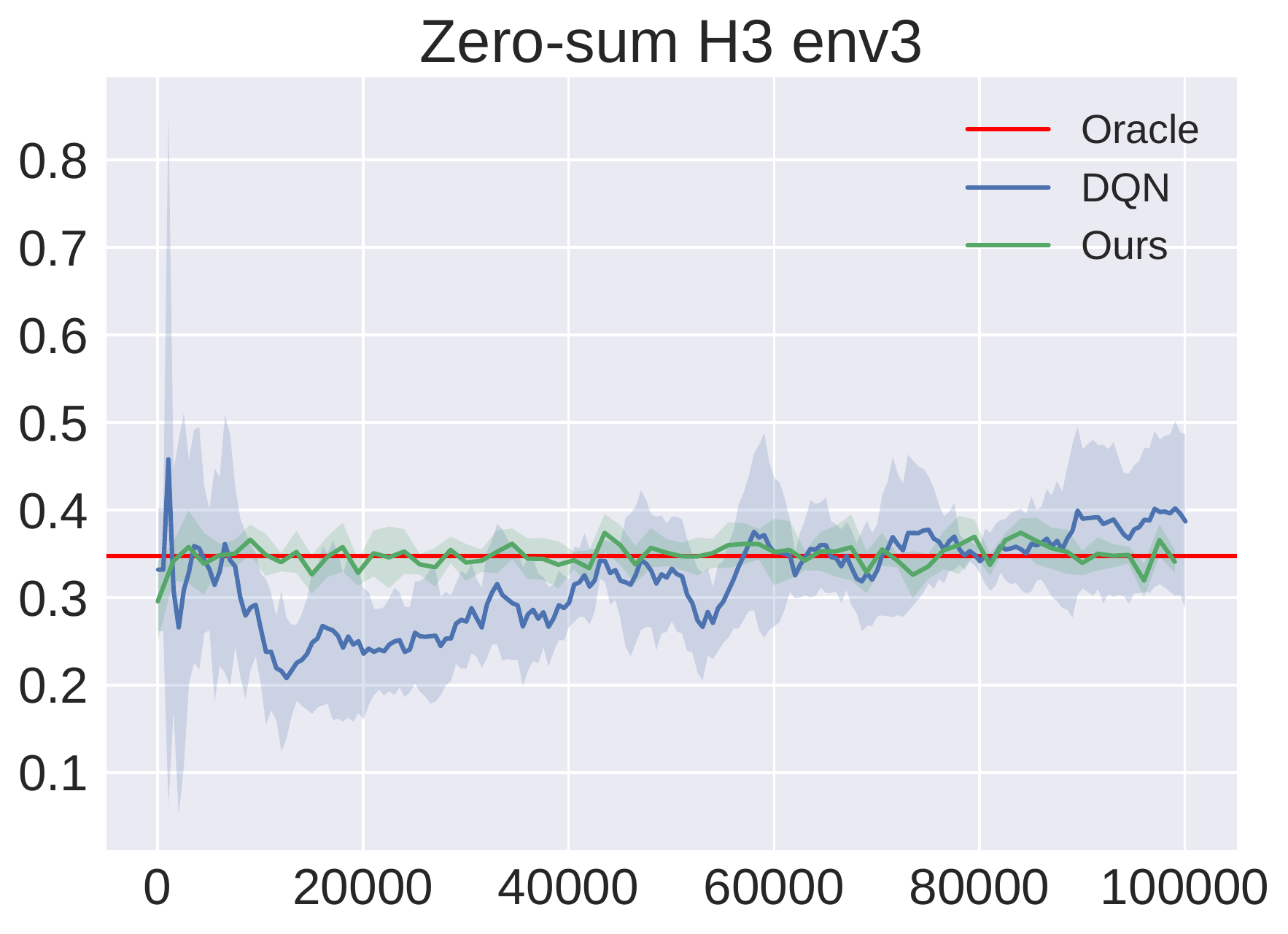}
    \centering
    \includegraphics[width=0.28\linewidth]{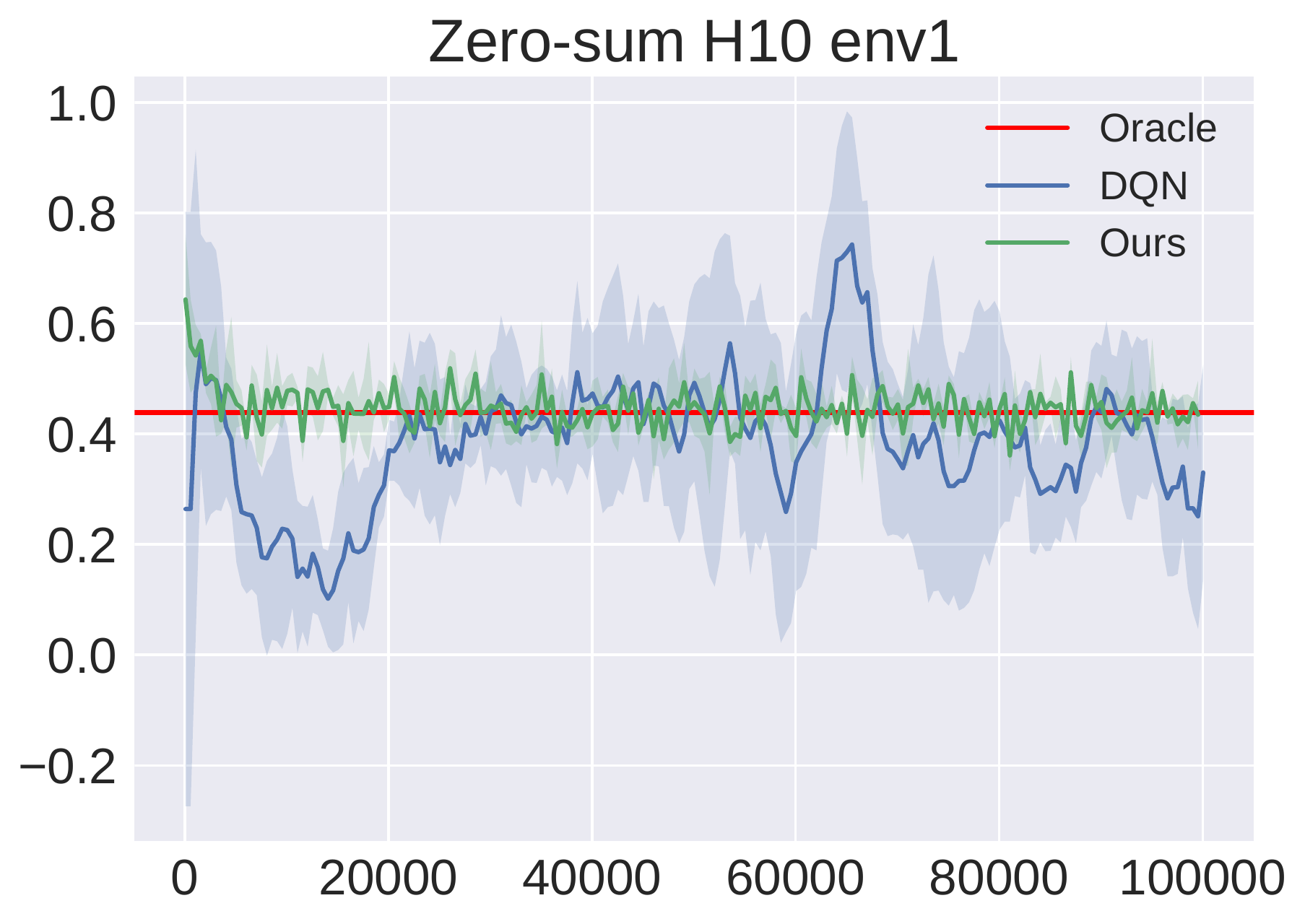}
    \includegraphics[width=0.28\linewidth]{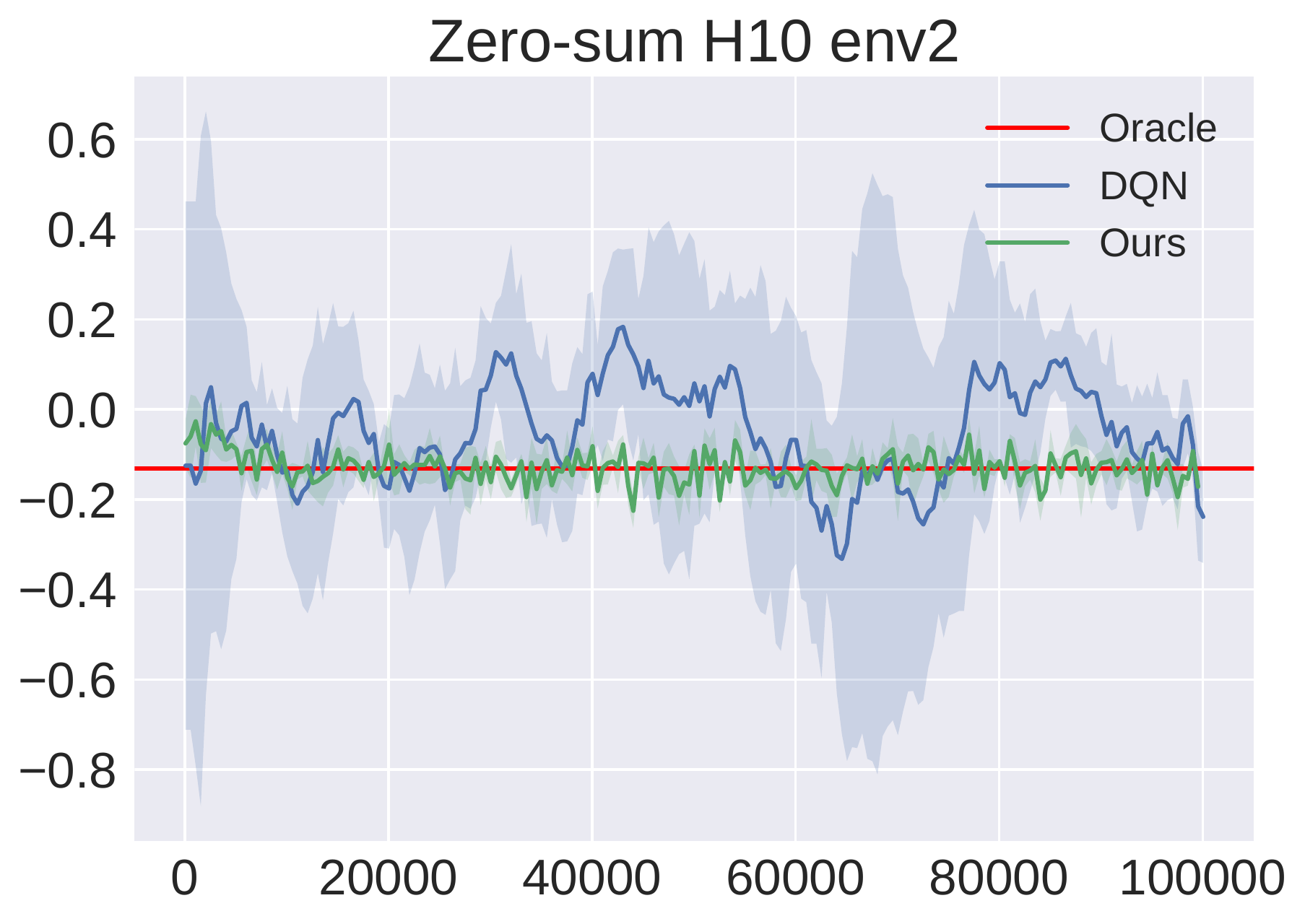}
    \includegraphics[width=0.28\linewidth]{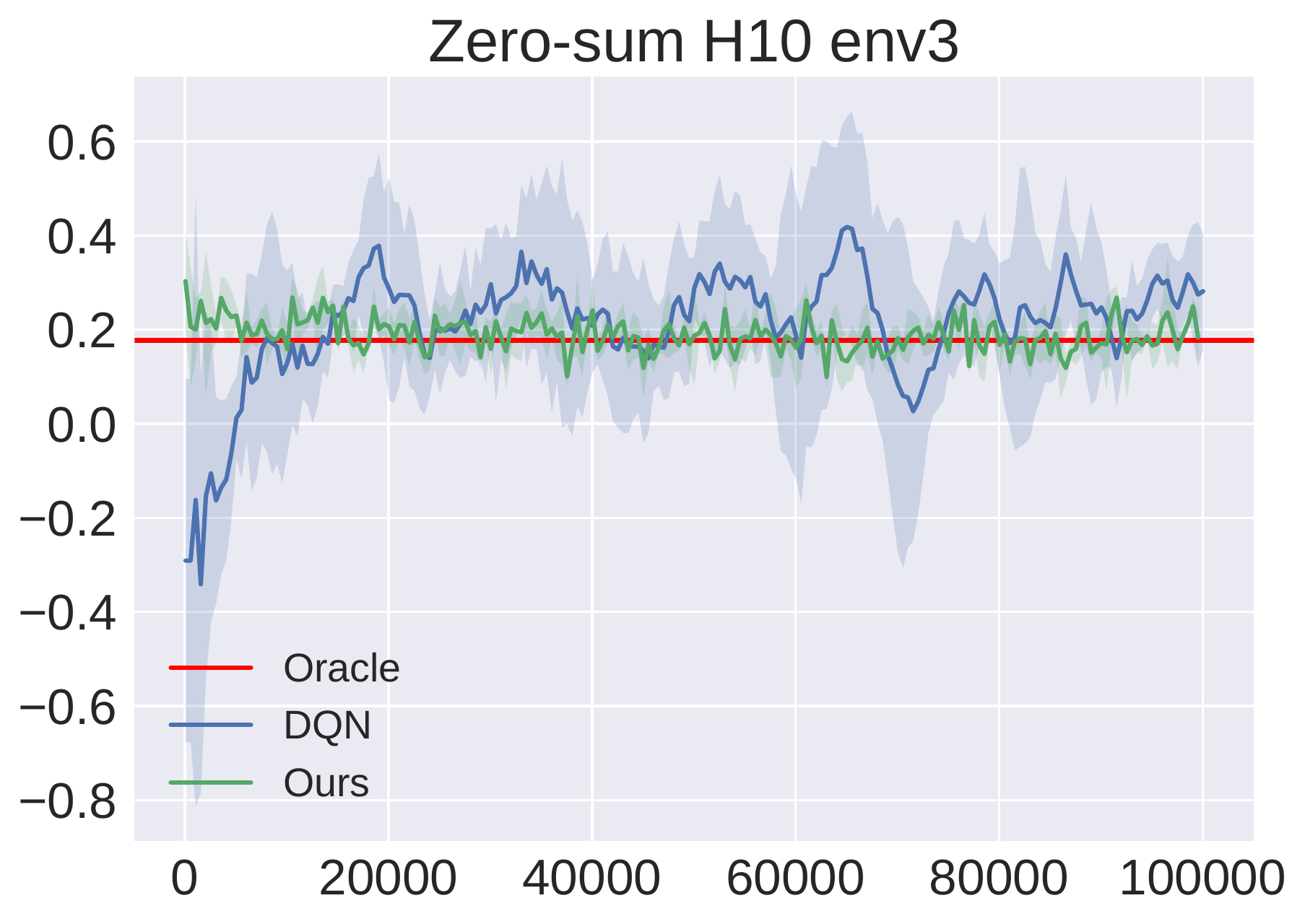}
    \caption{Training curve in the zero-sum setting. We evaluate each method over 5 random seeds and report the mean and standard deviation of the moving average of evaluation returns, wherein for each evaluation we perform 1000 runs. We use ``Oracle'' to denote the ground truth NE values of the Markov game. The x-axis denotes the number of episodes and the y-axis denotes the value of returns.}
    \label{app:fig:exp:zero_sum}
\end{figure}

\subsection{General-sum experiment details} \label{app:exp:gensum}
In this section we complete the remaining details for the general-sum experiment. We include the training curve in Fig.~\ref{app:exp:fig:gensum}.
\begin{figure}[t]
\centering
\includegraphics[width=0.3\linewidth]{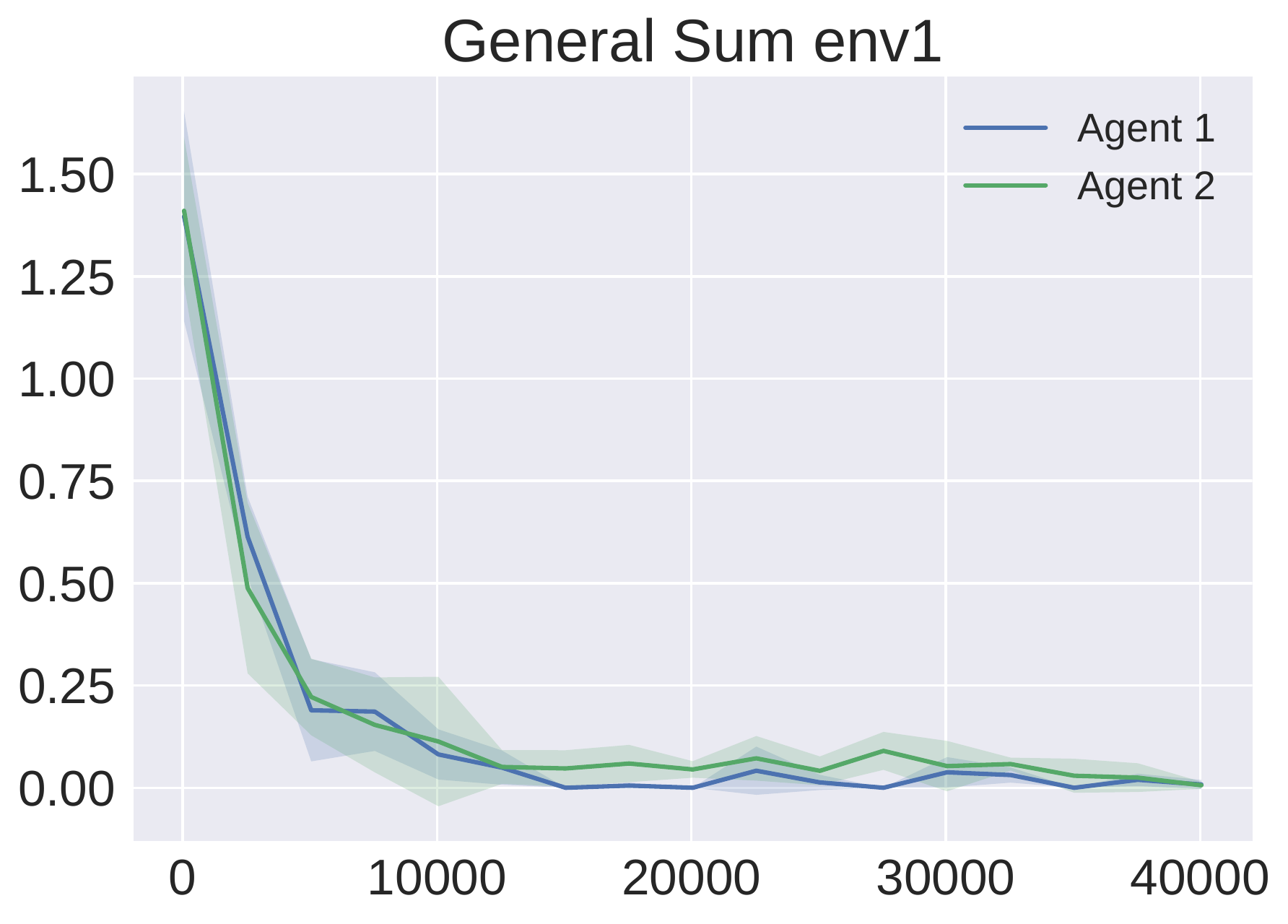}
\includegraphics[width=0.3\linewidth]{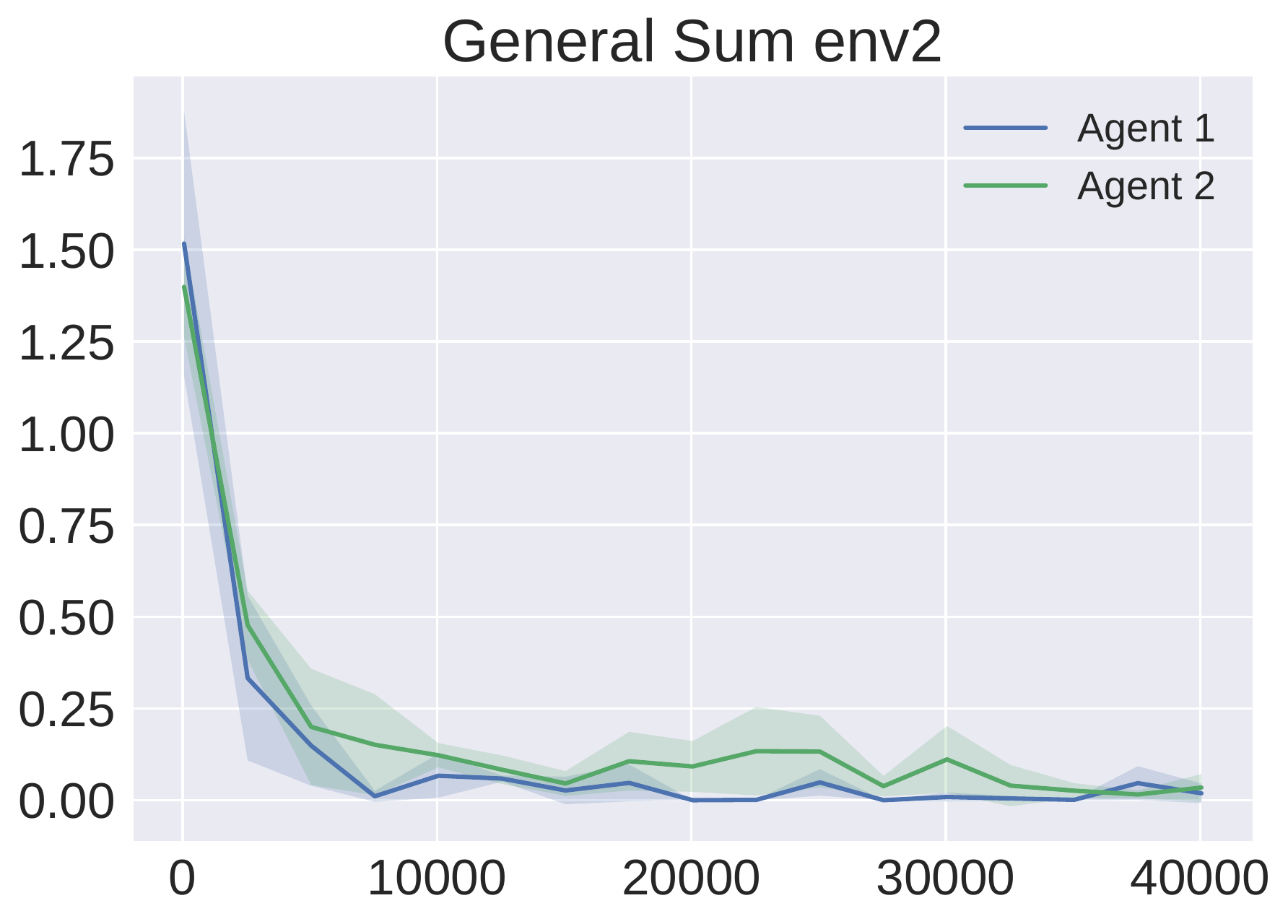}
\includegraphics[width=0.3\linewidth]{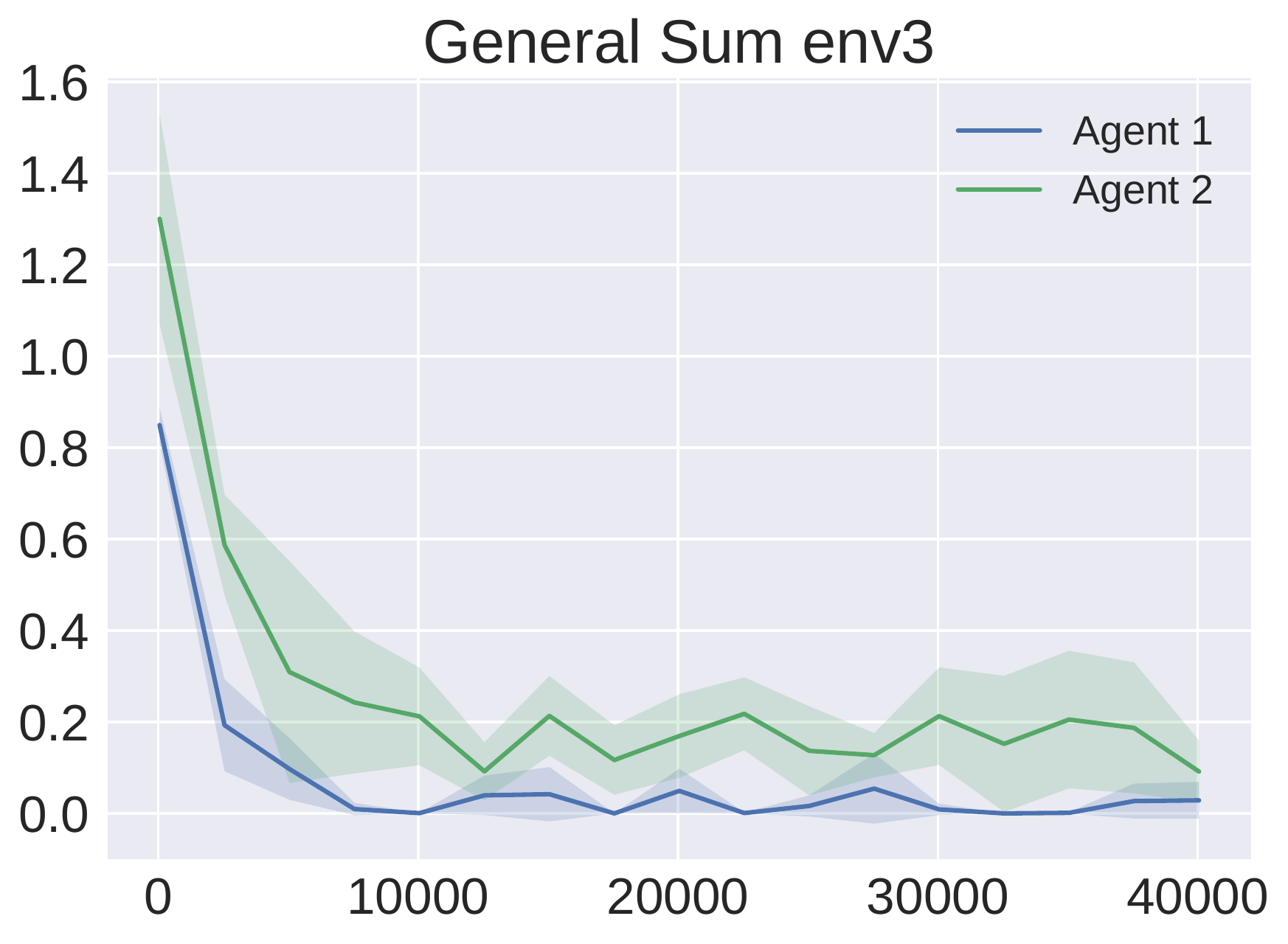}
\caption {Training curve of \ouralg in the general sum setting.  In this setting, the y-axis denotes exploitability instead of raw returns.}
\label{app:exp:fig:gensum}
\end{figure}

\begin{algorithm}[t]
    \caption{Model-free Representation Learning in Practice}
    \label{alg:rep_learn_practice}
    \begin{algorithmic}[1]
        \STATE\textbf{Input:} Dataset $\mathcal{D}$, step $h$, regularization $\lambda$, iterations $T$.
        \STATE Denote least squares loss: $\mathcal{L}_{\lambda,\mathcal{D}}(\phi,\theta,f):=\mathbb{E}_{\mathcal{D}}\left[\left(\phi(s,\bm{a})^\top\theta-f(s)\right)^2\right]+\lambda\Vert\theta\Vert_2^2$.
        \STATE Initialize $\phi_0\in\Phi_h$ arbitrarily;
        \FOR{$t=0,1,\ldots,T$}
            \STATE Discriminator selection: $f_t=\argmax_{f\in\mathcal{F}_h}[\min_\theta\mathcal{L}_{\lambda,\mathcal{D}}(\phi_t,\theta,f)-\min_{\tilde{\phi}\in\Phi,\tilde{\theta}}\mathcal{L}_{\lambda,\mathcal{D}}(\tilde{\phi},\tilde{\theta},f)]$
            \STATE Feature selection: $\phi_{t+1}=\argmin_{\phi\in\Phi_h}\sum_{i=1}^t\min_{\theta_i}\mathcal{L}_{\lambda,\mathcal{D}}(\phi,\theta_i,f_i),\ \hat{\phi}\leftarrow\phi_{t+1}$
        \ENDFOR
        \STATE\textbf{Return } $\hat{\phi}$ 
    \end{algorithmic}
\end{algorithm}

\subsection{Hyperparameters}
In this section, we include the hyperparameter for \ouralg in Table.~\ref{app:table:hyperparam:ours}, and the hyperparameter for DQN in Table.~\ref{app:table:huperparam:dqn:short} and Table.~\ref{app:table:huperparam:dqn:long}.

\begin{table}[ht] 
\caption{Hyperparameters for \ouralg.}
\centering
\begin{tabular}{ccc} 
\toprule
                                                & Value Considered          & Final Value  \\ 
\hline
Decoder $\phi$ learning rate                    & \{1e-2\}                  & 1e-2         \\
Discriminator $f$ learning rate                 & \{1e-\}                   & 1e-2         \\
Discriminator $f$ hidden layer size             & \{128,256,512\}           & 256          \\
RepLearn Iteration $T$                          & \{10,20,30,50\}           & 10           \\
Decoder $\phi$ number of gradient steps         & \{64,128,256\}            & 256           \\
Discriminator $f$ number of gradient steps      & \{64,128,256\}            & 256          \\
Decoder $\phi$ batch size                       & \{128,256,512\}           & 512          \\
Discriminator $f$ batch size                    & \{128,256,512\}           & 512          \\
RepLearn regularization coefficient $\lambda$   & \{0.01\}                  & 0.01         \\
Decoder $\phi$ softmax temperature              & \{1,0.5,0.1\}             & 1            \\
LSVI bonus coefficient $\beta$                  & \{0.1,0.5,1\}             & 0.1          \\
LSVI regularization coefficient $\lambda$       & \{1\}                     & 1            \\
Warm up samples                                 & \{0,200\}                 & 0            \\
\toprule
\end{tabular}
\label{app:table:hyperparam:ours}
\end{table}

\begin{table}[ht]
    \centering
    \caption{Hyperparameters for DQN in short horizon environment.}
    \begin{tabular}{ccc}
    \toprule
                                &  Value considered             & Final Value   \\
    \hline
    Target update interval      &  \{1000\}                     & 1000          \\
    $\epsilon_0$                &  \{1\}                        & 1             \\
    $\epsilon_N$                &  \{0.01\}                     & 0.01          \\
    $\epsilon$ decay frequency  &  \{8000\}                     & 8000          \\
    Batch size                  &  \{8000\}                     & 8000          \\
    Optimizer                   &  \{Adam\}                     & Adam          \\
    Learning Rate               &  \{0.0001\}                   & 0.0001        \\
    Hidden layer                &  \{[32,32,32]\}               & [32,32,32]    \\
    Self-play $\delta$          &  \{1.5\}                      & 1.5           \\
    \toprule
    \end{tabular}
    \label{app:table:huperparam:dqn:short}
\end{table}

\begin{table}[ht]
    \centering
    \caption{Hyperparameters for DQN in long horizon environment.}
    \begin{tabular}{ccc}
    \toprule
                                &  Value considered             & Final Value   \\
    \hline
    Target update interval      &  \{1000\}                     & 1000          \\
    $\epsilon_0$                &  \{1\}                        & 1             \\
    $\epsilon_N$                &  \{0.01\}                     & 0.01          \\
    $\epsilon$ decay frequency  &  \{8000\}                     & 8000          \\
    Batch size                  &  \{8000\}                     & 8000          \\
    Optimizer                   &  \{Adam\}                     & Adam          \\
    Learning Rate               &  \{0.0001\}                   & 0.0001        \\
    Hidden layer                &  \{[32,32,32]\}               & [32,32,32]    \\
    Self-play $\delta$          &  \{1.5,2\}                    & 2             \\
    \toprule
    \end{tabular}
    \label{app:table:huperparam:dqn:long}
\end{table}

\end{document}